\documentclass{article}

\usepackage{microtype}
\usepackage{booktabs} 

\usepackage{hyperref}

\usepackage{floatrow} 

\usepackage[accepted]{icml2019}

\usepackage{amsfonts}
\usepackage{amsmath}
\usepackage[algo2e,plain,linesnumbered,ruled]{algorithm2e}
\usepackage{enumitem} 
\usepackage[pdftex]{graphicx}
\usepackage{epstopdf}

\usepackage{amsthm}
\usepackage{amssymb}
\usepackage{subcaption}

\usepackage{amsthm}
\usepackage{hyperref}
\usepackage{bm}

\usepackage{multirow}
\usepackage{ctable} 
\usepackage{float}
\restylefloat{table} 
\allowdisplaybreaks 

\usepackage{adjustbox}

   \newcommand{\myComment}[1]{}  

\def\diag{\text{diag}}

\def\bar{\overline}

\def\norm#1{\left\|#1\right\|}
\def\normtwo#1{|||#1|||}
\def\bar{\overline}
\def\tilde{\widetilde}

\def\<{\left<} \def\>{\right>}

\def\sym{\mathrm{sym}}
\newtheorem{theorem}{Theorem}

\newtheorem{proposition}{Proposition}
\newtheorem{corollary}{Corollary}
\newtheorem{definition}{Definition}
\newtheorem{lemma}{Lemma}

 \newenvironment{talign*}
 {\csname align*\endcsname}
 {\endalign}
 
\def\x{\mathbf x} \def\y{\mathbf y}  
\def\u{\mathbf u} \def\v{\mathbf v}  
 \def\e{\mathbf e}

\newcommand{\abs}[1]{\left|#1\right|}
\newcommand{\bchi}{\boldsymbol{\chi}}
\newcommand{\vectornorm}[1]{\|#1\|}

\def\pp{p_{\mathrm{in}}^{+}}
\def\qp{p_{\mathrm{out}}^{+}}

\def\ppm{p_{\mathrm{in}}^{-}}
\def\qm{p_{\mathrm{out}}^{-}}

\def\diag{\mathrm{diag}}
\def\one{\mathbf{1}}

\def\to{\rightarrow}

\usepackage{color}

\newcommand{\SPMM}{SPM}

\usepackage{cancel}

\usepackage{appendix}

\icmltitlerunning{Spectral Clustering of Signed Graphs via Matrix Power Means}

\begin{document}

\twocolumn[
\icmltitle{Spectral Clustering of Signed Graphs via Matrix Power Means}




\begin{icmlauthorlist}
\icmlauthor{Pedro Mercado}{SB,TU}
\icmlauthor{Francesco  Tudisco}{FT}
\icmlauthor{Matthias Hein}{TU}
\end{icmlauthorlist}

\icmlaffiliation{SB}{Saarland University}
\icmlaffiliation{FT}{University of Strathclyde}
\icmlaffiliation{TU}{University of T\"{u}bingen}

\icmlcorrespondingauthor{Pedro Mercado}{pedro@cs.uni-saarland.de}

\icmlkeywords{Machine Learning, ICML}

\vskip 0.3in
]



\printAffiliationsAndNotice{}  
%

%

\begin{abstract}
Signed graphs encode positive (attractive) and negative (repulsive) relations between nodes.
We extend spectral clustering to signed graphs  via the one-parameter family of \textit{Signed Power Mean Laplacians},
defined as the matrix power mean of normalized standard and signless Laplacians of positive and negative edges. 
We provide a thorough analysis of the proposed approach in the setting of a general Stochastic Block Model that includes models such as the Labeled Stochastic Block Model and the Censored Block Model. 
We show that in expectation the signed power mean Laplacian captures the  ground truth clusters under reasonable settings where state-of-the-art approaches fail. Moreover, we prove that the eigenvalues and  eigenvector of the signed power mean Laplacian concentrate around their expectation under reasonable conditions in the general Stochastic Block Model. 
Extensive experiments on random graphs and real world datasets confirm the theoretically predicted behaviour of the signed power mean Laplacian and show that it compares favourably with state-of-the-art methods. 
\end{abstract}

\section{Introduction}\label{sec:Introduction}
The analysis of graphs has received a significant amount of attention due to their capability to encode interactions that naturally arise in social networks.
Yet, the vast majority of graph methods has been focused on the case where interactions are of the same type, leaving aside the case where different kinds of interactions are available~\cite{Leskovec:2010:SNS}.
Graphs and networks with both positive and negative edge weights arise naturally in a number of social, biological and economic contexts. Social dynamics and relationships are intrinsically positive and negative:  users of online social networks such as Slashdot and Epinions, for example, 
can express positive interactions, like friendship and trust, and negative ones, like enmity and distrust. Other important application settings are the analysis of gene expressions in biology \cite{fujita2012functional} or the analysis of financial and economic time sequences \cite{ziegler2010visual,pavlidis2006financial}, where similarity and variable dependence measures  commonly used may attain both positive and negative values (e.g.\ the Pearson correlation coefficient).  

Although the majority of the literature has focused on graphs that encode only positive interactions, the analysis of signed graphs can be traced back to social balance theory \cite{Cartwright:1956:Structural,Harary:1954:Notion,Davis:1967:Clustering},  
where the concept of a $k$-balance signed graph is introduced.  
The analysis of signed networks has been then pushed forward through the study of a variety of tasks in signed graphs, as for example
edge prediction~\cite{kumar2016wsn,Leskovec:2010:PPN,LeFalher:2017a},
node classification~\cite{Bosch:2018:nodeclassification,tang2016nodeClassification}, 
node embeddings~\cite{Chiang:2011:ELC,derr:2018:signedGraphConvolutionalNetwork,Kim:2018:SRL:3178876.3186117, Wang:2017:signedNetworkEmbedding, Yuan:2017:SNE}, 
node ranking~\cite{Chung:2013,Shahriari:2014}, 
and clustering~\cite{Chiang:2012:Scalable, Kunegis:2010:spectral, Mercado:2016:Geometric,sedoc:2017a,Doreian:2009:Partitioning,Knyazev:2018,kirkley2018balance,Cucuringu:2019:sponge,Cucuringu:2019:MB0}.
See~\cite{tang2015survey,gallier2016spectral} for recent surveys on the topic.

In this paper we present a novel extension of spectral clustering for signed graphs. Spectral clustering~\cite{Luxburg:2007:tutorial} is a well established technique for non-signed graphs, which partitions the set of the nodes based on a $k$-dimensional node embedding obtained using the first eigenvectors of the graph Laplacian. 
\textbf{Our contributions are as follows:} We introduce the family of \textit{Signed Power Mean (\SPMM) Laplacians}: a one-parameter family of graph matrices for signed graphs that blends the information from positive and negative interactions through the matrix power mean, a general class of matrix means that contains the arithmetic, geometric, and harmonic mean as special cases. 
This is inspired by recent extensions of spectral clustering which merge the information encoded by positive and negative interactions through different types of arithmetic~\cite{Chiang:2012:Scalable, Kunegis:2010:spectral} and geometric~\cite{Mercado:2016:Geometric} means of the standard and signless graph Laplacians. 
 We analyze the performance of the signed power mean Laplacian in a general Signed Stochastic Block Model. We first provide an anlysis in expectation showing that the smaller is the parameter of the signed power mean Laplacian, the less restrictive are the conditions  that ensure to recover the ground truth clusters. 
In particular, we show that the limit cases $+\infty$ and $-\infty$ are related to the boolean operators \texttt{AND} and $\texttt{OR}$, respectively, 
in the sense that for the limit case $+\infty$ clusters are recovered only if both positive \textit{and} negative interactions are informative, 
whereas for $-\infty$ clusters are recovered if positive \textit{or}
negative interactions are informative. 
This is consistent with related work in the context of unsigned multilayer graphs~\cite{Mercado:2018:powerMean}.
Second, we show that the eigenvalues and eigenvectors of the signed power mean Laplacian concentrate around their mean, so that our results hold also for the case where
one samples from the stochastic block model. 
Our result extends with minor changes to the unsigned multilayer graph setting considered in \cite{Mercado:2018:powerMean}, where just the expected case has been studied. 
To our knowledge these are the first concentration results for matrix power means under any stochastic block model for signed graphs.
Finally, we show that the signed power mean Laplacian compares favorably with state-of-the-art approaches through extensive numerical experiments on diverse real world datasets.
All the proofs have been moved to the supplementary material.

\textbf{Notation.} A signed graph is a pair ${G^\pm=(G^+,G^-)}$, where ${G^+=(V,W^+)}$ and ${G^-=(V,W^-)}$
encode positive and negative edges, respectively, with positive symmetric adjacency matrices $W^+$ and $W^-$, and a common vertex set $V=\{v_1,\ldots,v_n\}$. Note that this definition allows the simultaneous presence of both positive and negative interactions between the same two nodes. This is a major difference with respect to the alternative 
point of view where $G^\pm$ is associated to a single symmetric matrix $W$ with positive and negative entries. In this case $W=W^+-W^-$, with $W^+_{ij}=\max\{0,W_{ij}\}$ and 
$W^-_{ij}=-\min\{0,W_{ij}\}$, implying that every interaction is either positive or negative, but not both at the same time. 
We denote by $D_{ii}^+ = \sum_{j=1}^n w^+_{ij}$ and  $D_{ii}^- = \sum_{j=1}^n w^-_{ij}$
the diagonal matrix of the degrees of $G^+$ and $G^-$, respectively, 
and $\bar{D} = D^+ + D^-$.

\section{Related work}\label{sec:SpectralClusteringForSignedGraphs}
The study of clustering of signed graphs can be traced back to the theory of social balance~\cite{Cartwright:1956:Structural,Harary:1954:Notion,Davis:1967:Clustering}, where a signed graph is called $k$-balanced if the set of vertices can be partitioned into $k$ sets such that within the subsets there are only positive edges, and between them only negative. 

Inspired by the notion of $k$-balance, different approaches for signed graph clustering
have been introduced. In particular, many of them aim to extend spectral clustering to signed graphs by proposing novel signed graph Laplacians.  A related approach is correlation clustering~\cite{Bansal2004}. 
Unlike spectral clustering, where the number of clusters is fixed a-priori, correlation clustering approximates the optimal number of clusters by identifying a partition that is as close as possible to be $k$-balanced. 
In this setting, the case where the number of clusters is constrained has been considered in~\cite{Giotis:2006}.

We briefly introduce the standard and signless Laplacian and review different definitions of
Laplacians on signed graphs. The final clustering algorithm to find $k$ clusters is the same for all of them: compute the smallest $k$ eigenvectors of the corresponding Laplacian, use the eigenvectors to 
embed the nodes into $\mathbb{R}^k$, obtain the final clustering by doing $k$-means in the embedding space. However, we will see below that in some cases we have to slightly deviate
from this generic principle by using the $k-1$ smallest eigenvectors instead. 
%
%
%
%
%
%

\textbf{Laplacians of Unsigned Graphs:} 
In the following all weight matrices are non-negative and symmetric.
Given an assortative graph $G=(V,W)$, standard spectral clustering is based on the Laplacian and its normalized version defined as:

$\,$\hfill  $L = D - W \quad\qquad\,\, L_\sym = D^{-1/2}LD^{-1/2}$ \hfill $\,$

where $D_{ii} = \sum_{j=1}^n w_{ij}$ is the diagonal matrix of the degrees of $G$.
Both Laplacians are symmetric positive semidefinite and the multiplicity of the eigenvalue $0$ is
equal to the number of connected components in $G$.

For disassortative graphs, i.e.\ when edges carry only dissimilarity information, the goal is to identify clusters such that the amount of edges between clusters is larger than the one inside clusters. Spectral clustering is extended to this setting by considering the signless Laplacian matrix and its normalized version (see e.g.\ \cite{Liu2015,Mercado:2016:Geometric}), defined as:

$\,$\hfill  $Q = D + W \quad\qquad\,\, Q_\sym = D^{-1/2}QD^{-1/2}$ \hfill $\,$

Both Laplacians  are positive semi-definite, and the smallest eigenvalue is zero if and only if the graph has a bipartite component~\cite{Desai:1994:characterization}. 

\textbf{Laplacians of Signed Graphs:} 
Signed graphs encode both positive and negative interactions. In the ideal $k$-balanced case positive interactions present an assortative behaviour, whereas negative interactions present a disassortative behaviour. With this in mind, several novel definitions of \textit{signed Laplacians} have been proposed. We briefly review them for later reference.

In~\cite{Chiang:2012:Scalable} the balance ratio Laplacian and its normalized version are defined as:

$\,$\hfill  $L_{BR} = D^+ - W^+ + W^-, \quad\,\, L_{BN} = \bar{D}^{-1/2} L_{BR} \bar{D}^{-1/2}$ \hfill $\,$

whereas in~\cite{Kunegis:2010:spectral} the signed ratio Laplacian and its normalized version have been defined as:

$\,$\hfill  $L_{SR} = \bar{D} - W^+ + W^-, \quad\,\, L_{SN} = \bar{D}^{-1/2} L_{SR} \bar{D}^{-1/2}$ \hfill $\,$

The signed Laplacians $L_{BR}$ and $L_{BN}$ need not be positive semidefinite, while the signed Laplacians $L_{SR}$ and $L_{SN}$ are positive semidefinite with eigenvalue zero if and only if the graph is 2-balanced.

In the context of correlation clustering, in~\cite{saade:2015} the Bethe Hessian matrix is defined as:

$\,$\hfill  $H = (\alpha-1)I - \sqrt{\alpha}(W^+ - W^-) + \overline{D}$ \hfill $\,$

where $\alpha$ is the average node degree $\alpha=\frac{1}{n}\sum_{i=1}^n \overline{D}_{ii}$. 
The Bethe Hessian $H$ need not be positive definite. In fact, eigenvectors with negative eigenvalues bring information of clustering structure~\cite{saade:2014}.

Let  $L^+=D^+-W^+$ and $Q^- = D^-+W^-$ be the Laplacian and signless Laplacian of $G^+$ and $G^-$, respectively. As noted in \cite{Mercado:2016:Geometric}, $L_{SR}=L^+ + Q^-$ i.e.\ it coincides with twice the arithmetic mean of $L^+$ and $Q^-$. Note that the same holds for $H$ when the average degree $\alpha$ is equal to one, i.e.\ $H=L_{SR}$ when $\alpha=1$. In \cite{Mercado:2016:Geometric}, the arithmetic mean and geometric mean of the normalized Laplacian and its signless version are used to define new Laplacians for signed graphs:

$\,$\hfill  $L_{AM} = L^+_\sym + Q^-_\sym, \quad\,\, L_{GM} = L^+_\sym \# Q^-_\sym$ \hfill $\,$

where $A\#B=A^{-1/2}(A^{1/2}BA^{1/2})^{1/2}A^{-1/2}$ 
is the geometric mean of $A$ and $B$, 
$L^+_\sym=(D^+)^{-1/2}L^+(D^+)^{-1/2}$ and $Q^-_\sym=(D^-)^{-1/2}Q^-(D^-)^{-1/2}$.
While the computation of $L_{GM}$ is more challenging, in~\cite{Mercado:2016:Geometric} it is shown that the clustering assignment obtained with the geometric mean Laplacian $L_{GM}$ outperforms all other signed Laplacians.

Both the arithmetic and the geometric means are special cases of a much richer one-parameter family of means known as power means. Based on this observation, we introduce 
the \textit{Signed Power Mean Laplacian} in Section~\ref{subsec:MatrixMeansAndSpectralClusteringForSignedGraphs}, 
defined via a matrix version of the family of power means which we briefly review below.

\subsection{Matrix Power Means}\label{subsec:MatrixMeansForSignedGraphs}

The scalar power mean of two non-negative scalars $a,b$ is a one-parameter family of means defined for $p\in\mathbb{R}$ as $m_p(a,b) = \big(\frac{a^p+b^p}{2}\big)^{1/p}$. 
Particular cases are the arithmetic, geometric and harmonic means, as shown in Table~\ref{table:powerMeans}.
Moreover, the scalar power mean is monotone in the parameter $p$, 
i.e.\  $m_{p}(a,b)\leq m_{q}(a,b)$ when $p\leq q$ (see \cite{bullen2013handbook} , Ch.~3, Thm.~1), which yields the well known arithmetic-geometric-harmonic mean inequality $m_{-1}(a,b)\leq m_{0}(a,b) \leq m_{1}(a,b)$.
As matrices do not commute, several matrix extensions of the scalar power mean have been introduced, which typically agree if the matrices commute, see e.g.\ Chapter 4 in~\cite{bhatia2009positive}. 
We consider the following matrix extension of the scalar power mean:
\begin{definition}[\cite{bhagwat_subramanian_1978}]\label{definition:MatrixPowerMean}
Let $A, B$ be symmetric positive definite matrices, and $p\in\mathbb{R}$.
The matrix power mean of $A, B$ with exponent $p$ is
\begin{equation*}\label{eq:MatrixPowerMean}
 M_{p}(A, B)=\left({\dfrac  {A^{p} + B^{p}}{2}}  \right)^{ 1/p  } 
\end{equation*}
where $Y^{1/p}$ is the unique positive definite solution of the matrix equation $X^p = Y$.
\end{definition}
Please note that this definition can be extended to positive semidefinite matrices~\cite{bhagwat_subramanian_1978} for $p>0$, as $M_{p}(A,B)$ exists, whereas for $p<0$ a diagonal shift is necessary to ensure that the matrices $A,B$ are positive definite.
\setlength{\textfloatsep}{5pt}
\begin{table}
\centering
\small
\begin{tabular}{c|c|c}
\specialrule{1.5pt}{.1pt}{.1pt}                                               
       $p$       & $m_p(a,b)$     & $name$ \\ 
\specialrule{1.5pt}{.1pt}{.1pt}                                               
$p\to \infty$    & $\max \{a,b\}$ & maximum   \\
$p=1$            & $(a+b)/2$      & arithmetic mean   \\
$p\to 0$         & $\sqrt{ab}$    & geometric mean   \\
$p=-1$           & $2\big(\frac 1 a + \frac 1 b\big)^{-1}$      & harmonic mean   \\
$p\to -\infty$    & $\min \{a,b\}$ & minimum  
\end{tabular}
\caption{Particular cases of scalar power means}
\label{table:powerMeans}
\end{table}
\subsection{The Signed Power Mean Laplacian}\label{subsec:MatrixMeansAndSpectralClusteringForSignedGraphs}
Given a signed graph ${G^\pm=(G^+,G^-)}$ we define the Signed Power Mean (\SPMM) Laplacian $L_p$ of $G^\pm$ as
\begin{equation}\label{eq:signedMatrixMeanLaplacian}
 L_p=M_p(L^+_\sym, Q^-_\sym).
\end{equation}
For the case $p<0$ the matrix power mean requires positive definite matrices, hence we use in this case the matrix power mean of diagonally shifted Laplacians, i.e.\ $L^+_\sym + \varepsilon I$ and $Q^-_\sym + \varepsilon I$. Our following theoretical analysis holds for all possible shifts $\varepsilon>0$, whereas we discuss in the supplementary material the numerical robustness with respect to $\varepsilon$. The clustering algorithm for identifying $k$ clusters in signed graphs is given in Algorithm~\ref{alg:powerMean}.  Please note that for $p\geq 1$ we deviate from the usual scheme and use the first $k-1$ 
eigenvectors rathen than the first $k$.
%
%
The reason is a result of the analysis in the stochastic block model in Section~\ref{sec:SBM}. In general, the main influence of the parameter $p$ of the power mean is on the ordering
of the eigenvalues.
In Section~\ref{sec:SBM} we will see that this significantly influences the performance of  
different instances of \SPMM{} Laplacians, in particular, 
the arithmetic and geometric mean discussed in  \cite{Mercado:2016:Geometric} are suboptimal for the recovery of the ground truth clusters. For the computation of the matrix power mean we adapt the scalable Krylov subspace-based algorithm proposed in \cite{Mercado:2018:powerMean}.

\begin{algorithm2e}[t]
  \DontPrintSemicolon
	\caption{\small Spectral clustering of signed graphs with $L_p$}\label{alg:powerMean}
	\small
	\KwIn{Symmetric matrices $W^{+},W^{-}$, number $k$ of clusters to construct. }
	\KwOut{Clusters $C_1,\ldots,C_k$.}
	Let $k'=k-1$ if $p\geq 1$ and $k'=k$ if $p< 1$.
	\;
	 Compute eigenvectors $\u_1, \ldots,\u_{k'}$ corresponding to the $k'$ smallest eigenvalues of~$L_p$.\;
	 Set $U=(\u_1, \ldots,\u_{k'})$ and cluster the rows of $U$ with $k$-means into clusters $C_1,\ldots,C_k$.	\;
\end{algorithm2e}

\section{Stochastic Block Model Analysis of the Signed Power Mean Laplacian}\label{sec:SBM}
In this section we analyze the signed power mean Laplacian $L_p$ under a general Signed Stochastic Block Model. Our results here are twofold. First, we derive new conditions in expectation that guarantee 
that the eigenvectors corresponding to the smallest eigenvalues of $L_p$ recover the ground truth clusters. These conditions reveal that, in this setting, the state-of-the-art signed graph matrices are suboptimal as compared to $L_p$ for negative values of $p$. Second, we show that our result in expectation transfer to sampled graphs as we prove conditions that ensure that both eigenvalues and eigenvectors of $L_p$ concentrate around their expected value with high probability. We verify our results by several experiments where the clustering performance of state-of-the-art matrices and $L_p$ are compared on random graphs following the Signed Stochastic Block Model.

All proofs hold for an arbitrary diagonal shift $\varepsilon >0$, whereas the shift is set to $\varepsilon=\log_{10}(1+\abs{p})+10^{-6}$ in the numerical experiments. Numerical robustness with respect to $\varepsilon$ is discussed in the supplementary material.

The Stochastic Block Model (\textbf{SBM}) is a well-established generative model for graphs and a canonical tool for studying clustering methods~\cite{Holland:SBM:1983,rohe2011spectral,Abbe:2018:recentDevelopments}. Graphs drawn from the SBM show a prescribed clustering structure, as the probability of an edge between two nodes depends only on the clustering membership of each node.
We introduce our SBM for signed Graphs (\textbf{SSBM}): we consider $k$ ground truth clusters $\mathcal{C}_1, \ldots, \mathcal{C}_k$, all of them of size $\abs{\mathcal{C}}=\frac{n}{k}$, and parameters $\pp,\qp,\ppm,\qm\in[0,1]$ where 
$\pp$ (resp. $\ppm$) is the probability of observing an edge inside clusters in $G^+$ (resp. $G^-$)
and $\qp$ (resp. $\qm$) is the probability of observing an edge between clusters in $G^+$ (resp. $G^-$).
Calligraphic letters are used for the expected adjacency matrices:  
$\mathcal{W^{+}}$ and $\mathcal{W^{-}}$ are the expected adjacency matrix of $G^+$ and $G^-$, respectively, where 
$\mathcal{W}^+_{i,j}=\pp$ and $\mathcal{W}^-_{i,j}=\ppm$ if $v_i,v_j$ belong to the same cluster, whereas
$\mathcal{W}^+_{i,j}=\qp$ and $\mathcal{W}^-_{i,j}=\qm$ if $v_i,v_j$ belong to different  clusters.

Other extensions of the  SBM to the signed setting have been considered. Particularly relevant examples are the Labelled Stochastic Block Model (\textbf{LSBM})~\cite{heimlicher2012community} and the Censored Block Model (\textbf{CBM})~\cite{Abbe:2014:Decoding}. 
In the context of signed graphs, both LSBM and CBM assume that an observed edge can be either positive or negative, but not both. 
Our SSBM, instead, allows the simultaneous presence of both positive and negative edges between the same pair of nodes, as the parameters $\pp,\qp,\ppm,\qm$ in SSBM are independent. Moreover, the edge probabilities defining both the LSBM and the CBM can be recovered as special cases of the SSBM. In particular, the LSBM corresponds to  the  SSBM for the choices
 \begin{itemize}[topsep=-3pt,leftmargin=*]\setlength\itemsep{-3pt}
 \item[]$\pp=\overline{p}_{\mathrm{in}}\mu^+, \quad\,\,\,\, \ppm=\overline{p}_{\mathrm{in}}\mu^- \hfill (\text{within clusters})$
 \item[]$\qp=\overline{p}_{\mathrm{out}}\nu^+, \,\, \qm=\overline{p}_{\mathrm{out}}\nu^- \hfill  (\text{between clusters})$
 \end{itemize}
where $\overline{p}_{\mathrm{in}}$ and $\overline{p}_{\mathrm{out}}$ are edge probabilities within and between clusters, respectively, whereas $\mu^+$ and $\mu^- = 1-\mu^+$ (resp. $\nu^+$ and $\nu^-=1-\nu^+$) are the probabilities of assigning a positive and negative label to an edge within (resp. between) clusters.
Similarly, the CBM corresponds to the SSBM for the particular choices $\overline{p}_{\mathrm{in}}=\overline{p}_{\mathrm{out}}$, 
$\mu^+=\nu^-=(1-\eta)$ and $\mu^-=\nu^+=\eta$ where $\eta$ is a noise parameter.

Our goal is to identify conditions in terms of $k,\pp,\qp,\ppm$, and $\qm$, such that $\mathcal{C}_1, \ldots, \mathcal{C}_k$ are recovered by the smallest eigenvectors of the signed power mean Laplacian.
Consider the following $k$~vectors:
\begin{itemize}[topsep=-3pt,leftmargin=*]\setlength\itemsep{-5pt}
 \item[] \centering $\boldsymbol \chi_1  = \one, \qquad \boldsymbol \chi_i = (k-1)\one_{\mathcal{C}_i}-\one_{\overline{\mathcal{ C}_i}}$\, . 
\end{itemize}
$i=2,\ldots,k$.
The node embedding given by $\{\boldsymbol \chi_i\}_{i=1}^{k}$ is informative
in the sense that applying $k$-means on $\{\boldsymbol \chi_i\}_{i=1}^{k}$ trivially recovers the ground truth  clusters $\mathcal{C}_1, \ldots, \mathcal{C}_k$
as all nodes of a cluster are mapped to the same point. Note that the constant vector $\boldsymbol \chi_1$ could be omitted as it does not add clustering information.
We derive conditions for the 
SSBM 
such that $\{\boldsymbol \chi_i\}_{i=1}^{k}$ are the smallest eigenvectors of the signed power mean Laplacian in expectation.
%
%
%
%
%

\begin{theorem}\label{theorem:mp_in_expectation}
 Let $\mathcal{L}_p=M_p(\mathcal L^{+}_\sym , \mathcal Q^{-}_\sym)$ and let $\varepsilon>0$ be the diagonal shift. 
 \begin{itemize}[topsep=-3pt,leftmargin=*]\setlength\itemsep{-3pt}
 \item If $p\geq 1$, then 
 $\{ \boldsymbol \chi_i \}_{i=2}^{k}$ correspond to the $(k$-$1)$-smallest eigenvalues of 
 $\mathcal{L}_p$ if and only if $m_p(\rho^+_\varepsilon,\rho^-_\varepsilon)<1+\varepsilon$;
 \item If $p< 1$, then 
 $\{ \boldsymbol \chi_i \}_{i=1}^{k}$ correspond to the $k$-smallest eigenvalues of 
 $\mathcal{L}_p$ if and only if $m_p(\rho^+_\varepsilon,\rho^-_\varepsilon)<1+\varepsilon$;
 \end{itemize}
 with $\rho^+_\varepsilon=1-(p_{\mathrm{in}}^+ - p_{\mathrm{out}}^+)/(p_{\mathrm{in}}^+ + (k -1)p_{\mathrm{out}}^+)+\varepsilon$ and $\rho^-_\varepsilon=1+(p_{\mathrm{in}}^- - p_{\mathrm{out}}^-)/(p_{\mathrm{in}}^- + (k -1)p_{\mathrm{out}}^-)+\varepsilon$.
\end{theorem}

\myComment{
\begin{proof}
 Please see Section~\ref{appendix:theorem:mp_in_expectation-PROOF}.
\end{proof}
}

Note that Theorem~\ref{theorem:mp_in_expectation} is the reason why Alg.~\ref{alg:powerMean} uses only the first $k-1$ eigenvectors for $p\geq 1$. The problem is that the constant eigenvector need not be among the first $k$ eigenvectors in the SSBM for $p\geq 1$. However, as it is constant and thus uninformative in the embedding, this does not lead to any loss of information.
The following Corollary shows that the limit cases of $L_p$ are related to the boolean operators \texttt{AND} and \texttt{OR}.
\begin{corollary}\label{corollary:mp_limit_cases}
Let $\mathcal{L}_p=M_p(\mathcal L^{+}_\sym , \mathcal Q^{-}_\sym)$.
\begin{itemize}[topsep=-3pt,leftmargin=*]\setlength\itemsep{-3pt}
\item $\{ \boldsymbol \chi_i \}_{i=2}^{k}$ correspond to the $(k$-$1)$-smallest eigenvalues of 
 $\mathcal{L}_\infty$ if and only if 
$\pp > \qp$ \textbf{and}  $\ppm < \qm$,
\item $\{ \boldsymbol \chi_i \}_{i=1}^{k}$ correspond to the $k$-smallest eigenvalues of 
 $\mathcal{L}_{-\infty}$ if and only if
$\pp > \qp$ \textbf{or}  $\ppm < \qm$.
\end{itemize}
\end{corollary}
\myComment{
\begin{proof}
 Please see Section~\ref{corollary:mp_limit_cases-PROOF}.
\end{proof}
}

The conditions for $\mathcal{L}_{\infty}$ are the most conservative ones, as they require that $G^+$ \texttt{\underline{and}} $G^-$ are informative, i.e.\ $G^+$ has to be assortative \texttt{\underline{and}} $G^-$ disassortative. Under these conditions every clustering method for signed graphs should be able to identify the ground truth clusters in expectation.
On the other hand, the less restrictive conditions for the recovery of the ground truth clusters correspond to the limit case $\mathcal{L}_{-\infty}$. If $G^+$ \texttt{\underline{or}} $G^-$ are informative, then the ground truth clusters are recovered, that is,  $\mathcal{L}_{-\infty}$ only requires that $G^+$ is assortative \texttt{\underline{or}} $G^-$ is disassortative.
In particular, the following corollary shows that smaller values of $p$ require less restrictive conditions to ensure the identification of the informative eigenvectors.

\begin{corollary}\label{corollary:contention}
Let $q\leq p$. If $\{\bchi_i\}_{i=\theta(p)}^k$ correspond to the $k$-smallest eigenvalues of $\mathcal L_{p}$, then
$\{\bchi_i\}_{i=\theta(q)}^k$ correspond to the $k$-smallest eigenvalues of $\mathcal L_{q}$,
where $\theta(x)=1$ if $x\leq 0$ and $\theta(x)=2$ if $x>0$.
\end{corollary}
To better understand the different conditions we have derived, we visualize them in Fig.~\ref{fig:SBM:inExpectation}, 
where the $x$-axis corresponds to how assortative $G^+$ is, while the $y$-axis corresponds to how disassortative $G^-$ is. The conditions of the limit case $\mathcal{L}_{\infty}$, i.e. 
the case where $G^+$ \texttt{\underline{and}} $G^-$ have to be informative, correspond to the upper-right region, dark blue region in Fig.~\ref{fig:SBM:inExpectation:AND}, and correspond to  the $25\%$ of all possible configurations of the SBM. The conditions for the limit case $\mathcal{L}_{-\infty}$, i.e. the case where $G^+$ \texttt{\underline{or}} $G^-$ has to be informative, instead correspond to all possible configurations of the SBM except for the bottom-left region.
This is depicted in Fig.~\ref{fig:SBM:inExpectation:OR} and corresponds to the $75\%$ of all possible configurations under the SBM.

\begin{figure*}[!th]
\floatbox[{\capbeside\thisfloatsetup{capbesideposition={right,top},capbesidewidth=.25\textwidth}}]{figure}[\FBwidth]%
  {
  \hfill
 \begin{subfigure}[b]{0.22\textwidth}
 \includegraphics[width=1\textwidth,trim=160 40 10 60]{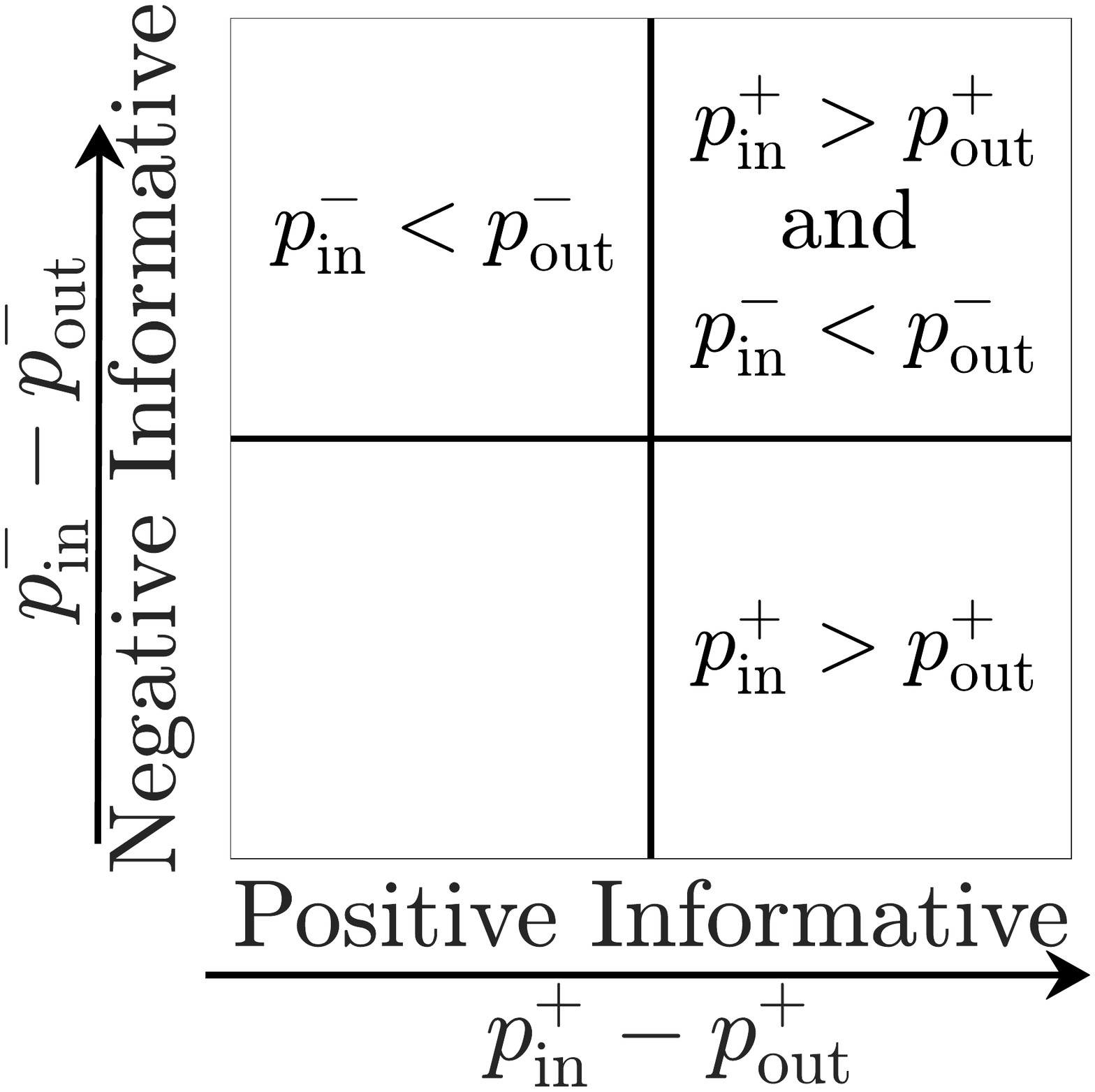}
\caption{SBM Diagram\hspace{30pt} }
 \label{fig:SBM:inExpectation:Diagram}
 \end{subfigure}
  \hfill
 \begin{subfigure}[b]{0.22\textwidth}
 \includegraphics[width=1\textwidth,trim=160 40 10 60]{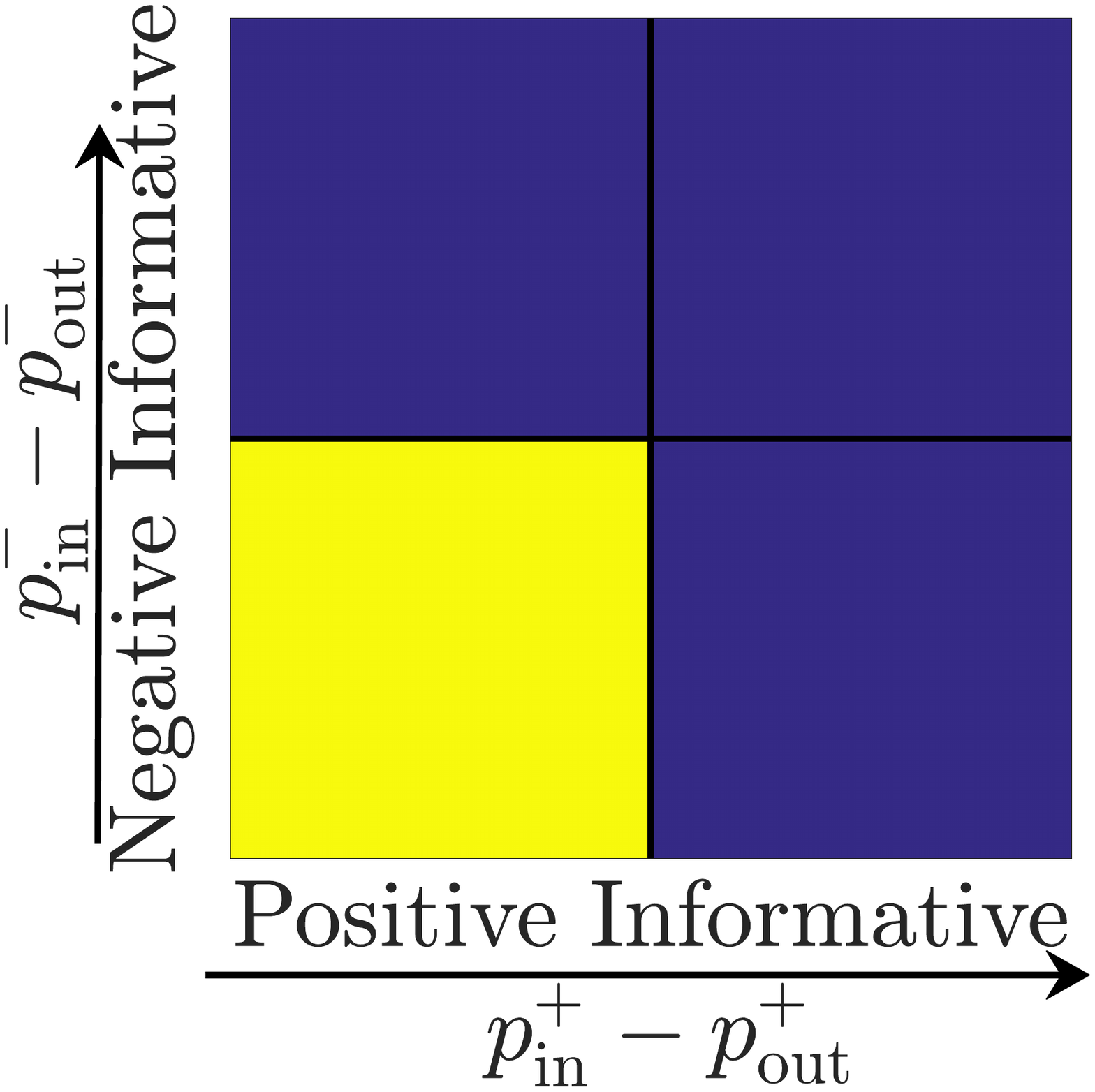}
 \caption{$L_{-\infty}$ (\texttt{OR})\hspace{30pt} }
 \label{fig:SBM:inExpectation:OR}
 \end{subfigure}
  \hfill
  \begin{subfigure}[b]{0.22\textwidth}
 \includegraphics[width=1\textwidth,trim=160 40 10 60]{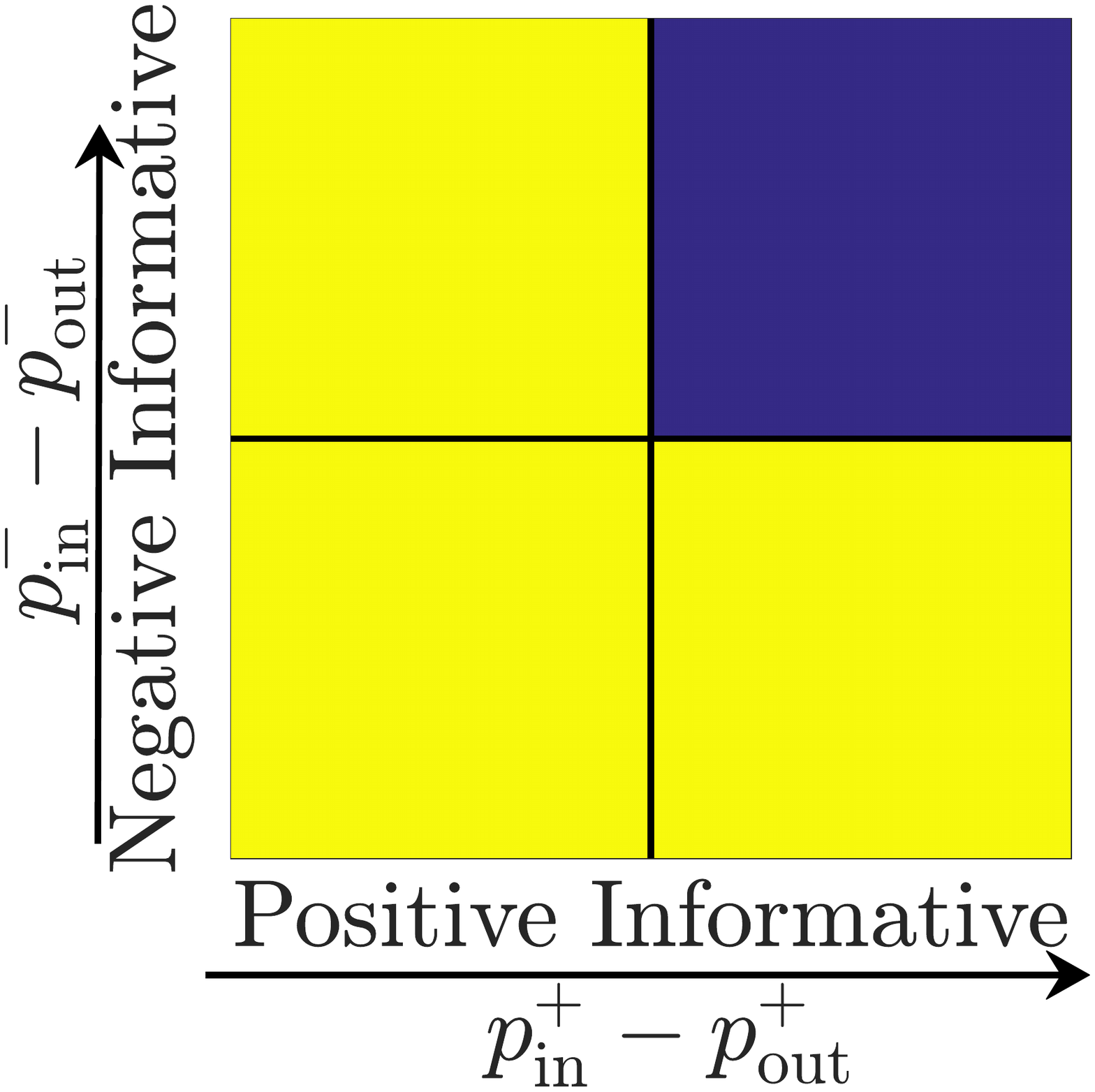}
 \caption{$L_{\infty}$ (\texttt{AND})\hspace{30pt} }
 \label{fig:SBM:inExpectation:AND}
 \end{subfigure}
    }
    {
\caption{
  Stochastic Block Model (SBM) for signed graphs. From left to right:
  Fig.~\ref{fig:SBM:inExpectation:Diagram} SBM Diagram.
  Fig.~\ref{fig:SBM:inExpectation:OR} SBM for $L_{-\infty}$(\texttt{OR}),
  Fig.~\ref{fig:SBM:inExpectation:AND} SBM for $L_{\infty}$(\texttt{AND}).
  according to Corollary~\ref{corollary:mp_limit_cases}.
  }\label{fig:SBM:inExpectation}
  }
\end{figure*}

\begin{figure*}[!htb]
 \centering
\textbf{Recovery of Clusters in Expectation\\}
\includegraphics[width=0.4\linewidth, clip,trim=0 50 0 480]{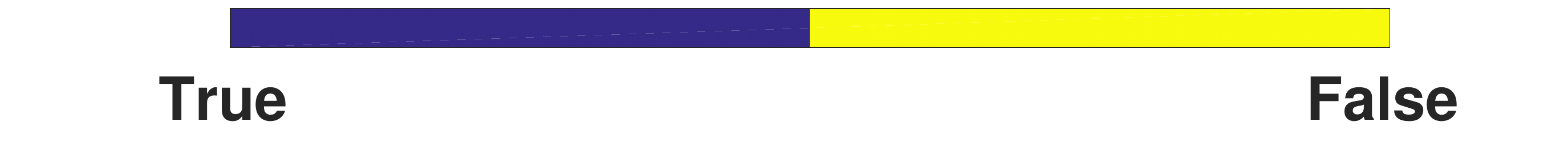}\\
 \begin{subfigure}[b]{0.19\textwidth}
 \includegraphics[width=1\textwidth,trim=160 40 10 60]{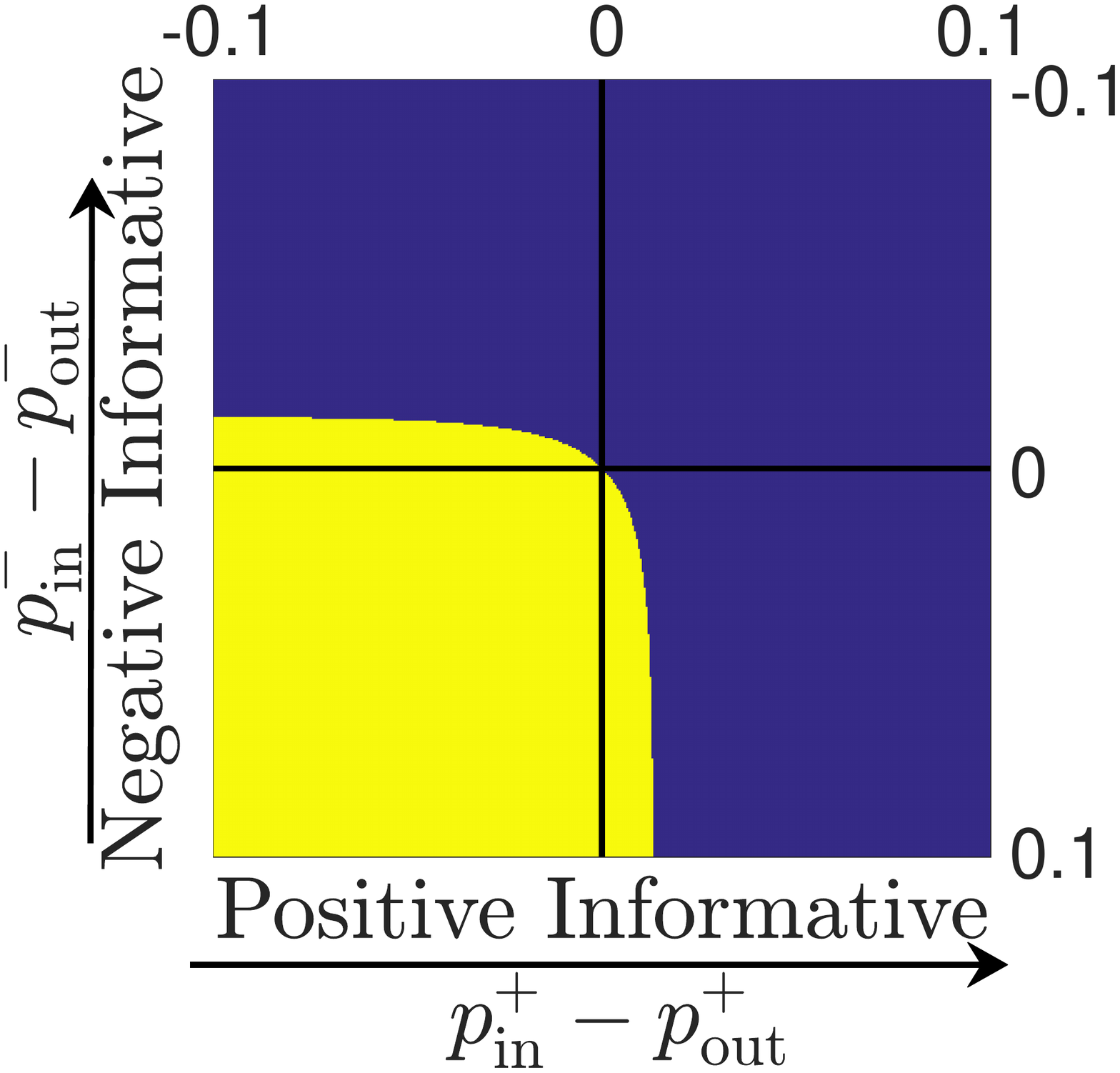}
 \caption{$\mathcal L_{-10}$\hspace{30pt} }
 \label{fig:SBM:inExpectation:Minus10}
 \vspace{2.5pt}
 \end{subfigure}
 \hfill
 \begin{subfigure}[b]{0.19\textwidth}
 \includegraphics[width=1\textwidth,trim=160 40 10 60]{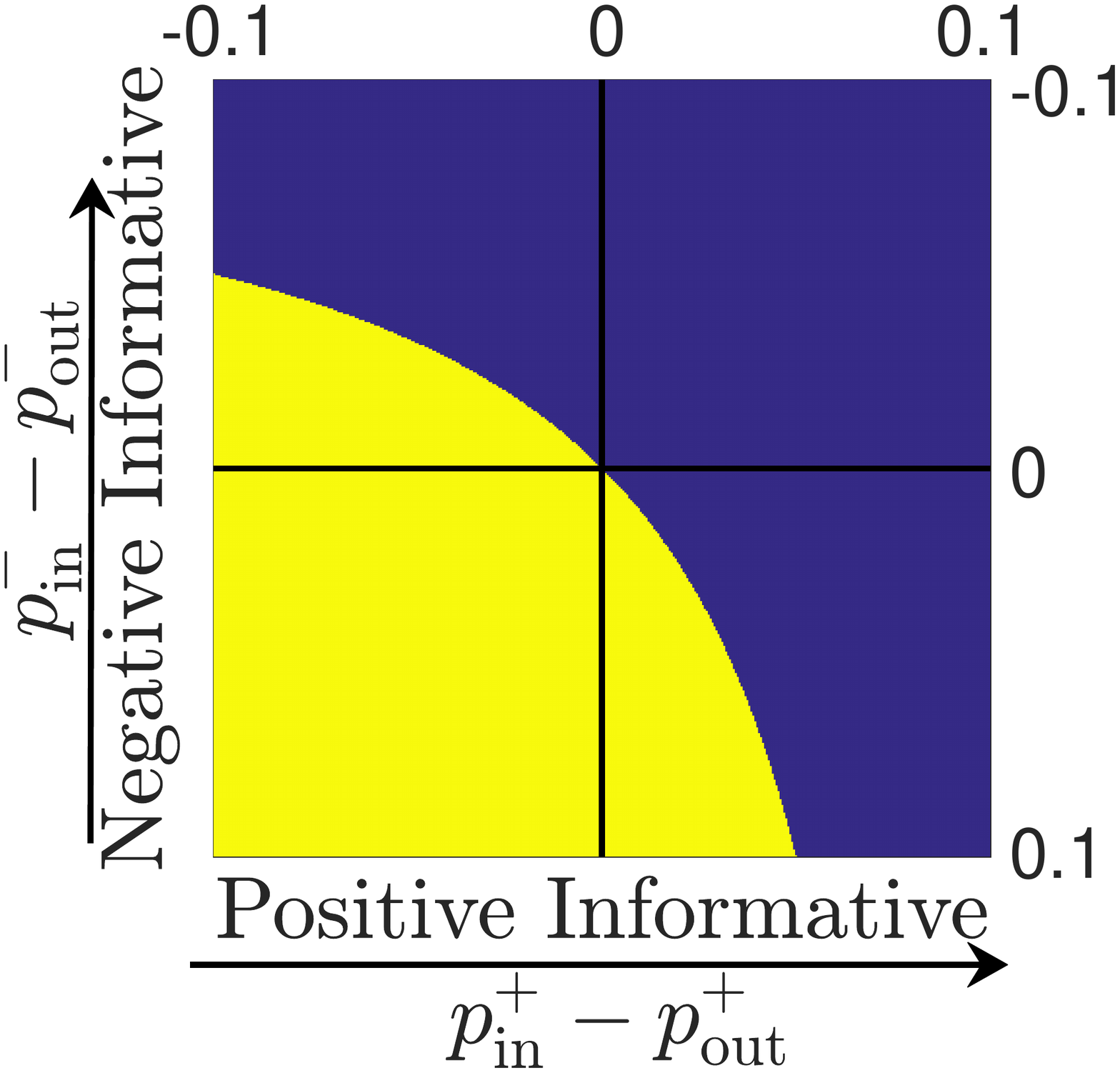}
 \caption{$\mathcal L_{-1}$\hspace{30pt} }
 \label{fig:SBM:inExpectation:Minus1}
 \vspace{2.5pt}
 \end{subfigure}
 \hfill
 \begin{subfigure}[b]{0.19\textwidth}
 \includegraphics[width=1\textwidth,trim=160 40 10 60]{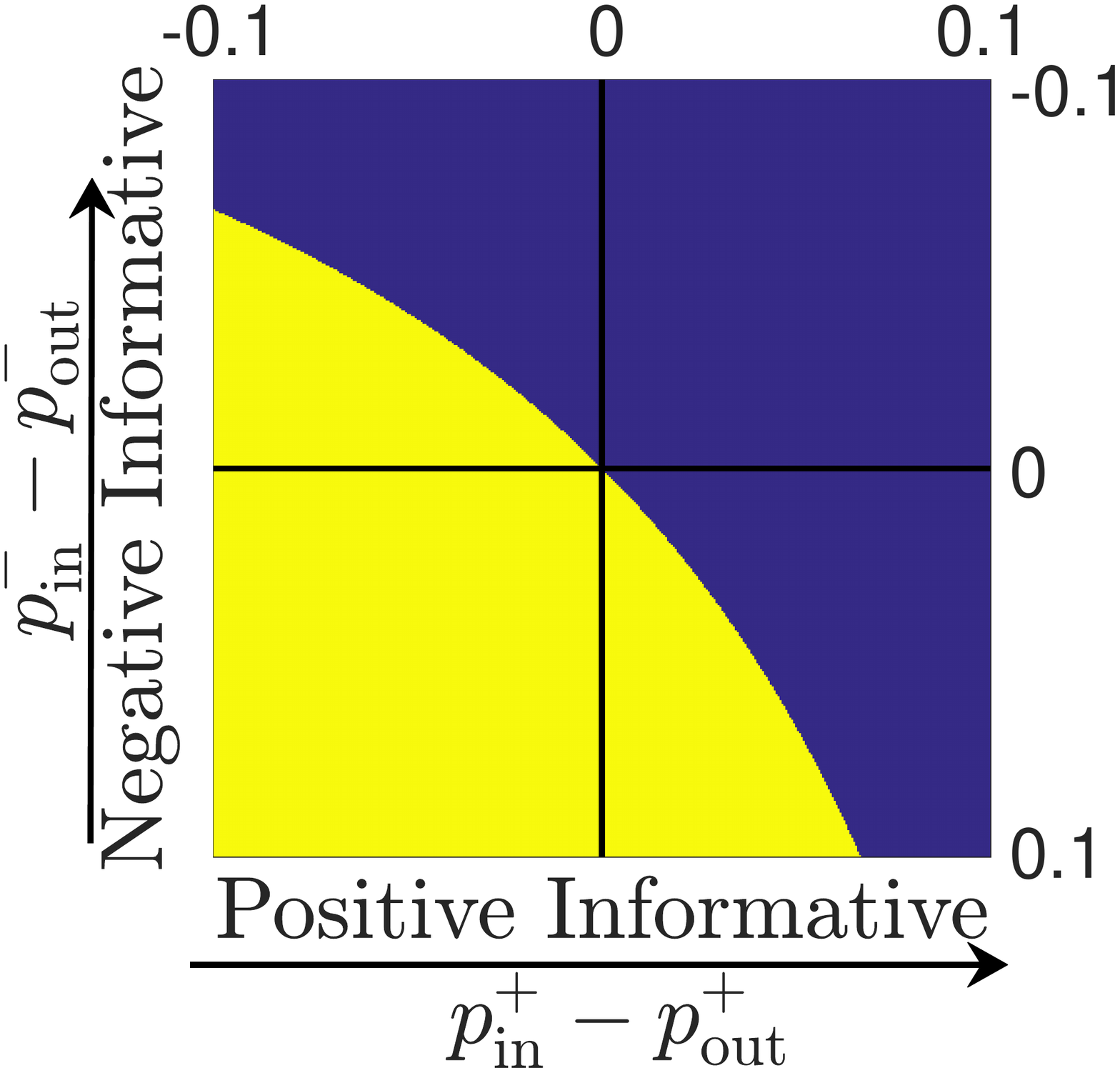}
 \caption{$\mathcal L_{0}$\hspace{30pt} }
 \label{fig:SBM:inExpectation:Zero}
 \vspace{2.5pt}
 \end{subfigure}
 \hfill
 \begin{subfigure}[b]{0.19\textwidth}
 \includegraphics[width=1\textwidth,trim=160 40 10 60]{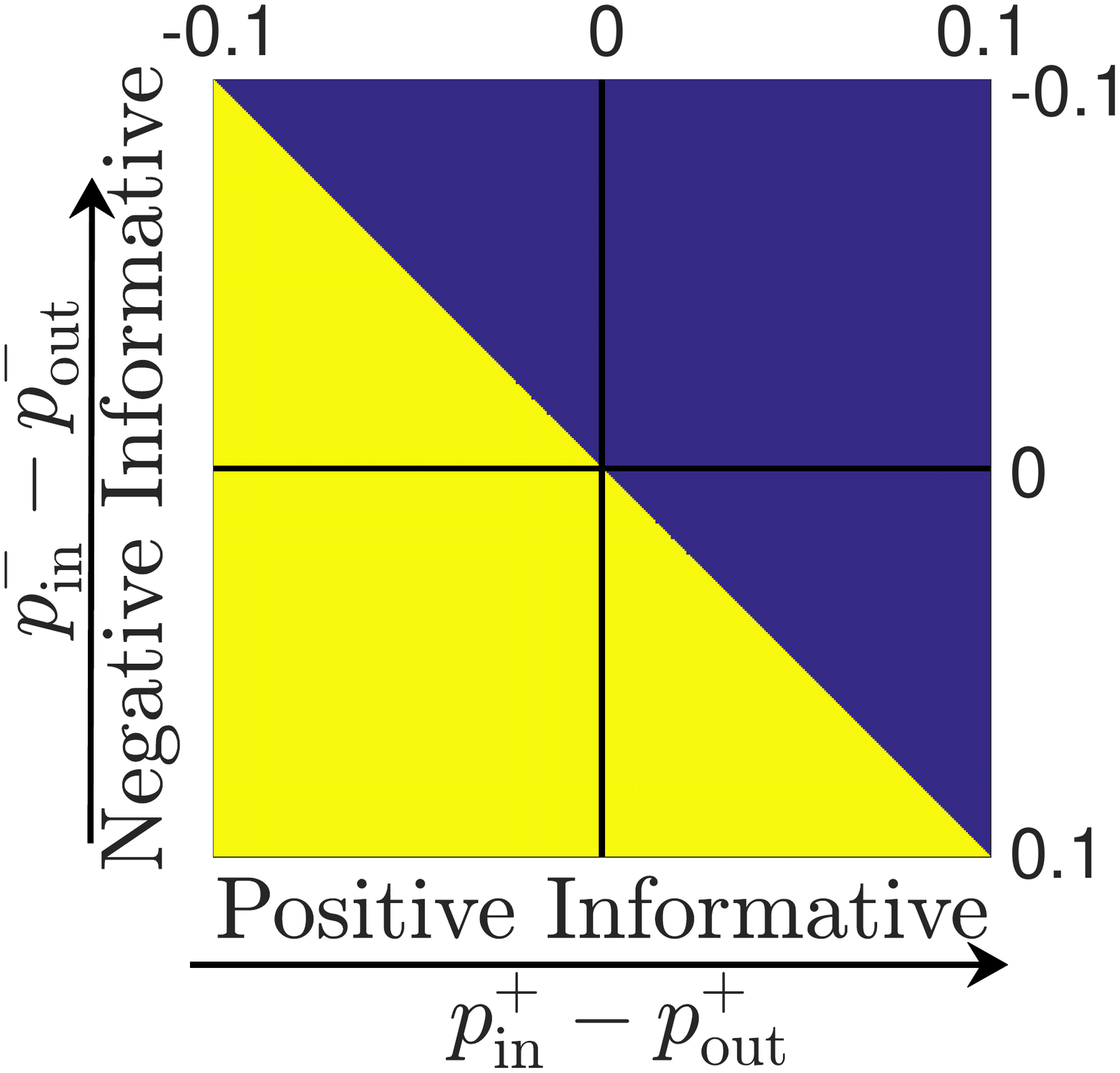}
 \caption{$\mathcal L_{1}$\hspace{30pt} }
 \label{fig:SBM:inExpectation:Plus1}
 \vspace{2.5pt}
 \end{subfigure}
 \hfill
 \begin{subfigure}[b]{0.19\textwidth}
 \includegraphics[width=1\textwidth,trim=160 40 10 60]{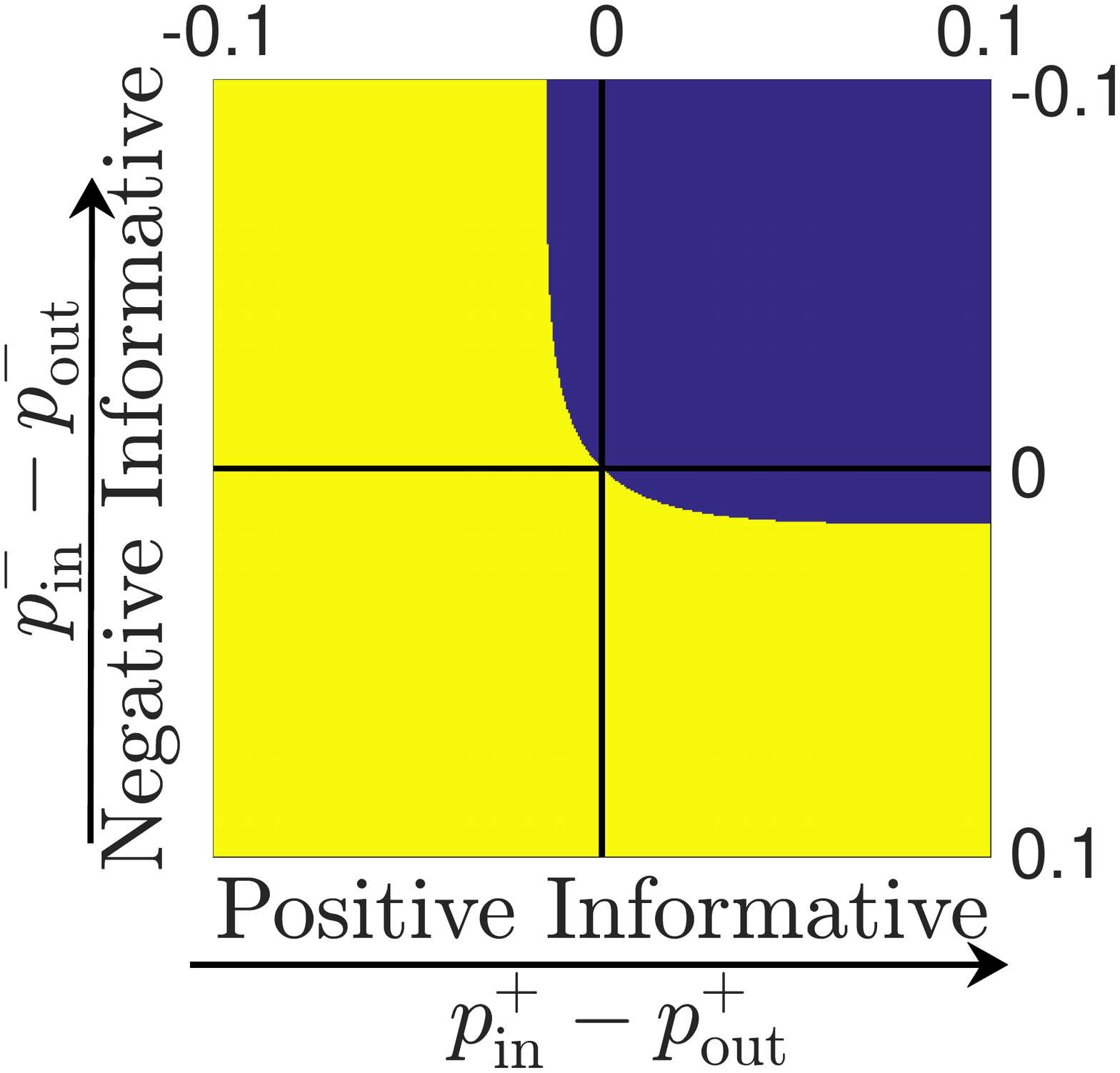}
 \caption{$\mathcal L_{10}$\hspace{30pt} }
 \label{fig:SBM:inExpectation:Plus10}
 \vspace{2.5pt}
 \end{subfigure}
\textbf{Clustering Error\\}
\includegraphics[width=0.4\linewidth, clip,trim=0 50 0 480]{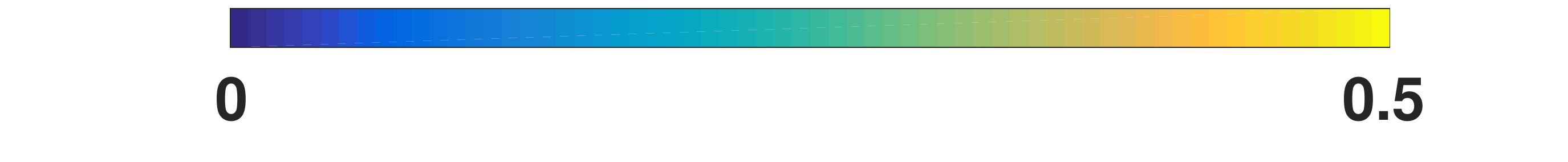}\\
 \vspace{5pt}
 \begin{subfigure}[b]{0.19\textwidth}
 \includegraphics[width=1\textwidth,trim=160 40 10 60]{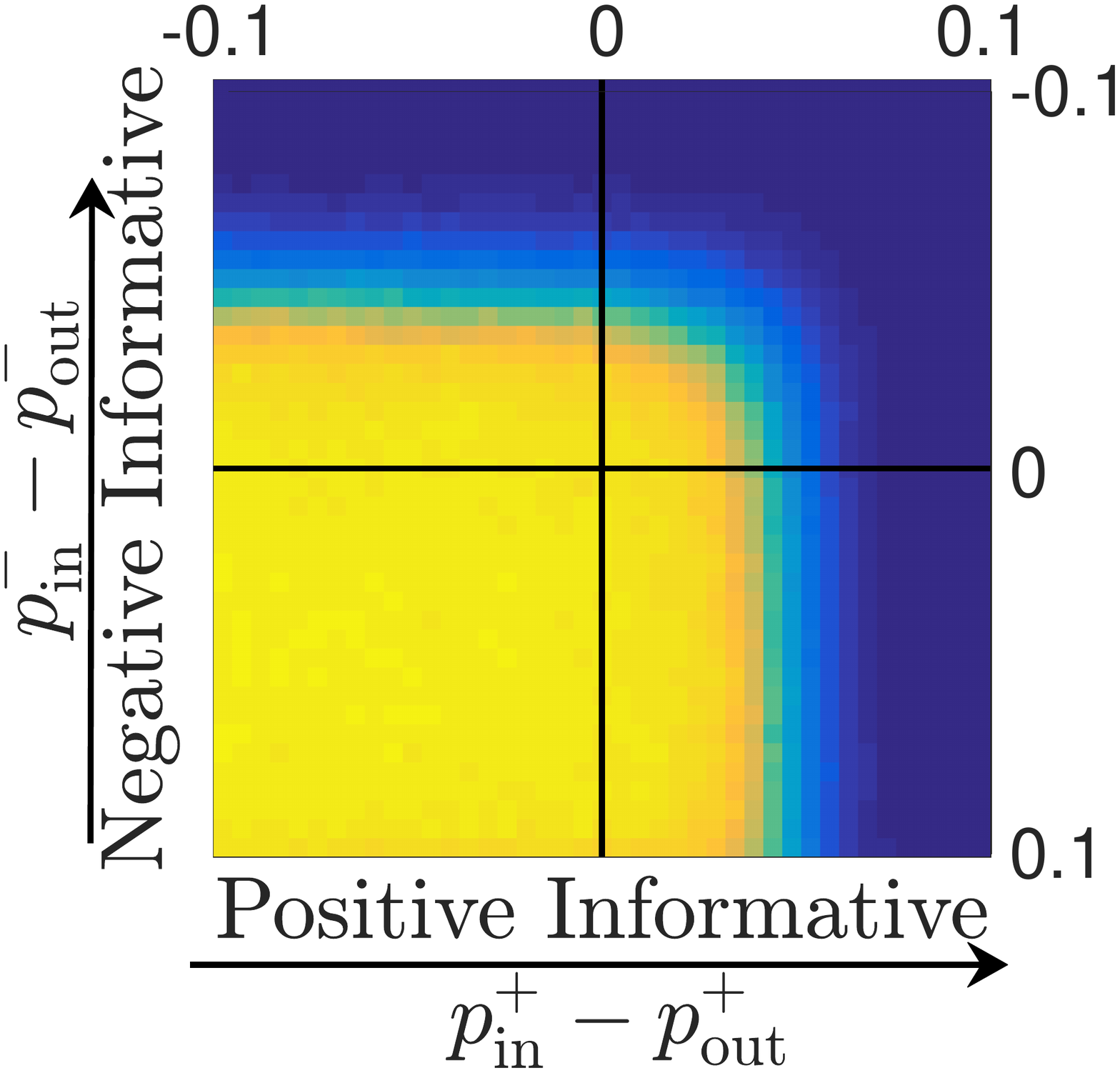}
 \caption{$L_{-10}$\hspace{30pt} }
 \label{fig:SBM:sparse:Minus10}
 \end{subfigure}
 \hfill
 \begin{subfigure}[b]{0.19\textwidth}
 \includegraphics[width=1\textwidth,trim=160 40 10 60]{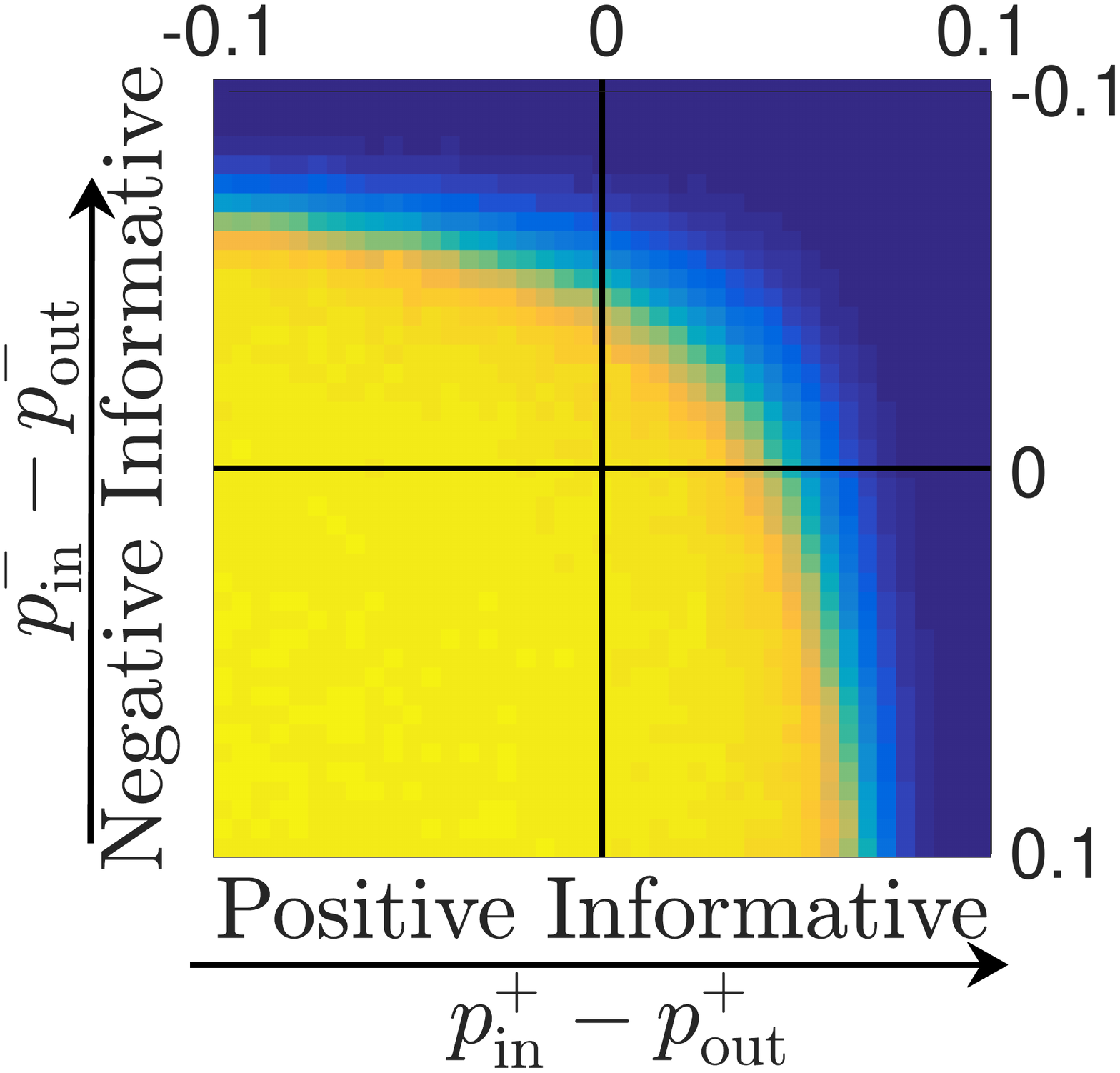}
 \caption{$L_{-1}$\hspace{30pt} }
 \label{fig:SBM:sparse:Minus1}
 \end{subfigure}
 \hfill
 \begin{subfigure}[b]{0.19\textwidth}
 \includegraphics[width=1\textwidth,trim=160 40 10 60]{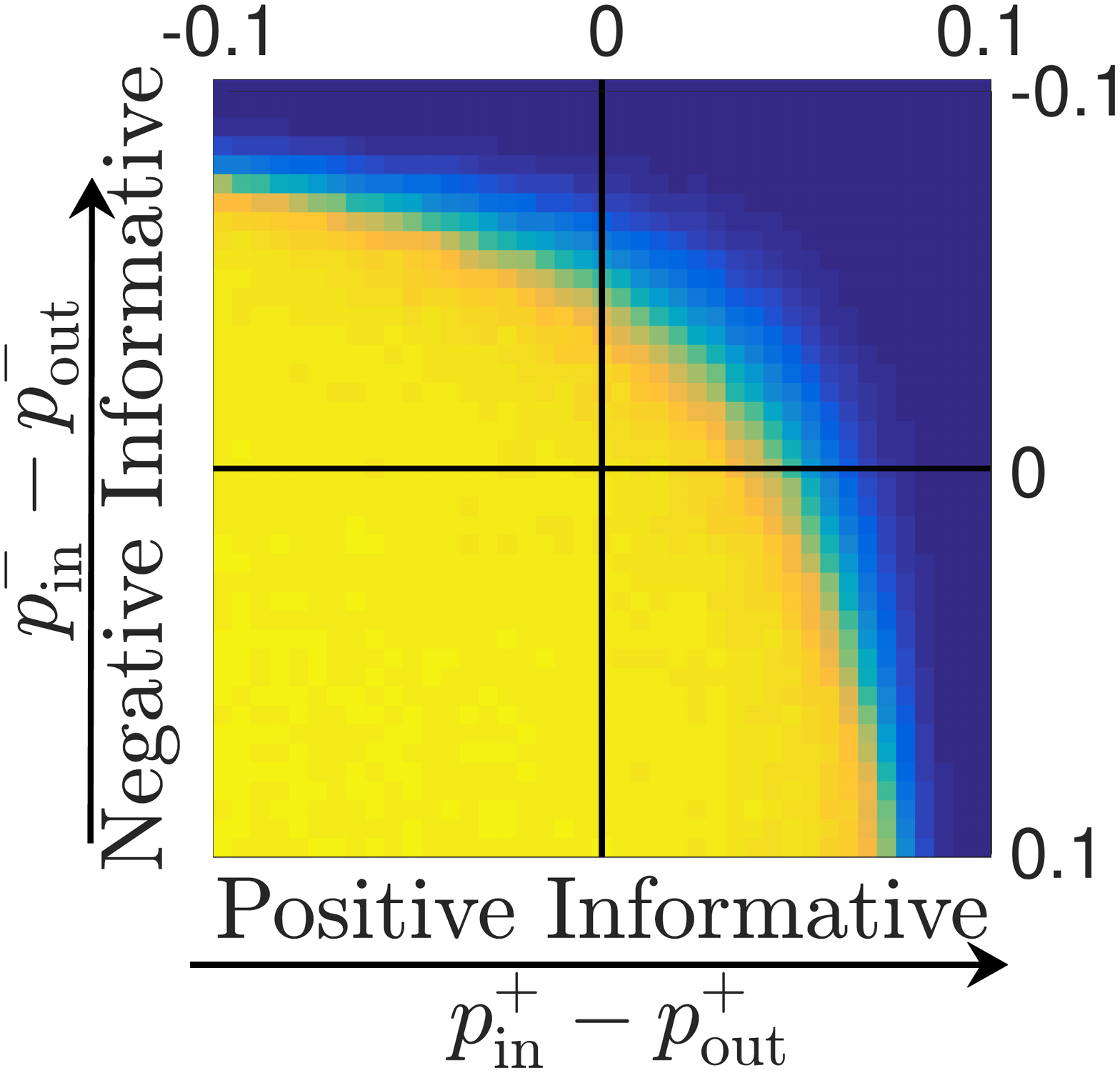}
 \caption{$L_{0}$\hspace{30pt} }
 \label{fig:SBM:sparse:Zero}
 \end{subfigure}
 \hfill
 \begin{subfigure}[b]{0.19\textwidth}
 \includegraphics[width=1\textwidth,trim=160 40 10 60]{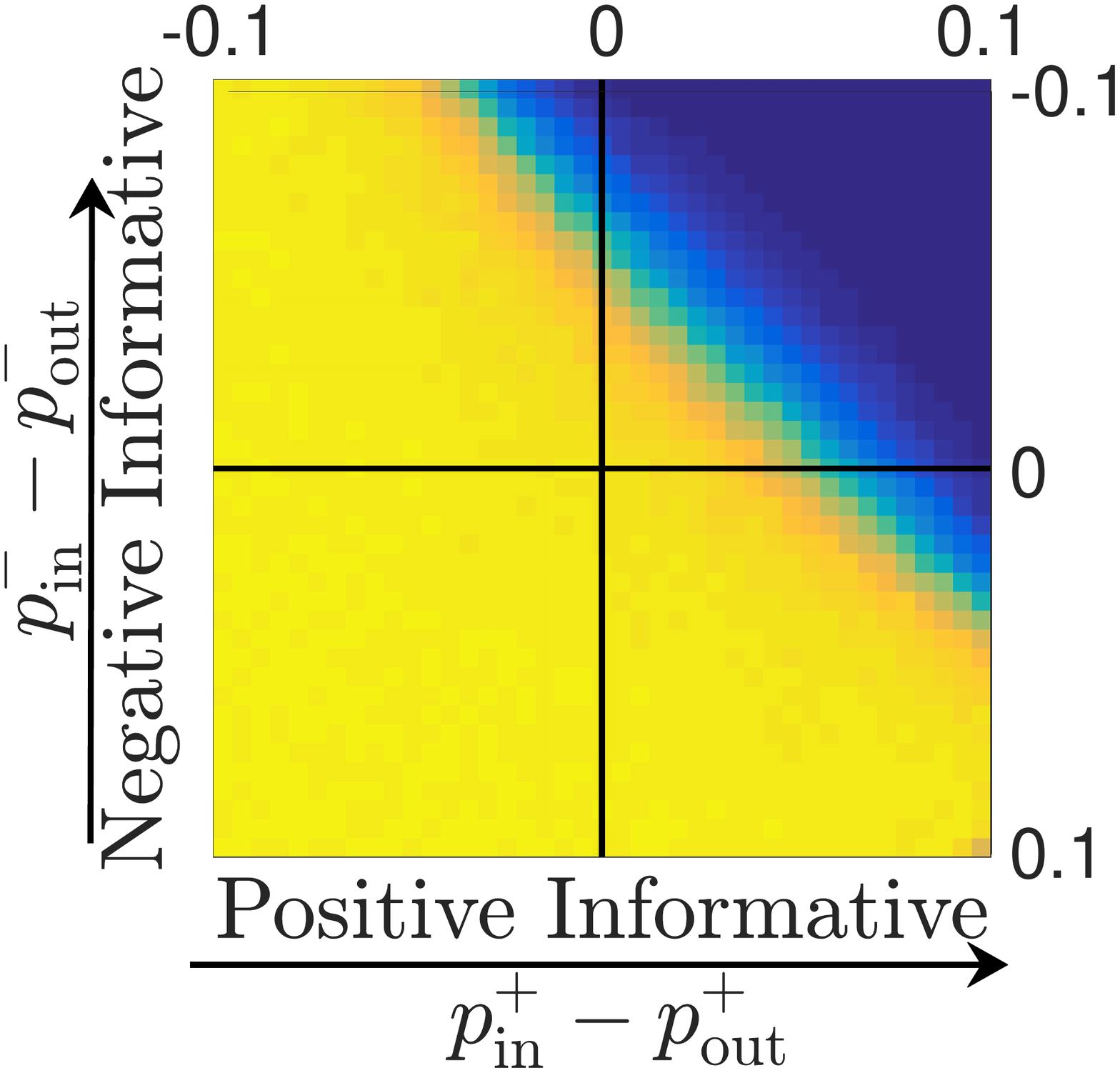}
 \caption{$L_{1}$ \hspace{30pt} }
 \label{fig:SBM:sparse:Plus1}
 \end{subfigure}
 \hfill
 \begin{subfigure}[b]{0.19\textwidth}
 \includegraphics[width=1\textwidth,trim=160 40 10 60]{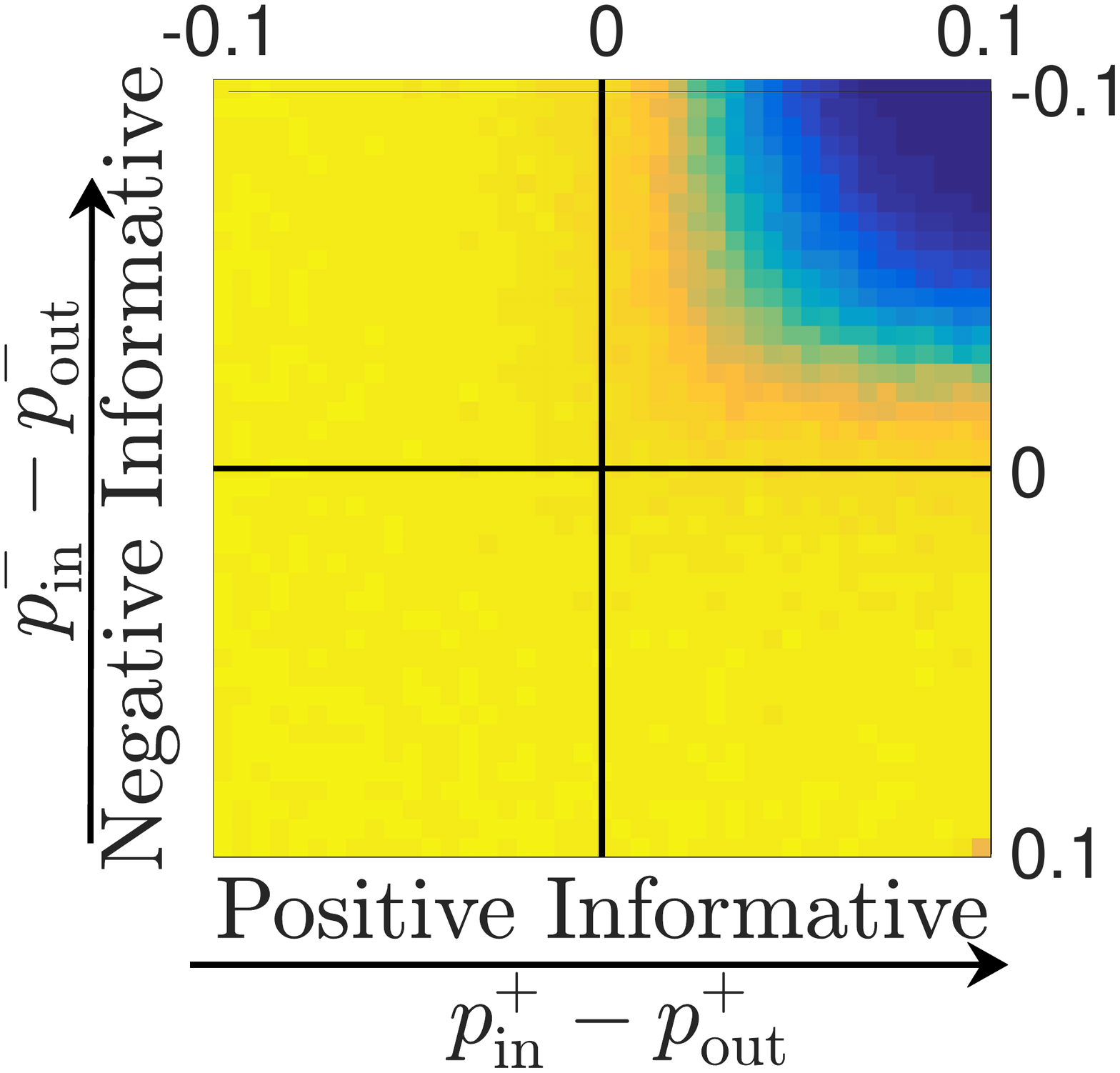}
 \caption{$L_{10}$ \hspace{30pt} }
 \label{fig:SBM:sparse:Plus10}
 \end{subfigure}
\hfill
\\
 \vspace{5pt}
 \begin{subfigure}[b]{0.19\textwidth}
 \includegraphics[width=1\textwidth,trim=160 40 10 60]{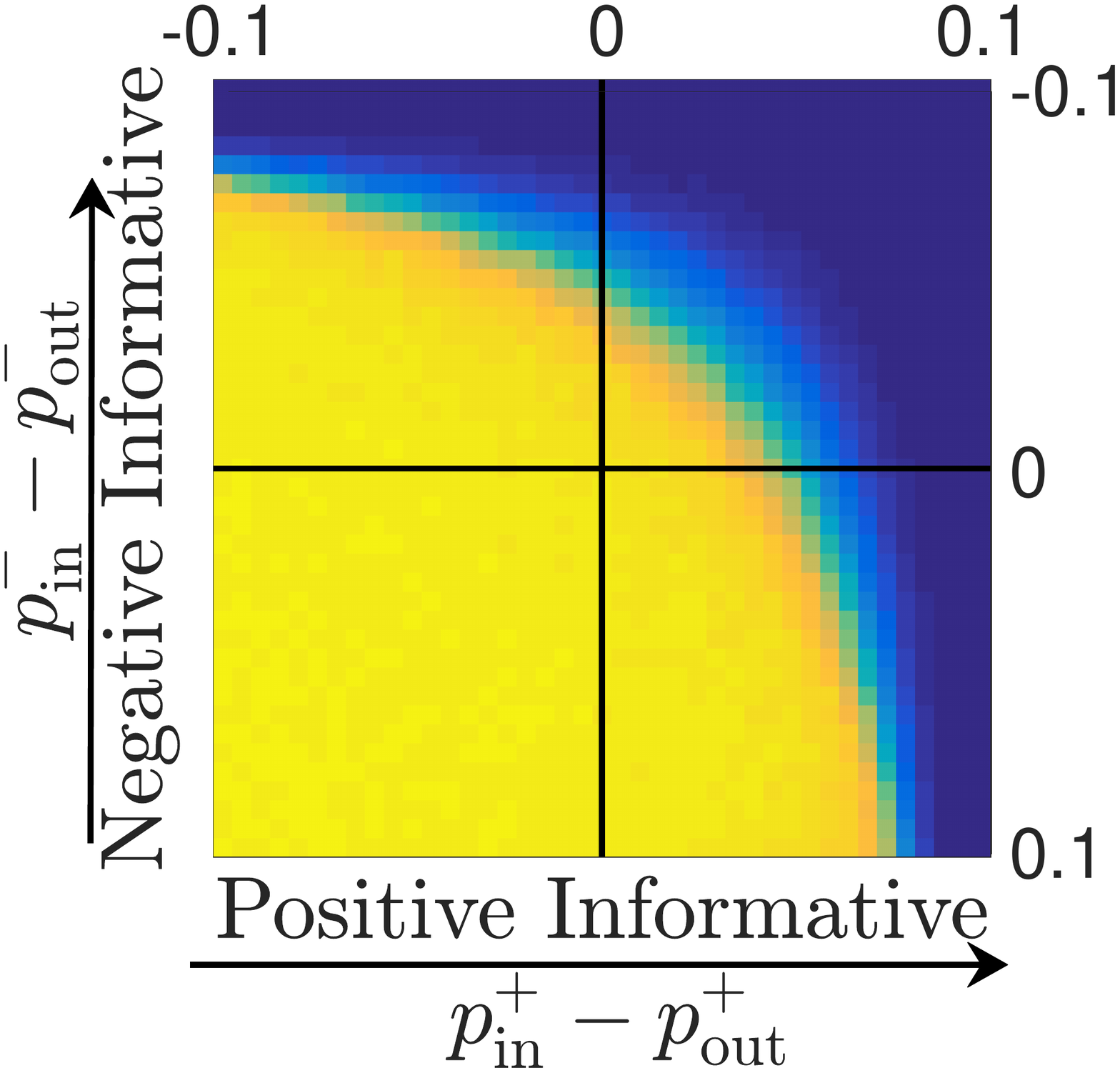}
 \caption{$L_{GM}$\hspace{30pt} }
 \label{fig:SBM:sparse:L_GM}
 \end{subfigure}
 \hfill
 \begin{subfigure}[b]{0.19\textwidth}
 \includegraphics[width=1\textwidth,trim=160 40 10 60]{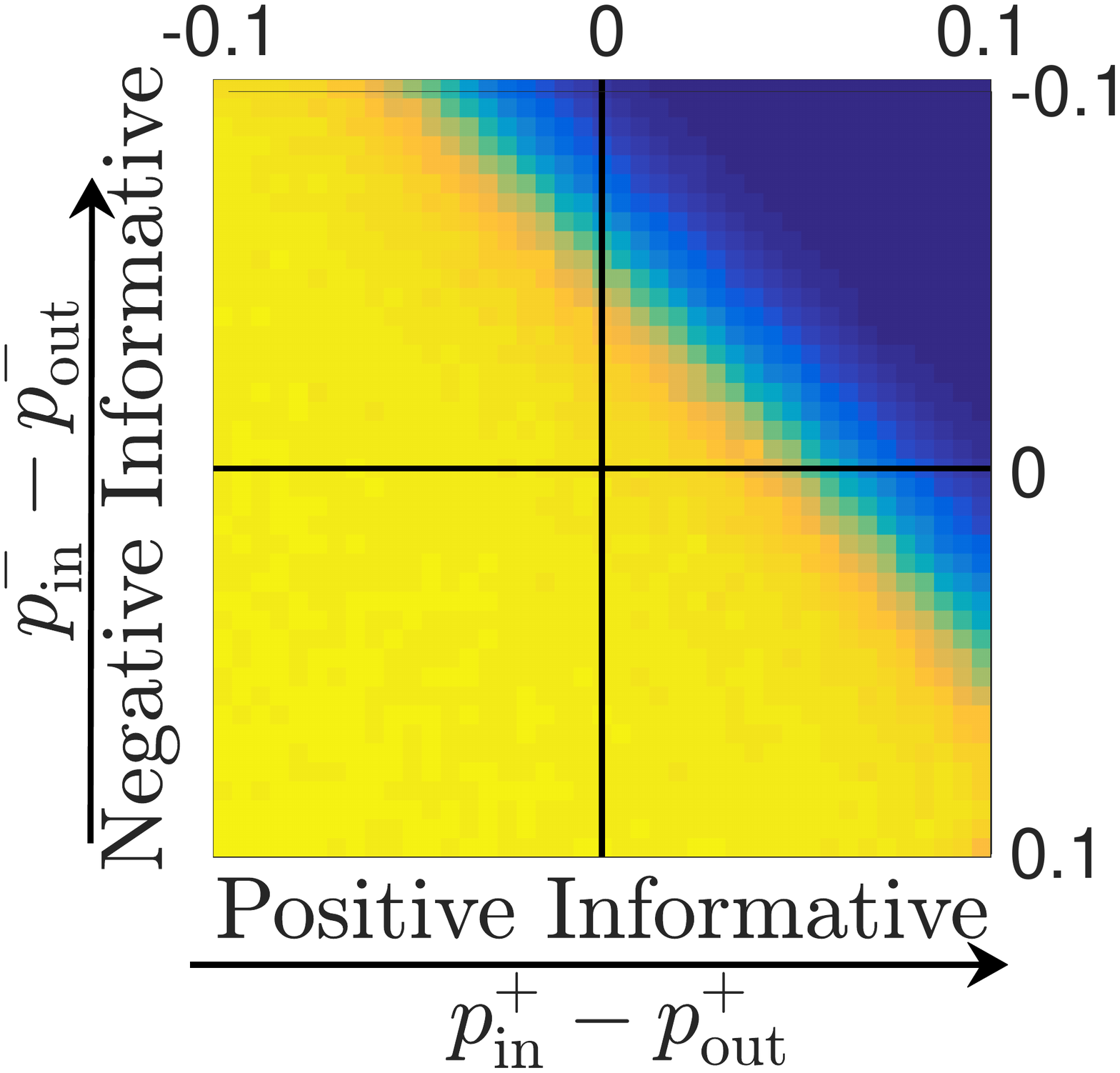}
 \caption{$L_{SN}$\hspace{30pt} }
 \label{fig:SBM:sparse:L_SN}
 \end{subfigure}
 \hfill
 \begin{subfigure}[b]{0.19\textwidth}
 \includegraphics[width=1\textwidth,trim=160 40 10 60]{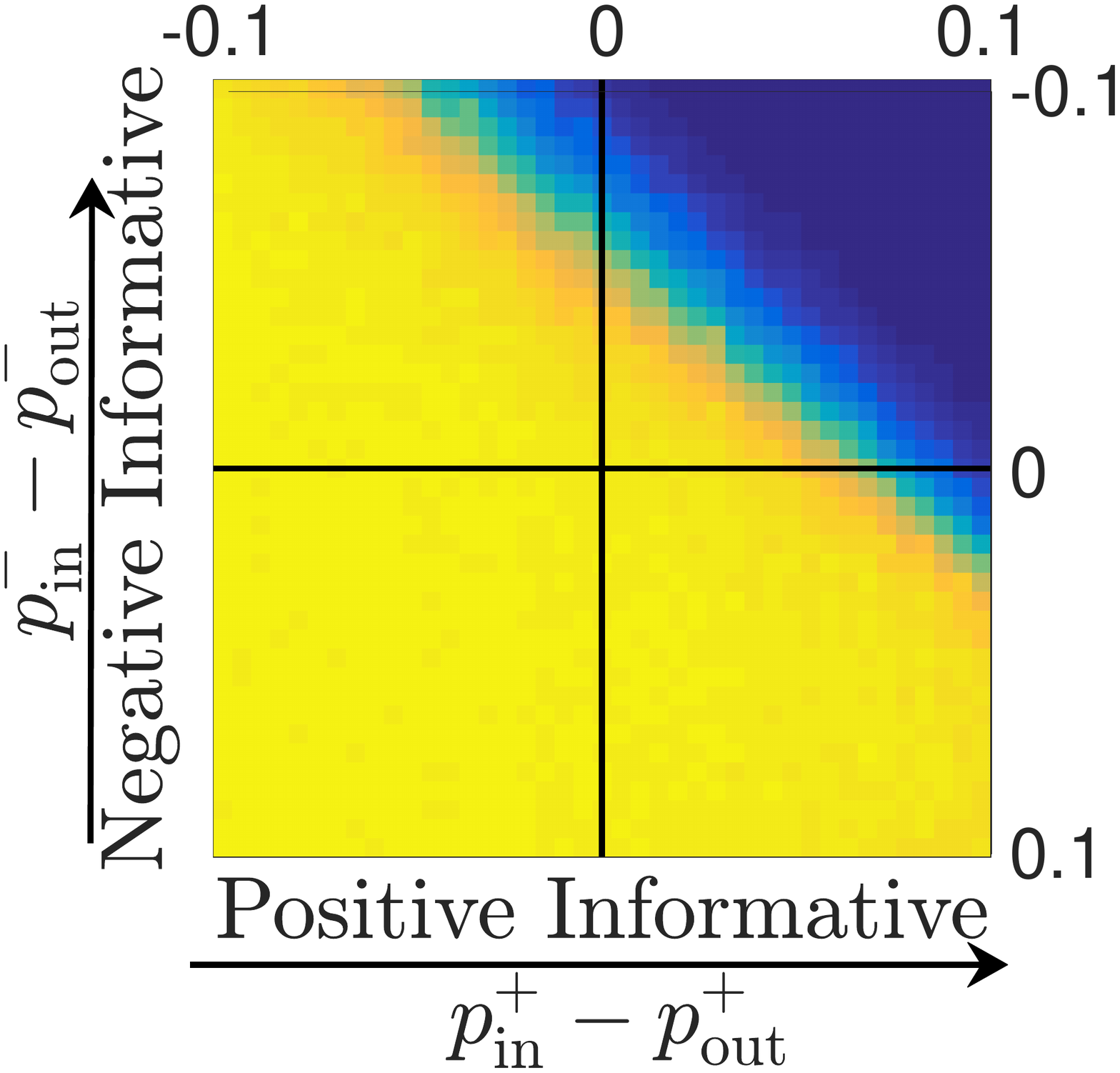}
 \caption{$L_{BN}$\hspace{30pt} }
 \label{fig:SBM:sparse:L_BN}
 \end{subfigure}
 \hfill
 \begin{subfigure}[b]{0.19\textwidth}
 \includegraphics[width=1\textwidth,trim=160 40 10 60]{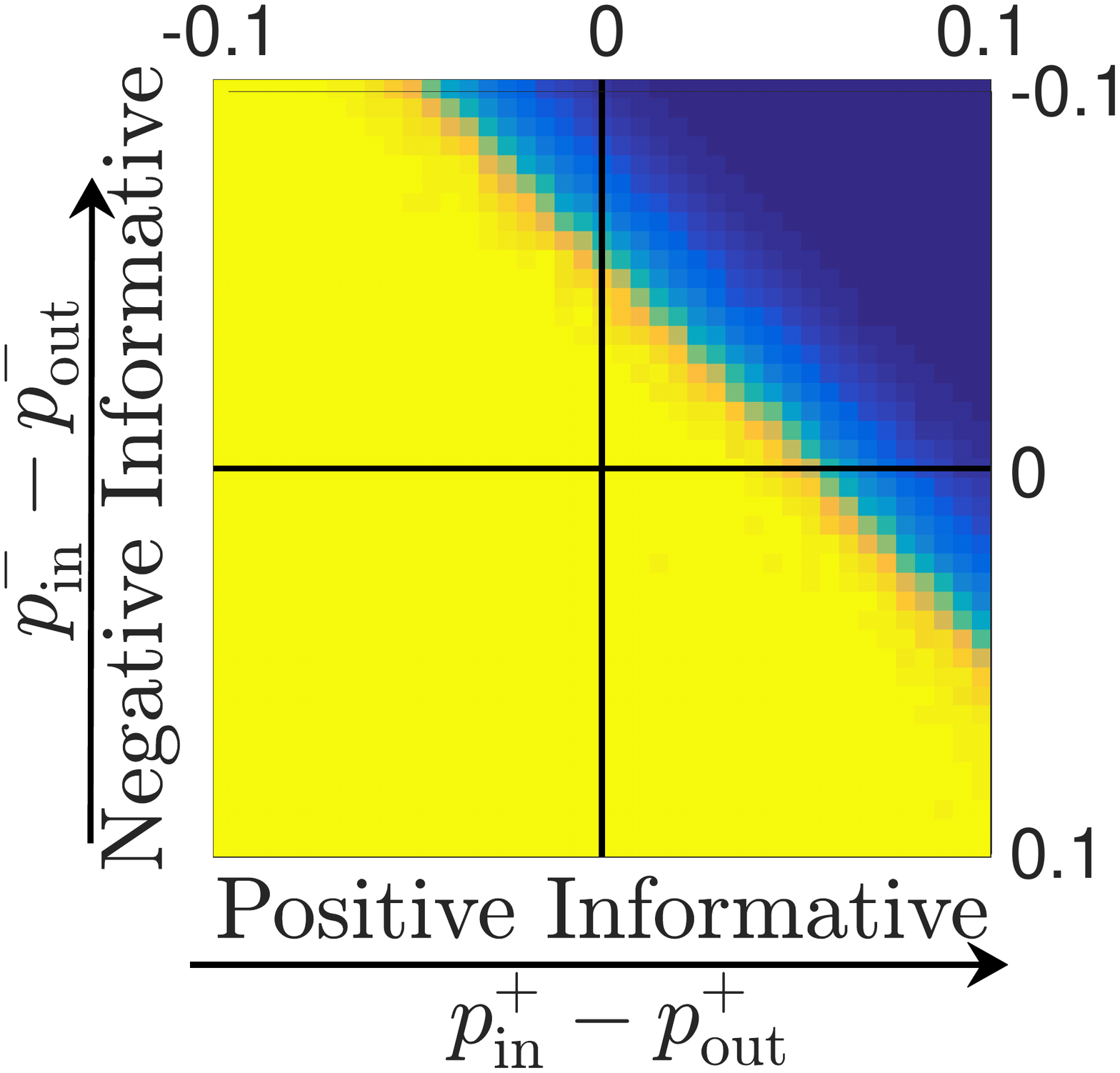}
 \caption{$H$\hspace{30pt} }
 \label{fig:SBM:sparse:BH}
 \end{subfigure}
  \hfill
 \begin{subfigure}[b]{0.19\textwidth}
 \color{white}
 \includegraphics[draft ,width=1\textwidth,trim=160 40 10 60]{random_bethe_hessian.pdf}
 \color{black}
 \label{fig:SBM:sparse:BHH}
 \end{subfigure}
 \hfill
 \vspace{-2.5pt}
\caption{
Performance visualization for two clusters for different parameters of the SBM.
\textbf{Top row}: In dark blue the settings where the signed power mean Laplacians $\mathcal L_p$ identify the ground truth clusters in expectation for the SBM, see Theorem~\ref{theorem:mp_in_expectation}, whereas yellow indicates failure.
\textbf{Middle/Bottom row}: average clustering error (dark blue: small error, yellow: large error) of the signed power mean Laplacian $L_p$ and $L_{GM},L_{SN},L_{BN},H$  for 50 samples from the
SBM. 
 \vspace{-10pt}
}
 \label{fig:SBM:inExpectation:GENERAL}
\end{figure*}

In Fig.~\ref{fig:SBM:inExpectation:GENERAL} we present the corresponding conditions for recovery in expectation for the cases $p\in\{-10,-1,0,1,10\}$.
We can visually verify that the larger the value of $p$ the smaller is the region where the conditions of Theorem~\ref{theorem:mp_in_expectation} hold. In particular, one can compare the change of conditions as one moves from the signed harmonic ($\mathcal{L}_{-1}$), geometric ($\mathcal{L}_0$), to the arithmetic ($\mathcal{L}_1$) mean Laplacians verifying the ordering described in~Corollary~\ref{corollary:contention}. Moreover, we clearly observe that $\mathcal{L}_{-10}$ and $\mathcal{L}_{10}$  are already quite close to the conditions necessary for the limit cases $\mathcal{L}_{-\infty}$ and $\mathcal{L}_{\infty}$, respectively.

In the  middle row of Fig.~\ref{fig:SBM:inExpectation:GENERAL} we show the average clustering error for each power mean Laplacian when sampling 50 times from the SSBM following the diagram presented in Fig.~\ref{fig:SBM:inExpectation:Diagram} and fixing the sparsity of $G^+$ and $G^-$ by setting $\pp+\qp=0.1$ and $\ppm+\qm=0.1$ with two clusters each of size 100. We observe that the areas with low clustering error qualitatively match the regions where in expectation we have recovery of the clusters. However, due to the sampling which can make one of the graphs $G^+$ and $G^-$ quite sparse and as we just consider graphs with 200 nodes, the area of low clustering error is smaller in comparison to the region of guaranteed recovery in expectation due to the sampling variance in the stochastic block model.

In the bottom row of Fig.~\ref{fig:SBM:inExpectation:GENERAL} we show the clustering error for the state of the art methods $L_{GM},L_{SN},L_{BM}$ and $H$. We can see that $L_{GM}$ presents a similar performance as the signed power mean  Laplacian $L_{0}$. 
The next Theorem shows that the geometric mean Laplacian $\mathcal L_{GM}$ and the limit $p\rightarrow 0$ of the signed power mean Laplacian agree in expectation for the SSBM. This implies via~Corollary~\ref{corollary:contention} that this operator is inferior to the signed power mean Laplacian for $p<0$. 
This is why we use in the experiments on real world graphs later on always $p<0$. 

\begin{theorem}\label{theorem:geometricMeanLaplacian}
Let $\mathcal L_{GM}=\mathcal{L}_{\sym}^+ \# \mathcal{Q}_\sym^-$  and $\mathcal L_{0}$ be the signed power mean Laplacian with $p\to 0$ of the expected signed graph. Then, $\mathcal L_{0} = \mathcal L_{GM}$.
\end{theorem}

 \myComment{
\begin{proof}
Please see Section~\ref{theorem:geometricMeanLaplacian-PROOF}.
\end{proof}
}

\begin{figure*}[!t]
\centering
\begin{subfigure}[]{0.24\linewidth}
\includegraphics[width=1\linewidth, clip,trim=130 40 170 40]{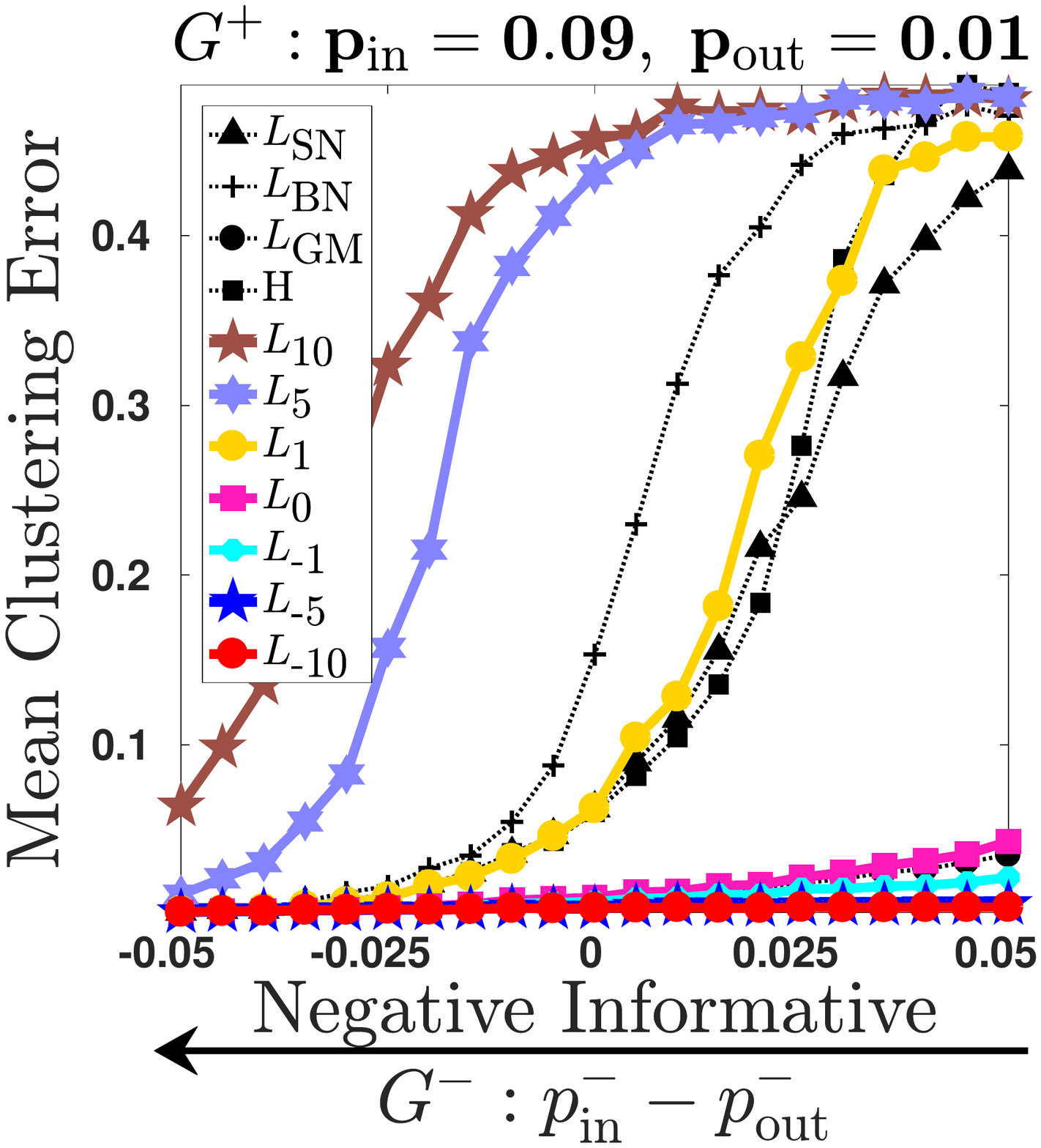}\hspace*{\fill}
 \vspace{-2.5pt}
\caption{}
\label{subfig:Wpos_fixed}
\end{subfigure}%
%
%
\begin{subfigure}[]{0.24\linewidth}
 \includegraphics[width=1\linewidth, clip,trim=130 40 170 40]{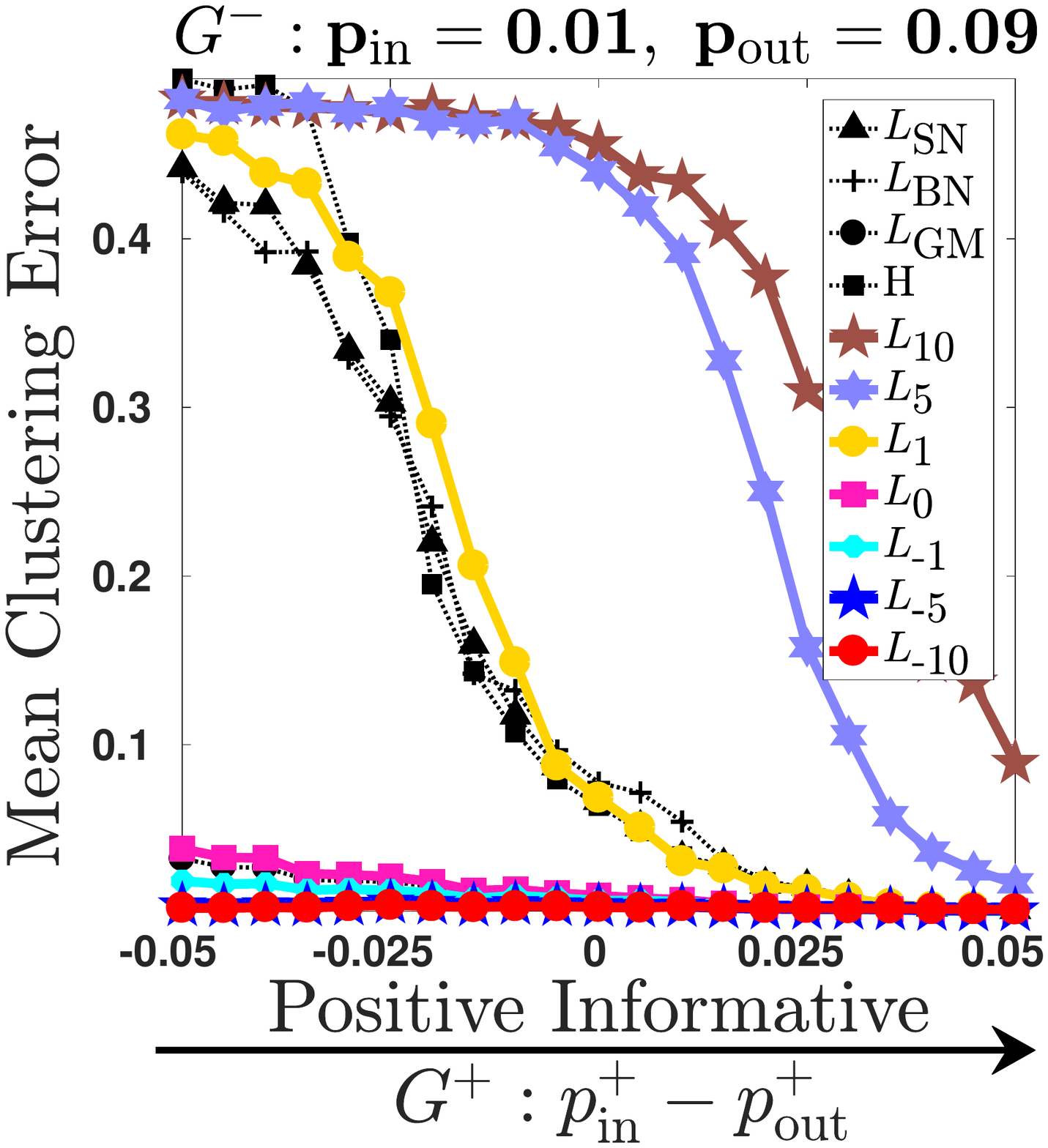}\hspace*{\fill}
 \vspace{-2.5pt}
\caption{}
\label{subfig:Wneg_fixed}
\end{subfigure}%
\hfill
\begin{minipage}{.5\textwidth}
\begin{subfigure}[b]{0.24\textwidth}
 \includegraphics[width=1\textwidth,trim=160 90 125 60,clip]{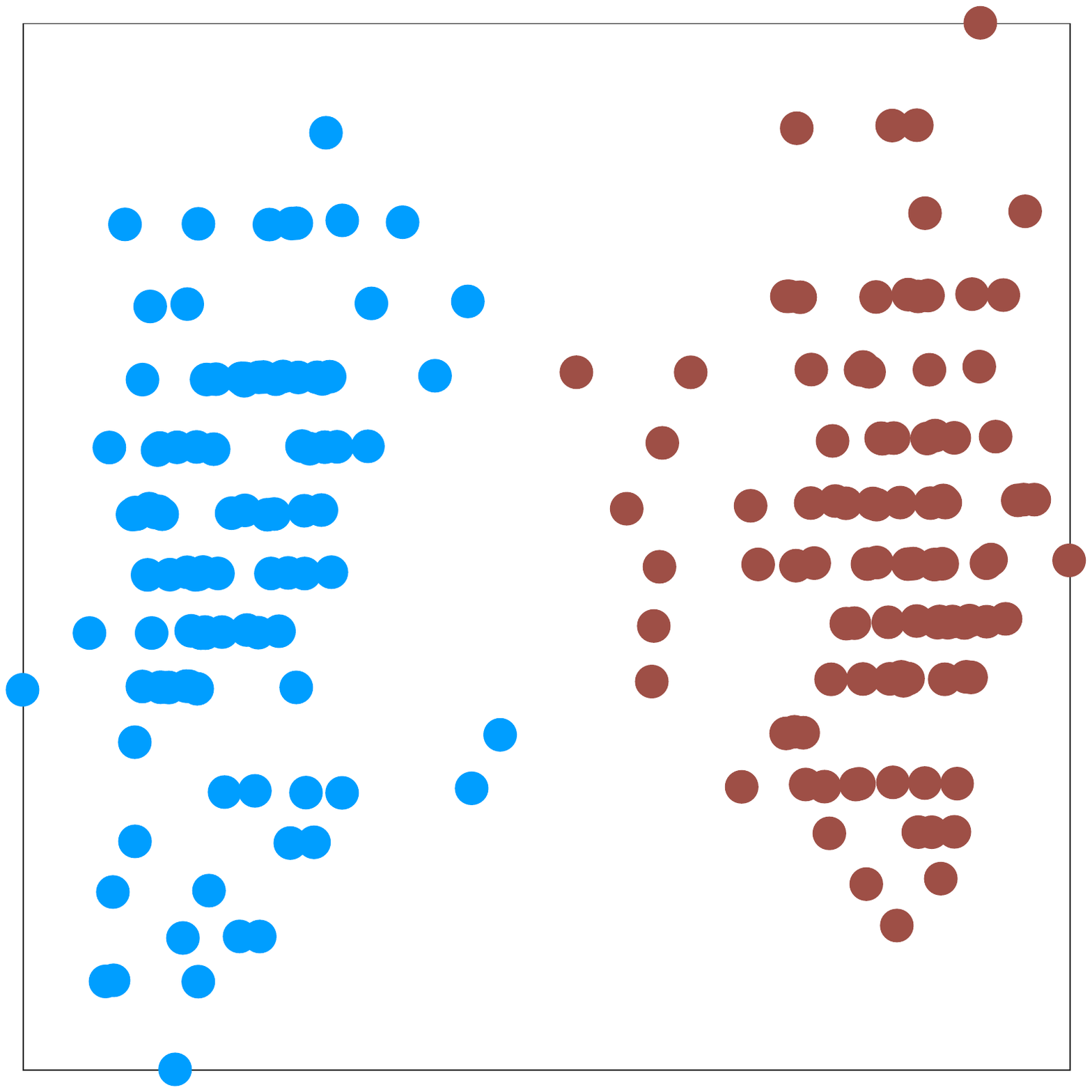}
 \caption{$L_{-10}$}
 \label{fig:SBM:sparse:Minus10:embedding}
 \end{subfigure}
 \begin{subfigure}[b]{0.24\textwidth}
 \includegraphics[width=1\textwidth,trim=160 90 125 60,clip]{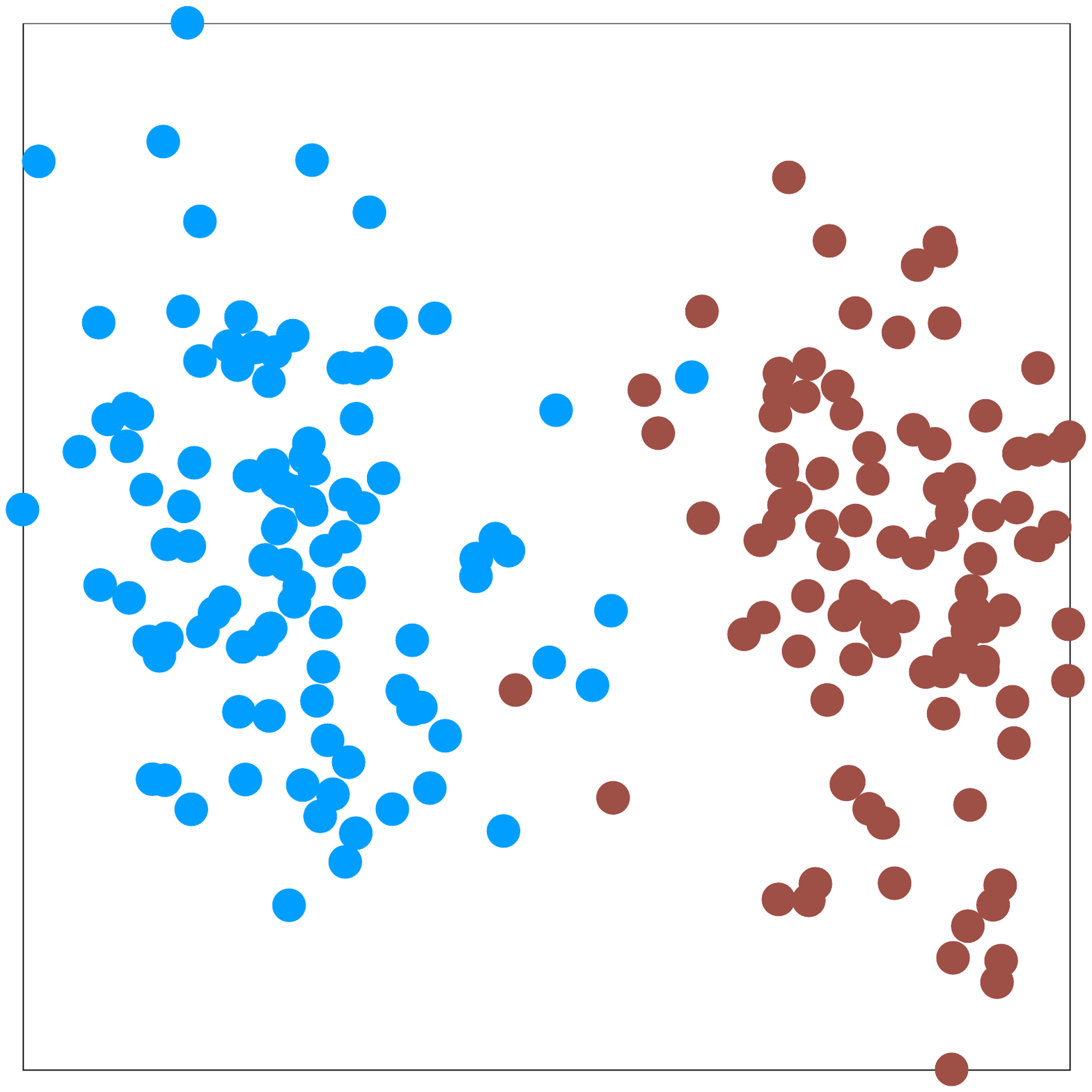}
 \caption{$L_{-1}$}
 \label{fig:SBM:sparse:Minus1:embedding}
 \end{subfigure}
 \begin{subfigure}[b]{0.24\textwidth}
 \includegraphics[width=1\textwidth,trim=160 90 125 60,clip]{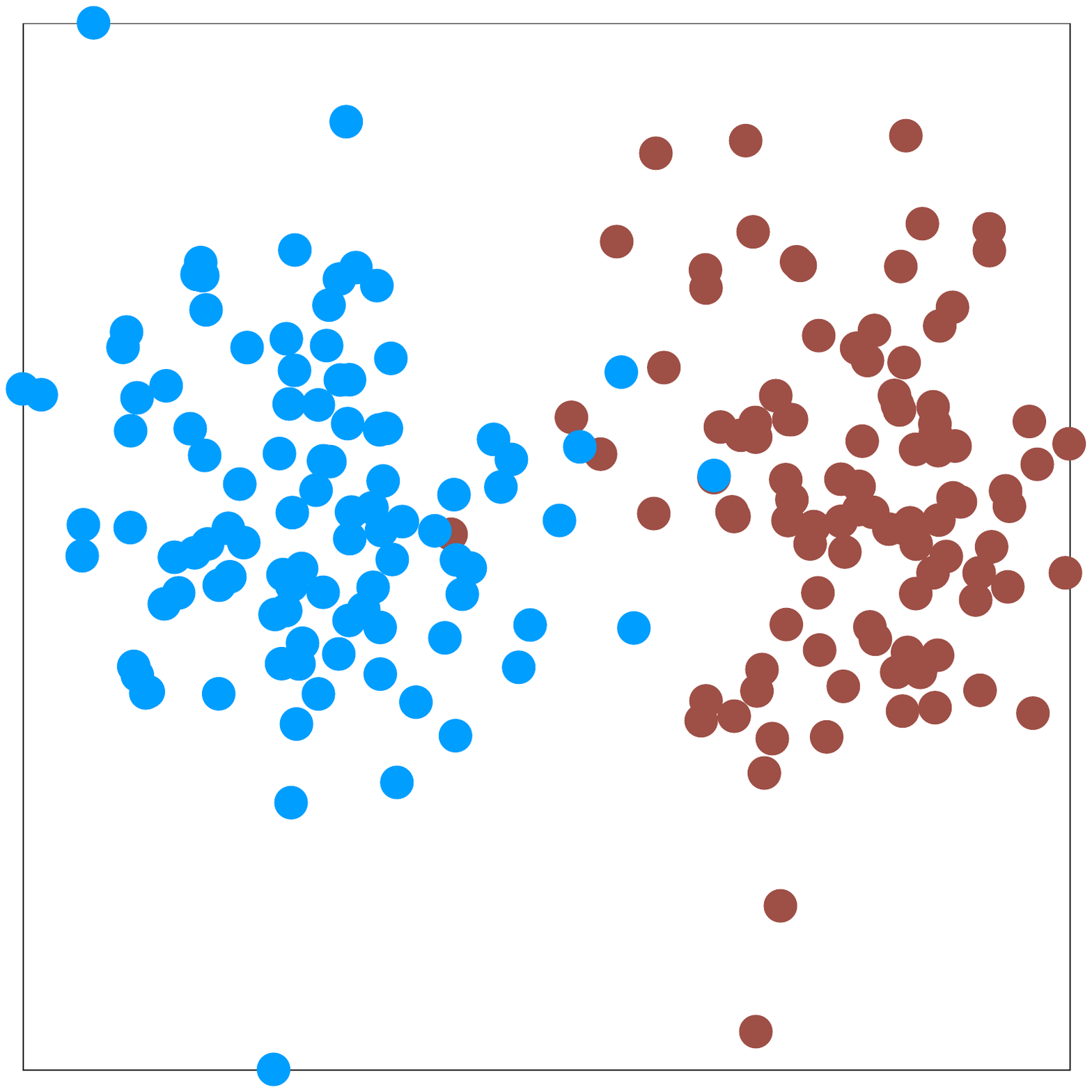}
 \caption{$L_{0}$}
 \label{fig:SBM:sparse:Zero:embedding}
 \end{subfigure}
 \begin{subfigure}[b]{0.24\textwidth}
 \includegraphics[width=1\textwidth,trim=160 90 125 60,clip]{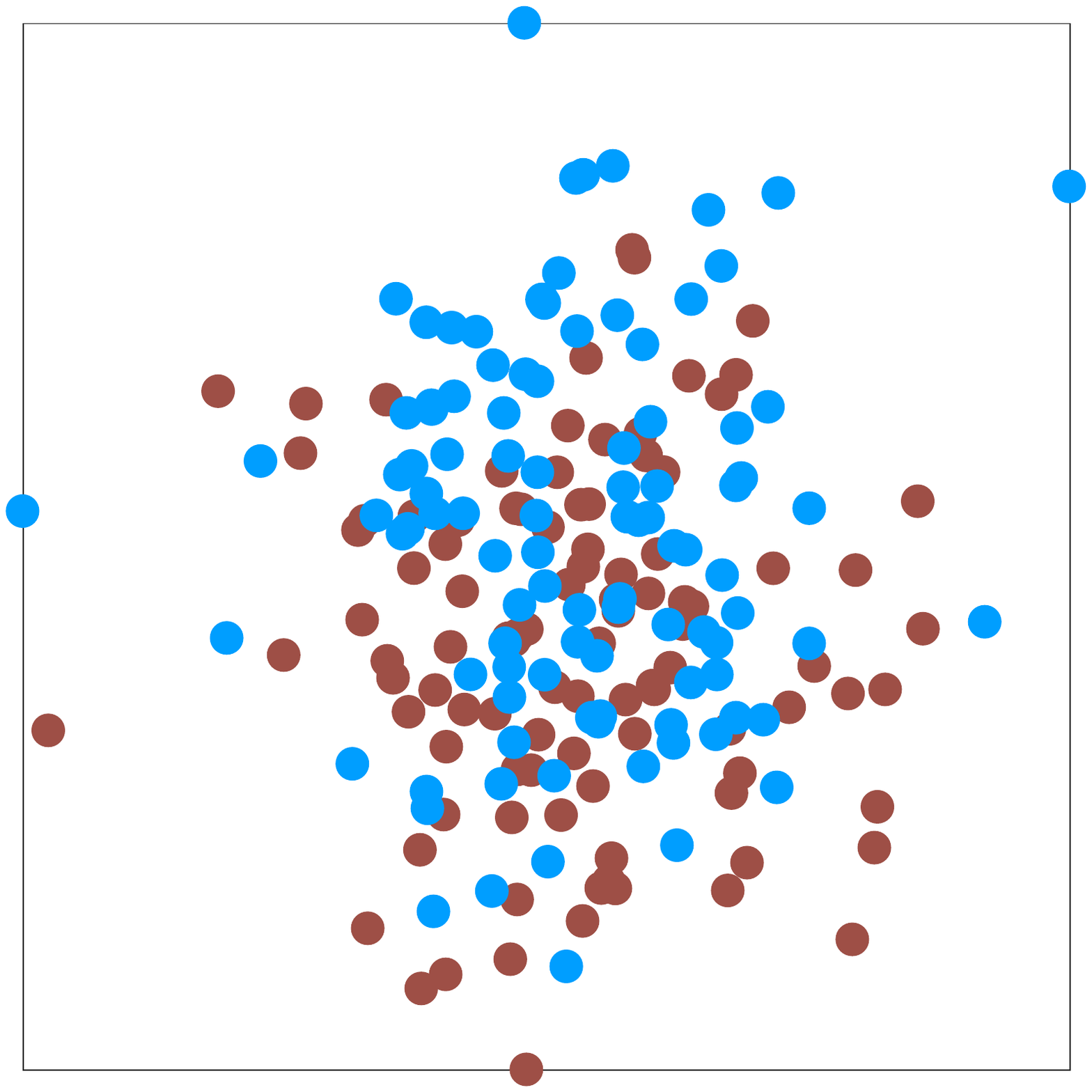}
 \caption{$L_{1}$}
 \label{fig:SBM:sparse:Plus1:embedding}
 \end{subfigure}
 \\
  \begin{subfigure}[b]{0.24\textwidth}
 \includegraphics[width=1\textwidth,trim=160 90 125 60,clip]{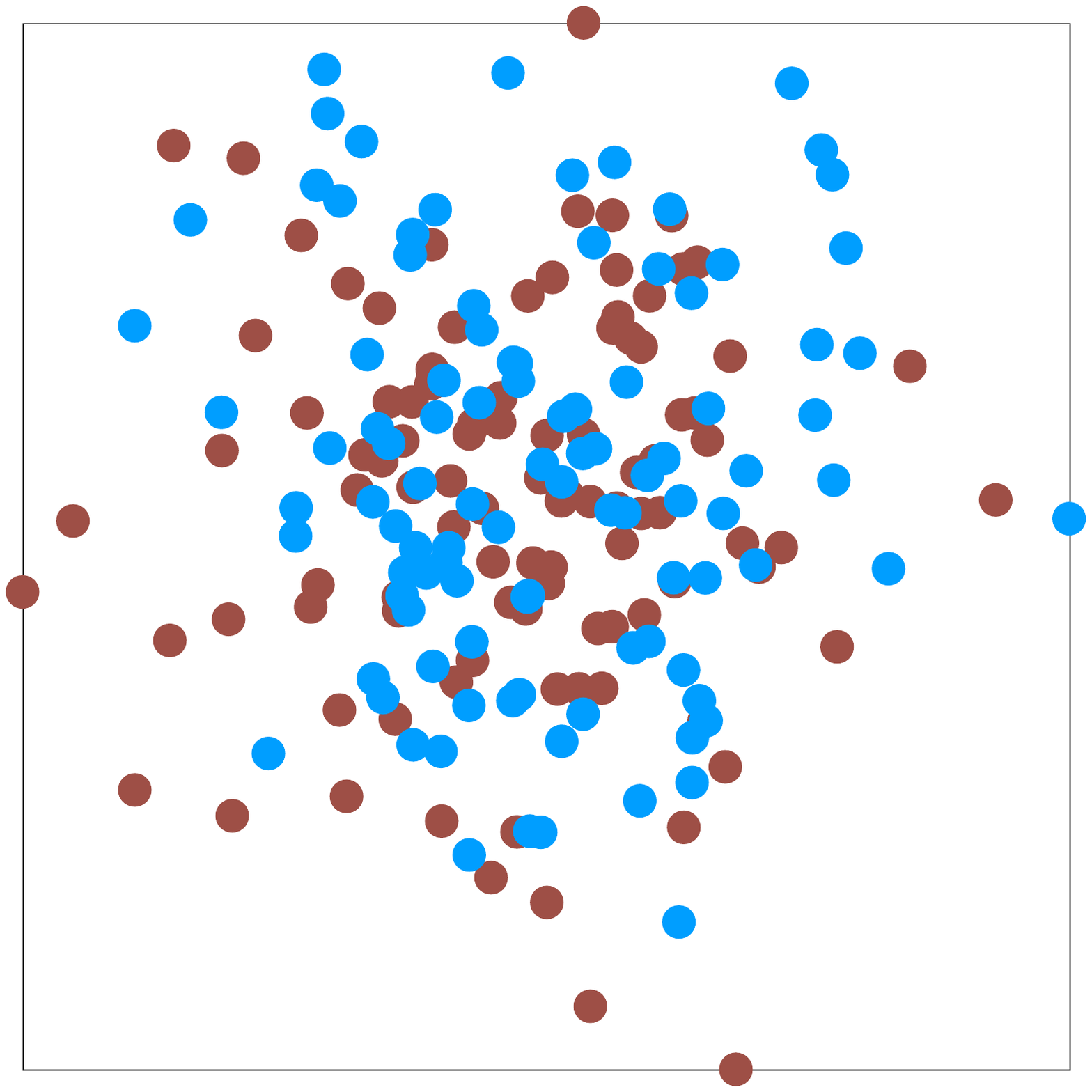}
 \caption{$L_{10}$}
 \label{fig:SBM:sparse:Plus10:embedding}
 \end{subfigure}
 \begin{subfigure}[b]{0.24\textwidth}
 \includegraphics[width=1\textwidth,trim=160 90 125 60,clip]{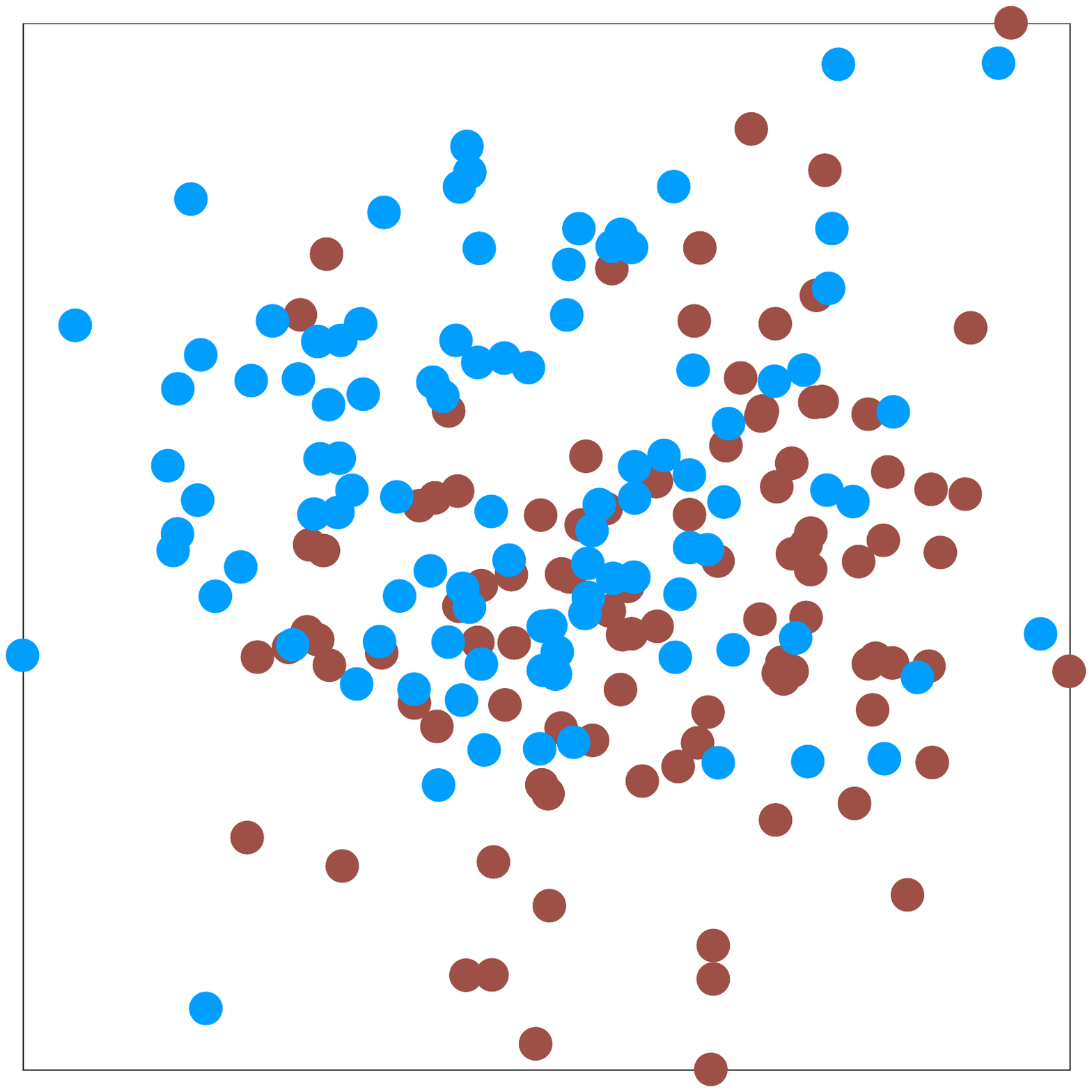}
 \caption{$L_{SN}$}
 \label{fig:SBM:sparse:LSN:embedding}
 \end{subfigure}
 \begin{subfigure}[b]{0.24\textwidth}
 \includegraphics[width=1\textwidth,trim=160 90 125 60,clip]{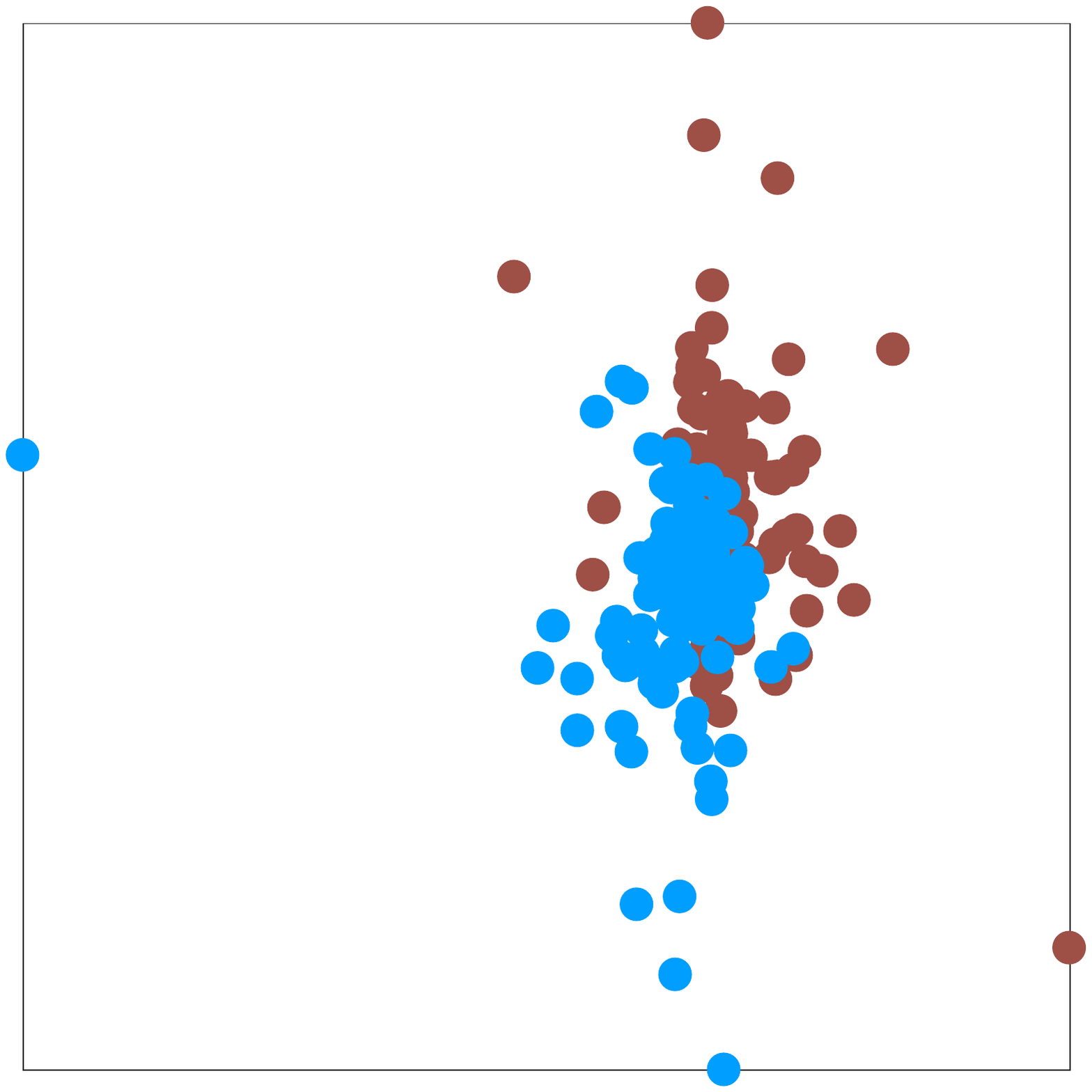}
 \caption{$L_{BN}$}
 \label{fig:SBM:sparse:LBN:embedding}
 \end{subfigure}
 \begin{subfigure}[b]{0.24\textwidth}
 \includegraphics[width=1\textwidth,trim=160 90 125 60,clip]{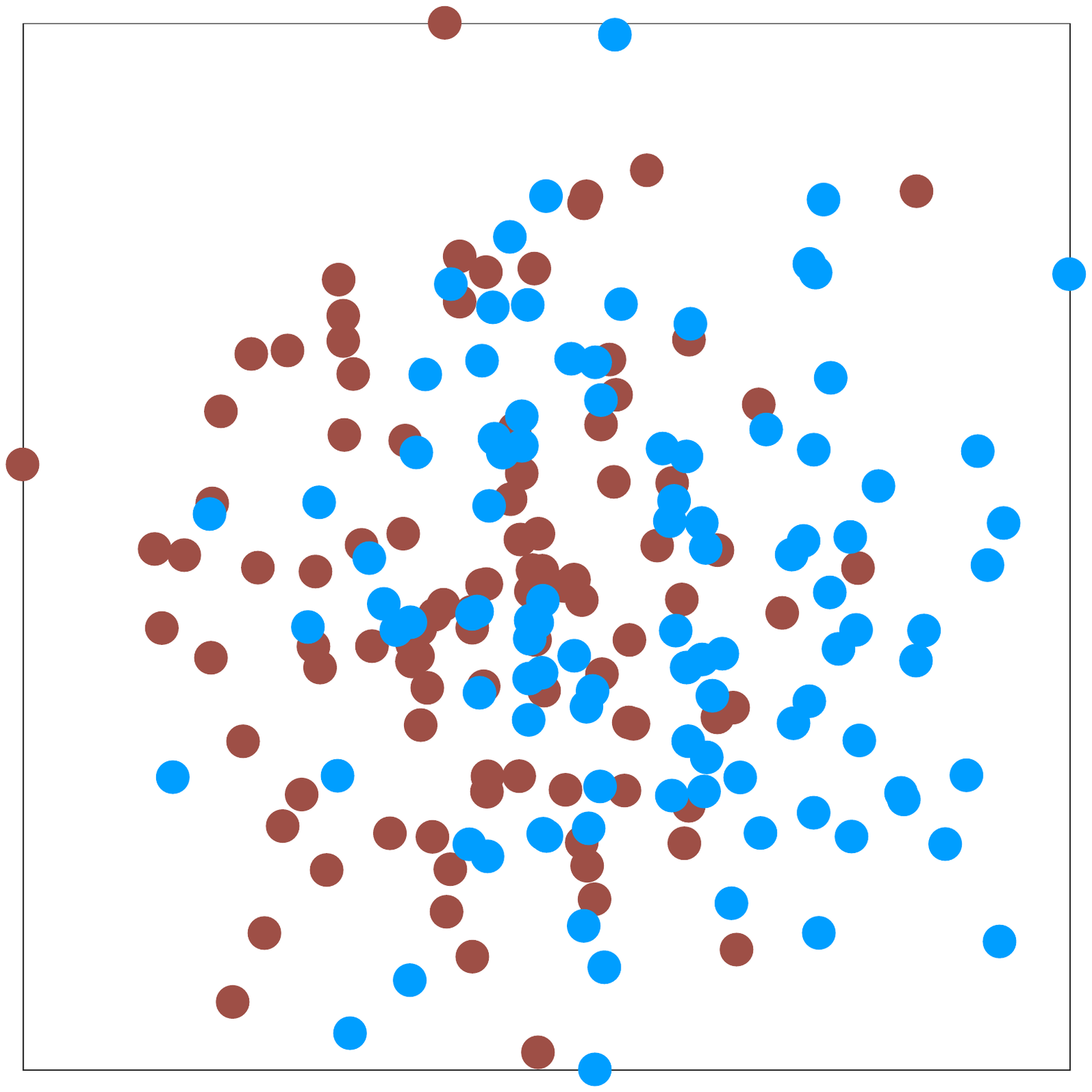}
 \caption{$H$}
 \label{fig:SBM:sparse:H:embedding}
 \end{subfigure}
  \end{minipage}
\vspace{-10pt}
\caption{
\textbf{Left}: Mean clustering error under the SSBM,
with two clusters of size 100 and 50 runs.
In Fig.~\ref{subfig:Wpos_fixed}: $G^+$ is informative, i.e. assortative with $\pp=0.09$ and $\qp=0.01$.
In Fig.~\ref{subfig:Wneg_fixed}: $G^-$ is informative, i.e. disassortative with $\ppm=0.01$ and $\qm=0.09$. 
\textbf{Right}: Node embeddings induced by eigenvectors of different signed Laplacians for a random graph drawn from SSBM for 2 clusters of size 100, 
$\pp=0.025,\qp=0.075,\ppm=0.01,\qm=0.09$.
\vspace{-1.5em}}
\label{fig:Wpos-Wneg-fixed}
\end{figure*}

In the bottom row of Fig.~\ref{fig:SBM:inExpectation:GENERAL} we can observe that $L_{SN},L_{BN}$ and $H$ present a similar behaviour to the arithmetic mean Laplacian $L_1$.
A quick computation shows that for the case where both $G^+,G^-$ have the same node degree in expectation,
the conditions of Theorem~\ref{theorem:mp_in_expectation} for $\mathcal{L}_1$ reduce to $\ppm+\qp < \pp+\qm$. It turns out that this condition is also required by $\mathcal L_{SN},\mathcal L_{BN}$ and $\mathcal H$, as the following shows.
\begin{theorem}[\cite{Mercado:2016:Geometric}]\label{theorem:signedLaplacians}
Let $\mathcal L_{BN}$ and $\mathcal L_{SN}$ be the balanced normalized Laplacian and signed normalized Laplacian of the expected signed graph. The following statements are equivalent:
 \begin{itemize}[topsep=-3pt,leftmargin=*]\setlength\itemsep{-3pt}
   \item $\{ \boldsymbol \chi_i \}_{i=1}^{k}$ are the eigenvectors corresponding to the $k$-smallest eigenvalues of $\mathcal{L}_{BN}$.
   \item $\{ \boldsymbol \chi_i \}_{i=1}^{k}$ are the eigenvectors corresponding to the $k$-smallest eigenvalues of $\mathcal{L}_{SN}$.
   \item  inequalities $\ppm+(k-1)\qm < \pp+(k-1)\qp$ and $\ppm+\qp < \pp+\qm$ hold.
 \end{itemize}
 \end{theorem}
 
Finally, we present conditions in expectation for the Bethe Hessian to identify the ground truth clustering.
\begin{theorem}\label{theorem:bethe_hessian_V1}
 Let $\mathcal{H}$ be the Bethe Hessian of the expected signed graph. Then 
 $\{ \boldsymbol \chi_i \}_{i=2}^{k}$ are the eigenvectors corresponding to the $(k-1)$-smallest negative eigenvalues of $\mathcal{H}$ if and only if the following conditions hold:
 \begin{enumerate}[topsep=-3pt,leftmargin=*]\setlength\itemsep{-3pt}
  \item  $\max\{ 0,\frac{2(d^+ + d^-)-1}{\sqrt{d^+ + d^-}\abs{\mathcal{C}}}\}  < (\pp-\qp) - (\ppm-\qm)$
  \item  $\qp<\qm$ 
 \end{enumerate}
 Moreover, for the limit case $\abs{V}\rightarrow\infty$ the first condition reduces to
 $\ppm + \qp < \pp + \qm$.
\end{theorem}
Please see the supplementary material for a further analysis in expectation.
We can observe that 
the 
first condition in Theorem~\ref{theorem:bethe_hessian_V1} is related to conditions of $\mathcal L_1$ and $\mathcal L_{SN},\mathcal L_{BN}$ through the inequality $\ppm + \qp < \pp + \qm$. This explains why the performance of the Bethe Hessian $H$ resembles the one of arithmetic Laplacians $L_{SN},L_{BN},L_1$. 
A more detailed comparison between the conditions of Theorems \ref{theorem:mp_in_expectation}, \ref{theorem:signedLaplacians} and \ref{theorem:bethe_hessian_V1} is detailed in the supplementary material.

Note that our analysis in expectation considers the dense regime where the average degree increases with the number of nodes and hence our results in expectation are verified under the SSBM setting here considered, showing that $L_{SN},L_{BN},L_1,H$ have a similar performance. 
However, in the case of sparse graphs, it is known that the Bethe Hessian is asymptotically optimal in the information-theoretic transition limit~\cite{saade:2014,saade:2015}.
Please see the supplementary material for an evaluation under the CBM.

We now zoom in on a particular setting of Fig.~\ref{fig:SBM:inExpectation:GENERAL}. Namely, the case where $G^+$ (resp.$G^-$) is fixed to be informative, whereas the remaining graph transitions from informative to uninformative. The corresponding results are in Fig.~\ref{fig:Wpos-Wneg-fixed}.
In Fig.~\ref{subfig:Wpos_fixed} we consider the case where $G^+$ is informative with parameters $\pp=0.09$ and $\qp=0.01$ (this corresponds to $\pp-\qp=0.08$ in Fig.~\ref{fig:SBM:inExpectation:GENERAL} ), and $G^-$ goes from being informative ($\ppm<\qm$) to non-informative ($\ppm\geq\qm$). We confirm that the power mean Laplacian $L_p$ presents smaller clustering errors for smaller values of $p$. Moreover, it is clear that in the case $p<0$, $L_p$ is able to recover clusters even in the case where $G^-$ is not informative, whereas for $p>0$, $L_p$ requires both $G^+$ and $G^-$ to be informative. We observe that the smallest (resp. largest) clustering errors correspond to $L_{-10}$ (resp. $L_{10}$), corroborating Corollary~\ref{corollary:contention}.
Further, we can observe that $L_{GM}$ and $L_{0}$ have a similar performance, as well as $L_{SN},L_{BN},L_1,H$, as observed before, confirming Theorem~\ref{theorem:geometricMeanLaplacian} and Theorem~\ref{theorem:bethe_hessian_V1}, respectively.
\newline
In Fig.~\ref{subfig:Wneg_fixed} similar observations hold for the case where $G^-$ is informative with parameters $\ppm=0.01$ and $\qm=0.09$ (this corresponds to $\ppm-\qm=-0.08$ in Fig.~\ref{fig:SBM:inExpectation:GENERAL}), and $G^+$ goes from being non-informative ($\pp\leq\qp$) to informative ($\pp>\qp$). 
Within this setting we present the eigenvector-based node embeddings of each method for the case 
$\pp=0.025,\qp=0.075,\ppm=0.01,\qm=0.09$,
in right hand side of Fig.~\ref{fig:Wpos-Wneg-fixed}.
For $L_{-10},L_{-1},L_{0}$ the embeddings split the clusters properly, whereas 
remaining embeddings are not informative, verifying the 
effectivity
of $L_p$ with 
$p< 0$.
\begin{figure*}[!htb]
 \centering
 \hfill
 \begin{subfigure}[b]{0.13\textwidth}
 \includegraphics[angle=-90,width=1\textwidth,trim=50 130 75 110]{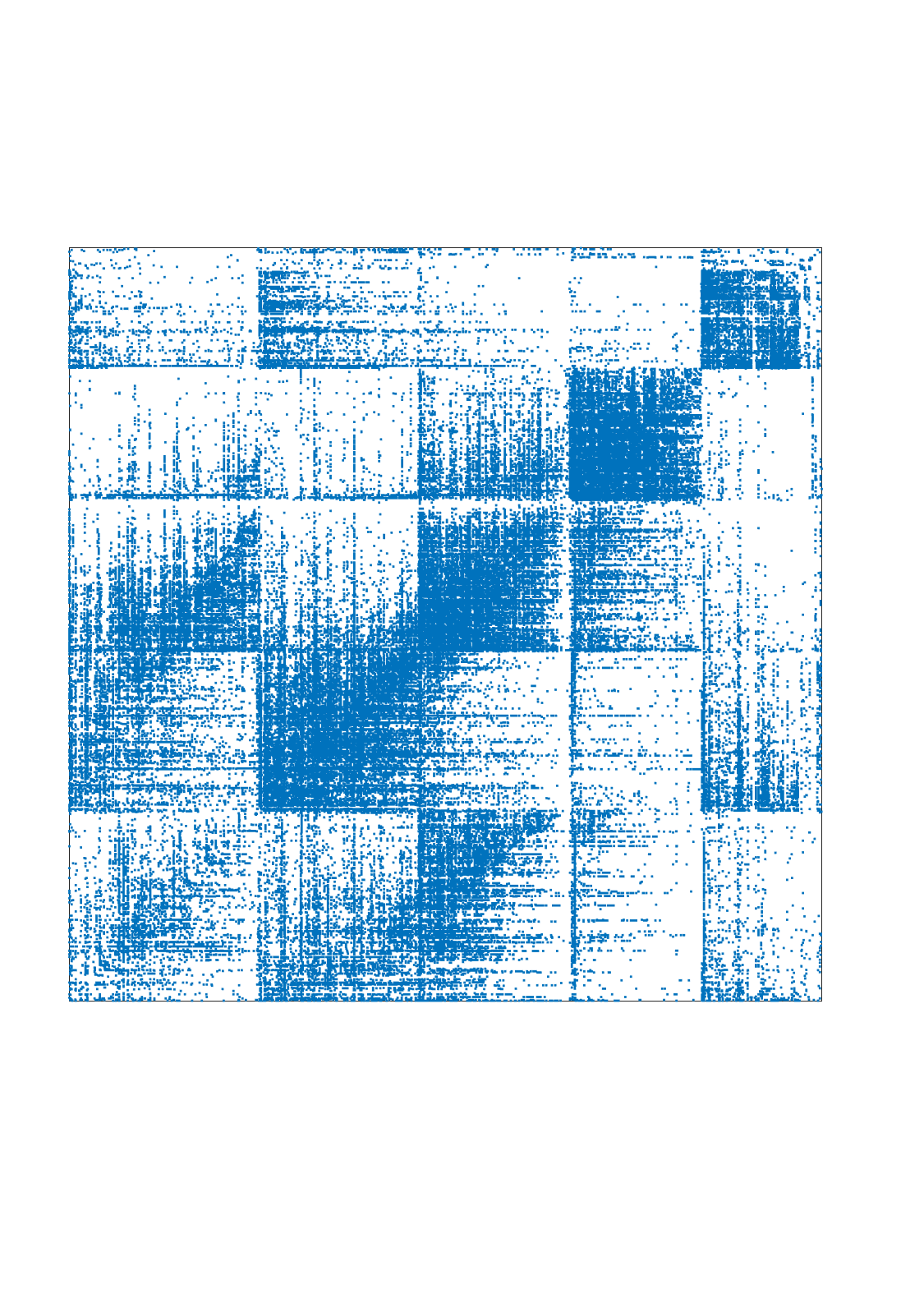}
 \label{fig:wikipedia:blue:Minus10}
 \end{subfigure}
 \hfill
 \begin{subfigure}[b]{0.13\textwidth}
 \includegraphics[angle=-90,width=1\textwidth,trim=50 130 75 110]{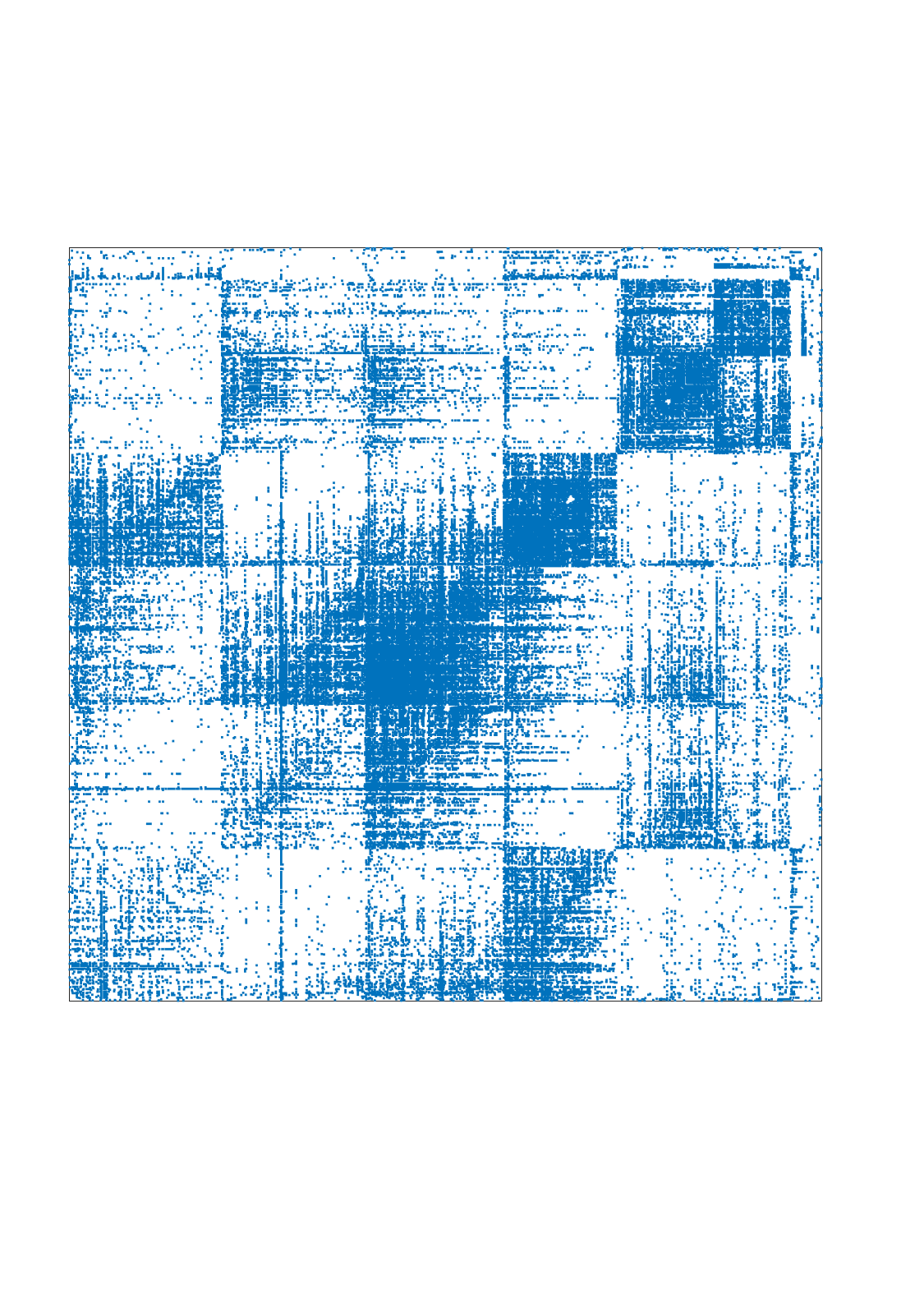}
 \label{fig:wikipedia:blue:Minus5}
 \end{subfigure}
 \hfill
 \begin{subfigure}[b]{0.13\textwidth}
 \includegraphics[angle=-90,width=1\textwidth,trim=50 130 75 110]{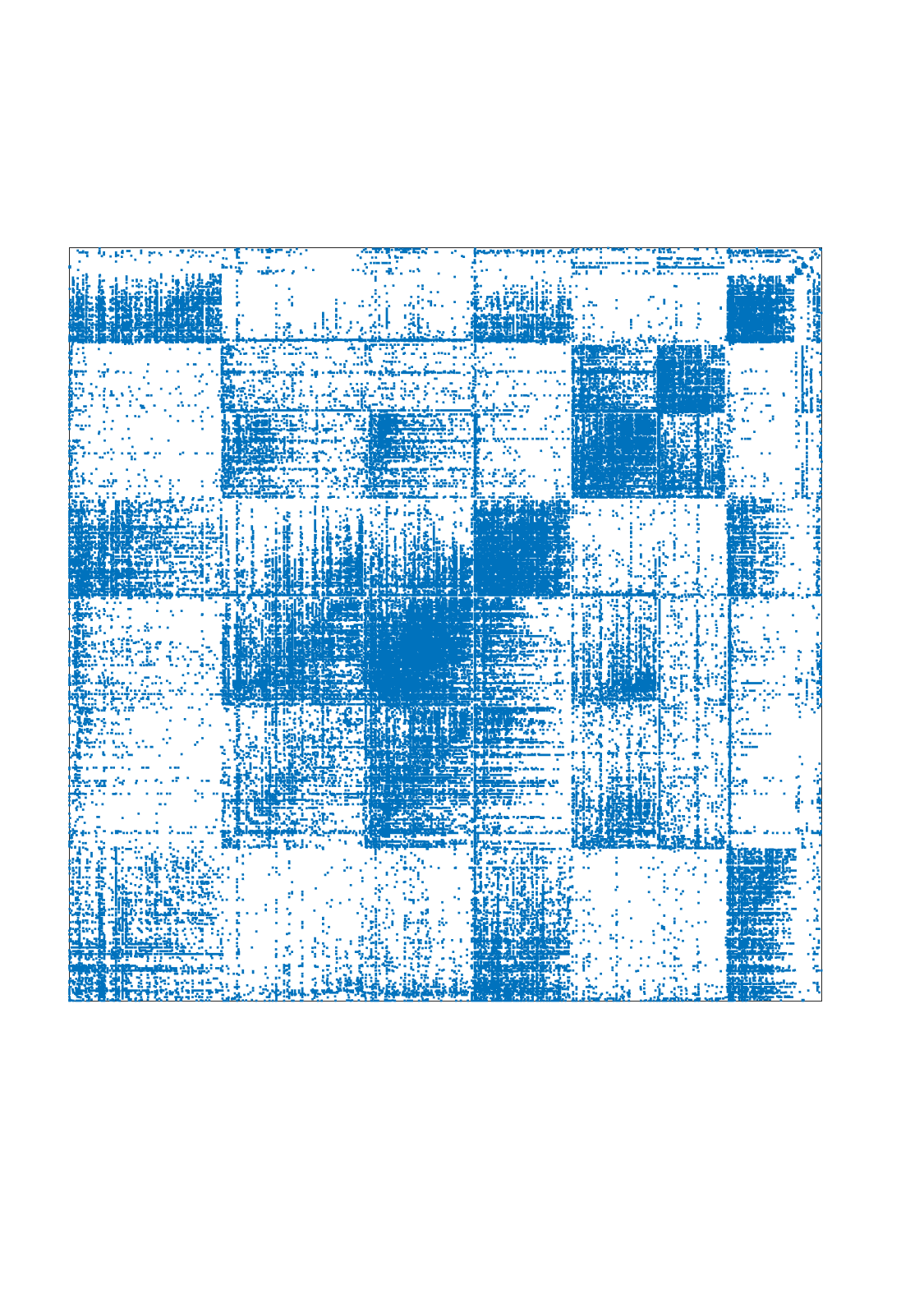}
 \label{fig:wikipedia:blue:Minus2}
 \end{subfigure}
 \hfill
 \begin{subfigure}[b]{0.13\textwidth}
 \includegraphics[angle=-90,width=1\textwidth,trim=50 130 75 110]{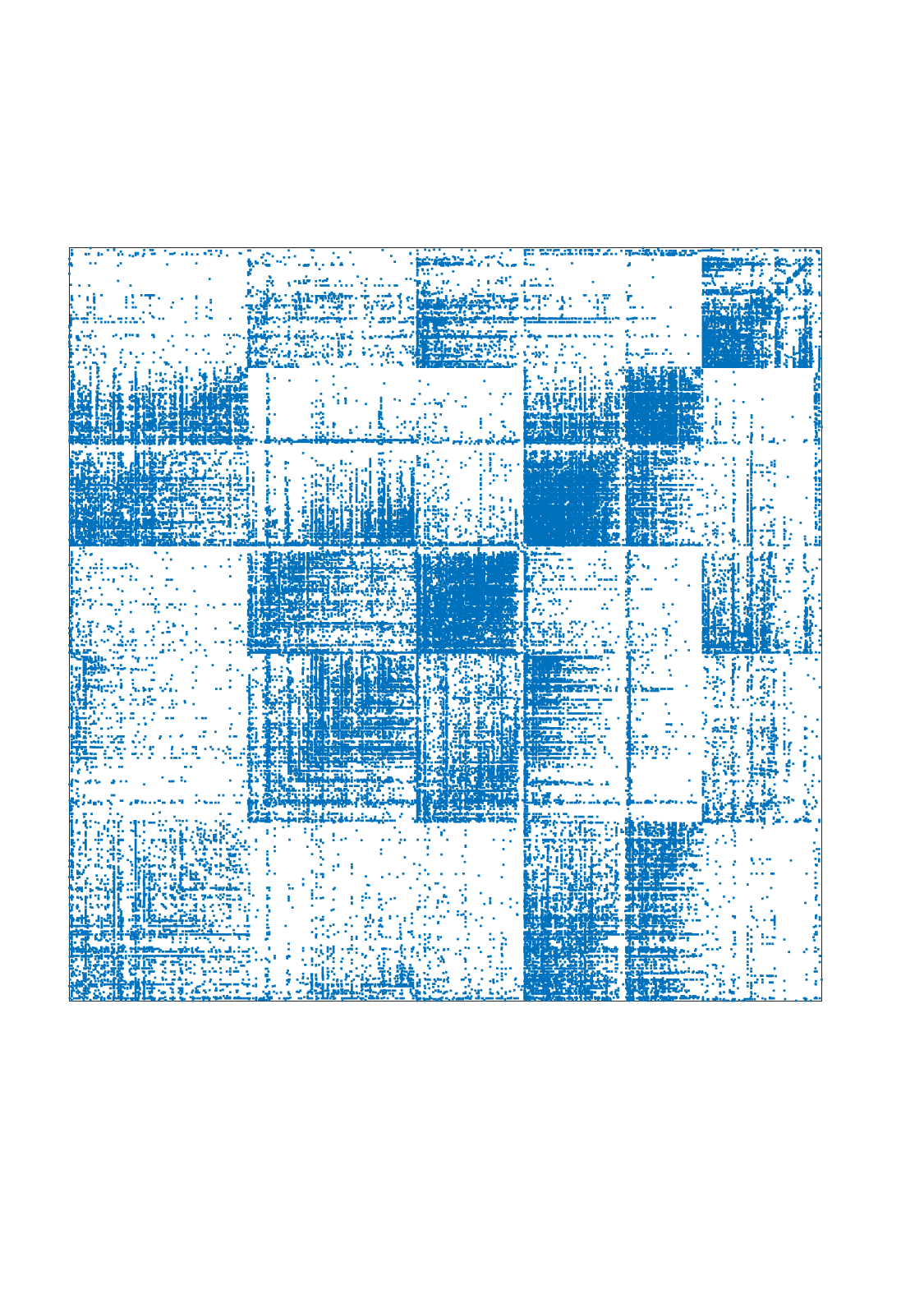}
 \label{fig:wikipedia:blue:Minus1}
 \end{subfigure}
 \hfill
  \begin{subfigure}[b]{0.13\textwidth}
 \includegraphics[angle=-90,width=1\textwidth,trim=50 130 75 110]{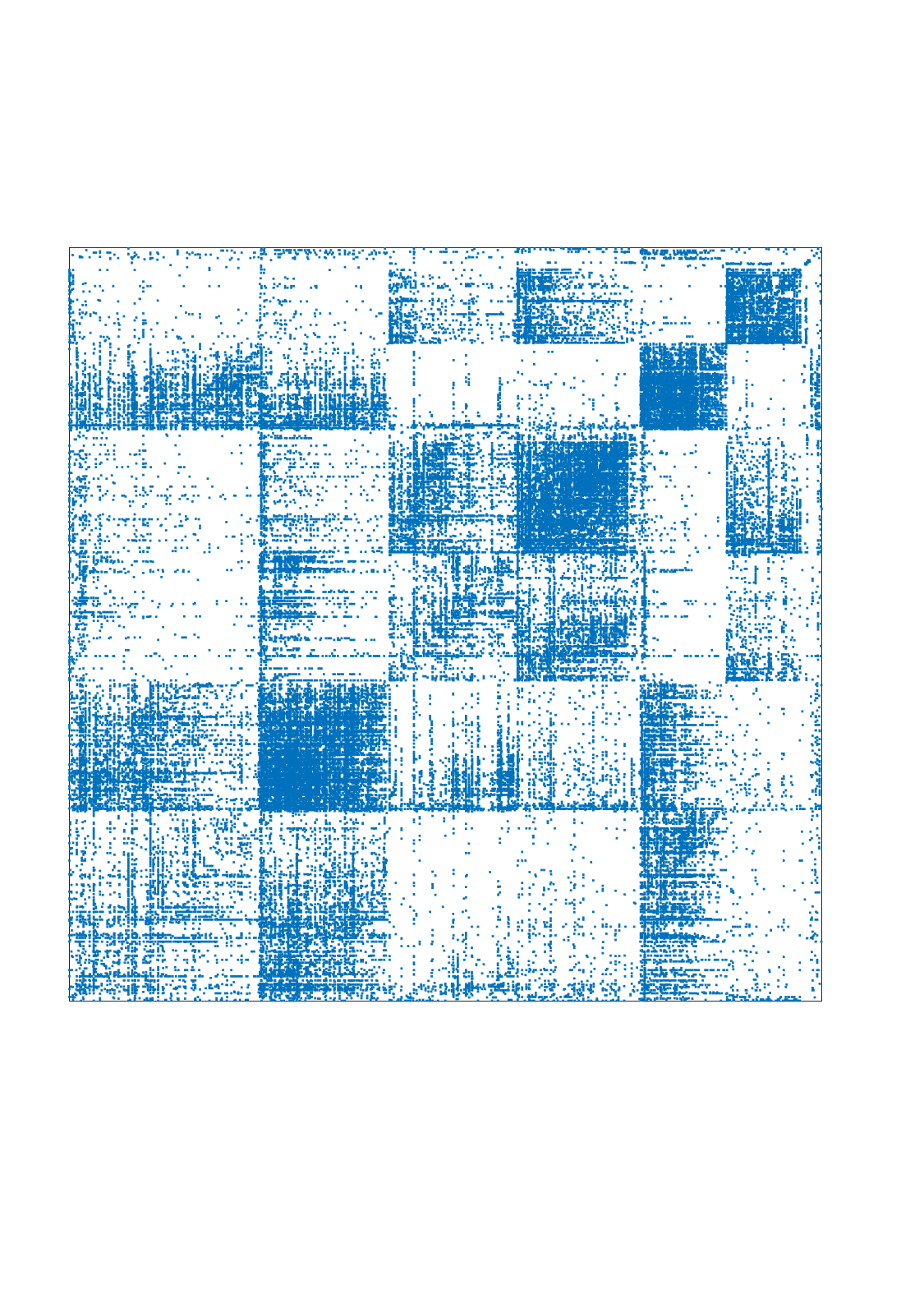}
 \label{fig:wikipedia:blue:zero}
 \end{subfigure}
 \hfill
  \begin{subfigure}[b]{0.13\textwidth}
 \includegraphics[angle=-90,width=1\textwidth,trim=50 130 75 110]{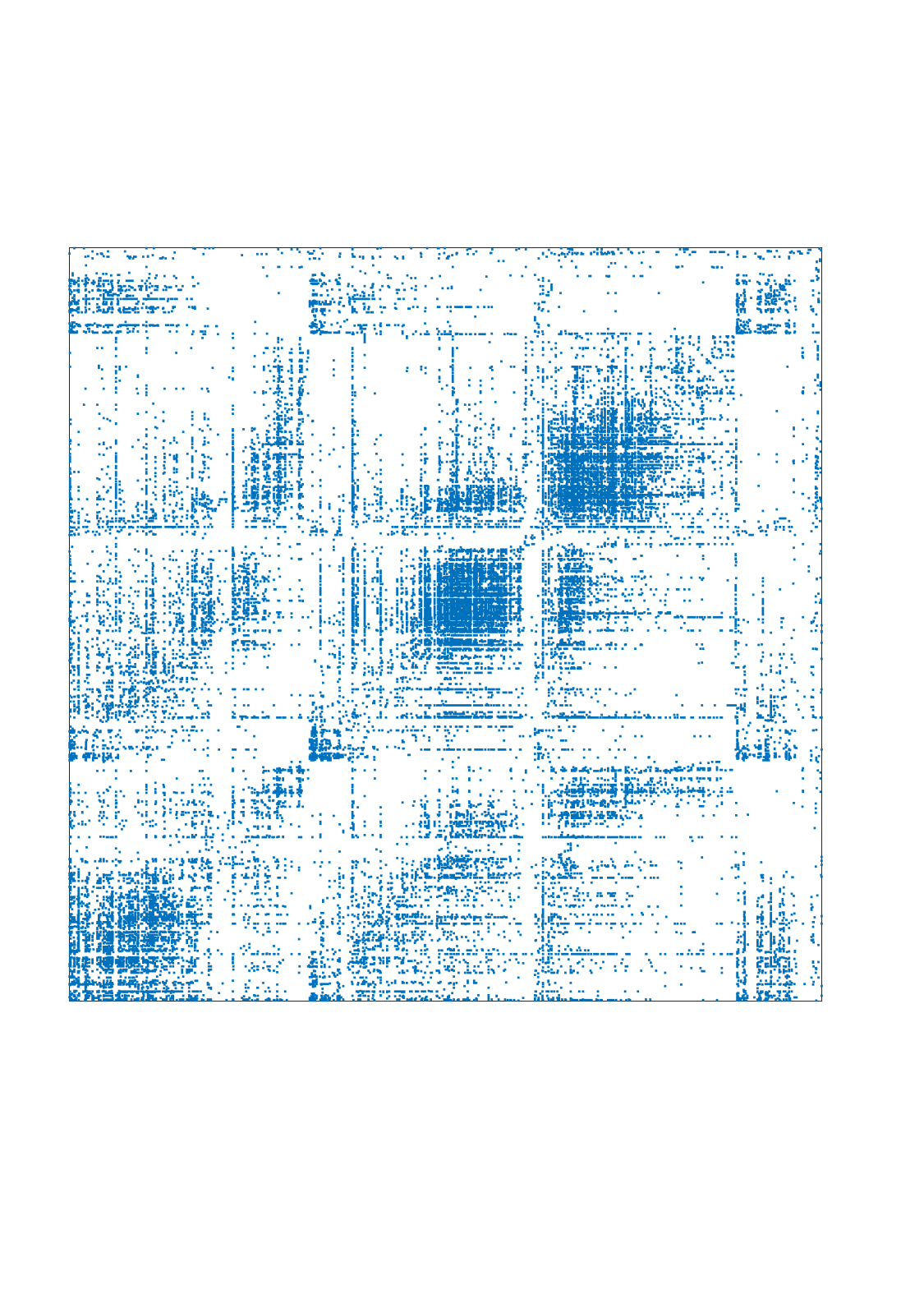}
 \label{fig:wikipedia:blue:Plus1}
 \end{subfigure}
 \hfill
 \hfill
\\
 \vspace{25pt}
 \hfill
 \begin{subfigure}[b]{0.13\textwidth}
 \includegraphics[angle=-90,width=1\textwidth,trim=50 130 75 110]{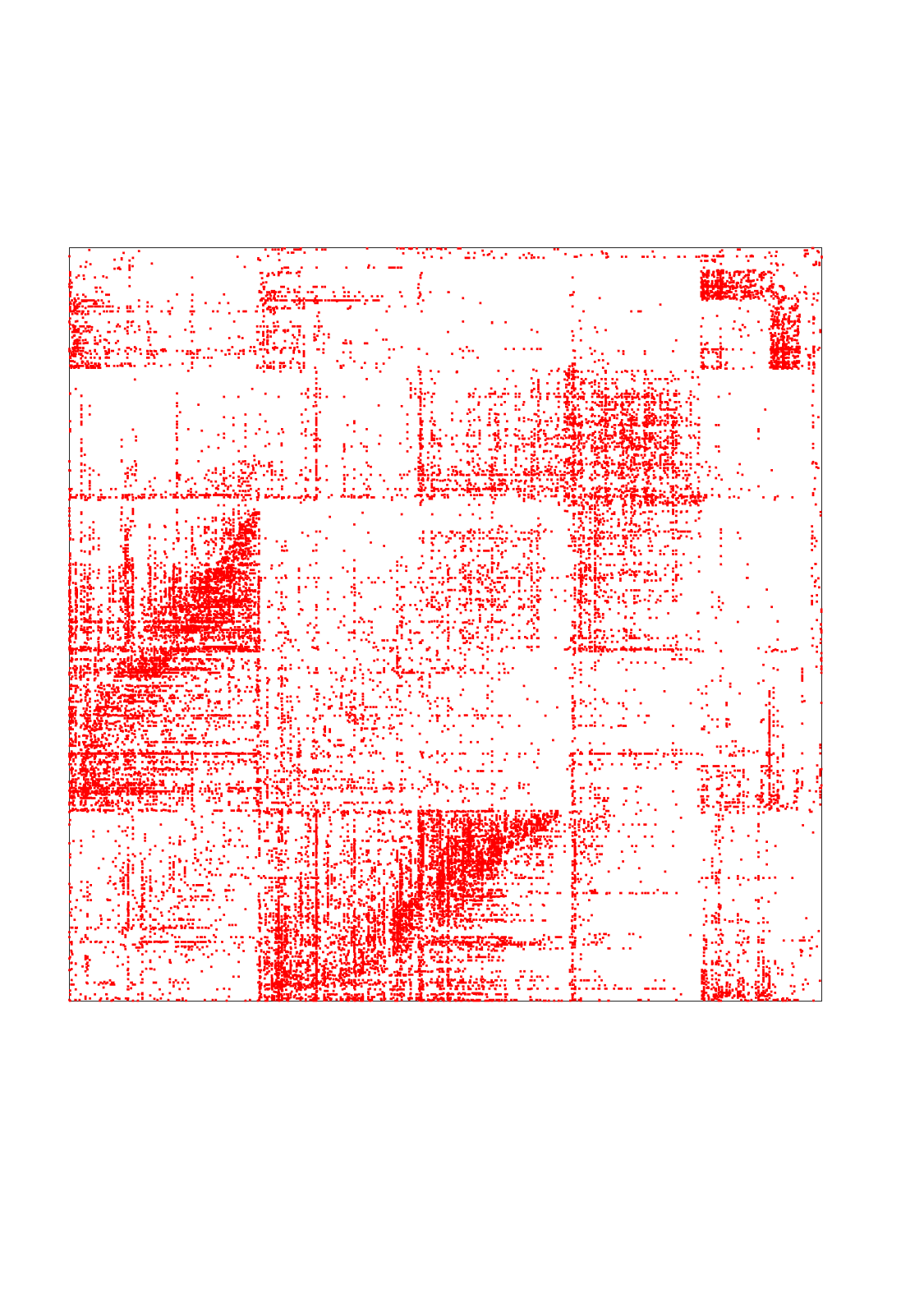}
 \caption{$L_{-10}$ }
 \label{fig:wikipedia:red:Minus10}
 \end{subfigure}
 \hfill
 \begin{subfigure}[b]{0.13\textwidth}
 \includegraphics[angle=-90,width=1\textwidth,trim=50 130 75 110]{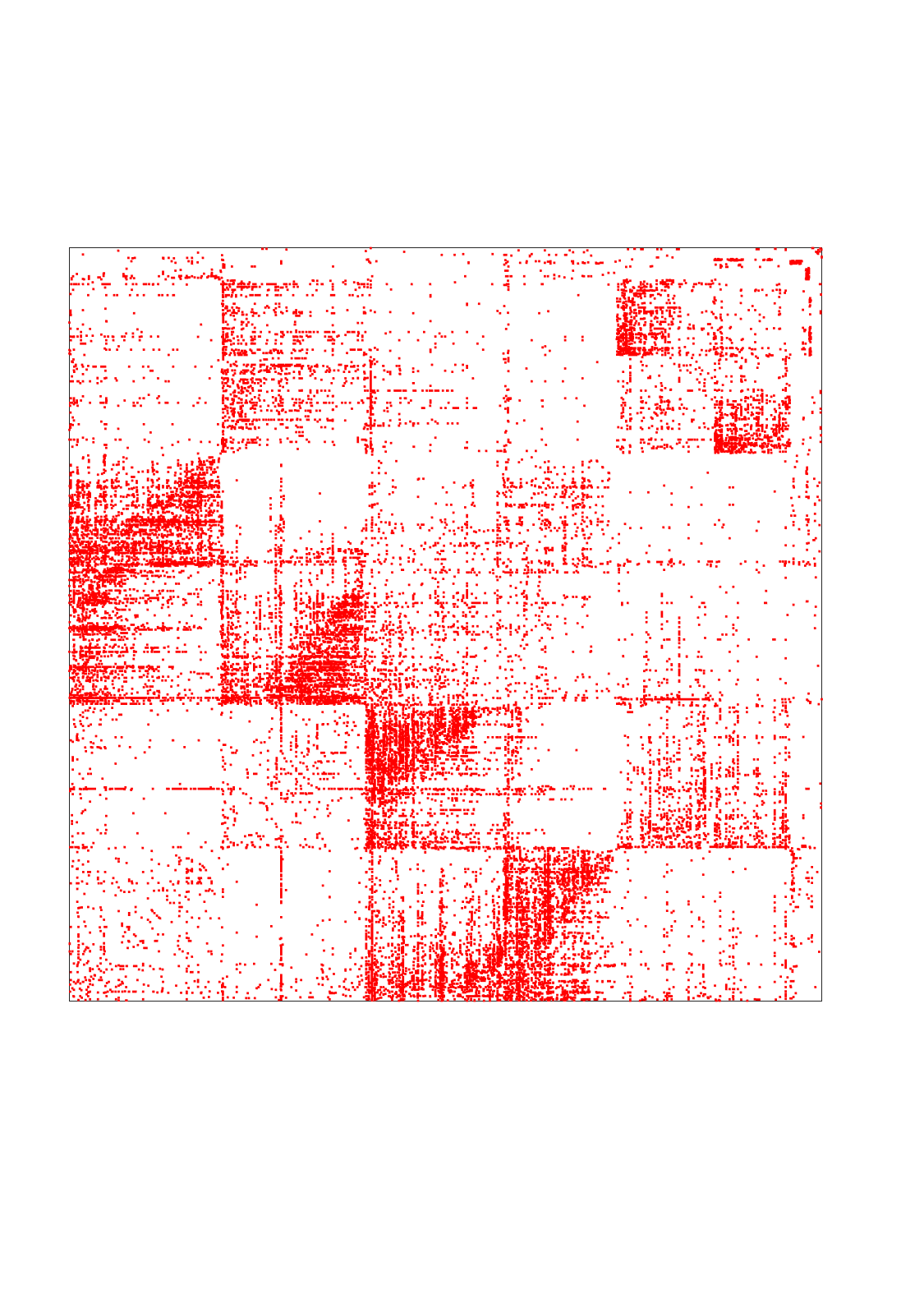}
 \caption{$L_{-5}$ }
 \label{fig:wikipedia:red:Minus5}
 \end{subfigure}
 \hfill
 \begin{subfigure}[b]{0.13\textwidth}
 \includegraphics[angle=-90,width=1\textwidth,trim=50 130 75 110]{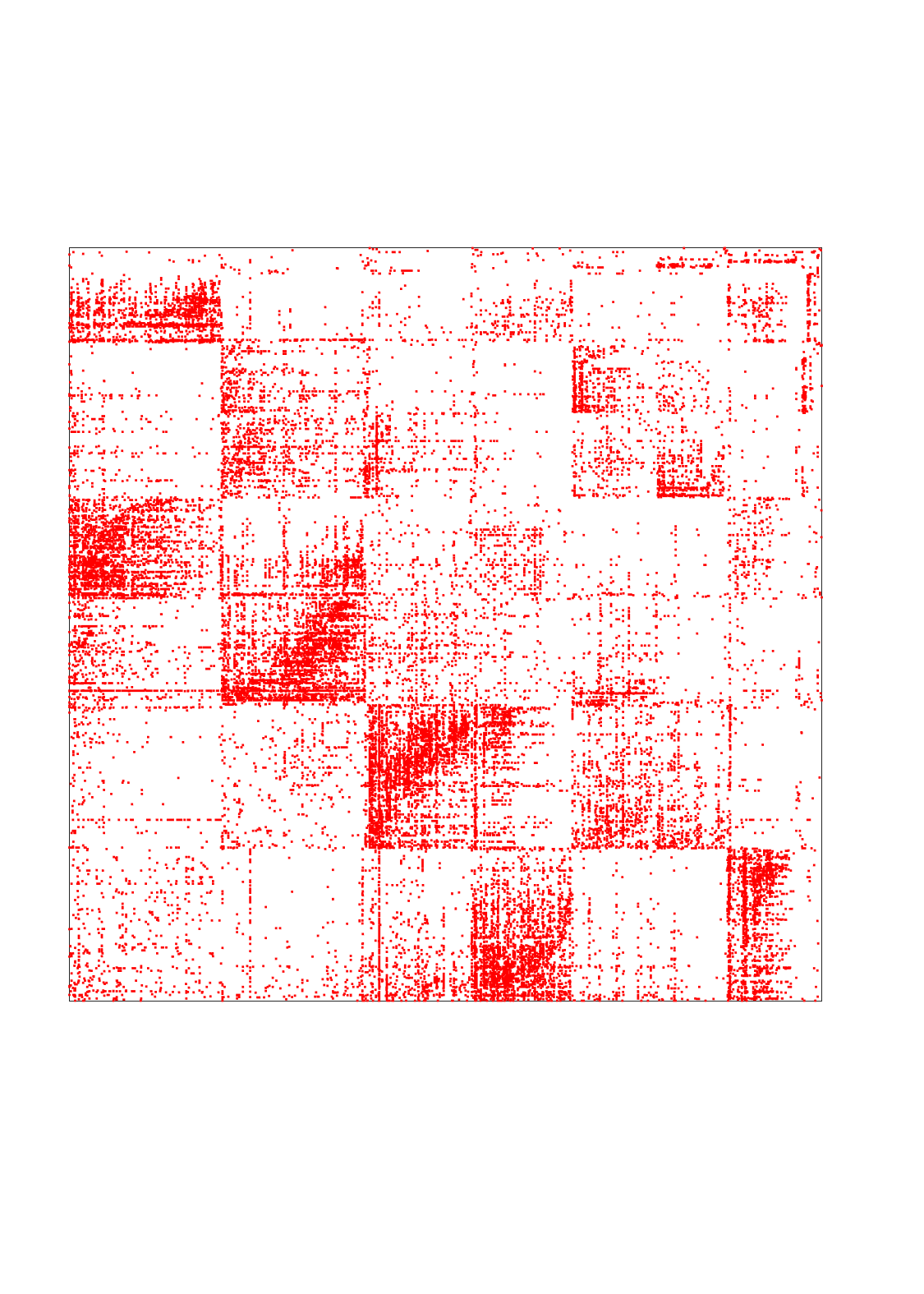}
 \caption{$L_{-2}$ }
 \label{fig:wikipedia:red:Minus2}
 \end{subfigure}
 \hfill
 \begin{subfigure}[b]{0.13\textwidth}
 \includegraphics[angle=-90,width=1\textwidth,trim=50 130 75 110]{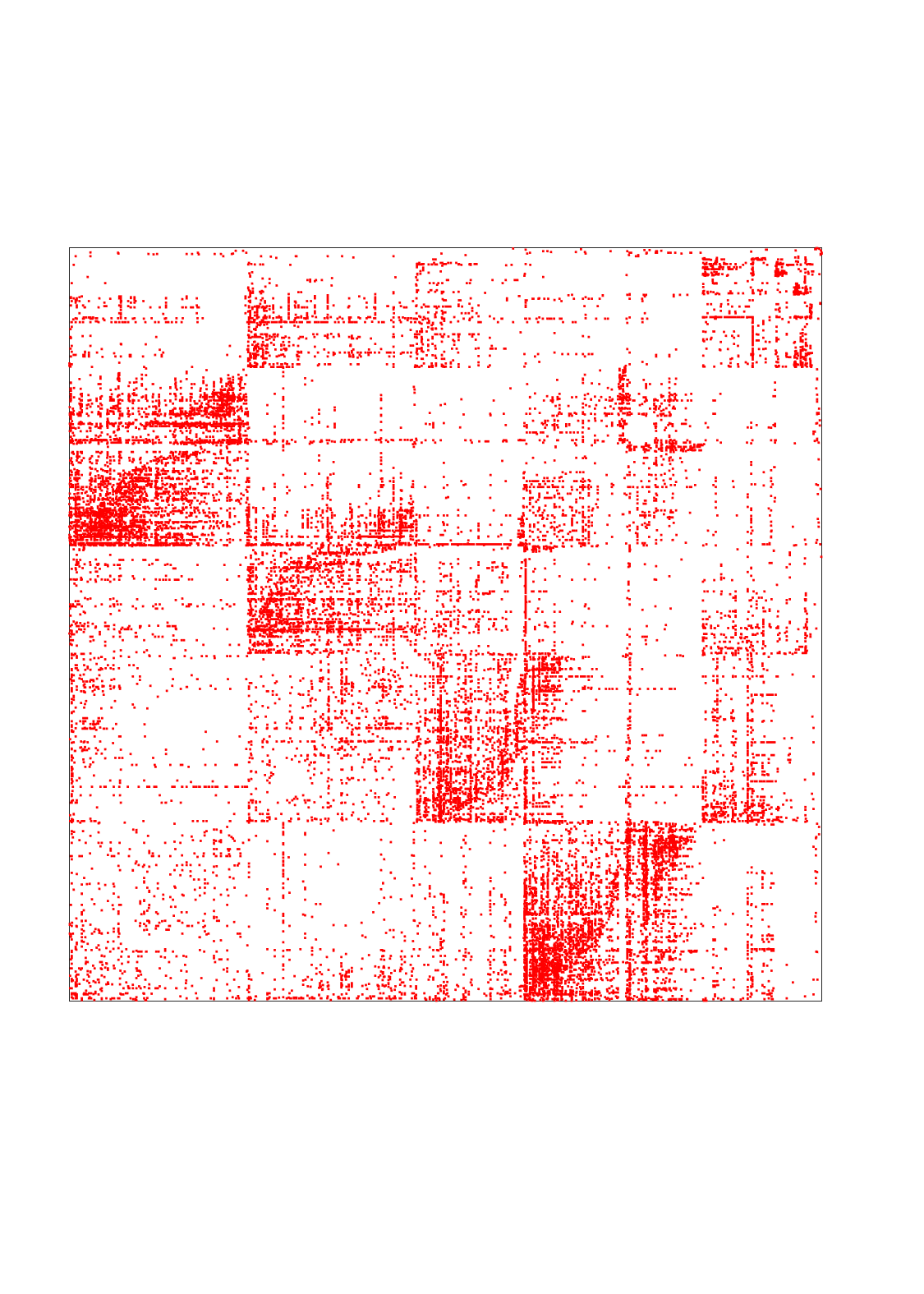}
 \caption{$L_{-1}$  }
 \label{fig:wikipedia:red:Minus1}
 \end{subfigure}
 \hfill
  \begin{subfigure}[b]{0.13\textwidth}
 \includegraphics[angle=-90,width=1\textwidth,trim=50 130 75 110]{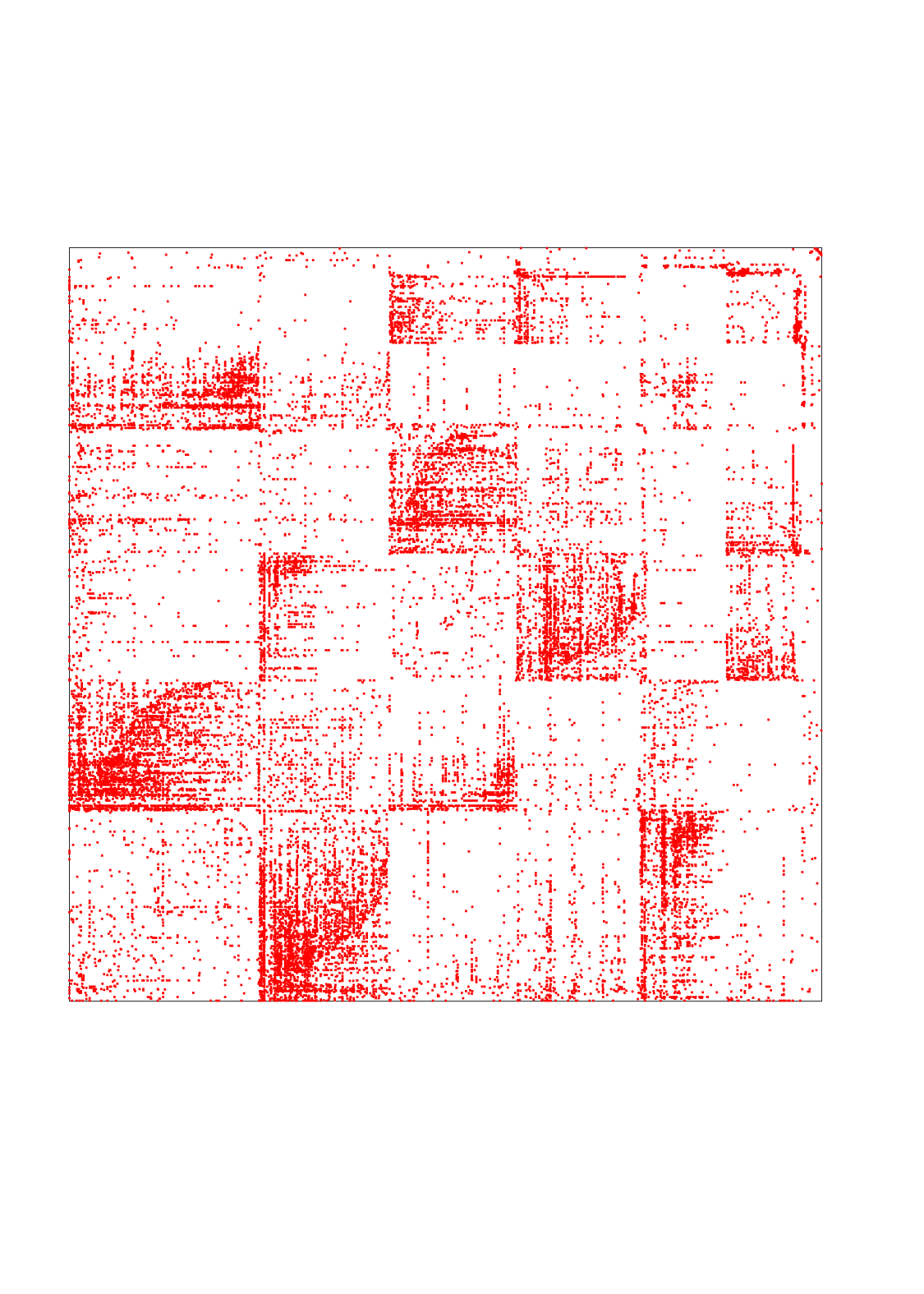}
 \caption{$L_{0}$  }
 \label{fig:wikipedia:red:zero}
 \end{subfigure}
 \hfill
  \begin{subfigure}[b]{0.13\textwidth}
 \includegraphics[angle=-90,width=1\textwidth,trim=50 130 75 110]{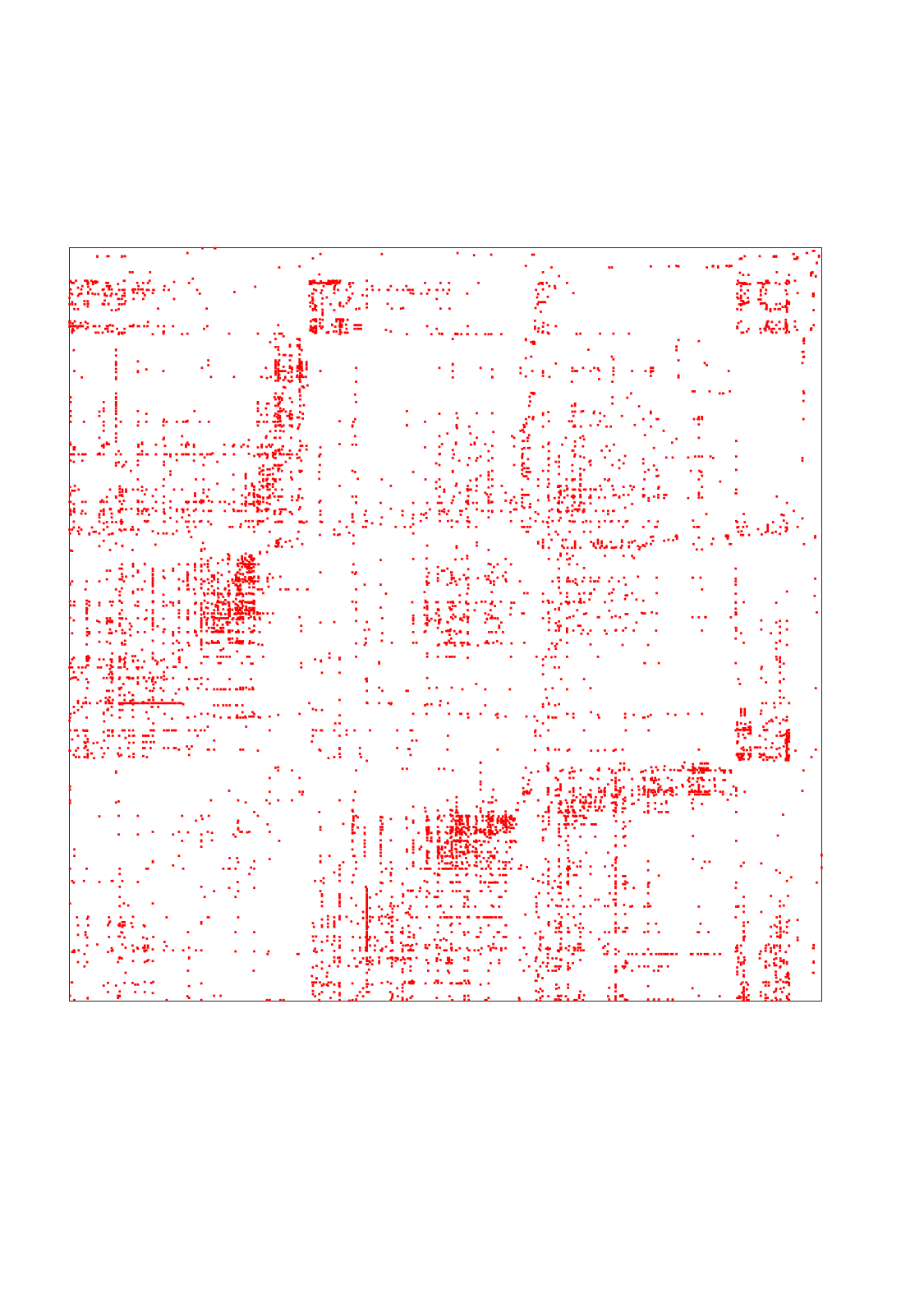}
 \caption{$L_{1}$  }
 \label{fig:wikipedia:red:Plus1}
 \end{subfigure}
 \hfill
\caption{
Adjacency matrices $W^+$ and $W^-$ sorted through clustering of Wikipedia Elections dataset by the proposed Signed Power mean Laplacians $L_{-10},L_{-5},L_{-2},L_{-1},L_{0},L_{1}$.
\textbf{Top row}: zoom-in visualization of positive edges $W^+$. \textbf{Bottom row}: zoom-in visualization of negative edges $W^-$. See supplementary material for more details.
 \vspace{-10pt}
}
 \label{fig:wikipedia}
\end{figure*}

%
%
\subsection{Consistency of the Signed Power Mean Laplacian for the Stochastic Block Model}\label{Concentration}
In this section we prove two novel concentration bounds for signed power mean Laplacians of signed graphs drawn from the SSBM. The bounds show that, for large graphs, our previous results in expectation transfer to sampled graphs with high probability. We first show in Theorem~\ref{theorem:ConcentrationBound-signedGraph-identity-SBM} that $L_p$ is close to $\mathcal{L}_p$. Then, in Theorem~\ref{theorem:concentration-eigenvectors}, we show that eigenvalues and eigenvectors of $L_p$ are close to those  of $\mathcal{L}_p$.
We derive this result by tracing back the consistency of the matrix power mean to the consistency of the standard and signless Laplacian established in~\cite{chung2011spectra}.

The consistency of spectral clustering on unsigned graphs for the SBM has been studied in~\cite{lei2015,sarkar2015,rohe2011spectral} and more recently
consistency of several variants of spectral clustering has been shown~\cite{Qin:2013:regularized, joseph2016, chaudhuri:2012:regularized,Le:2018:concentration,fasino2017modularity,davis2018}. 
Moreover, while the case of multilayer graphs under the SBM has been previously analyzed ~\cite{Han:2015:CED:3045118.3045279,heimlicher2012community,jog:2015a,Paul:2017a,xu:2014a,xu:2017a,NIPS2016_6196}, there are no consistency results for matrix power means
for multilayer graphs as studied in \cite{Mercado:2018:powerMean}.
While our main emphasis is on the analysis of the \SPMM{} Laplacian, our proofs are general enough to cover also the consistency of the matrix power means
for unsigned multilayer graphs \cite{Mercado:2018:powerMean}.
In Thm.~\ref{theorem:ConcentrationBound-signedGraph-identity-SBM} we show that the \SPMM{} Laplacian $L_p$ for the SSBM is concentrated around $\mathcal{L}_p$, with high probability for large $n$. 
The following results hold for general shifts $\varepsilon$.

\begin{theorem}\label{theorem:ConcentrationBound-signedGraph-identity-SBM}
Let $p$ be a non-zero integer, let

$\,$\hfill $C_p = \left\{\begin{array}{ll}
						(2p)^{1/p}(2+\varepsilon)^{1-1/p} & p\geq 1 \\ 
						\abs{2p}^{1/\abs{p}}\varepsilon^{-(3+1/\abs{p})} & p\leq-1
						 \end{array}\right.$\hfill$\,$

and choose $\epsilon>0$. 
If 
$\frac{n}{k}(\pp\!+\!(k\!-\!1)\qp)> 3\ln (8n/\epsilon)$, 
and 
$\frac{n}{k}(\ppm\!+\!(k\!-\!1)\qm)> 3\ln (8n/\epsilon)$, 
then with probability at least $1-\epsilon$, we have 

\resizebox{\columnwidth}{!}{
 $
 \norm{L_p - \mathcal{L}_p} \leq 
 C_{p} \, 
 m_{\abs{p}}^{1/\abs{p}}\!
 \bigg(\!\!
  \sqrt{\frac{3\ln(8n / \epsilon)}{\frac{n}{k}(\pp\!+\!(k-1)\qp)}}, \sqrt{\frac{3\ln(8n / \epsilon)}{\frac{n}{k}(\ppm\!+\!(k-1)\qm)}}
 \bigg)
 $
 }
\end{theorem}

\myComment{
\begin{proof}
 Please see Appendix~\ref{proof:concentration}.
\end{proof}
}
In Thm~\ref{theorem:ConcentrationBound-signedGraph-identity-SBM} we take the spectral norm.
A more general version of Theorem~\ref{theorem:ConcentrationBound-signedGraph-identity-SBM} for the 
inhomogeneous Erd\H{o}s-R\'{e}nyi model, where edges are formed independently with probabilities $p^+_{ij},p^-_{ij}$ is given in the supplementary material. Theorem~\ref{theorem:ConcentrationBound-signedGraph-identity-SBM} builds on top of concentration results of~\cite{chung2011spectra} proven for the unsigned case $\norm{L^+_\sym - \mathcal L^+_\sym}$.  
We can see that the deviation of $L_p$ from $\mathcal{L}_p$ depends on the power mean of the individual deviations of $L^+_\sym$ and $Q^-_\sym$ from $\mathcal L^+_\sym$ and $\mathcal Q^-_\sym$, respectively. 
Note that the larger the size $n$ of the graph is, the stronger is the concentration of $L_p$ around $\mathcal{L}_p$.

The next Theorem shows that the eigenvectors corresponding to the smallest eigenvalues of $L_p$ are close to the corresponding eigenvectors of $\mathcal{L}_p$. This is a key result showing consistency of our spectral clustering technique with $L_p$ for signed graphs drawn from the  SSBM.

\begin{theorem}\label{theorem:concentration-eigenvectors}
Let $p\neq 0$ be an integer. Let ${V_k,\mathcal{V}_k\in\mathbb{R}^{n\times k}}$ be orthonormal matrices whose columns are the eigenvectors of the $k$ smallest eigenvalues of 
$L_p$ and $\mathcal{L}_p$,
respectively. Let $\rho^+_\varepsilon$, $\rho^-_\varepsilon$ and $C_p$ be defined as in Theorems \ref{theorem:mp_in_expectation} and \ref{theorem:ConcentrationBound-signedGraph-identity-SBM}, respectively. Define $\tilde k = k-1$, if $p\geq 1$, and $\tilde k = k$, if $p\leq -1$ and choose $\epsilon >0$.  \\
If $m_p(\rho^+_\varepsilon, \rho^-_\varepsilon)<1+\varepsilon$, $\delta^+:=\frac{n}{k}(\pp\!+\!(k\!-\!1)\qp)> 3\ln (8n/\epsilon)$, 
and 
$\delta^-:=\frac{n}{k}(\ppm\!+\!(k\!-\!1)\qm)> 3\ln (8n/\epsilon)$, then there exists an orthogonal matrix $O_{\tilde k}\in\mathbb{R}^{\tilde{k}\times \tilde{k}}$ such that, with probability at least $1-\epsilon$, we have

 \resizebox{1.0\columnwidth}{!}{
 $
 \vectornorm{V_{\tilde k}-\mathcal{V}_{\tilde k}O_{\tilde k}}\leq \dfrac{\sqrt{8\tilde k} C_{p}\,  m_{\abs{p}}^{1/\abs{p}}\!
 \Big(\!\!
  \sqrt{\frac{3\ln(8n / \epsilon)}{\delta^+}}, \sqrt{\frac{3\ln(8n / \epsilon)}{\delta^-}}
 \Big)}{
(1+\varepsilon)-m_p(\rho^+_\varepsilon, \rho^-_\varepsilon)
}
 $
 }
\end{theorem}
\myComment{
\begin{proof}
 Please see Section~\ref{theorem:concentration-eigenvectors-PROOF}.
\end{proof}
}
Note that the main difference compared to Thm.~\ref{theorem:ConcentrationBound-signedGraph-identity-SBM} is the
spectral gap  $\gamma_p = (1+\varepsilon)-m_p(\rho^+_\varepsilon,\rho^-_\varepsilon)$
of $\mathcal{L}_p$, which is the difference of the eigenvalues corresponding to the informative versus non-informative eigenvectors of $\mathcal{L}_p$. Thus the stronger the clustering structure the tighter is the concentration of the eigenvectors. Moreover, from the monotonicity of $m_p$ we have $\gamma_p\geq\gamma_q$ for $p<q$,  and thus for $p\leq-1$ the spectral gap increases with $|p|$, ensuring a  stronger concentration of eigenvectors for smaller values of $p$. 

 \section{Experiments on Wikipedia-Elections}\label{section:Wikipedia-experiments}
We now evaluate the Signed Power Mean Laplacian $L_p$ with $p\in\{-10,-5,-2,-1,0,1 \}$ on Wikipedia-Elections dataset~\cite{snapnets}. 
In this dataset each node represents an editor requesting to become administrator 
and positive (resp. negative) edges represent supporting (resp. against) votes to the corresponding admin candidate.


While~\cite{Chiang:2012:Scalable} conjectured that this dataset has no clustering structure, recent works~\cite{Mercado:2016:Geometric,Cucuringu:2019:sponge} have shown that indeed there is clustering structure. As noted in~\cite{Mercado:2016:Geometric}, using the geometric mean Laplacian $L_{GM}$ and looking for $k$ clusters unveils the presence of a large non-informative cluster and $k-1$ remaining smaller clusters which show relevant clustering structure.

Our results verify these recent findings. We set the number of clusters to identify to $k=30$ and in Fig.~\ref{fig:wikipedia} we portray the portion of adjacency matrices of positive and negative edges $W^+$ and $W^-$ corresponding to $k-1$ clusters sorted according to the corresponding identified clusters. 
We can see that when $p\leq 0$ the Signed Power Mean Laplacian $L_p$ identifies clustering stucture, 
whereas this structure is overlooked by the arithmetic mean case $p=1$.
%
%
Moreover, we can see that different powers identify slightly different clusters: this happens as this dataset does not necessarily follow the Signed Stochastic Block Model, and hence we do not fully retrieve the same behaviour studied in Section~\ref{sec:SBM}.

Further experiments on UCI datasets are available in the supplementary material, suggesting that the $L_{GM}$ together with $L_p$ is a reasonable option under different settings.

\textbf{Acknowledgments}
%
The  work  of  F.T.  has been funded by the Marie Curie Individual Fellowship MAGNET n.  744014.

%

\clearpage
\newpage
\bibliographystyle{icml2019}
\bibliography{bibliographyUpdated,referencesUpdated}
\vfill
\newpage
\appendix

\appendixpage

This section contains all the proofs of results mention in the main paper. It is organized as follows:
\begin{itemize}[topsep=-3pt,leftmargin=*]\setlength\itemsep{-3pt}
 \item Section~\ref{appendix:theorem:mp_in_expectation-PROOF} contains the proof of Theorem~\ref{theorem:mp_in_expectation},
 \item Section~\ref{corollary:mp_limit_cases-PROOF} contains the proof of Corollary~\ref{corollary:mp_limit_cases},
 \item Section~\ref{corollary:contention-PROOF} contains the proof of Corollary~\ref{corollary:contention},
 \item Section~\ref{theorem:geometricMeanLaplacian-PROOF} contains the proof of Theorem~\ref{theorem:geometricMeanLaplacian},
 \item Section~\ref{section:resultsOnBetheHessian-Proofs} contains the proof of Theorem~\ref{theorem:bethe_hessian_V1}
 \item Section~\ref{proof:concentration} contains the proof of Theorem~\ref{theorem:ConcentrationBound-signedGraph-identity-SBM},
 \item Section~\ref{theorem:concentration-eigenvectors-PROOF} contains the proof of~\ref{theorem:concentration-eigenvectors}.
 \item Section~\ref{section:Experiments with the Censored Block Model} contains experiments on the Censored Block Model where the superiority of the Bethe Hessian on the sparse regime is verified,
 \item Section~\ref{section:experiments} contains experiments on UCI datasets,
 \item Section~\ref{sec:OnDiagonalShift} contains an analysis on diagonal shift of the signed power mean Laplacian,
 \item Section~\ref{section:computation} contains a description on the numerical scheme for the computation of eigenvectors withtout ever computing the power mean Laplacian matrix,
 \item Section~\ref{section:proportionOfCasesWhereConditionsHold} contains a further study on conditions of expectation of the power mean Laplacian and state of the art approaches.
 \item Section~\ref{sec:wikipedia-experiments-appendix} contains a more detailed inspection of results on Wikipedia-Elections dataset presented in Section~\ref{section:Wikipedia-experiments}.
\end{itemize}

Before starting with the general results, we first mention a couple of basic results.

The following theorem states the monotonocity of the scalar power mean.
\begin{theorem}[\cite{bullen2013handbook}, Ch.~3, Thm.~1]\label{theorem:mp_monotone}
 Let $p<q$ then $m_{p}(a,b)\leq m_{q}(a,b)$
with equality if and only if $a=b$.
\end{theorem}

The following lemma shows the effect of the matrix power mean when matrices have a common eigenvector.
\begin{lemma}
[\cite{Mercado:2018:powerMean}]
\label{lemma:eigenvalues_and_eigenvectors_of_generalized_mean_V2}
Let $\u$ be an eigenvector of both $A$ and $B$,  with corresponding eigenvalues $\alpha$ and $\beta$.
Then 
$\u$ is an eigenvector of $M_p(A,B)$ with eigenvalue $m_p(\alpha,\beta)$.
\end{lemma}

\section{Proof of Theorem \ref{theorem:mp_in_expectation}}\label{appendix:theorem:mp_in_expectation-PROOF}

\begin{proof}(Proof of Theorem \ref{theorem:mp_in_expectation})
We first show that $\boldsymbol \chi_1,\ldots,\boldsymbol \chi_k$ are eigenvectors of $\mathcal W^{+}$ and $\mathcal W^{-}$. 
For $\boldsymbol \chi_1$ we have,
\begin{equation*}
 \mathcal W^{+} \boldsymbol \chi_1 = \mathcal W^{+} \one = \abs{\mathcal C}(\pp + (k-1)\qp)\one = d^+ \one = \lambda_1^+ \one
\end{equation*}
For the remaining vectors $\boldsymbol \chi_2,\ldots,\boldsymbol \chi_k$ we have
\begin{align*}
 \mathcal W^{+} \boldsymbol \chi_i 
 &= \mathcal W^{+} \big( (k-1)\one_{\mathcal{C}_i}-\one_{\overline{\mathcal{ C}_i}} \big)\\
 &= \mathcal W^{+} \big( k\one_{\mathcal{C}_i}-(\one_{\mathcal{C}_i} + \one_{\overline{\mathcal{ C}_i}}) \big)  \\
 &= \mathcal W^{+} \big( k\one_{\mathcal{C}_i}-\one \big)  \\
 &=  k\abs{\mathcal C}(\pp\one_{\mathcal C_i} + \qp\one_{\overline{\mathcal C_i}}) -d^+\one \\
 &=  k\abs{\mathcal C}(\pp\one_{\mathcal C_i} + \qp\one_{\overline{\mathcal C_i}}) -d^+(\one_{\mathcal C_i}+\one_{\overline{\mathcal{ C}_i}}) \\
 &=  \abs{\mathcal C}(k\pp-d^+)\one_{\mathcal C_i} + \abs{\mathcal C}(k\qp-d^+)\one_{\overline{\mathcal C_i}} \\
 &=  \abs{\mathcal C}(k-1)(\pp-\qp)\one_{\mathcal C_i} - \abs{\mathcal C}(\pp-\qp)\one_{\overline{\mathcal C_i}} \\
 &= \abs{\mathcal C} (\pp-\qp) \big( (k-1)\one_{\mathcal{C}_i}-\one_{\overline{\mathcal{ C}_i}} \big) \\
 &= \abs{\mathcal C} (\pp-\qp) \boldsymbol \chi_i  \\
 &= \lambda_i \boldsymbol \chi_i  \\
\end{align*}
The same procedure holds for $\mathcal W^{-}$.
Thus, we have shown that $\boldsymbol \chi_1,\ldots,\boldsymbol \chi_k$ are eigenvectors of both $\mathcal W^{+}$ and $\mathcal W^{-}$. 
In particular, we have seen that
\begin{align*}
 \lambda^+_1 = \abs{\mathcal{C}}(\pp+(k-1)\qp), \quad \lambda^+_i = \abs{\mathcal{C}}(\pp-\qp)\\
 \lambda^-_1 = \abs{\mathcal{C}}(\ppm+(k-1)\qm), \quad \lambda^-_i = \abs{\mathcal{C}}(\ppm-\qm)
\end{align*}
for $i=2,\ldots,k$.
Further, as both matrices $\mathcal{W}^+$ and $\mathcal{W}^-$ share all their eigenvectors, they are simultaneously diagonalizable, that is there exists a 
non-singular matrix $\Sigma$ such that $\Sigma^{-1}\mathcal{W}^{\pm}\Sigma =\Lambda^\pm$, 
where $\Lambda^+$ and $\Lambda^-$ are diagonal matrices  $\Lambda^\pm = \diag(\lambda_1^{\pm}, \dots, \lambda_k^\pm, 0, \dots, 0)$.

As we assume that all clusters are of the same size $\abs{C}$, the expected signed graph is a regular graph with degrees $d^+$ and $d^-$. Hence, the normalized Laplacian and normalized signless Laplacian of the expected signed graph can be expressed as 
\begin{align*}
 \mathcal L^{+}_\sym &= \Sigma(I-\frac{1}{d^+}\Lambda^{+})\Sigma^{-1}\\
 \mathcal Q^{-}_\sym &= \Sigma(I+\frac{1}{d^-}\Lambda^{-})\Sigma^{-1}\\
\end{align*}
Thus, we can observe that
\begin{align*}
 \lambda^+_1(\mathcal L^{+}_\sym) &= 0,       \qquad\qquad\,\, \lambda^-_1(\mathcal Q^{-}_\sym) = 2\\
 \lambda^+_i(\mathcal L^{+}_\sym) &=1-\rho^+, \qquad \lambda^-_i(\mathcal Q^{-}_\sym) =1+\rho^-\\
 \lambda^+_j(\mathcal L^{+}_\sym) &=1, \qquad\qquad\,\, \lambda^-_j(\mathcal Q^{-}_\sym) =1\\
\end{align*}
for $i=2,\ldots,k$, and $j=k+1,\ldots,\abs{V}$,
where
\begin{align*}
 \rho^+ &=(\pp-\qp)/(\pp+(k-1)\qp) \\
 \rho^- &=(\ppm-\qm)/(\ppm+(k-1)\qm) 
\end{align*}

By obtaining the signed power mean Laplacian on diagonally shifted matrices, 
$$\mathcal L_p = M_p(\mathcal L^{+}_\sym+\varepsilon I, \mathcal Q^{-}_\sym+\varepsilon I)$$
we have by Lemma~\ref{lemma:eigenvalues_and_eigenvectors_of_generalized_mean_V2}
\begin{equation}\label{eq:eigenvaluesOfPowerMeanInExpectation}
  \begin{aligned}
  \lambda_1(\mathcal L_p) &= m_p(\lambda^+_1+\varepsilon, \lambda^-_1+\varepsilon)=m_p(\varepsilon,2+\varepsilon) \\
  \lambda_i(\mathcal L_p) &= m_p(1-\rho^+ +\varepsilon, 1+\rho^- +\varepsilon)\\
  \lambda_j(\mathcal L_p) &= m_p(\lambda^+_j+\varepsilon, \lambda^-_j+\varepsilon)=1+\varepsilon
  \end{aligned}
\end{equation}
Observe that $\lambda_j(\mathcal L_p)$, with $j=k+1,\ldots,\abs{V}$, corresponds to eigenvectors that do not yield an informative embedding. Hence, we do not want this eigenvalue to belong to the bottom of the spectrum of $\mathcal L_p$. Thus, for the case of $\boldsymbol \chi_2,\ldots,\boldsymbol \chi_k$,
we can see that they will be located at the bottom of the spectrum if the following condition holds:
$$
\lambda_i(\mathcal L_p) = m_p(1-\rho^+ +\varepsilon, 1+\rho^- +\varepsilon) < 1+\varepsilon = \lambda_j(\mathcal L_p)
$$
It remains to analyze the case of the constant eigenvector $\boldsymbol \chi_1$.
Note that its associated eigenvalue 
$\lambda_1(\mathcal L_1)$ has the following relationship to the non-informative eigenvectors:
$$\lambda_1(\mathcal L_1) = m_1(\varepsilon, 2+\varepsilon)= 1+\varepsilon = \lambda_j(\mathcal L_p)$$
By Theorem~\ref{theorem:mp_monotone} we know that the scalar power mean is monotone in its parameter $p$, and thus, 
for the case $p<1$ we observe
\begin{align*}
 \lambda_1(\mathcal L_p)&=m_p(\varepsilon, 2\!+\!\varepsilon)<m_1(\varepsilon, 2\!+\!\varepsilon)=\lambda_1(\mathcal L_1)=\lambda_j(\mathcal L_p) 
\end{align*}
and for the case $p\geq 1$ we observe
\begin{align*}
 \lambda_1(\mathcal L_p)&=m_p(\varepsilon, 2\!+\!\varepsilon)\geq m_1(\varepsilon, 2\!+\!\varepsilon)=\lambda_1(\mathcal L_1)=\lambda_j(\mathcal L_p)
\end{align*}

This means that for positive powers 
$p\geq 1$
, the constant eigenvector $\boldsymbol \chi_1$ does not belong to the bottom of the spectrum, whereas for 
$p<1$
it always does.

With this in mind, we reach the desired result, namely
 \begin{itemize}[topsep=-3pt,leftmargin=*]
 \item Let 
 $p\geq 1$
. $\{ \boldsymbol \chi_i \}_{i=2}^{k}$ correspond to the $(k$-$1)$-smallest eigenvalues of 
 $\mathcal{L}_p$ if and only if $m_p(\boldsymbol{\mu}+\varepsilon\one)<1+\varepsilon$,
 \item Let 
 $p< 1$
. $\{ \boldsymbol \chi_i \}_{i=1}^{k}$ correspond to the $k$-smallest eigenvalues of 
 $\mathcal{L}_p$ if and only if $m_p(\boldsymbol{\mu}+\varepsilon\one)<1+\varepsilon$
 \end{itemize}

\end{proof}

\section{Proof of Corollary~\ref{corollary:mp_limit_cases}}\label{corollary:mp_limit_cases-PROOF}
\begin{proof}[Proof of Corollary~\ref{corollary:mp_limit_cases}]
Following the proof from Theorem~\ref{theorem:mp_in_expectation},
 recall that 
 $\lim_{p\to\infty}m_p(x)=\max \{x_1,\dots,x_T\}$ and 
 $\lim_{p\to -\infty}m_p(x)=\min \{x_1,\dots,x_T\}$.
 
 Thus, 
 $m_\infty(\boldsymbol\mu+\varepsilon I) = \max(1-\rho^+,1+\rho^-)$,
 and hence
 $m_\infty(\boldsymbol\mu+\varepsilon I)  < 1+\varepsilon$
 if and only if $\rho^+ > 0$ and $\rho^- < 0$, yielding the desired conditions.
 
 The case for $p\to -\infty$ is analogous:
 $m_{-\infty}(\boldsymbol\mu+\varepsilon I) = \min(1-\rho^+,1+\rho^-)$
 and thus
 $m_{-\infty}(\boldsymbol\mu+\varepsilon I)  < 1+\varepsilon$
 if and only if $\rho^+ > 0$ or $\rho^- < 0$, yielding the desired conditions.
\end{proof}
\section{Proof of Corollary~\ref{corollary:contention}}\label{corollary:contention-PROOF}

\begin{proof}[Proof of Corollary~\ref{corollary:contention}]
If $\lambda_1,\ldots,\lambda_k$ resp. ($\lambda_2,\ldots,\lambda_k$) are among the $k$ (resp. $k-1$)-smallest eigenvalues of 
$\mathcal L_{p}$, then by Theorem \ref{theorem:mp_in_expectation}, we have $m_p(\boldsymbol{\mu}+\epsilon\one)<1+\epsilon$. By Theorem~\ref{theorem:mp_monotone} we have $m_q(\boldsymbol{\mu}+\epsilon\one)\leq m_p(\boldsymbol{\mu}+\epsilon\one)$, Theorem \ref{theorem:mp_in_expectation} concludes the proof.
\end{proof}

\section{Proof of Theorem~\ref{theorem:geometricMeanLaplacian}}\label{theorem:geometricMeanLaplacian-PROOF}

\begin{proof}[Proof of Theorem~\ref{theorem:geometricMeanLaplacian}]
Following the proof from Theorem~\ref{theorem:mp_in_expectation} we can see that $\mathcal{L}_\sym^+$ and $\mathcal{Q}_\sym^{-}$ share all of their eigenvectors. 
Let $\u$ be an eigenvector of $\mathcal{L}_\sym^+$ and $\mathcal{Q}_\sym^{-}$ with eigenvalues $\alpha$ and $\beta$, respectively.

By Lemma~\ref{lemma:eigenvalues_and_eigenvectors_of_generalized_mean_V2} we have 
$$
\mathcal{L}_0\u = m_0(\alpha,\beta)\u
$$

Moreover, from Theorem 1 in~\cite{Mercado:2016:Geometric} we know that 
$$
(\mathcal{L}_\sym^+ \# \mathcal{Q}_\sym^{-}) \u = \sqrt{\alpha\beta}\u
$$

Further, $m_0(\alpha,\beta)=\sqrt{\alpha\beta}$. Hence, as $\mathcal{L}_0$ and $\mathcal{L}_\sym^+ \# \mathcal{Q}_\sym^{-}$ have in common all eigenvectors and eigenvalues, we conclude that $\mathcal{L}_0=\mathcal{L}_\sym^+ \# \mathcal{Q}_\sym^{-}$.
\end{proof}

\section{Proof of Theorem~\ref{theorem:ConcentrationBound-signedGraph-identity-SBM}}\label{proof:concentration}

For the proof of Theorem~\ref{theorem:ConcentrationBound-signedGraph-identity-SBM}, we first  present Theorem~\ref{theorem:ConcentrationBound-signedGraph-general} which is a general version that allows to choose different diagonal shifts of the Laplacians together with different edge probabilities.
\begin{theorem}\label{theorem:ConcentrationBound-signedGraph-general}
Let $G^+$ and $G^-$ be random  graphs with independent edges 
$\mathbb{P}(W^{+}_{i,j}=1)=p^{+}_{ij}$ 
and
$\mathbb{P}(W^{-}_{i,j}=1)=p^{-}_{ij}$.
Let $\delta^+$, $\delta^-$ be the minimum expected degrees of $G^{+}$ and $G^{-}$, respectively.
Let 
$C_{p}^+ =  p^{1/p} \beta^{1-1/p} $, and 
$C_{p}^- = \abs{p}^{1/\abs{p}} \alpha^{-(3+1/\abs{p})}  $.
Choose $\epsilon>0$. 
Then there exist constants $k^+ = k^+(\epsilon/2)$ and $k^- = k^-(\epsilon/2)$
such that if $\delta^+> k^+ \ln n$, and $\delta^-> k^- \ln n$ then with probability at least $1-\epsilon$, 

 $
 \norm{L_p - \mathcal{L}_p} \leq C_{p}^+ 
 m_p
 \bigg(2
 \sqrt{\frac{3\ln(8n / \epsilon)}{\delta^+}},2\sqrt{\frac{3\ln(8n / \epsilon)}{\delta^-}}
 \bigg)^{1/p}
 $
 
for $p\geq1$, with $p$ integer and

 $
 \norm{L_p - \mathcal{L}_p} \leq 
 C_{p}^- 
 m_{\abs{p}}
 \bigg(
 2 \sqrt{\frac{3\ln(8n / \epsilon)}{\delta^+}},2 \sqrt{\frac{3\ln(8n / \epsilon)}{\delta^-}}
 \bigg)^{1/\abs{p}}
 $
 
for $p\leq -1$, with $p$ integer,
where
$L_p=M_p(L^{+}_\sym+\alpha I, Q^{-}_\sym+\alpha I)$, 
and 
$\mathcal L_p=M_p(\mathcal L^{+}_\sym+\alpha I, \mathcal Q^{-}_\sym+\alpha I)$.
%
\end{theorem}
%
%
Before starting the proof of Theorem~\ref{theorem:ConcentrationBound-signedGraph-general}, we present an upper bound on the matrix power mean.
\begin{theorem}\label{theorem:UpperBoundOfMatrixPowerMean}
Let 
$A_1,\ldots,A_T,B_1,\ldots,B_T$ be symmetric matrices
where $\alpha\!\leq\!\lambda(A_i)\!\leq\!\beta $, $\alpha\!\leq\!\lambda(B_i)\!\leq\!\beta$ for $i=1,\ldots T$ and 
$\alpha,\beta>0$.

%
Let 
$C_p^+ = p^{1/p} \beta^{1-1/p}$
and
$C_p^- = \abs{p}^{1/\abs{p}} \alpha^{-(3+1/\abs{p})} $.
Then, for $p\geq1$, with $p$ integer
  \begin{align*}
  &\norm{M_p(A_1,\ldots,A_T) - M_p(B_1,\ldots,B_T)}
  \\
  &\qquad\leq C_p^+ m_p\big( \norm{A_1-B_1},\ldots,\norm{A_T-B_T} \big)^{1/p}
  \end{align*}
and, for $p\leq-1$, with $p$ integer
 \begin{align*}
  &\norm{M_p(A_1,\ldots,A_T) - M_p(B_1,\ldots,B_T)}
  \\
  &\qquad\leq C_p^- m_{\abs{p}}\big( \norm{A_1-B_1},\ldots,\norm{A_T-B_T} \big)^{1/\abs{p}}
 \end{align*}
 \end{theorem}
 \begin{proof}
 The proof is contained in Section~\ref{appendix:proof_upper_bound}.
 \end{proof}
 Observe that the upper bound in Theorem~\ref{theorem:UpperBoundOfMatrixPowerMean} is general in the sense that it is suitable for symmetric definite matrices with bounded spectrum, and for an arbitrary number of matrices.

We are now ready to prove Theorem~\ref{theorem:ConcentrationBound-signedGraph-general}.
\begin{proof}[Proof of Theorem~\ref{theorem:ConcentrationBound-signedGraph-general}]
Let
\begin{align*}
 A_1 &= L_\sym^{+}, \qquad\qquad\qquad\quad\, B_1 = \mathcal  L_\sym^{+} \\
 A_2 &= Q_\sym^{-}, \qquad\qquad\qquad\quad\, B_2 = \mathcal Q_\sym^{-}
 \end{align*}
with he corresponding signed power mean Laplacian
\begin{align*}
 L_p &= M_p(L_\sym^{+}, Q_\sym^{-}), \quad\quad\,\,\mathcal{L}_p = M_p(\mathcal{L}_\sym^{+}, \mathcal{Q}_\sym^{-})
\end{align*}
We start with the case $p\geq 1$, with $p$ integer.
Let $C_p^+ = p^{1/p} \beta^{1-1/p}$. By Theorem~\ref{theorem:UpperBoundOfMatrixPowerMean} we have
   \begin{equation*}
  \norm{L_p - \mathcal{L}_p} 
  \leq 
  C_p^+ 
  m_p(\norm{L_\sym^{+}-\mathcal{L}_\sym^{+}},\norm{Q_\sym^{-}-\mathcal{Q}_\sym^{-}})^{1/p}
  \end{equation*}
  
Let $\bm{\gamma} = (\gamma_1, \gamma_2)$
where

$\,$\hfill  $\gamma_1 = 2 \sqrt{\frac{3\ln(8n / \epsilon)}{\delta^+}} \quad\qquad\,\, \gamma_2 = 2 \sqrt{\frac{3\ln(8n / \epsilon)}{\delta^-}}$ \hfill $\,$
%
%

Define $a=c \, m_p(\bm{\gamma})$ and $c=C_p^+$.
Then,
%
\begin{align}
 &\mathbb{P} \big( \norm{L_p - \mathcal{L}_p} > a \big)\nonumber 
 \\
 \leq& \mathbb{P} \Bigg( c \, m_p(\norm{L_\sym^{+}-\mathcal L_\sym^{+}}, \norm{Q_\sym^{-}-\mathcal Q_\sym^{-}}) > a \Bigg)\nonumber
 \\
 =&
 \mathbb{P} \Bigg(  m_p(\norm{L_\sym^{+}-\mathcal L_\sym^{+}}, \norm{Q_\sym^{-}-\mathcal Q_\sym^{-}}) > \frac{a}{c} \Bigg)\nonumber
 \\
 =&
 \mathbb{P} \Bigg( \norm{L_\sym^{+}-\mathcal L_\sym^{+}}^p +
 \norm{Q_\sym^{-}-\mathcal Q_\sym^{-}}^p > 2\bigg(\frac{a}{c} \bigg)^p \Bigg)\nonumber
 \\
 =&
 \mathbb{P} \Bigg( \norm{L_\sym^{+}-\mathcal L_\sym^{+}}^p +
 \norm{Q_\sym^{-}-\mathcal Q_\sym^{-}}^p > \sum_{i=1}^2 \gamma_i^p \Bigg)\nonumber
 \\
   \leq&
 \mathbb{P} 
 \Bigg( 
 \bigg\{ \norm{L_\sym^{+}-\mathcal L_\sym^{+}}^p > \gamma_1^p \bigg\} 
 \cup\nonumber
 \\
 &\qquad\qquad\bigg\{ \norm{Q_\sym^{-}-\mathcal Q_\sym^{-}}^p > \gamma_2^p \bigg\} 
 \Bigg)\nonumber
 \\
 \leq&
 \mathbb{P}\bigg( \norm{L_\sym^{+}-\mathcal L_\sym^{+}}^p > \gamma_1^p \bigg)
 \nonumber
 \\
 &\qquad+\mathbb{P}\bigg( \norm{Q_\sym^{-}-\mathcal Q_\sym^{-}}^p > \gamma_2^p \bigg)
 \label{eq:ApplyBoolesInequality}
  \\
 =&
 \mathbb{P}\bigg( \norm{L_\sym^{+}-\mathcal L_\sym^{+}} > \gamma_1 \bigg)
 \nonumber
 \\
 &\qquad+
 \mathbb{P}\bigg( \norm{Q_\sym^{-}-\mathcal Q_\sym^{-}} > \gamma_2 \bigg)\nonumber
  \\
 =&
 \mathbb{P}\bigg( \norm{L_\sym^{+}-\mathcal L_\sym^{+}} > 2 \sqrt{\frac{3\ln(8n / \epsilon)}{\delta^+}} \bigg)\nonumber
 \\
 &\qquad+
 \mathbb{P}\bigg( \norm{Q_\sym^{-}-\mathcal Q_\sym^{-}} > 2 \sqrt{\frac{3\ln(8n / \epsilon)}{\delta^-}} \bigg)
 \nonumber\\
 =&
 \mathbb{P}\bigg( \norm{L_\sym^{+}-\mathcal L_\sym^{+}} > 2 \sqrt{\frac{3\ln(4n / \hat{\epsilon})}{\delta^+}} \bigg)\nonumber
 \\
 &\qquad+
 \mathbb{P}\bigg( \norm{Q_\sym^{-}-\mathcal Q_\sym^{-}} > 2 \sqrt{\frac{3\ln(4n / \hat{\epsilon})}{\delta^-}} \bigg)
 \nonumber\\
 =&
 \mathbb{P}\bigg( \norm{L_\sym^{+}-\mathcal L_\sym^{+}} > 2 \sqrt{\frac{3\ln(4n / \hat{\epsilon})}{\delta^+}} \bigg)\nonumber
 \\
 &\qquad+
 \mathbb{P}\bigg( \norm{L_\sym^{-}-\mathcal L_\sym^{-}} > 2 \sqrt{\frac{3\ln(4n / \hat{\epsilon})}{\delta^-}} \bigg)
 \nonumber\\
 \leq&
 \hat{\epsilon} + \hat{\epsilon}\label{eq:ApplyRadcliffeChungBound}
 \\
 =&
 \epsilon \nonumber
\end{align}
where $\hat{\epsilon}=\epsilon/2$. 
Inequality~\eqref{eq:ApplyBoolesInequality} follows from Boole's inequality.
Inequality~\eqref{eq:ApplyRadcliffeChungBound} comes from applying Theorem~\ref{thm:chung} from~\cite{chung2011spectra} to $G^{+}$ and $G^{-}$, with corresponding minimum expected degree $\delta^+$, and $\delta^-$, respectively, and $\hat{\epsilon}$,
and 
\begin{align*}
 \norm{Q_\sym^{-}-\mathcal Q_\sym^{-}} 
 &= \norm{(I+T)-(I+\mathcal{T})}\\
 &= \norm{(I-T)-(I-\mathcal{T})}\\
 &= \norm{L_\sym^{-}-\mathcal L_\sym^{-}}
\end{align*}
where
\begin{align*}
 T &= (D^{-})^{-1/2}W^{-}(D^{-})^{-1/2} \\
 \mathcal{T} &= (\mathcal D^{-})^{-1/2} \mathcal W^{-}(\mathcal D^{-})^{-1/2}
\end{align*}

Thus, 
$$
\mathbb{P} \bigg( \norm{L_p - \mathcal{L}_p} \geq a \bigg) < \epsilon
$$
and hence
$$
\mathbb{P} \bigg( \norm{L_p - \mathcal{L}_p} \leq a \bigg) < 1-\epsilon
$$
completing the proof for the case $p\geq 1$.

For the proof of the case $p\leq -1$ with $p$ integer, let 
 $c=\abs{p}^{1/\abs{p}} \alpha^{-(3+1/\abs{p})}$, and proceed as for the previous case with $\abs{p}$.
\end{proof}

We now finally give the proof for Theorem~\ref{theorem:ConcentrationBound-signedGraph-identity-SBM}.
\begin{proof}[Proof of Theorem~\ref{theorem:ConcentrationBound-signedGraph-identity-SBM}]

We will adapt to our particular case the general version presented in Theorem~\ref{theorem:ConcentrationBound-signedGraph-general}. We do this by showing that our Stochastic Block Model approach together with the shift of our model are particular cases of Theorem~\ref{theorem:ConcentrationBound-signedGraph-general}.

First,
note that the spectrum of the normalized Laplacians $L^+_\sym$ and $Q^-_\sym$ is upper bounded by two, i.e. $\lambda(L^+_\sym),\lambda(Q^-_\sym)\in [0,2]$. Hence, by adding a diagonal shift we get 
$\lambda(L^+_\sym+\alpha I),\lambda(Q^-_\sym+\alpha I)\in [\alpha,2+\alpha]$.
Letting $\alpha=\varepsilon$ and $\beta=2+\alpha$ we get the shift corresponding to the particular case from Theorem~\ref{theorem:ConcentrationBound-signedGraph-identity-SBM}. 
%

Further, observe that our SBM model is obtained by setting $p^{+}_{ij}=\pp$ and $p^{-}_{ij}=\ppm$ if $v_i,v_j$ belong to the same cluster and $p^{+}_{ij}=\qp$ and $p^{-}_{ij}=\qm$ if $v_i,v_j$ belong to different clusters.

Moreover,
under the Stochastic Block Model here considered, the induced expected graphs are regular, and thus all nodes have the same degree. 
Hence, the minimum expected degrees of $G^+$ and $G^-$ are
\begin{align*}
 \delta^+ &= \frac{n}{k}(\pp+(k-1)\qp)
 \\
 \delta^- &= \frac{n}{k}(\ppm+(k-1)\qm)
\end{align*}

Thus, taking these settings into Theorem~\ref{theorem:ConcentrationBound-signedGraph-general}
we get the desired result,
except that the condition on the minimum expected degrees is that there exists constants $k^+=k^+(\epsilon/2)$, and $k^-=k^-(\epsilon/2)$ such that the desired concentration holds.

To overcome this, observe in the proof of Theorem~\ref{thm:chung} (p.9) that the condition $\delta>k \ln(n)$ comes from the requirement 
$$
\sqrt{\frac{3\ln(4n/\epsilon)}{\delta}}<1
$$
Thus, by setting $\delta > 3\ln(4n/\epsilon)
$
the condition is fulfilled.
In our case, this yields to
$\delta^+ > 3\ln(4n/\hat{\epsilon}) = 3\ln(8n/\epsilon)$
and
$\delta^- > 3\ln(4n/\hat{\epsilon}) = 3\ln(8n/\epsilon)$
, leading to the desired result.
%


%
 
\end{proof}

\section{Proof of Theorem~\ref{theorem:UpperBoundOfMatrixPowerMean}}~\label{appendix:proof_upper_bound}

Before going into the proof, a set of preliminary results are necessary.

In what follows, for Hermitian matrices $A$ and $B$ we mean by $A\preceq B$ that $B-A$ is positive semidefinite (see \cite{Bhatia1997MatrixAnalysis}, Ch. 5, and \cite{Tropp:2015}, Ch. 2.1.8 for more details). We now proceed with the definition of a operator monotone function:

\begin{definition}[\cite{Tropp:2015} Ch. 8.4.2, \cite{Bhatia1997MatrixAnalysis} Ch. 5.]
Let $f$:$I\to\mathbb{R}$ be a function on an interval $I$ of the real line. The function $f$ is operator monotone on $I$ when $A\preceq B$ implies $f(A)\preceq f(B)$ for all Hermitian matrices $A$ and $B$ whose eigenvalues are contained in $I$.
\end{definition}

The following result states that the negative inverse is operator monotone.
\begin{proposition}[\cite{Bhatia1997MatrixAnalysis}, Prop. \uppercase\expandafter{\romannumeral 5\relax}.1.6, \cite{Tropp:2015}, Prop. 8.4.3]\label{proposition:operatorMonotone}
The function $f(t)=-\frac{1}{t}$ is operator monotone on $(0,\infty)$.
\end{proposition}

The following result states that the effect of operator monotone functions can be upper bounded in a helpful way.

\begin{theorem}[\cite{Bhatia1997MatrixAnalysis}, Theorem. \uppercase\expandafter{\romannumeral 10\relax}.3.8]\label{theorem:fA-FB}
 Let $f$ be an operator monotone function on $(0,\infty)$ and let $A,B$ be two positive definite matrices that are bounded below by $a$; i.e. $A\geq aI$ and $B\geq aI$ for the positive number $a$. Then for every unitarily invariant norm
 \begin{equation*}
  \normtwo{f(A) - f(B)} \leq f'(a) \normtwo{A-B}
 \end{equation*}
\end{theorem}

Applying this to the case of the negative inverse leads to the following Corollary.
\begin{corollary}\label{corollary:theorem:fA-FB-inverse}
 Let $A,B$ be two positive definite matrices that are bounded below by $a$; i.e. $A\geq aI$ and $B\geq aI$ for the positive number $a$. Then for every unitarily invariant norm
 \begin{equation*}
  \normtwo{A^{-1} - B^{-1}} \leq \frac{1}{a^2} \normtwo{A-B}
 \end{equation*} 
\end{corollary}
\begin{proof}
 Let $f(t) = -\frac{1}{t}$. Then, by Proposition~\ref{proposition:operatorMonotone} we know that $f$ is operator monotone.
 Since $f'(t) = 1/t^2$, it follows from Theorem~\ref{theorem:fA-FB}
 \begin{align*}
  &\normtwo{A^{-1} - B^{-1}} = \normtwo{f(A) - f(B)} \leq \\
  & \qquad\qquad\qquad\qquad f'(a) \normtwo{A-B} = \frac{1}{a^2} \normtwo{A-B}
 \end{align*}
\end{proof}

The next results states a useful result on positive powers between zero and one.

\begin{corollary}[\cite{Bhatia1997MatrixAnalysis}, Eq. \uppercase\expandafter{\romannumeral 10\relax}.2]\label{corollary:fA-FB-smallerThan-fAB}
Let $A,B$ be two positive semidefinite matrices. Then, for $0\leq r \leq 1$
\begin{equation*}
 \norm{A^r-B^r}\leq \norm{A-B}^r
\end{equation*}
\end{corollary}

Its equivalent to positive integer powers is stated in the following result.
\begin{proposition}[See~\cite{Bhatia1997MatrixAnalysis}, Eq. \uppercase\expandafter{\romannumeral 9\relax}.4]\label{proposition:XmMinusYm}
For any two matrices $X,Y$, and for $m=1,2,\ldots,$
\begin{equation*}
 \norm{X^m-Y^m}\leq mM^{m-1}\norm{X-Y}
\end{equation*}
where $M=\max(\norm{X},\norm{Y})$.
\end{proposition}

Next we show that the spectrum of the matrix power mean is well bounded for positive powers larger than one.

\begin{proposition}\label{proposition:upper-lower-bound-of-eigenvalues-of-power-mean}
Let $A_1,\ldots,A_T$ be symmetric positive definite matrices that are bounded below and above by $\alpha$ and $\beta$; i.e. $\alpha I \leq A_i \leq \beta I$ for positive numbers $\alpha$ and $\beta$. 
Then, for $p\geq1$, with $p$ integer
\begin{equation*}
 \alpha \leq \lambda(M_{p}(A_1 , \ldots , A_T)) \leq \beta
\end{equation*}
\end{proposition}
\begin{proof}
 Let $S_p(A_1 , \ldots , A_T) = \frac{1}{T}\sum _{{i=1}}^{T}A_{i}^{p}$. Then
 \begin{align*}
  \langle x, S_p(A_1 , \ldots , A_T)x \rangle &= \langle x, \Bigg( \frac{1}{T}\sum _{{i=1}}^{T}A_{i}^{p} \Bigg)x \rangle \\
  &= \frac{1}{T}\sum _{{i=1}}^{T} \langle x, A_{i}^{p} x \rangle 
 \end{align*}
 Thus, we obtain the following upper bound
 \begin{align*}
  \max_{\norm{x}=1} \langle x, S_p(A_1 , \ldots , A_T)x \rangle 
  &= 
  \max_{\norm{x}=1} \frac{1}{T}\sum _{{i=1}}^{T} \langle x, A_{i}^{p} x \rangle \\
  &\leq 
  \frac{1}{T}\sum _{{i=1}}^{T} \max_{\norm{x}=1}  \langle x, A_{i}^{p} x \rangle \\
  &\leq \beta^p
 \end{align*}
 Hence, $\lambda_{max}(S_p(A_1 , \ldots , A_T)) \leq \beta^p$, and thus we obtain the corresponding upper bound $\lambda_{max}(M_p(A_1 , \ldots , A_T)) \leq \beta$.
 \newline\newline
 In a similar way we obtain the following lower bound,
 \begin{align*}
  \min_{\norm{x}=1} \langle x, S_p(A_1 , \ldots , A_T)x \rangle 
  &= 
  \min_{\norm{x}=1} \frac{1}{T}\sum _{{i=1}}^{T} \langle x, A_{i}^{p} x \rangle \\
  &\geq 
  \frac{1}{T}\sum _{{i=1}}^{T} \min_{\norm{x}=1}  \langle x, A_{i}^{p} x \rangle \\
  &\geq \alpha^p
 \end{align*}
 Hence, $\lambda_{min}(S_p(A_1 , \ldots , A_T)) \geq \alpha^p$, and thus we obtain the corresponding lower bound $\lambda_{min}(M_p(A_1 , \ldots , A_T)) \geq \alpha$.
 \newline\newline
Therefore, $\alpha \leq \lambda(M_{p}(A_1 , \ldots , A_T)) \leq \beta$.
\end{proof}

We now present results for $p\geq 1$ of Theorem~\ref{theorem:UpperBoundOfMatrixPowerMean}.
\subsection{Results for the case $p\geq 1$}
The following two propositions are the main ingredients for the upper bound presented in Theorem~\ref{theorem:UpperBoundOfMatrixPowerMean} for the case $p\geq 1$.
\begin{proposition}\label{proposition:takeOutFractionalPower}
Let $A_1,\ldots,A_T,B_1,\ldots,B_T$ be symmetric positive semidefinite matrices. Then, for $p\geq~1$ with $p$ integer
 \begin{align*}
  &\norm{M_p(A_1,\ldots,A_T) - M_p(B_1,\ldots,B_T)} \\ 
  &\qquad\leq \norm{M_p^p(A_1,\ldots,A_T) - M_p^p(B_1,\ldots,B_T)}^\frac{1}{p}
 \end{align*}
 \begin{proof}
  Let $S_p(A_1,\ldots,A_T) = {\frac  {1}{T}}\sum _{{i=1}}^{T}A_{i}^{p}$
  and $r=1/p$. Then,
  \begin{align*}
   &\norm{M_p(A_1,\ldots,A_T) - M_p(B_1,\ldots,B_T)} \\
   &\qquad= \norm{S_p^{1/p}(A_1,\ldots,A_T) - S_p^{1/p}(B_1,\ldots,B_T)} \\
   &\qquad= \norm{S_p^{r}(A_1,\ldots,A_T) - S_p^{r}(B_1,\ldots,B_T)} \\
   &\qquad\leq \norm{S_p(A_1,\ldots,A_T) - S_p(B_1,\ldots,B_T)}^{r} \\
   &\qquad= \norm{S_p(A_1,\ldots,A_T) - S_p(B_1,\ldots,B_T)}^{1/p} \\
   &\qquad= \norm{M_p^p(A_1,\ldots,A_T) - M_p^p(B_1,\ldots,B_T)}^{1/p}
  \end{align*}
  where the inequality comes from Corollary~\ref{corollary:fA-FB-smallerThan-fAB},
giving the desired result.
 \end{proof}
\end{proposition}
\begin{proposition}\label{proposition:takeOutIntegerPower}
Let $A_1,\ldots,A_T,B_1,\ldots,B_T$ be symmetric positive semidefinite matrices such that 
${\lambda(A_i)\!\leq\!\beta}$ and ${\lambda(B_i)\!\leq\!\beta}$ for $i=1,\ldots T$. Then, for $p\geq~1$,
  \begin{align*}
  &\norm{M_p^p(A_1,\ldots,A_T) - M_p^p(B_1,\ldots,B_T)} \\
  &\qquad\qquad\leq p \beta^{p-1} m_p\big(\norm{A_1-B_1},\ldots,\norm{A_T-B_T} \big)
  \end{align*}
\end{proposition}
\begin{proof}
Let $\beta_i = \max(\norm{A_i}, \norm{B_i})$. Then,
\begin{align*}
 &\norm{M_p^p(A_1,\ldots,A_T) - M_p^p(B_1,\ldots,B_T)} \\
 &\qquad= \norm{\left({\frac  {1}{T}}\sum _{{i=1}}^{T}A_{i}^{p}\right) - \left({\frac  {1}{T}}\sum _{{i=1}}^{T}B_{i}^{p}\right)}\\
 &\qquad= \norm{ {\frac  {1}{T}}\sum _{{i=1}}^{T}A_{i}^{p} - B_{i}^{p} }\\
 &\qquad\leq \frac{1}{T}\sum_{{i=1}}^{T}\norm{A_{i}^{p} - B_{i}^{p} }\\
 &\qquad\leq \frac{1}{T}\sum_{{i=1}}^{T} p(\beta_i)^{p-1} \norm{A_{i} - B_{i} }\\
 &\qquad\leq \frac{1}{T}\sum_{{i=1}}^{T} p\beta^{p-1} \norm{A_{i} - B_{i} }\\
 &\qquad= p\beta^{p-1} \Bigg( \frac{1}{T} \sum_{i=1}^T \norm{A_i-B_i} \Bigg)\\
 &\qquad= p \beta^{p-1} m_1\big(\norm{A_1-B_1},\ldots,\norm{A_T-B_T} \big)\\
 &\qquad\leq p \beta^{p-1} m_p\big(\norm{A_1-B_1},\ldots,\norm{A_T-B_T} \big)
 \end{align*}
 where: 
the first inequality follows from the triangular inequality,
the second inequality follows from Proposition~\ref{proposition:XmMinusYm},
the third inequality follows as $\beta_i \leq \beta$, and
the last inequality comes from the monotonicity of the scalar power means.
\end{proof}
The next Lemma contains the proof corresponding to the case of positive powers of Theorem~\ref{theorem:UpperBoundOfMatrixPowerMean}.
\begin{lemma}[Theorem~\ref{theorem:UpperBoundOfMatrixPowerMean} for the case $p\geq 1$]\label{proposition:MpUpperBoundForPositivePowers}
Let $A_1,\ldots,A_T,B_1,\ldots,B_T$ be symmetric positive semidefinite matrices where $\lambda(A_i)\leq \beta$ and $\lambda(B_i)\leq \beta$ for $i=1,\ldots T$. 
Let $C_p^+=p^{1/p} \beta^{1-1/p}$.
Let $p\geq1$, with $p$ integer
Then,
  \begin{align*}
  &\norm{M_p(A_1,\ldots,A_T) - M_p(B_1,\ldots,B_T)} \\
  &\qquad\leq C_p^+ m_p\big(\norm{A_1-B_1},\ldots,\norm{A_T-B_T} \big)^{1/p}
  \end{align*}
\end{lemma}
\begin{proof}
 We can see that
 \begin{align*}
  &\norm{M_p(A_1,\ldots,A_T) - M_p(B_1,\ldots,B_T)} \\
  &\qquad\leq \norm{M_p^p(A_1,\ldots,A_T) - M_p^p(B_1,\ldots,B_T)}^\frac{1}{p}\\
  &\qquad\leq \Bigg( p \beta^{p-1} m_p\big(\norm{A_1-B_1},\ldots,\norm{A_T-B_T} \big)\Bigg)^\frac{1}{p}\\
  &\qquad= C_p^+ m_p\big(\norm{A_1-B_1},\ldots,\norm{A_T-B_T} \big)^{1/p}
 \end{align*}
where the first inequality comes from Proposition~\ref{proposition:takeOutFractionalPower}, and the second inequality comes from Proposition~\ref{proposition:takeOutIntegerPower}.
\end{proof}

\subsection{Results for the case $p\leq-1$}
The following two propositions are the main ingredients for the upper bound presented in Theorem~\ref{theorem:UpperBoundOfMatrixPowerMean} for the case $p\leq -1$.
\begin{proposition}\label{proposition:takeOutFractionalPower-negativePower}
Let $A_1,\ldots,A_T,B_1,\ldots,B_T$ be symmetric positive definite matrices
where
$\alpha  \leq \lambda(A_i)$ and $\alpha \leq \lambda(B_i)$ for $i=1,\ldots T$,
and $\alpha>0$.
Then, for $p\leq-1$, with $p$ integer
 \begin{align*}
  &\norm{M_p(A_1,\ldots,A_T) - M_p(B_1,\ldots,B_T)} \\ 
  &\leq \frac{1}{\alpha^2}\norm{M_{\abs{p}}^{\abs{p}}(A_1^{-1},\ldots,A_T^{-1}) - M_{\abs{p}}^{\abs{p}}(B_1^{-1},\ldots,B_T^{-1})}^{1/\abs{p}} \\
 \end{align*}
 \begin{proof}

Let $S_p(A_1,\ldots,A_T) = {\frac  {1}{T}}\sum _{{i=1}}^{T}A_{i}^{p}$
  . Then,
  \begin{align*}
  &\norm{M_p(A_1,\ldots,A_T) - M_p(B_1,\ldots,B_T)} \\
  &= \norm{S_p^{1/p}(A_1,\ldots,A_T) - S_p^{1/p}(B_1,\ldots,B_T)} \\
  &= \norm{S_p^{1/\abs{p}}(A_1,\ldots,A_T)^{-1} - S_p^{1/\abs{p}}(B_1,\ldots,B_T)^{-1}} \\
  &= \norm{S_{\abs{p}}^{1/\abs{p}}(A_1^{-1},\ldots,A_T^{-1})^{-1} - S_{\abs{p}}^{1/\abs{p}}(B_1^{-1},\ldots,B_T^{-1})^{-1}} \\
  &= \norm{M_{\abs{p}}(A_1^{-1},\ldots,A_T^{-1})^{-1} - M_{\abs{p}}(B_1^{-1},\ldots,B_T^{-1})^{-1}}\\
   &\leq  \frac{1}{\alpha^2}\norm{M_{\abs{p}}(A_1^{-1},\ldots,A_T^{-1}) - M_{\abs{p}}(B_1^{-1},\ldots,B_T^{-1})}\\
   &\leq \frac{1}{\alpha^2}\norm{M_{\abs{p}}^{\abs{p}}(A_1^{-1},\ldots,A_T^{-1}) - M_{\abs{p}}^{\abs{p}}(B_1^{-1},\ldots,B_T^{-1})}^{1/\abs{p}} \\
   \end{align*}
   where
   the first inequality follows from Corollary~\ref{corollary:theorem:fA-FB-inverse} and Proposition~\ref{proposition:upper-lower-bound-of-eigenvalues-of-power-mean},
   whereas the second inequality follows from Proposition~\ref{proposition:takeOutFractionalPower}.
 \end{proof}
\end{proposition}
\begin{proposition}\label{proposition:takeOutIntegerPower-negativePower}
Let $A_1,\ldots,A_T,B_1,\ldots,B_T$ be symmetric positive definite matrices such that 
${\alpha\!\leq\!\lambda(A_i)}$ and ${\alpha\!\leq\!\lambda(B_i)}$ for $i=1,\ldots T$. Then, for $p\leq~-1$, with $p$ integer
  \begin{align*}
  &\norm{M_{\abs{p}}^{\abs{p}}(A_1^{-1},\ldots,A_T^{-1}) - M_{\abs{p}}^{\abs{p}}(B_1^{-1},\ldots,B_T^{-1})} \\
  &\qquad\leq \abs{p} \alpha^{-(1+\abs{p})} m_{\abs{p}}\big(\norm{A_1-B_1},\ldots,\norm{A_T-B_T} \big)
  \end{align*}
\end{proposition}
\begin{proof}
Let $\alpha_i = \min(\norm{A_i}, \norm{B_i})$, 
then it clearly follows that
$\frac{1}{\alpha_i} = \max(\norm{A_i^{-1}}, \norm{B_i^{-1}})$.
Thus,
\begin{align*}
 &\norm{M_{\abs{p}}^{\abs{p}}(A_1^{-1},\ldots,A_T^{-1}) - M_{\abs{p}}^{\abs{p}}(B_1^{-1},\ldots,B_T^{-1})}\\
 &\qquad= \norm{\left({\frac  {1}{T}}\sum _{{i=1}}^{T}(A_{i}^{-1})^{\abs{p}}\right) - \left({\frac  {1}{T}}\sum _{{i=1}}^{T}(B_{i}^{-1})^{\abs{p}}\right)}\\
 &\qquad= \norm{ {\frac  {1}{T}}\sum _{{i=1}}^{T}(A_{i}^{-1})^{\abs{p}} - (B_{i}^{-1})^{\abs{p}} }\\
 &\qquad\leq \frac{1}{T}\sum_{{i=1}}^{T}\norm{(A_{i}^{-1})^{\abs{p}} - (B_{i}^{-1})^{\abs{p}} }\\
 &\qquad\leq \frac{1}{T}\sum_{{i=1}}^{T} \abs{p}\Big(\frac{1}{\alpha_i}\Big)^{\abs{p}-1} \norm{A_{i}^{-1} - B_{i}^{-1} }\\
 &\qquad\leq \abs{p}\Big(\frac{1}{\alpha}\Big)^{\abs{p}-1} \Bigg( \frac{1}{T} \sum_{i=1}^T \norm{A_i^{-1}-B_i^{-1}} \Bigg)\\
 &\qquad\leq \abs{p}\Big(\frac{1}{\alpha}\Big)^{\abs{p}-1} \Bigg( \frac{1}{T\alpha^2} \sum_{i=1}^T \norm{A_i-B_i} \Bigg)\\
 &\qquad= \abs{p}\Big(\frac{1}{\alpha}\Big)^{\abs{p}+1} \Bigg( \frac{1}{T} \sum_{i=1}^T \norm{A_i-B_i} \Bigg)\\
 &\qquad= \abs{p} \Big(\frac{1}{\alpha}\Big)^{\abs{p}+1} m_1\big(\norm{A_1-B_1},\ldots,\norm{A_T-B_T} \big)\\
 &\qquad\leq \abs{p} \Big(\frac{1}{\alpha}\Big)^{\abs{p}+1} m_{\abs{p}}\big(\norm{A_1-B_1},\ldots,\norm{A_T-B_T} \big)\\
 &\qquad= \abs{p} \alpha^{-(1+\abs{p})} m_{\abs{p}}\big(\norm{A_1-B_1},\ldots,\norm{A_T-B_T} \big)
 \end{align*}
 where: 
the first inequality follows from the triangular inequality,
the second inequality follows from Proposition~\ref{proposition:XmMinusYm},
the third inequality follows as $\alpha \leq \alpha_i$, and
the fourth inequality follows as Corollary~\ref{corollary:theorem:fA-FB-inverse}, and
the last inequality comes from the monotonicity of the scalar power means.
\end{proof}
The next Lemma contains the proof corresponding to the case of negative powers of Theorem~\ref{theorem:UpperBoundOfMatrixPowerMean}.
\begin{lemma}[Theorem~\ref{theorem:UpperBoundOfMatrixPowerMean} for the case $p\leq -1$]\label{proposition:MpUpperBoundForNegativePowers-2}
Let $A_1,\ldots,A_T,B_1,\ldots,B_T$ be symmetric positive definite matrices where 
$\alpha \leq \lambda(A_i)$
and 
$\alpha \leq \lambda(B_i)$
for $i=1,\ldots T$. 
Let $C_p^-=\abs{p}^{1/\abs{p}} \alpha^{-(3+1/\abs{p})}$.
Let $p\leq-1$ with $p$ integer.
Then,
  \begin{align*}
  &\norm{M_p(A_1,\ldots,A_T) - M_p(B_1,\ldots,B_T)} \\
  &\qquad\leq C_p^- m_{\abs{p}}\big(\norm{A_1-B_1},\ldots,\norm{A_T-B_T} \big)^{1/\abs{p}}
  \end{align*}
\end{lemma}
\begin{proof}
%
%
 \begin{align*}
  &\norm{M_p(A_1,\ldots,A_T) - M_p(B_1,\ldots,B_T)} \\
  &\leq \frac{1}{\alpha^2}\norm{M_{\abs{p}}^{\abs{p}}(A_1^{-1},\ldots,A_T^{-1}) - M_{\abs{p}}^{\abs{p}}(B_1^{-1},\ldots,B_T^{-1})}^{1/\abs{p}} \\
  &\leq \frac{1}{\alpha^2}\Bigg( \abs{p} \alpha^{-(1+\abs{p})} m_{\abs{p}}\big(\norm{A_1-B_1},\ldots,\norm{A_T-B_T} \big)\Bigg)^{1/\abs{p}}\\
  &= C_p^- m_{\abs{p}}\big(\norm{A_1-B_1},\ldots,\norm{A_T-B_T} \big)^{1/\abs{p}}
 \end{align*}
where the first inequality comes from Proposition~\ref{proposition:takeOutFractionalPower-negativePower}, and the second inequality comes from Proposition~\ref{proposition:takeOutIntegerPower-negativePower}.
\end{proof}

We are now ready to prove the result of Theorem~\ref{theorem:UpperBoundOfMatrixPowerMean}.
\begin{proof}[Proof of Theorem~\ref{theorem:UpperBoundOfMatrixPowerMean}]
 For the case $p\geq1$ see Lemma~\ref{proposition:MpUpperBoundForPositivePowers}.
 For the case $p\leq -1$ see Lemma~\ref{proposition:MpUpperBoundForNegativePowers-2}.

\end{proof}
\section{Proof of Theorem~\ref{theorem:concentration-eigenvectors}}\label{theorem:concentration-eigenvectors-PROOF}
Before giving the proof of Theorem~\ref{theorem:concentration-eigenvectors} we need to present two auxiliary results. 

The following is an auxiliary technical result that extends an implicit result stated in~\cite{rohe2011spectral}(p.1908-1909) for the Frobenius norm to the case of the operator norm.
\begin{lemma}\label{lemma:rephraseKahanDavisOurs}
Let $X,\mathcal{X} \in \mathbb{R}^{n \times k}$ be matrices with orthonormal columns.
Let $U,V$ be orthonormal matrices and $\Sigma$ a diagonal matrix such that
\begin{equation*}
 \mathcal{X}^T X = U \Sigma V^T
\end{equation*}
where the diagonal entries of $\Sigma$ are the cosines of the principal
angles between the column space of $X$ and the column space of $\mathcal{X}$.
Let $O=UV^T$. Then,
\begin{equation*}
 \frac{1}{\sqrt{2}} \| X - \mathcal{X}O \| \leq \| \sin \Theta(\mathcal{X},X) \|
\end{equation*}
\end{lemma}
\begin{proof}
 For the proof we will make use of the identity $X^T\mathcal{X}O = (V\Sigma U^T)UV^T = V \Sigma V^T$, 
 and the fact that $\|A\| = \sqrt{\lambda_{\max}(A^TA)}$. That is,
\begin{align*}
 (X - &\mathcal{X}O)^T (X - \mathcal{X}O) \\ &= (X^T - O^T \mathcal{X}^T) (X - \mathcal{X}O) \\
                                         &= X^T X - X^T\mathcal{X}O - O^T \mathcal{X}^T X + O^T \mathcal{X}^T \mathcal{X}O \\
                                         &= I - X^T\mathcal{X}O - O^T \mathcal{X}^T X + O^T O \\
                                         &= I - X^T\mathcal{X}O - O^T \mathcal{X}^T X + I \\
					 &= 2I - X^T\mathcal{X}O - O^T \mathcal{X}^T X \\
					 &= 2I - V \Sigma V^T - V \Sigma V^T \\
					 &= 2(I - V \Sigma V^T) \\
\end{align*}
Thus,
\begin{align*}
 \| X - \mathcal{X}O \|^2 &= \lambda_{\max}\Big((X - \mathcal{X}O)^T (X - \mathcal{X}O)\Big)\\
								 &=    2\lambda_{\max}(I - V \Sigma V^T)\\
                                                                 &=    2\max_{i}(1 - \cos{\Theta_i})\\
                                                                 &\leq 2\max_{i}(1 - \cos^2{\Theta_i})\\
                                                                 &=    2\max_{i}(\sin^2{\Theta_i})\\
                                                                 &=    2\|\sin{\Theta}\|^2 \\
\end{align*}
Hence, $\frac{1}{\sqrt{2}}\| X - \mathcal{X}O \| \leq \| \sin \Theta(\mathcal{X},X) \|$
\end{proof}

The next result is a useful representation of the Davis-Kahan theorem. It is a technical adaption from the Frobenius norm to the operator norm based on~Lemma~\ref{lemma:rephraseKahanDavisOurs} and Theorem~\ref{thm:davisKahan_Yu}.
 \begin{theorem}\label{theorem:davisKahan_ourVersion-operatorNorm}
Let $\Sigma,\hat{\Sigma} \in \mathbb{R}^{p \times p}$ be symmetric, with eigenvalues $\mu_1 \geq \ldots \geq \mu_p$ and $\hat{\mu}_1 \geq \ldots \geq \hat{\mu}_p$ respectively.  Fix $1 \leq r \leq s \leq p$ and assume that $\min(\mu_{r-1} - \mu_r,\mu_s - \mu_{s+1}) > 0$, where $\mu_0 := \infty$ and $\mu_{p+1} := -\infty$.  Let $d := s - r + 1$, and let $V = (v_r,v_{r+1},\ldots,v_s) \in \mathbb{R}^{p \times d}$ and $\hat{V} = (\hat{v}_r,\hat{v}_{r+1},\ldots,\hat{v}_s) \in \mathbb{R}^{p \times d}$ have orthonormal columns satisfying $\Sigma v_j = \mu_j v_j$ and $\hat{\Sigma}\hat{v}_j = \hat{\mu}_j \hat{v}_j$ for $j= r,r+1,\ldots,s$.  Then
there exists an orthogonal matrix $O \in \mathbb{R}^{d \times d}$ such that 
\begin{equation*}
   \frac{1}{\sqrt{2}} \| V - \hat{V}O \| \leq \frac{2d^{1/2} \vectornorm{ \hat{\Sigma} - \Sigma}}{\min(\mu_{r-1} - \mu_r,\mu_s - \mu_{s+1})}
\end{equation*} 
\end{theorem}
\begin{proof}

  By theorem~\ref{thm:davisKahan_Yu} we have
\begin{equation*}
 \|\sin \Theta(\hat{V},V)\|_{\mathrm{F}} \leq \frac{2\min(d^{1/2}\|\hat{\Sigma} - \Sigma\|,\|\hat{\Sigma} - \Sigma\|_\mathrm{F})}{\min(\mu_{r-1} - \mu_r,\mu_s - \mu_{s+1})}. 
\end{equation*}

From lemma~\ref{lemma:rephraseKahanDavisOurs} we can see that 
\begin{equation*}
\frac{1}{\sqrt{2}}\| V - \hat{V}O \| \leq \| \sin \Theta(\hat{V},V) \|
\end{equation*}
Moreover, as $\sin \Theta(\hat{V},V)$ is a diagonal matrix, it holds that
  \begin{align*}
  \vectornorm{\sin \Theta(\hat{V},V)}^2 &= \max_i( \sin^2 \Theta_i) \\
  &\leq \sum_i^p \sin^2 \Theta_i \\
  &= \vectornorm{\sin \Theta(\hat{V},V)}^2_F
 \end{align*}
Thus, 
\begin{equation*}
 \frac{1}{\sqrt{2}}\| V - \hat{V}O \| \leq \| \sin \Theta(\hat{V},V) \| \leq \vectornorm{\sin \Theta(\hat{V},V)}_F
\end{equation*}

Further, it is straightforward to see that
\begin{equation*}
\min(d^{1/2}\|\hat{\Sigma} - \Sigma\|,\|\hat{\Sigma} - \Sigma\|_\mathrm{F}) \leq  d^{1/2}\|\hat{\Sigma} - \Sigma\|
\end{equation*}
Thus, all in all, we have
\begin{align*}
   \frac{1}{\sqrt{2}}\| V - \hat{V}O \| 
   &\leq \| \sin \Theta(\hat{V},V) \| \\
   &\leq \vectornorm{\sin \Theta(\hat{V},V)}_F \\
   &\leq \frac{2d^{1/2} \vectornorm{ \hat{\Sigma} - \Sigma}}{\min(\mu_{r-1} - \mu_r,\mu_s - \mu_{s+1})}
\end{align*}
which completes the proof.
\end{proof}

We are now ready to give the proof of Theorem~\ref{theorem:concentration-eigenvectors}.

\begin{proof}[Proof of Theorem~\ref{theorem:concentration-eigenvectors}]

The proof is an application of the Davis-Kahan theorem as presented in Theorem~\ref{theorem:davisKahan_ourVersion-operatorNorm}.
Observe that in Theorem~\ref{theorem:davisKahan_ourVersion-operatorNorm} the eigenvalues are sorted in a decreasing way i.e. $\mu_1 \geq \cdots \geq \mu_n$, whereas in our case they are sorted in an increasing manner i.e. $\lambda_1 \leq \cdots \leq \lambda_n$.

Notationally, let the variables $p,s,r$ from Theorem~\ref{theorem:davisKahan_ourVersion-operatorNorm} be defined as $p=s=n,r=p-k+1$.

We first focus in the case for $p\leq -1$. 
For this case we are interested in the $k$-smallest eigenvalues, i.e. $\lambda_1,\ldots,\lambda_k$, which correspond to $\mu_p,\ldots,\mu_{r}$, where $\mu_p=\lambda_1$ and $\mu_{r}=\lambda_k$.

By definition, in Theorem~\ref{theorem:davisKahan_ourVersion-operatorNorm} we have that $\mu_{p+1}=-\infty$. Thus, 
$\mu_p-\mu_{p+1}=\infty$.
Further, we can see
%
$\mu_{r-1}-\mu_{r}=\lambda_{k+1}-\lambda_{k}=(1+\varepsilon)-m_p(1-\rho^++\varepsilon,1+\rho^-+\varepsilon)$ 
and hence by Eq.\ref{eq:eigenvaluesOfPowerMeanInExpectation}
\begin{align*}
&\min(\mu_{r-1} - \mu_r,\mu_s - \mu_{s+1}) = 
\\
&\qquad\qquad(1+\varepsilon)-m_p(1-\rho^++\varepsilon,1+\rho^-+\varepsilon)
\end{align*}
which by Theorem~\ref{theorem:davisKahan_ourVersion-operatorNorm} leads to the following inequality

$\,$\hfill
$\vectornorm{V_k-\mathcal{V}_k O_k} 
\leq 
\frac{2^{3/2}k^{1/2}}{\gamma} \vectornorm{L_p-\mathcal{L}_p}
=
\frac{\sqrt{8k}}{\gamma} \vectornorm{L_p-\mathcal{L}_p}
$
\hfill $\,$

By applying Theorem~\ref{theorem:ConcentrationBound-signedGraph-identity-SBM}, we know that if
\begin{align*}
 \delta^+ &= \frac{n}{k}(\pp+(k-1)\qp)>3\ln(8n/\epsilon), \text{ and }\\
 \delta^- &= \frac{n}{k}(\ppm+(k-1)\qm)>3\ln(8n/\epsilon)
\end{align*}
then with probability at least $1-\epsilon$

 \resizebox{1.04\hsize}{!}{
 $
 \vectornorm{V_k-\mathcal{V}_k O_k} \leq 
 \frac{\sqrt{8k}}{\gamma} 
 C_{p}^-
 m_{\abs{p}}^{1/\abs{p}}
 \bigg(\!\!
  \sqrt{\frac{3\ln(8n / \epsilon)}{\delta^+}}, \sqrt{\frac{3\ln(8n / \epsilon)}{\delta^-}}
 \bigg)
 $
 }
 
yielding the desired result.
The case for $p\geq 1$ is similar, where instead of $k$ the value ${k'=k-1}$ is used.

\end{proof}

\section{Main building block for our results}
In this section present two results from~\cite{chung2011spectra} that are the main building blocks for our results.

\begin{theorem}[\cite{chung2011spectra}]\label{thm:chung}
  Let $G$ be a random graph, where $\textsl{pr}(v_i \sim v_j) = p_{ij}$,
  and each edges is independent of each other edge. 
  Let $A$ be the adjacency matrix of $G$, 
  so $A_{ij}=1$ if $v_i \sim v_j$ and $0$ otherwise,
  and $\bar{A} = E(A)$, so $\bar{A}_{ij}=p_{i,j}$.
  Let $D$ be the diagonal matrix with $D_{ii}=\textsl{deg}(v_i)$, and $\bar{D}=E(D)$.
  Let $\delta$ be the minimum expected degree of $G$, and $L=I - D^{-1/2}AD^{-1/2}$ the (normalized) Laplacian matrix for $G$.
  Choose $\epsilon>0$. Then there exists a constant $k = k(\epsilon)$ such that if $\delta > k \ln n$, 
  then the probability at least $1-\epsilon$, the eigenvalues of $L$ and $\bar{L}$ satisfy
  \begin{equation*}
   \abs{ \lambda_j (L) - \lambda_j (\bar{L}) } \leq 2 \sqrt{\frac{3\ln(4n / \epsilon)}{\delta}}
  \end{equation*}
  for all $1\leq j \leq n$, where $\bar{L} = I - \bar{D}^{-1/2}\bar{A}\bar{D}^{-1/2}$.
\end{theorem}

Although this theorem is presented as the main result,
one can see in the proof of theorem~\ref{thm:chung} in~\cite{chung2011spectra},
that in deed what they proved was a concentration bound for $\norm{L - \bar{L}}$.

\begin{theorem}[\cite{chung2011spectra}]\label{thm:chung_normsVersion}
 Assume that conditions of Theorem~\ref{thm:chung} hold. 
 Choose $\epsilon>0$. Then there exists a constant $k = k(\epsilon)$ such that if $\delta > k \ln n$, 
 then 
 \begin{equation}
  \mathbb{P}\bigg( \norm{L - \bar{L}} \leq 2 \sqrt{\frac{3\ln(4n / \epsilon)}{\delta}} \bigg) > 1-\epsilon
 \end{equation}
 \end{theorem}
 
\begin{theorem}[\!\!\cite{Yu:2015:DavisKahan}]\label{thm:davisKahan_Yu}
Let $\Sigma,\hat{\Sigma} \in \mathbb{R}^{p \times p}$ be symmetric, with eigenvalues ${\mu_1 \geq \ldots \geq \mu_p}$ and $\hat{\mu}_1 \geq \ldots \geq \hat{\mu}_p$ respectively.  
Fix $1 \leq r \leq s \leq p$ and assume that 
$\min(\mu_{r-1} - \mu_r,\mu_s - \mu_{s+1}) > 0$, where $\mu_0 := \infty$ and $\mu_{p+1} := -\infty$.  Let $d := s - r + 1$, and let $V = (v_r,v_{r+1},\ldots,v_s) \in \mathbb{R}^{p \times d}$ and $\hat{V} = (\hat{v}_r,\hat{v}_{r+1},\ldots,\hat{v}_s) \in \mathbb{R}^{p \times d}$ have orthonormal columns satisfying $\Sigma v_j = \mu_j v_j$ and $\hat{\Sigma}\hat{v}_j = \hat{\mu}_j \hat{v}_j$ for $j= r,r+1,\ldots,s$.  Then
\begin{equation*}
\label{Eq:OurSinTheta}
\|\sin \Theta(\hat{V},V)\|_{\mathrm{F}} \leq \frac{2\min(d^{1/2}\|\hat{\Sigma} - \Sigma\|,\|\hat{\Sigma} - \Sigma\|_\mathrm{F})}{\min(\mu_{r-1} - \mu_r,\mu_s - \mu_{s+1})}. 
\end{equation*}
Moreover, there exists an orthogonal matrix $\hat{O} \in \mathbb{R}^{d \times d}$ such that 
\begin{equation*}
\label{Eq:OurDifference}
\|\hat{V}\hat{O} - V\|_{\mathrm{F}} \leq \frac{2^{3/2}\min(d^{1/2}\|\hat{\Sigma} - \Sigma\|,\|\hat{\Sigma} - \Sigma\|_\mathrm{F})}{\min(\mu_{r-1} - \mu_r,\mu_s - \mu_{s+1})}.
\end{equation*}
\end{theorem}

\section{Results on Bethe Hessian}\label{section:resultsOnBetheHessian-Proofs}

%

The following Lemma~\ref{lemma:BetheHessianAndKunegis} states that for the case where $\alpha=1$ the Bethe Hessian is equal to the arithmetic mean of Laplacians, i.e. the signed ratio Laplacian $L_{SR}$.

\begin{lemma}\label{lemma:BetheHessianAndKunegis}
Let $\alpha=1$. Then the Bethe Hessian is two times the arithmetic mean of $L^+$ and $Q^-$.
\end{lemma}
\begin{proof}[Proof of Lemma~\ref{lemma:BetheHessianAndKunegis}]
 Let $J^+,J^-$ be the positive and negative part of $J$, i.e.
 $J^+_{ij}=\max\{0,J_{ij}\}$ and $J^-_{ij}=-\min\{0,J_{ij}\}$.
 Let $D^+$  and $D^-$ be degree diagonal matrices of $J^+$ and $J-$ respectively, i.e. $\overline{D}=D^+ + D^-$.
 Then, 
 \begin{align*}
  H 
  &= (\alpha-1)I - \sqrt{\alpha}J + \overline{D} \\
  &= -J + \overline{D} \\
  &= -J^+  + J^- + D^+ + D^- \\
  &= (D^+ -J^+) + (D^- + J^-) \\
  &= L^+ + Q^- \\
  &= L_{SR}
 \end{align*}

\end{proof}

\begin{lemma}\label{lemma:bethe_hessian_V1:supplementaryMaterial}
 Let $\mathcal{H}$ be the Bethe hessian of the expected signed graph. Then 
 $\{ \boldsymbol \chi_i \}_{i=2}^{k}$ are the eigenvectors corresponding to the $(k-1)$-smallest negative eigenvalues of $\mathcal{H}$ if and only if the following conditions hold:
 \begin{enumerate}
  \item  $\max\{ 0,\frac{2(d^+ + d^-)-1}{\sqrt{d^+ + d^-}\abs{\mathcal{C}}}\}  < (\pp-\qp) - (\ppm-\qm)$
  \item  $\qp<\qm$ 
 \end{enumerate}
\end{lemma}

\begin{proof}
In our framework the we can see that $J=W^+-W^-$. In Section~\ref{appendix:theorem:mp_in_expectation-PROOF} we can see that expected adjacency matrices $\mathcal{W}^+$ and $\mathcal{W}^-$ have three distinct eigenvalues:
\begin{align*}
 \lambda^+_1 = \abs{\mathcal{C}}(\pp+(k-1)\qp), \quad \lambda^+_i = \abs{\mathcal{C}}(\pp-\qp)\\
 \lambda^-_1 = \abs{\mathcal{C}}(\ppm+(k-1)\qm), \quad \lambda^-_i = \abs{\mathcal{C}}(\ppm-\qm)
\end{align*}
for $i=2,\ldots,k$, with corresponding eigenvectors $\boldsymbol{\chi}_1,\ldots,\boldsymbol{\chi}_k$. Remaining eigenvalues are equal to zero.
Further, as both matrices $\mathcal{W}^+$ and $\mathcal{W}^-$ share all their eigenvectors, then the expected matrix $\mathcal{J}=\mathcal{W}^+ - \mathcal{W}^-$ has the same eigenvectors with eigenvalues being the difference between the positive and negative counterparts, i.e. $\mathcal{J}\boldsymbol{\chi}_i=\mu_i \boldsymbol{\chi}_i$ where
\begin{align*}
 \mu_i = \lambda^+_i - \lambda^-_i\,.
\end{align*}
As we assume that all clusters are of the same size $\abs{C}$, the expected signed graph is a regular graph with degrees $d^+$ and $d^-$. Thus, in expectation $\widehat{\alpha} = d^+ + d^-$, where
$d^+=\abs{\mathcal{C}}(\pp+(k-1)\qp)$ and $d^-=\abs{\mathcal{C}}(\ppm+(k-1)\qm)$.
Hence, the Bethe hessian of the expected signed graph can be expressed as
\begin{align*}
 \mathcal{H} 
 &= (\widehat{\alpha}-1)I - \sqrt{\widehat{\alpha}}\mathcal{J} + \overline{\mathcal D}\\
 &= (\widehat{\alpha}-1)I - \sqrt{\widehat{\alpha}}\mathcal{J} + \widehat{\alpha}I\\
 &= (2\widehat{\alpha}-1)I - \sqrt{\widehat{\alpha}}\mathcal{J}
\end{align*}
It is easy to see that the matrix $\mathcal{H}$ is some sort of a diagonal shift of $\mathcal{J}$, and thus they have the same eigenvectors. 
In particular we can observe that:
\begin{align*}
 \mathcal{H}\boldsymbol\chi_i 
 &= \big((2\widehat{\alpha}-1) I - \sqrt{\widehat{\alpha}}\mathcal{J}\big) \boldsymbol\chi_i\\
 &= (2\widehat{\alpha}-1)\boldsymbol\chi_i - \sqrt{\widehat{\alpha}}\mathcal{J} \boldsymbol\chi_i \\
 &= (2\widehat{\alpha}-1)\boldsymbol\chi_i - \sqrt{\widehat{\alpha}}\mu_i \boldsymbol\chi_i \\
 &= \big( (2\widehat{\alpha}-1)\boldsymbol - \sqrt{\widehat{\alpha}}\mu_i \big)\boldsymbol\chi_i 
\end{align*}
Hence, the corresponding eigenvalues of $\mathcal{H}$ are:
\begin{equation}\label{eq:general_eigval_bethe_hessian}
 \lambda_i = (2\widehat{\alpha}-1)\boldsymbol - \sqrt{\widehat{\alpha}}\mu_i\,.
\end{equation}

All in all, the corresponding eigenvalues of the expected Bethe hessian matrix $\mathcal{H}$ are:
\begin{align*}
\lambda_1 &=  (2\widehat{\alpha}-1) - \sqrt{\widehat{\alpha}}(d^+ - d^-) \,,
\\
\lambda_i &=  (2\widehat{\alpha}-1) - \sqrt{\widehat{\alpha}}\abs{\mathcal{C}}\big((\pp-\qp) - (\ppm-\qm)\big) \,,
\\
\lambda_j &=  (2\widehat{\alpha}-1) \,.
\end{align*}
for $i=2,\ldots,k$ and $j=k+1,\ldots,n$.

We now focus on the conditions that are necessary so that eigenvectors $\boldsymbol\chi_2,\ldots,\boldsymbol\chi_k$ have the smallest negative eigenvalues. This is based on the fact that informative eigenvectors of the Bethe Hessian $H$ have the smallest negative eigenvalue. From  Eq.\ref{eq:general_eigval_bethe_hessian} we can see that the general condition for eigenvalues of the Bethe Hessian in expectation $\mathcal{H}$ to be negative is
\begin{equation}
 \lambda_i < 0 \iff \frac{2\widehat{\alpha}-1}{\sqrt{\widehat{\alpha}}} < \mu_i\,.
\end{equation}
Hence the conditions to be analyzed are:
\begin{align*}
 &\lambda_i <\lambda_1, \,\,\,\,\text{ for }i=2,\ldots,k \\
 &\lambda_i <0, \,\,\,\,\text{ for }i=2,\ldots,k \\
 &\lambda_i <\lambda_j, \,\,\,\,\text{ for }i=2,\ldots,k \,\,\text{ and }j=k+1,\ldots,n \\
\end{align*}
Therefore we can easily see that the corresponding condition $\lambda_i <\lambda_1$ boils down to
\[
\qp<\qm
\]
whereas condition $\lambda_i <0$ is equivalent to 
\[
 \frac{2(d^+ + d^-)-1}{\sqrt{d^+ + d^-}\abs{\mathcal{C}}} < \big((\pp-\qp) - (\ppm-\qm)\big)
\]
and for the remaining condition $\lambda_i <\lambda_j$ the equivalent condition is
\[
 0 < (\pp-\qp) - (\ppm-\qm) 
\]
by putting together conditions for $\lambda_i <0$ and $\lambda_i <\lambda_j$ we get the desired result.
\end{proof}

\begin{lemma}\label{lemma:bethe_hessian_V1_limit}
 Let $\mathcal{H}$ be the Bethe hessian of the expected non-empty signed graph. 
 Let $\abs{V}\rightarrow\infty$. Then 
 $\{ \boldsymbol \chi_i \}_{i=2}^{k}$ are the eigenvectors corresponding to the $(k-1)$ smallest negative eigenvalues of $\mathcal{H}$ if and only if the following conditions hold:
 \begin{enumerate}
  \item  $\ppm + \qp < \pp + \qm$
  \item  $\qp<\qm$
 \end{enumerate}
\end{lemma}


\begin{proof}
From Lemma~\ref{lemma:bethe_hessian_V1:supplementaryMaterial} we have the following conditions for the recovery of informative eigenvectors on finite graphs:
\begin{enumerate}
  \item  $\max\{ 0,\frac{2(d^+ + d^-)-1}{\sqrt{d^+ + d^-}\abs{\mathcal{C}}}\}  < (\pp-\qp) - (\ppm-\qm)$
  \item  $\qp<\qm$ 
 \end{enumerate}
 Let 
 \begin{align*}
  c_2 &= \pp+(k-1)\qp + \ppm+(k-1)\qm
  \\
  c_3 &= (\pp-\qp) - (\ppm-\qm)
 \end{align*}
 
For the first condition of Lemma~\ref{lemma:bethe_hessian_V1:supplementaryMaterial} can be expressed as follows:
 \begin{align*}
  \frac{2(d^+ + d^-)-1}{\sqrt{d^+ + d^-}\abs{\mathcal{C}}}  &< \big((\pp-\qp) - (\ppm-\qm)\big) \iff 
  \\
  \frac{2c_2^{1/2}}{\abs{\mathcal C}^{1/2}} - \frac{1}{\abs{\mathcal C}^{3/2}c_2^{1/2}} 
  &< c_3\,.
 \end{align*}
Hence, in the limit where $\abs{C}\rightarrow\infty$ the above condition turns into
\begin{equation}\label{eq:2:lemma:bethe_hessian_V1_limit}
 0 < c_3 \iff \ppm + \qp < \pp + \qm\,.
\end{equation}
yielding the desired conditions.

\end{proof}

The following Lemma states the interesting fact that the Bethe Hessian works better for large graphs

\begin{lemma}\label{lemma:bethe_hessian_good_for_large_graphs}
Let $\mathcal{H}_n$ be the Bethe Hessian of the expected signed graph under the SBM with $n$ nodes. Let $\boldsymbol \chi^n=\{ \boldsymbol \chi_i \}_{i=2}^{k}$ where $\boldsymbol \chi_2,\ldots,\boldsymbol \chi_k \in\mathbb{R}^n$. 
Let $\frac{3}{2}<d^++d^-$.
Let $n<m$. If $\boldsymbol \chi^n$ are eigenvectors corresponding to the $(k-1)$-smallest negative eigenvalues of $\mathcal{H}_n$, then $\boldsymbol \chi^m$ are eigenvectors corresponding to the $(k-1)$-smallest negative eigenvalues of $\mathcal{H}_m$.
 \end{lemma}

 \begin{proof}
 
  In this proof we show that if for a given signed graph with $n$ nodes the conditions of Lemma~\ref{lemma:bethe_hessian_V1:supplementaryMaterial}, then conditions of Lemma~\ref{lemma:bethe_hessian_V1:supplementaryMaterial} hold for expected signed graphs with a larger number of nodes.
 
  By Lemma~\ref{lemma:bethe_hessian_V1:supplementaryMaterial}, we know that for a given graph in expectation with $n$ nodes, eigenvectors $\boldsymbol \chi^n=\{ \boldsymbol \chi_i \}_{i=2}^{k}$ correspond to the $(k-1)$-smallest negative eigenvalues of $\mathcal H_n$ if and only the following conditions hold:
 \begin{enumerate}
  \item  $\max\{ 0,\frac{2(d^+ + d^-)-1}{\sqrt{d^+ + d^-}\abs{\mathcal{C}}}\}  < (\pp-\qp) - (\ppm-\qm)$
  \item  $\qp<\qm$
 \end{enumerate}
 Observe that the right hand side of the above conditions does not depend on the number of nodes in the graph.
We proceed by analyzing the left hand side of the first condition:
\begin{equation}\label{lemma:bethe-hessian-lhs}
  \frac{2(d^+ + d^-)-1}{\abs{\mathcal C}\sqrt{d^+ + d^-}}\,.
 \end{equation}
 Note that under the Stochastic Block Model in consideration, all $k$ clusters are of size $\abs{\mathcal C}=\frac{n}{k}$.
 We now identify conditions such that the Equation~\ref{lemma:bethe-hessian-lhs} decreases with larger values of $\abs{\mathcal C}$.
 
 Let $x,\alpha\in\mathbb{R}$. Define the scalar function $g:\mathbb{R}_{>0}\to\mathbb{R}$ as
 \[
  g(x) = \frac{2\alpha x-1}{\sqrt{\alpha x^3}}
 \]
 Observe that we recover Equation~\ref{lemma:bethe-hessian-lhs} by letting $x=\abs{\mathcal{C}}$ and $\alpha=\pp+(k-1)\qp+\ppm+(k-1)\qm$ where $\alpha x = d^++d^-$.
 
The corresponding derivative is
 \[
  g'(x) = \frac{3-2\alpha x}{2x\sqrt{\alpha x^3}}
 \]
 Then 
\begin{equation}\label{lemma:bethe-hessian-lhs-derivative}
  g'(x)<0 \iff \frac{3}{2} < \alpha x\,.
\end{equation}
Hence, if $\frac{3}{2} < \alpha x$ then $g(y) < g(x)$ if and only if $x<y$. We now apply this result to our setting.

Let $\abs{\mathcal{C}_n}:=\abs{\mathcal{C}}=\frac{n}{k}$ and $\abs{\mathcal{C}_m}=\frac{m}{k}$ denote the cluster size of the expected signed graphs with $n$ and $m$ nodes, respectively. Let $\alpha=\pp+(k-1)\qp+\ppm+(k-1)\qm$.
Let $\frac{3}{2}<d^+ + d^-$. Then
\begin{equation}
 g(\abs{\mathcal{C}_m}) < g(\abs{\mathcal{C}_n})=\frac{2(d^+ + d^-)-1}{\abs{\mathcal C}\sqrt{d^+ + d^-}}
\end{equation}
if and only if $n<m$.
Hence, if conditions 1 and 2 hold for the expected graph $G$ with $n$ nodes and its expected absolute degree is larger than $\frac{3}{2}$, i.e. $\frac{3}{2}<d^+ + d^- $, then conditions~1~and~2 hold for expected graphs with a larger number of nodes,
leading to the desired result.
 \end{proof}
 
%

 \begin{figure}[!h]
\centering
\vskip.6em

%
\begin{subfigure}[]{0.48\linewidth}
\includegraphics[width=1\linewidth, clip,trim=130 40 170 40]{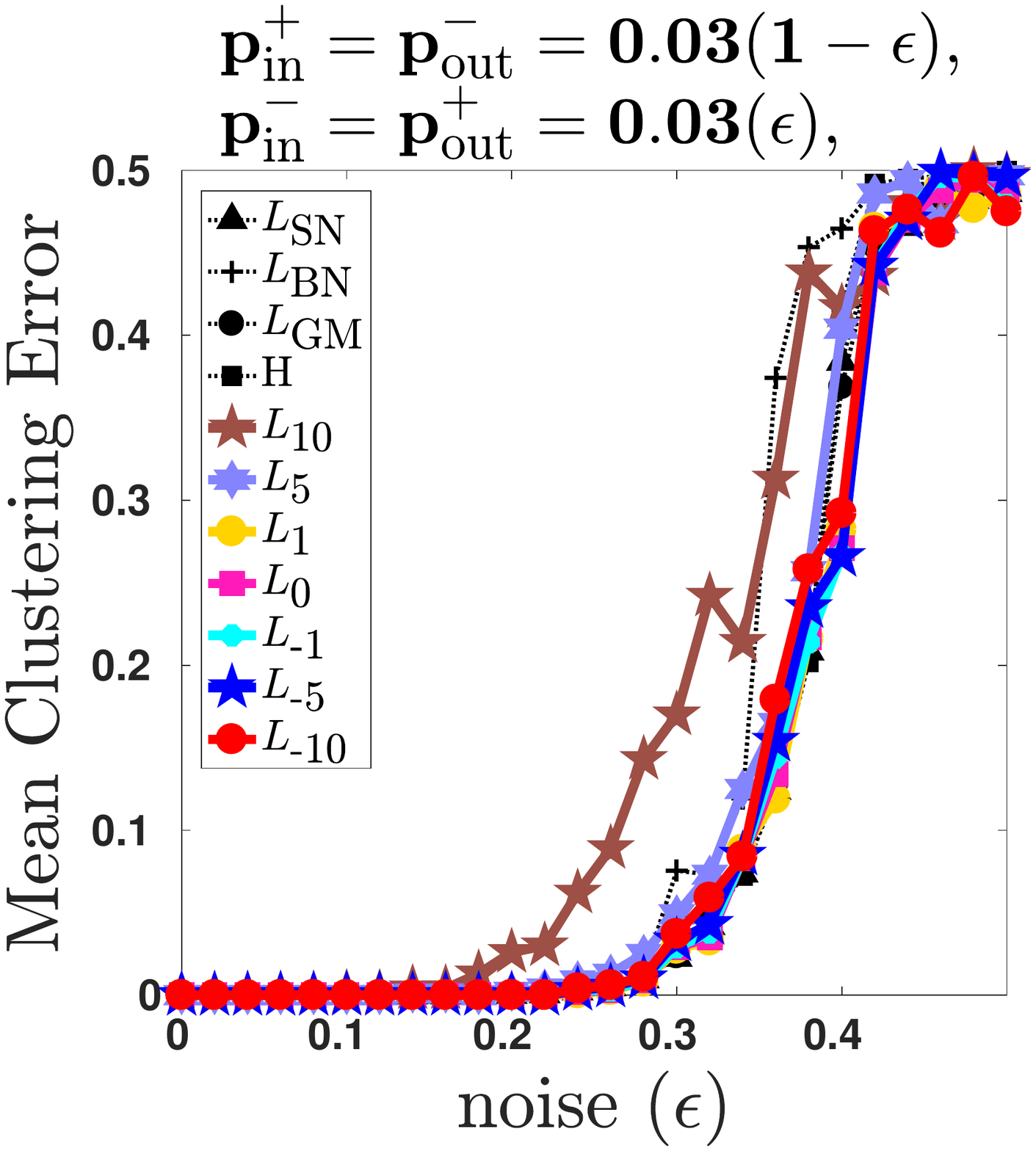}\hspace*{\fill}
\caption{$p=0.03,\epsilon\in[0,0.5]$}
\label{subfig:CBM:p_fixed}
\end{subfigure}%
\hfill
\begin{subfigure}[]{0.48\linewidth}
 \includegraphics[width=1\linewidth, clip,trim=130 40 170 40]{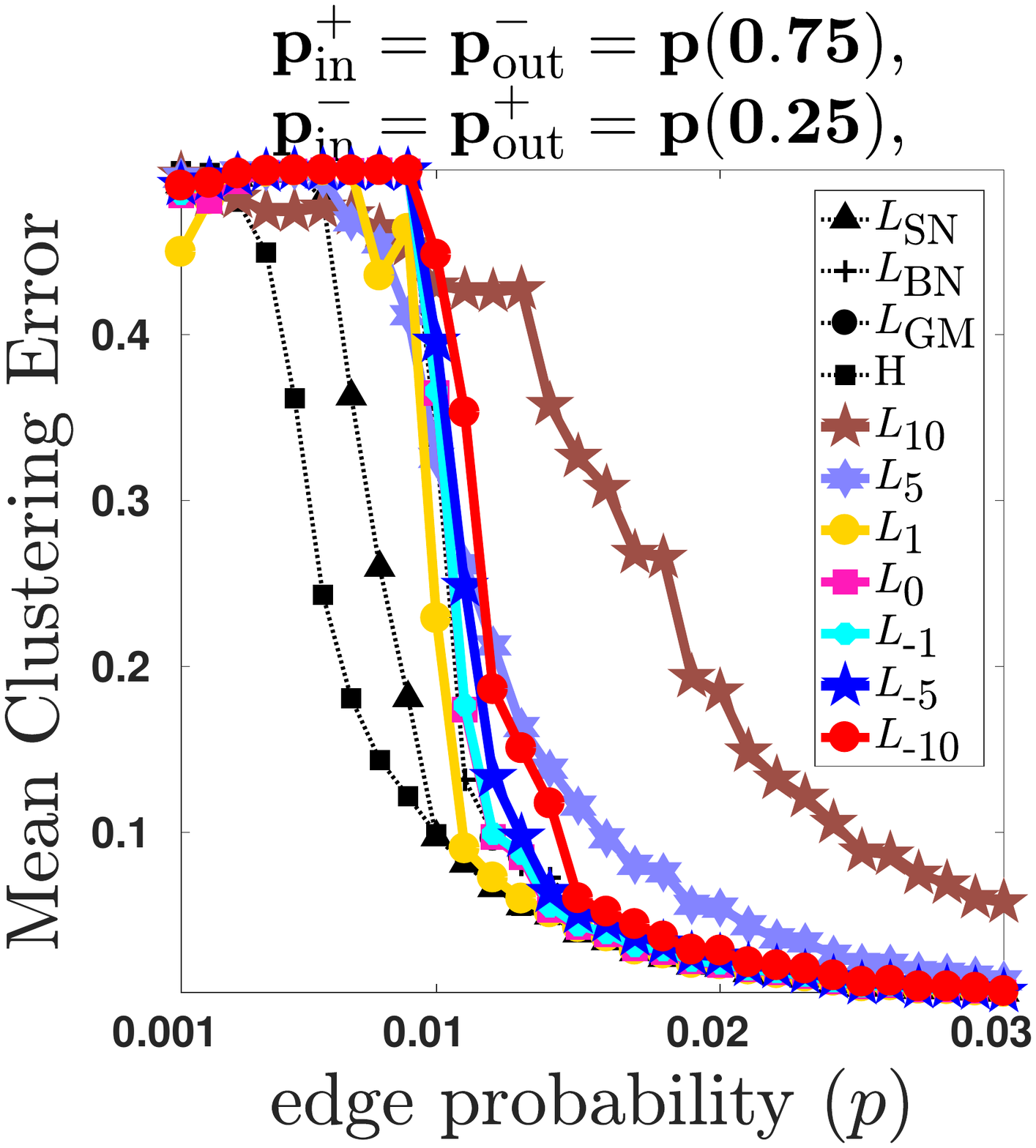}\hspace*{\fill}
\caption{$p\in[0.001,0.03],\epsilon=0.25$}
\label{subfig:CBM:epsilon_fixed}
\end{subfigure}%
%
\caption{
Mean clustering error under the Censored Block Model~\cite{saade:2015},
with two clusters of size 500 and 20 runs.
Fig.~\ref{subfig:CBM:p_fixed}: probability of observing and edge is fixed to $p=0.03$, and $\epsilon\in[0,0.5]$.
Fig.~\ref{subfig:CBM:epsilon_fixed}: probability of flipping sign of an edge is fixed to $\epsilon=0.25$, and $p\in[0.001,0.03]$
}
\label{fig:CBM}
\end{figure}
\begin{table*}[!h]    
\centering    
\small  
\adjustbox{max width=\textwidth}{
\begin{tabular}{cccccccccccccccccc}                                                                                                                                        
 \specialrule{1.5pt}{.1pt}{2pt}                                                                                                                                                                
 & & iris & wine & ecoli & australian & cancer & vehicle & german & image & optdig & isolet & USPS & pendig & 20new & MNIST \\                                                                                        
 \specialrule{1pt}{.1pt}{0pt}                                                                                                                                                                          
&\footnotesize{\# vertices} & 150 & 178 & 336 & 690 & 699 & 846 & 1000 & 2310 & 5620 & 7797 & 9298 & 10992 & 18846 & 70000 \\                                                            
&\footnotesize{\# classes} & 3 & 3 & 8 & 2 & 2 & 4 & 2 & 7 & 10 & 26 & 10 & 10 & 20 & 10 \\  
\specialrule{1pt}{-1pt}{1pt} 
\multirow{3}{*}{$H$} & \scriptsize{Best (\%)} & 14.1 & 14.1 & 10.9 & 7.8 & 0.0 & 20.3 & 34.4 & 0.0 & 3.1 & 0.0 & 0.0 & 0.0 & \textbf{45.3} & 0.0 \\                                                                 
& \scriptsize{Str.\ best (\%)} & 4.7 & 3.1 & 7.8 & 3.1 & 0.0 & 14.1 & 17.2 & 0.0 & 3.1 & 0.0 & 0.0 & 0.0 & \textbf{45.3} & 0.0 \\                                                                                    
& \scriptsize{Avg. error} & 16.8 & 32.2 & 23.9 & 41.7 & 13.2 & 58.4 & \textbf{29.5} & 64.0 & 32.5 & 54.0 & 41.3 & 39.2 & 88.5 & 48.2 \\                                                                              
\specialrule{1pt}{-1pt}{1pt} 
\multirow{3}{*}{$L_{SN}$} & \scriptsize{Best (\%)} & 10.9 & 10.9 & 14.1 & 4.7 & 15.6 & 12.5 & 12.5 & 17.2 & 26.6 & 7.8 & 14.1 & 7.8 & 26.6 & 14.1 \\                                                                 
& \scriptsize{Str.\ best (\%)} & 1.6 & 3.1 & 9.4 & 1.6 & 15.6 & 7.8 & 0.0 & 17.2 & 26.6 & 7.8 & 14.1 & 7.8 & 26.6 & 14.1 \\                                                                                          
& \scriptsize{Avg. error} & 17.5 & 32.2 & 24.5 & 42.8 & 8.8 & 57.2 & 29.9 & 53.9 & 24.9 & 51.2 & 38.6 & 37.8 & 89.0 & 45.8 \\                                                                                        
\specialrule{1pt}{-1pt}{1pt} 
\multirow{3}{*}{$L_{BN}$} & \scriptsize{Best (\%)} & 4.7 & 12.5 & 0.0 & 6.3 & 1.6 & 6.3 & 40.6 & 0.0 & 0.0 & 0.0 & 0.0 & 0.0 & 1.6 & 0.0 \\                                                                          
& \scriptsize{Str.\ best (\%)} & 1.6 & 1.6 & 0.0 & 0.0 & 1.6 & 4.7 & 17.2 & 0.0 & 0.0 & 0.0 & 0.0 & 0.0 & 1.6 & 0.0 \\                                                                                               
& \scriptsize{Avg. error} & 26.6 & 33.6 & 30.5 & 42.5 & 10.2 & 61.6 & 29.6 & 57.2 & 41.1 & 67.4 & 50.1 & 50.5 & 92.5 & 58.6 \\                                                                                       
\specialrule{1pt}{-1pt}{1pt} 
\multirow{3}{*}{$L_{AM}$} & \scriptsize{Best (\%)} & 6.3 & 20.3 & 7.8 & 6.3 & 0.0 & 20.3 & 15.6 & 9.4 & 0.0 & 0.0 & 0.0 & 0.0 & 4.7 & 0.0 \\                                                                         
& \scriptsize{Str.\ best (\%)} & 1.6 & 9.4 & 6.3 & 1.6 & 0.0 & 7.8 & 1.6 & 6.3 & 0.0 & 0.0 & 0.0 & 0.0 & 4.7 & 0.0 \\                                                                                                
& \scriptsize{Avg. error} & 19.0 & 32.7 & 24.4 & 42.7 & 11.6 & 58.1 & 29.7 & 47.7 & 33.5 & 49.6 & 44.7 & 48.3 & 89.7 & 56.1 \\                                                                                       
\specialrule{1pt}{-1pt}{1pt} 
\multirow{3}{*}{$L_{GM}$} & \scriptsize{Best (\%)} & 32.8 & 35.9 & 34.4 & 32.8 & 7.8 & 17.2 & \textbf{46.9} & 6.3 & 29.7 & 28.1 & 12.5 & 0.0 & 1.6 & \textbf{82.8} \\                                                
& \scriptsize{Str.\ best (\%)} & 1.6 & 7.8 & \textbf{21.9} & 23.4 & 6.3 & 14.1 & \textbf{25.0} & 6.3 & 28.1 & 28.1 & 9.4 & 0.0 & 1.6 & \textbf{82.8} \\                                                              
& \scriptsize{Avg. error} & 14.1 & 31.9 & 20.4 & 39.3 & 11.3 & 57.6 & \textbf{29.5} & 46.8 & 13.0 & 42.6 & 27.6 & 45.0 & 89.9 & \textbf{26.7} \\                                                                     
\specialrule{1pt}{-1pt}{1pt} 
\multirow{3}{*}{$L_{-1}$} & \scriptsize{Best (\%)} & 25.0 & \textbf{45.3} & \textbf{39.1} & \textbf{42.2} & 0.0 & 12.5 & 15.6 & \textbf{39.1} & 4.7 & \textbf{37.5} & 4.7 & 9.4 & 12.5 & 1.6 \\                      
& \scriptsize{Str.\ best (\%)} & 0.0 & \textbf{14.1} & 18.8 & \textbf{31.3} & 0.0 & 9.4 & 1.6 & 29.7 & 4.7 & \textbf{37.5} & 4.7 & 9.4 & 12.5 & 1.6 \\                                                               
& \scriptsize{Avg. error} & 13.8 & \textbf{29.8} & \textbf{20.3} & \textbf{38.2} & 8.3 & 56.2 & 29.8 & 39.7 & 16.3 & \textbf{42.1} & 25.2 & 32.9 & \textbf{88.3} & 32.3 \\                                           
\specialrule{1pt}{-1pt}{1pt} 
\multirow{3}{*}{$L_{-10}$} & \scriptsize{Best (\%)} & \textbf{73.4} & 43.8 & 25.0 & 34.4 & \textbf{76.6} & \textbf{31.3} & 20.3 & \textbf{39.1} & \textbf{37.5} & 26.6 & \textbf{71.9} & \textbf{82.8} & 7.8 & 1.6 \\
& \scriptsize{Str.\ best (\%)} & \textbf{42.2} & 7.8 & 10.9 & 18.8 & \textbf{75.0} & \textbf{25.0} & 4.7 & \textbf{31.3} & \textbf{35.9} & 26.6 & \textbf{68.8} & \textbf{82.8} & 7.8 & 1.6 \\                       
& \scriptsize{Avg. error} & \textbf{12.7} & 30.2 & 20.8 & 38.6 & \textbf{5.7} & \textbf{55.9} & 29.7 & \textbf{39.4} & \textbf{12.1} & 42.3 & \textbf{21.9} & \textbf{26.9} & 89.8 & 28.6                                                                                       
\end{tabular}  
}
\caption{
Experiments on UCI datasets. 
Positive edges generated by $k$-nearest neighbours, and 
negative edges generated by $k$-farthest neighbours.
Reported is the percentage of cases where each method achieves the smallest and stricly smallest clustering error, and the average clustering error.      \vspace{-.7em}
}  
\label{table:UCI-experiments}   
\end{table*} 
\vspace{-.5em}
\begin{table*}  
\adjustbox{max width=\textwidth}{
\centering                                                                                                                                                                                                   
\begin{tabular}{ccccccccccccccccccccc}                                                                                                                                                             
\hline                                                                                                                                                                                                       
 & & iris & wine & ecoli & australian & cancer & vehicle & german & image & optdig & isolet & USPS & pendigits & 20new & MNIST \\                                                                             
\specialrule{1pt}{.1pt}{0pt}                                                                                                                                                                                                                                                                                                                                                        
&\footnotesize{\# vertices} & 150  & 178  & 336  & 690  & 699  & 846  & 1000  & 2310  & 5620  & 7797  & 9298  & 10992  & 18846  & 70000  \\                                                    
&\footnotesize{\# classes} & 3  & 3  & 8  & 2  & 2  & 4  & 2  & 7  & 10  & 26  & 10  & 10  & 20  & 10  \\                                                                                        
\specialrule{1pt}{-1pt}{1pt} 
\multirow{3}{*}{$H$} & \scriptsize{Best (\%)} & 54.7 & \textbf{51.6} & 20.3 & 43.8 & 40.6 & 26.6 & 45.3 & 12.5 & 4.7 & 3.1 & 4.7 & 1.6 & 15.6 & 4.7 \\                                                      
& \scriptsize{Str.\ best (\%)} & 0.0 & 9.4 & 9.4 & 1.6 & 0.0 & 7.8 & 0.0 & 0.0 & 3.1 & 3.1 & 4.7 & 1.6 & 15.6 & 3.1 \\                                                                                       
& \scriptsize{Avg. error} & 3.6 & \textbf{14.3} & 15.9 & 8.6 & 0.8 & 32.0 & 8.5 & 16.9 & 6.5 & 47.8 & 15.7 & 10.4 & 87.2 & 10.7 \\                                                                           
\hline 
\multirow{3}{*}{$L_{SN}$} & \scriptsize{Best (\%)} & 59.4 & 42.2 & 17.2 & 42.2 & 45.3 & 37.5 & 48.4 & 17.2 & 15.6 & 20.3 & 12.5 & 12.5 & 28.1 & 9.4 \\                                                       
& \scriptsize{Str.\ best (\%)} & 0.0 & 4.7 & 9.4 & 0.0 & 0.0 & 17.2 & 0.0 & 3.1 & 10.9 & 18.8 & 12.5 & 12.5 & 28.1 & 7.8 \\                                                                                  
& \scriptsize{Avg. error} & 4.0 & 15.5 & 15.9 & 11.9 & 5.3 & 30.5 & 11.1 & 13.8 & 4.9 & 44.1 & 11.9 & 7.2 & 86.1 & 7.8 \\                                                                                    
\hline                                                                                                                                                                                                       
\multirow{3}{*}{$L_{BN}$} & \scriptsize{Best (\%)} & \textbf{68.8} & \textbf{51.6} & \textbf{45.3} & 42.2 & 53.1 & \textbf{50.0} & 50.0 & \textbf{50.0} & 37.5 & 28.1 & 35.9 & 35.9 & \textbf{35.9} & 42.2 \\
& \scriptsize{Str.\ best (\%)} & 3.1 & 12.5 & \textbf{40.6} & 0.0 & 1.6 & \textbf{28.1} & 0.0 & 35.9 & 34.4 & 26.6 & 35.9 & 35.9 & \textbf{35.9} & 42.2 \\                                                   
& \scriptsize{Avg. error} & 2.5 & 14.6 & \textbf{14.9} & 11.0 & 0.8 & \textbf{30.1} & 10.1 & 13.6 & 6.6 & 47.0 & 12.5 & 8.8 & \textbf{84.6} & 9.2 \\                                                         
\hline                                                                                                                                                                                                       
\multirow{3}{*}{$L_{AM}$} & \scriptsize{Best (\%)} & 42.2 & 25.0 & 26.6 & 0.0 & 0.0 & 23.4 & 0.0 & 45.3 & \textbf{45.3} & \textbf{35.9} & \textbf{46.9} & \textbf{48.4} & 18.8 & \textbf{45.3} \\            
& \scriptsize{Str.\ best (\%)} & \textbf{21.9} & \textbf{21.9} & 23.4 & 0.0 & 0.0 & 21.9 & 0.0 & \textbf{43.8} & \textbf{45.3} & \textbf{35.9} & \textbf{46.9} & \textbf{48.4} & 18.8 & \textbf{45.3} \\     
& \scriptsize{Avg. error} & \textbf{2.1} & 22.9 & 17.0 & 11.3 & \textbf{0.5} & 44.2 & 12.5 & \textbf{13.0} & \textbf{3.2} & \textbf{40.2} & \textbf{8.2} & \textbf{3.9} & 88.1 & \textbf{4.2} \\             
\hline                                                                                                                                                                                                       
\multirow{3}{*}{$L_{GM}$} & \scriptsize{Best (\%)} & 3.1 & 4.7 & 0.0 & \textbf{87.5} & \textbf{87.5} & 1.6 & \textbf{100.0} & 0.0 & 0.0 & 7.8 & 0.0 & 0.0 & 1.6 & 0.0 \\                                     
& \scriptsize{Str.\ best (\%)} & 0.0 & 0.0 & 0.0 & \textbf{56.3} & \textbf{46.9} & 1.6 & \textbf{48.4} & 0.0 & 0.0 & 7.8 & 0.0 & 0.0 & 1.6 & 0.0 \\                                                          
& \scriptsize{Avg. error} & 5.6 & 28.8 & 19.4 & \textbf{2.2} & 4.3 & 55.9 & \textbf{0.0} & 35.9 & 9.4 & 41.9 & 21.1 & 25.0 & 89.5 & 25.4 \\                                                                  
\hline                                                                                                                                                                                                                                                                                                                                                                                                             
\multirow{3}{*}{$L_{-1}$} & \scriptsize{Best (\%)} & 7.8 & 9.4 & 1.6 & 0.0 & 0.0 & 1.6 & 0.0 & 1.6 & 1.6 & 6.3 & 0.0 & 1.6 & 0.0 & 0.0 \\                                                                    
& \scriptsize{Str.\ best (\%)} & 0.0 & 4.7 & 0.0 & 0.0 & 0.0 & 1.6 & 0.0 & 1.6 & 1.6 & 6.3 & 0.0 & 1.6 & 0.0 & 0.0 \\                                                                                        
& \scriptsize{Avg. error} & 3.9 & 28.3 & 19.1 & 10.8 & 4.2 & 54.9 & 22.5 & 27.3 & 9.1 & 41.5 & 19.2 & 15.9 & 89.5 & 19.5 \\                                                                                  
\hline                                                                                                                                                                                                       
\multirow{3}{*}{$L_{-10}$} & \scriptsize{Best (\%)} & 6.3 & 0.0 & 4.7 & 0.0 & 0.0 & 0.0 & 0.0 & 0.0 & 0.0 & 0.0 & 0.0 & 0.0 & 0.0 & 0.0 \\                                                                   
& \scriptsize{Str.\ best (\%)} & 4.7 & 0.0 & 3.1 & 0.0 & 0.0 & 0.0 & 0.0 & 0.0 & 0.0 & 0.0 & 0.0 & 0.0 & 0.0 & 0.0 \\                                                                                        
& \scriptsize{Avg. error} & 8.1 & 29.3 & 20.0 & 5.6 & 3.6 & 56.4 & 7.8 & 38.5 & 11.2 & 42.1 & 21.4 & 26.3 & 89.9 & 25.9 
\end{tabular}
}
\caption{
Experiments on UCI datasets. 
Positive edges generated by $k$-nearest neighbours.
Negative edges generated are cannot links between nodes of different classes.
Reported is the percentage of cases where each method achieves the smallest and stricly smallest clustering error, and the average clustering error.      
\vspace{-.7em}
}                                                                                                                                                                                                  
\label{table:UCI-experiments-cannotLinks}   
\end{table*}    
\section{Experiments with the Censored Block Model}\label{section:Experiments with the Censored Block Model}
In this section we present a numerical evaluation of different methods under the Stochastic Block Model following the parameters corresponding to the Censored Block Model (\textbf{CBM}), following~\cite{saade:2015}.
Observe that the CBM is a particular case of the Stochastic Block Model for signed graphs as introduced in Section~\ref{sec:SBM}.
Following~\cite{saade:2015}, the CBM is has two parameters: probability of observing an edge ($\overline{p}$), and the probability of flipping the sign of an edge ($\epsilon$).
The CBM can be recovered from the SSBM introduced in Section~\ref{sec:SBM} by setting $\pp=\qm=\overline{p}(1-\epsilon)$ and $\ppm=\qp=\overline{p}\epsilon$. Observe that the parameter $\epsilon$ works as a noise parameter:
the noiseless setting corresponds to $\epsilon=0$, where positive and negative edges are only inside and between clusters, respectively. The case where $\epsilon=0.5$ corresponds to the case where no clustering structure is conveyed by the sign of the edges.

We present a numerical evaluation under the SSBM with parameters from CBM in Fig.~\ref{fig:CBM}. We consider two clusters and fix a priori its size to be of 500 nodes each. We present the clustering error out of 20 realizations from the SSM with parameters following the CBM. We consider two settings:
\textbf{First setting:} we fix the probability of observing an edge to $\overline{p}=0.03$, and evaluate over different values of $\epsilon\in[0,0.5]$. In Fig.~\ref{subfig:CBM:p_fixed} we can observe that there is no relevant difference in clustering error between methods. Further, as expected we can see that for small values of $\epsilon$ all methods perform well, and for larger values of $\epsilon$ the clustering error increases;
\textbf{Second setting:} we fix the probability of flipping the sign of an edge to $\epsilon=0.25$, and evaluate over different values of $\overline{p}\in[0.001,0.03]$. 
In Fig.~\ref{subfig:CBM:epsilon_fixed} we can observe that the performance of the Bethe Hessian is best for small values of $\overline{p}$, i.e. for sparser graphs. Following the Bethe Hessian are the arithmetic mean Laplacian $L_{1}$ together with the signed normalized Laplacian $L_{SN}$. 

Hence we have observed that for sufficiently dense graphs following the Censored Block Model, the performance of different methods is rather similar, whereas for sparser graphs the Bethe Hessian performs best, confirming the analysis presented in~\cite{saade:2015}. 
%
 \section{Experiments on UCI datasets}\label{section:experiments}
We evaluate the signed power mean Laplacian with $L_{-10},L_{-1}$ against
$L_{SN}$, $L_{BN}$, $L_{AM}$, $L_{GM}$ and $H$
using datasets from the UCI repository. We build $W^+$ from the $k^+$ nearest neighbor graph, whereas $W^-$ is obtained from the $k^-$ farthest neighbor graph. For each dataset we evaluate all clustering methods over all possible choices of $k^+,k^-\in\{3,5,7,10,15,20,40,60\}$, yielding in total 64 cases. 
We present the following statistics:
Best($\%$): proportion of cases where a method yields the smallest clustering error. 
Strictly~Best($\%$): proportion of cases where a method is the \textit{only one} yielding the smallest clustering error. 
Results are shown  in Table~\ref{table:UCI-experiments}. 

Observe that in 4 datasets $H$ and $L_{GM}$ present a competitive performance. For the remaining cases we can see that the best performance are obtained by the signed power mean Laplacians $L_{-1},L_{-10}$. This verifies the superiority of negative powers 
($p<0$)
to positive ($p>0$) powers of $L_p$ and related approaches like $L_{SN},L_{BN}$. Moreover, although the Bethe Hessian is known to be optimal under the sparse transition theoretic limit under the Censored Block Model~\cite{saade:2015}, in the context where graphs unlikely follow a SBM distribution we can see that it is outperformed by the signed power mean Laplacian $L_p$.

We consider a second setting  where we generate $k^-$~noiseless negative edges via cannot link constraints between nodes of different classes. The corresponding results are shown in Table~\ref{table:UCI-experiments-cannotLinks}. We observe in this setting that the arithmetic mean Laplacian $L_{AM}$ presents the best performance, followed by the geometric mean Laplacian $L_{GM}$ and the Balance Normalized Laplacian $L_{BN}$. This suggests that from the family of non-arithmetic based Laplacians the case of $L_{GM}$ is a reasonable option showing certain robustness to different signed graph regimes.

We emphasize that the eigenvectors of $L_p$ are calculated without ever computing the matrix itself, by adapting the method proposed in~\cite{Mercado:2018:powerMean}, described in Sec.\ref{section:computation}.
Also, please see Section~\ref{sec:OnDiagonalShift} for a performance comparison with respect to changes in the diagonal shift on UCI datasets.

%
 \begin{figure*}[t]
\centering
\vskip.0em

%
\begin{subfigure}[]{0.24\linewidth}
\includegraphics[width=1\linewidth, clip,trim=130 40 170 40]{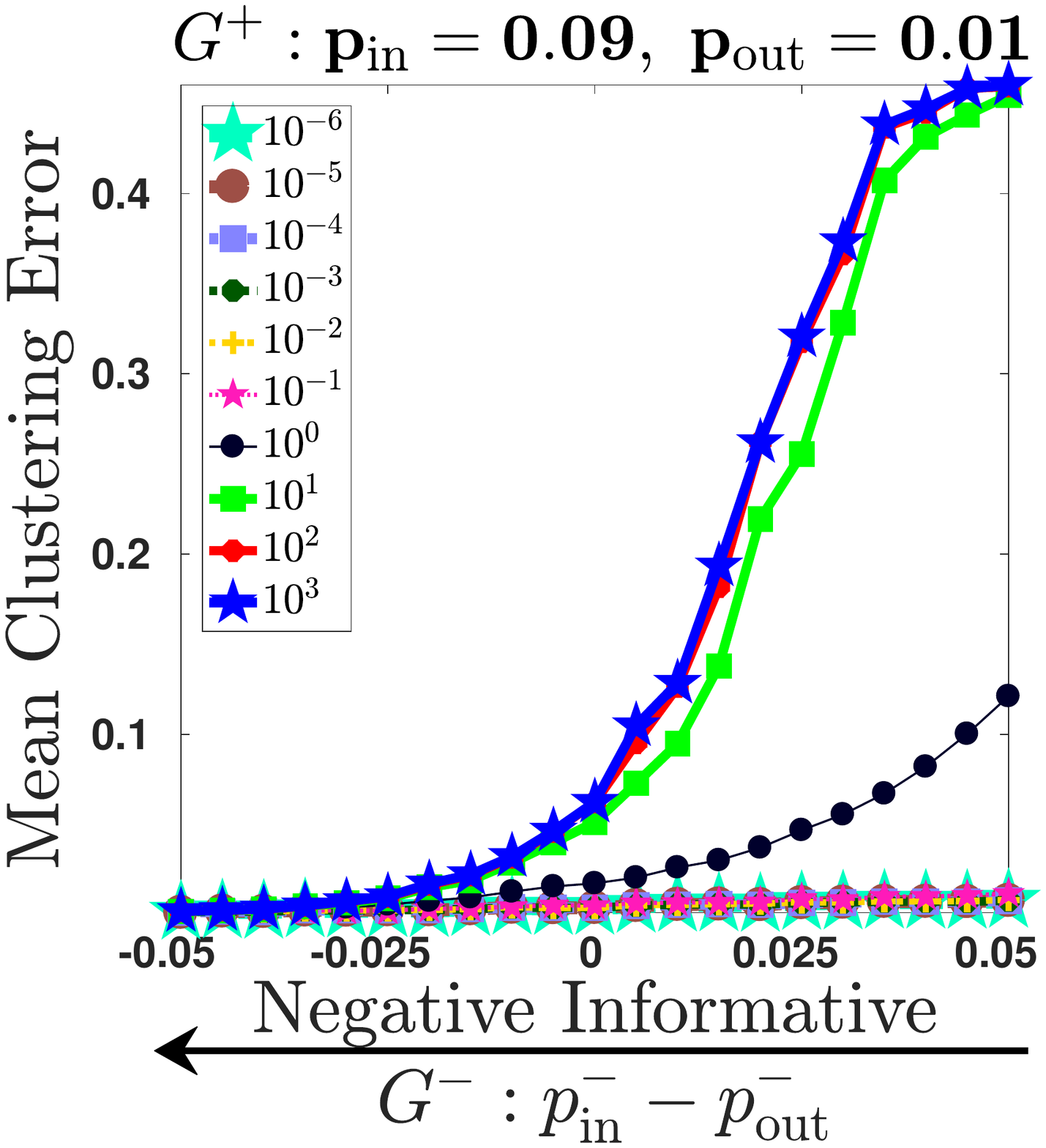}\hspace*{\fill}
\caption{$L_{-1}$}
\label{subfig:fig:SBM:diagonal_shift:fix_Wpos:L_{-1}}
\end{subfigure}%
\hfill
\begin{subfigure}[]{0.24\linewidth}
\includegraphics[width=1\linewidth, clip,trim=130 40 170 40]{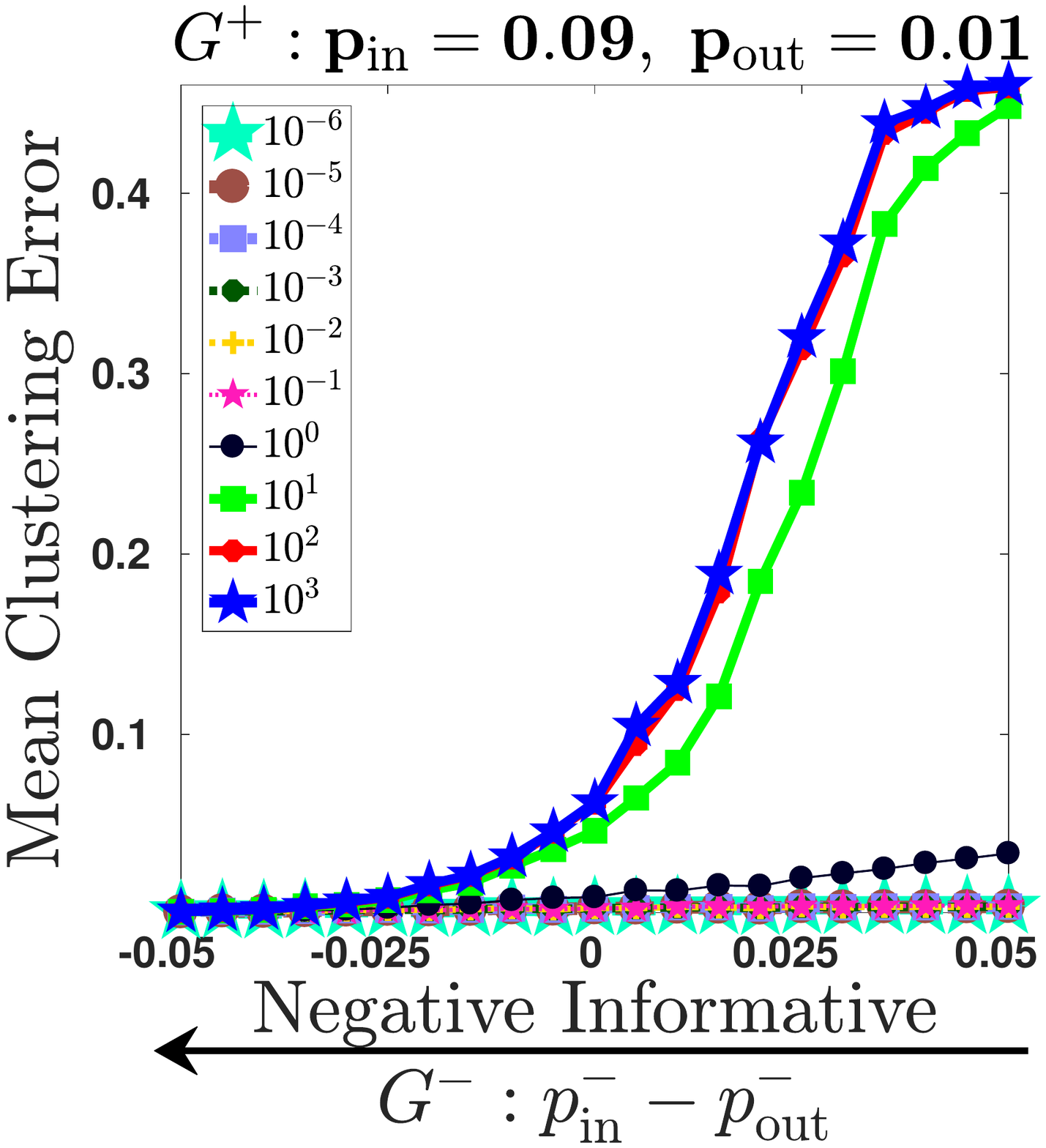}\hspace*{\fill}
\caption{$L_{-2}$}
\label{subfig:fig:SBM:diagonal_shift:fix_Wpos:L_{-2}}
\end{subfigure}%
\hfill
\begin{subfigure}[]{0.24\linewidth}
\includegraphics[width=1\linewidth, clip,trim=130 40 170 40]{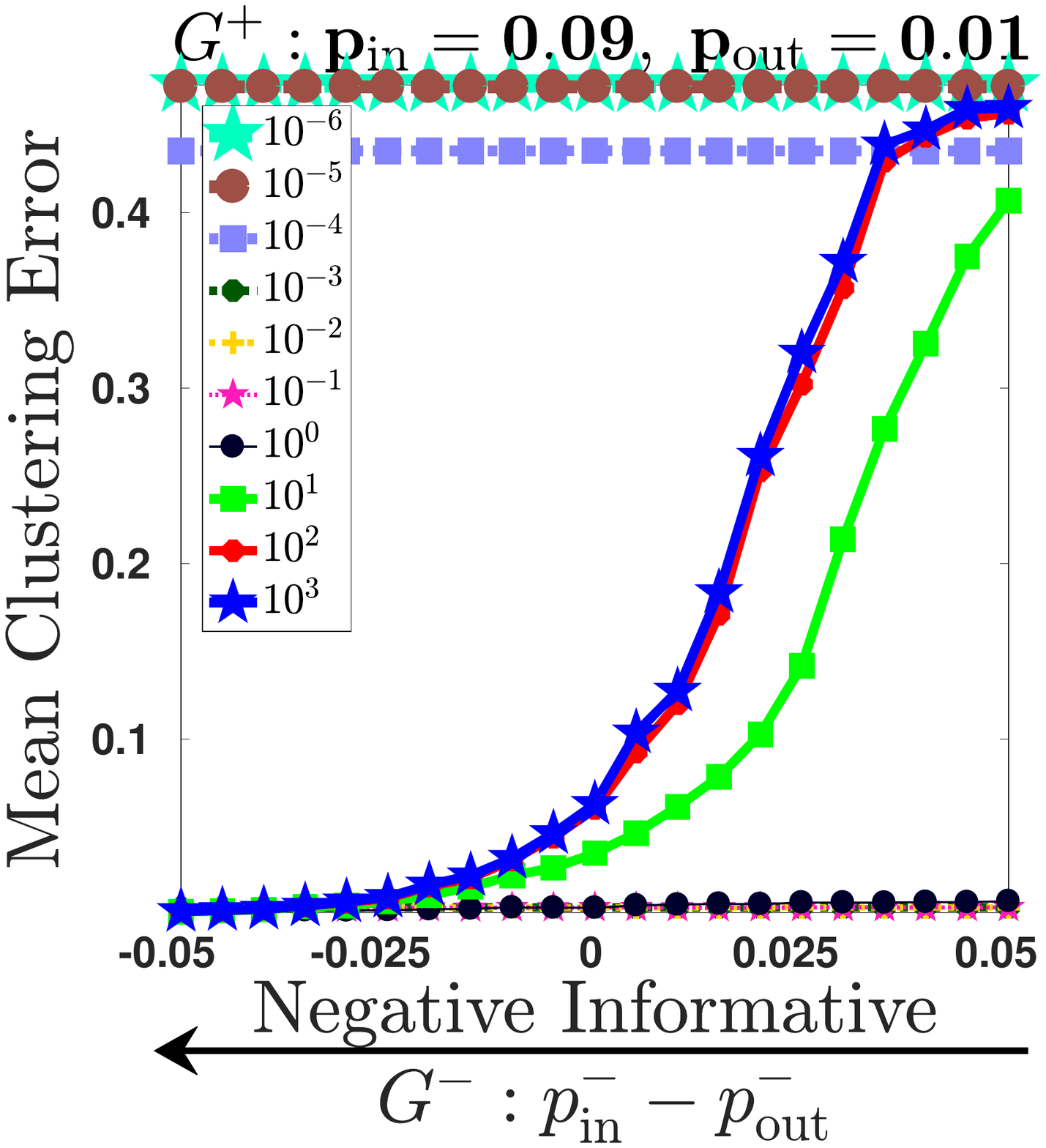}\hspace*{\fill}
\caption{$L_{-5}$}
\label{subfig:fig:SBM:diagonal_shift:fix_Wpos:L_{-5}}
\end{subfigure}%
\hfill
\begin{subfigure}[]{0.24\linewidth}
\includegraphics[width=1\linewidth, clip,trim=130 40 170 40]{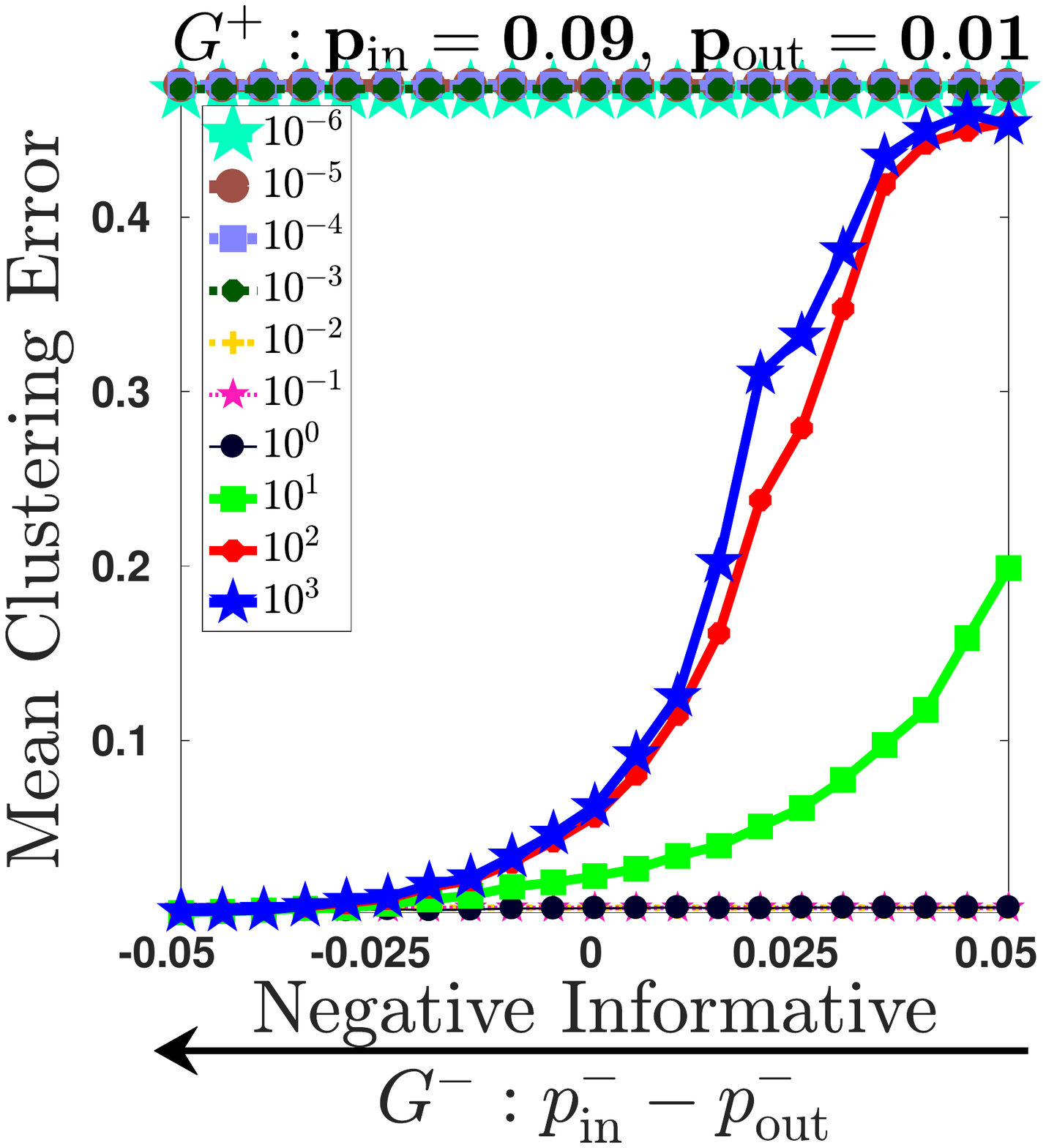}\hspace*{\fill}
\caption{$L_{-10}$}
\label{subfig:fig:SBM:diagonal_shift:fix_Wpos:L_{-10}}
\end{subfigure}
\hfill
\\
 \vspace{5pt}
\begin{subfigure}[]{0.24\linewidth}
\includegraphics[width=1\linewidth, clip,trim=130 40 170 40]{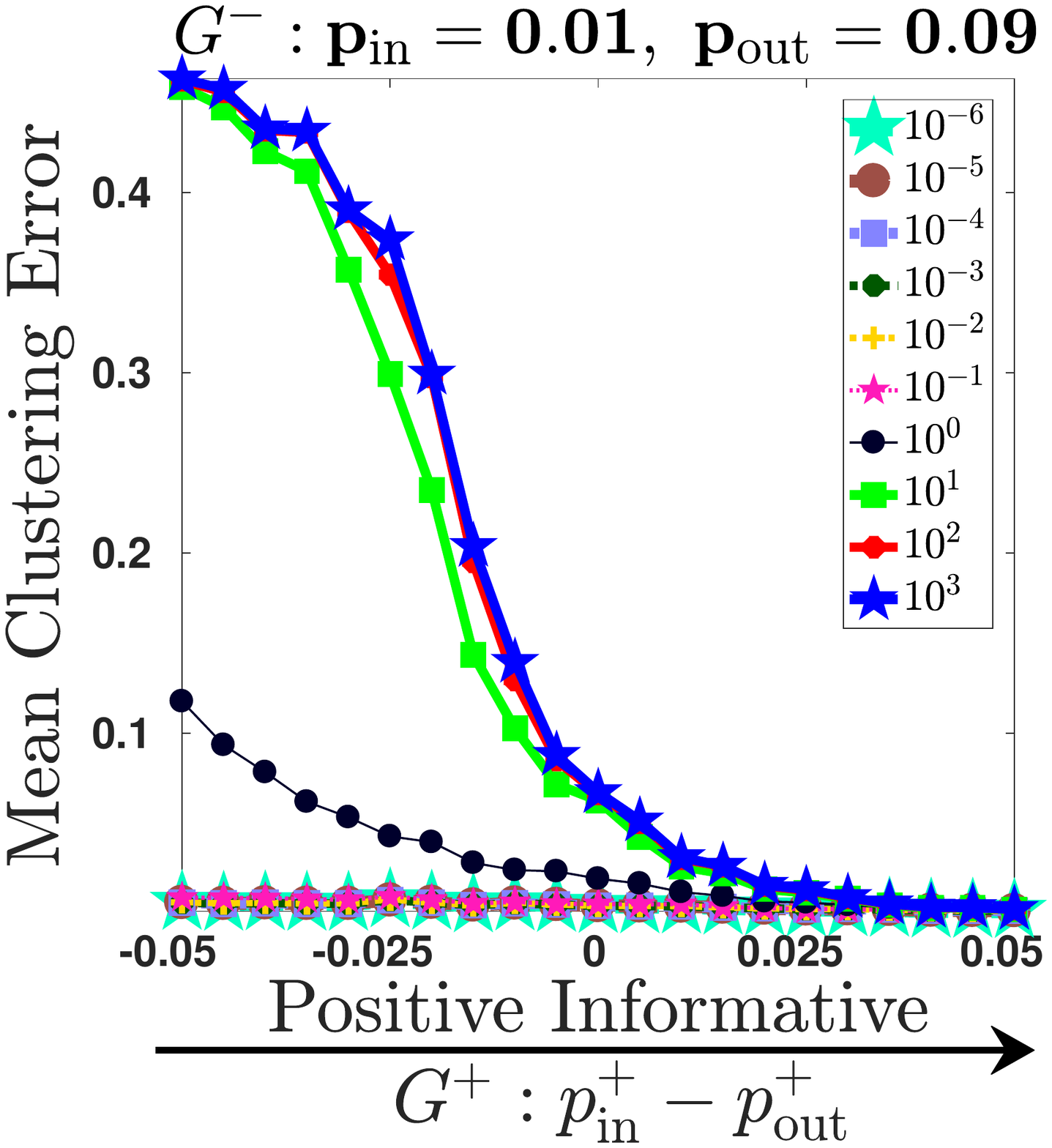}\hspace*{\fill}
\caption{$L_{-1}$}
\label{subfig:fig:SBM:diagonal_shift:fix_Wneg:L_{-1}}
\end{subfigure}%
\hfill
\begin{subfigure}[]{0.24\linewidth}
\includegraphics[width=1\linewidth, clip,trim=130 40 170 40]{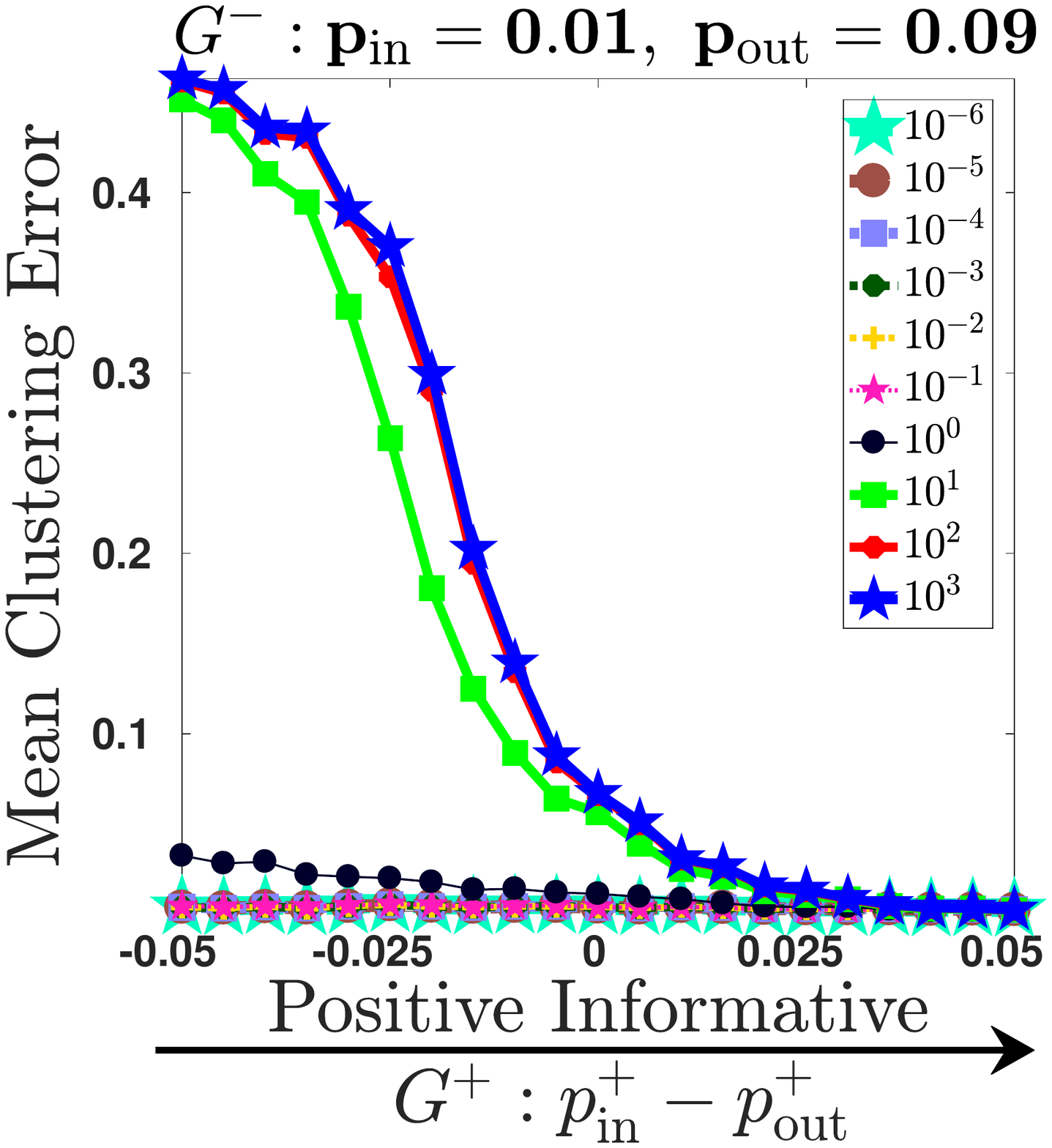}\hspace*{\fill}
\caption{$L_{-2}$}
\label{subfig:fig:SBM:diagonal_shift:fix_Wneg:L_{-2}}
\end{subfigure}%
\hfill
\begin{subfigure}[]{0.24\linewidth}
\includegraphics[width=1\linewidth, clip,trim=130 40 170 40]{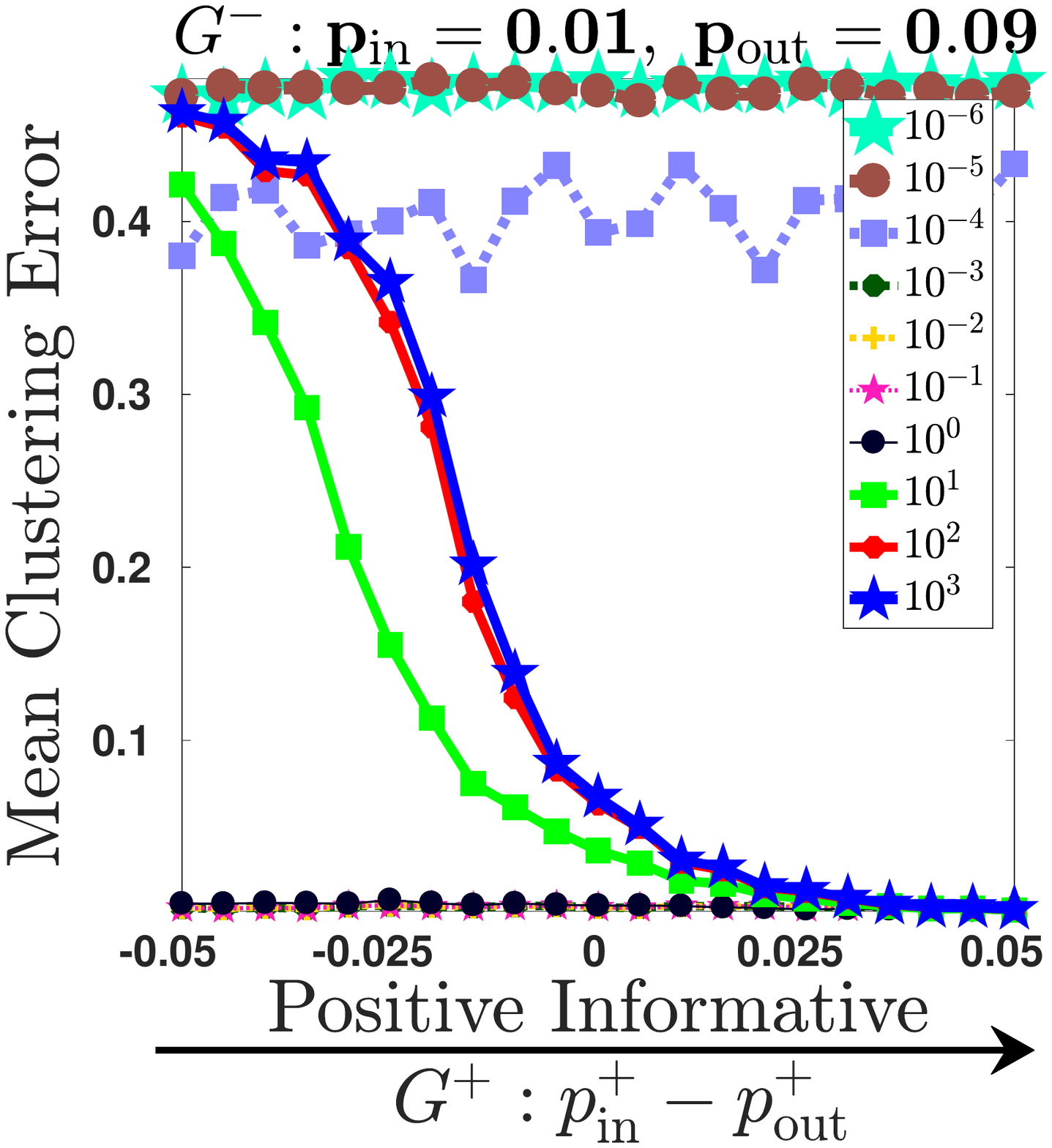}\hspace*{\fill}
\caption{$L_{-5}$}
\label{subfig:fig:SBM:diagonal_shift:fix_Wneg:L_{-5}}
\end{subfigure}%
\hfill
\begin{subfigure}[]{0.24\linewidth}
\includegraphics[width=1\linewidth, clip,trim=130 40 170 40]{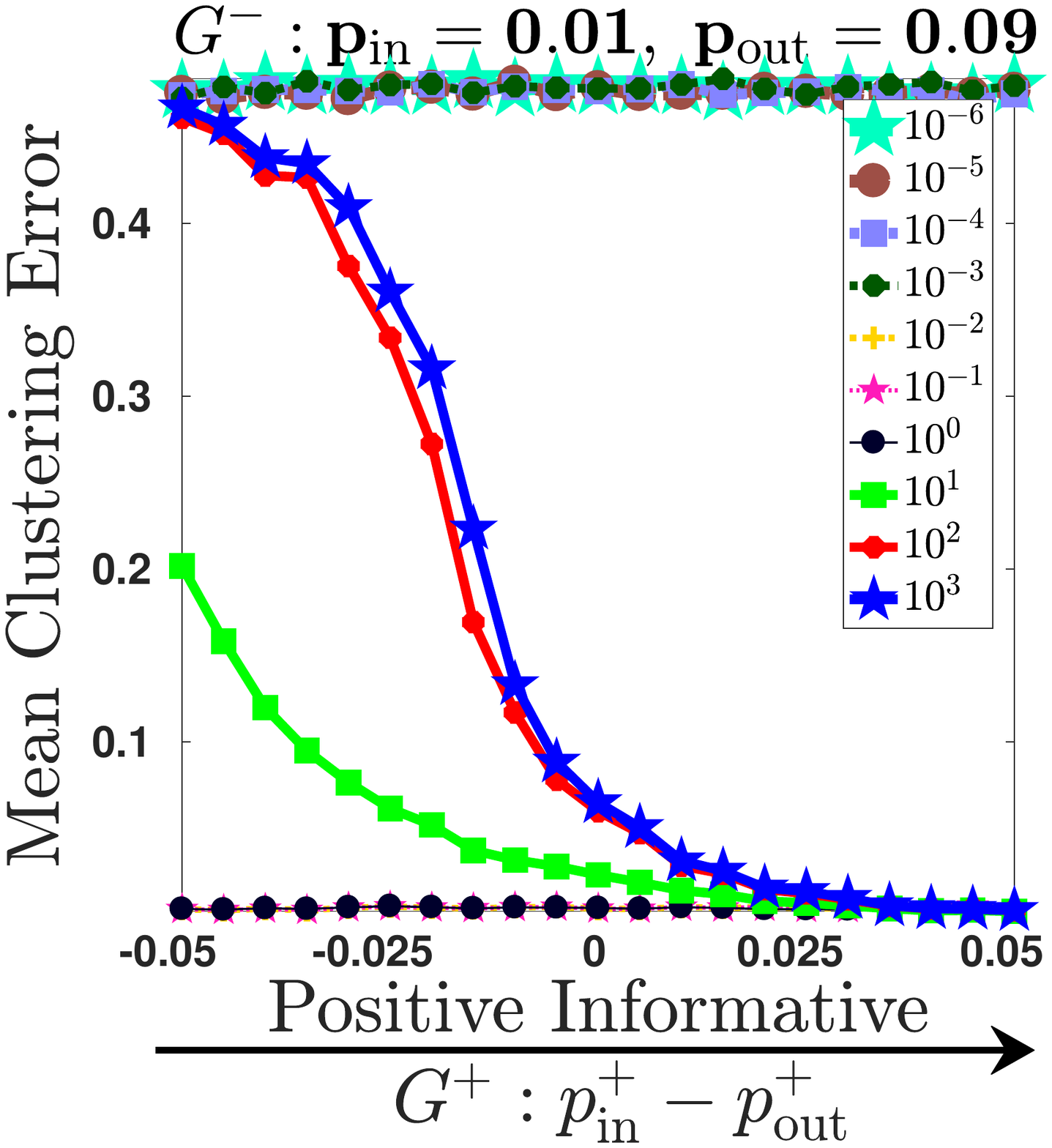}\hspace*{\fill}
\caption{$L_{-10}$}
\label{subfig:fig:SBM:diagonal_shift:fix_Wneg:L_{-10}}
\end{subfigure}%
\hfill
%
\caption{
Mean clustering error under SBM for different diagonal shifts with sparsity $0.1$. Details in Sec.~\ref{sec:OnDiagonalShift}.
}
\label{fig:SBM:diagonal_shift}
\end{figure*}
%
%
 \begin{figure*}[!h]
\centering
\vskip.0em
\begin{subfigure}[]{0.24\linewidth}
\includegraphics[width=1\linewidth, clip,trim=130 40 170 40]{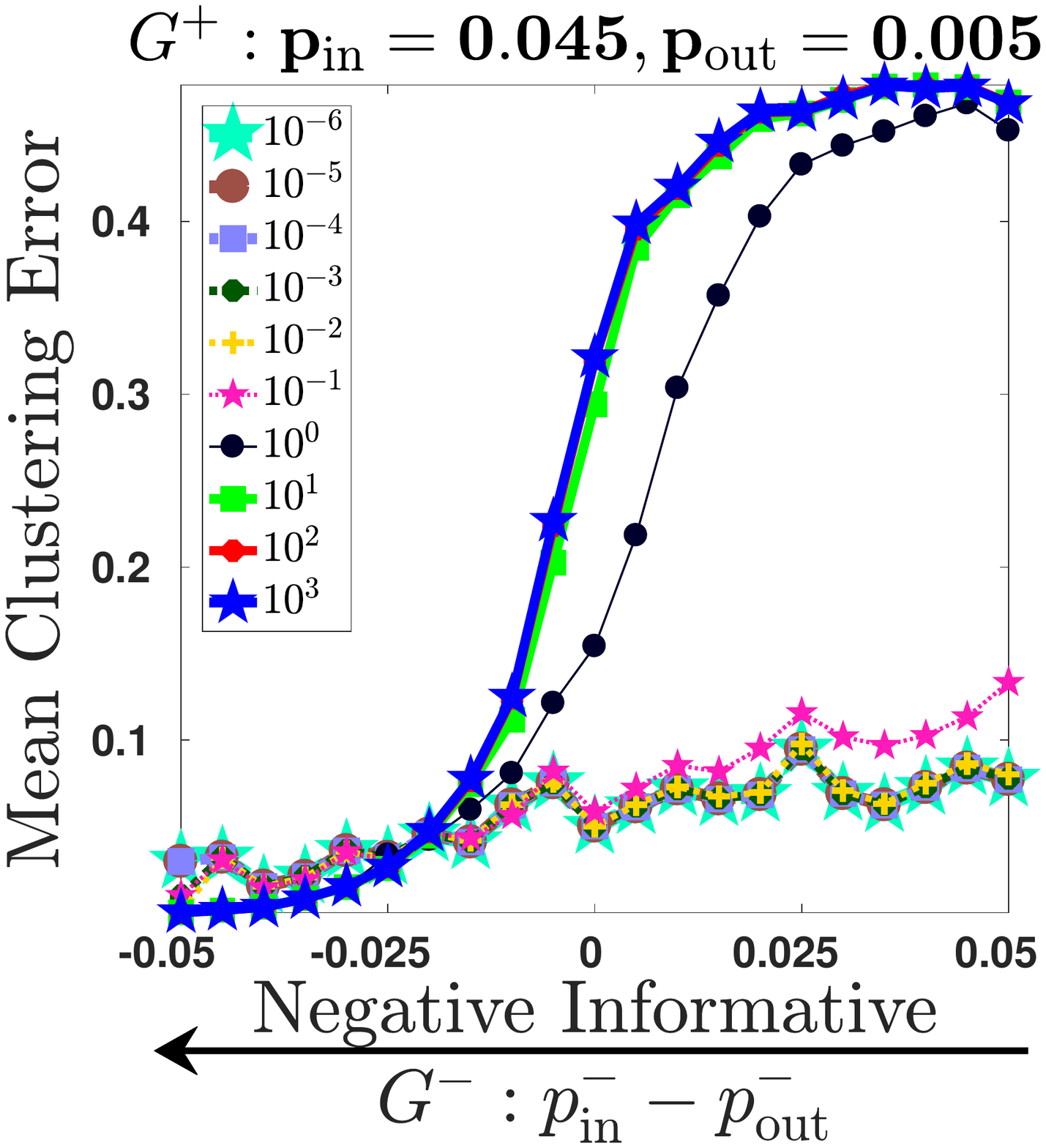}\hspace*{\fill}
\caption{$L_{-1}$}
\label{subfig:fig:SBM:diagonal_shift:sparsity_5:fix_Wpos:L_{-1}}
\end{subfigure}%
\hfill
\begin{subfigure}[]{0.24\linewidth}
\includegraphics[width=1\linewidth, clip,trim=130 40 170 40]{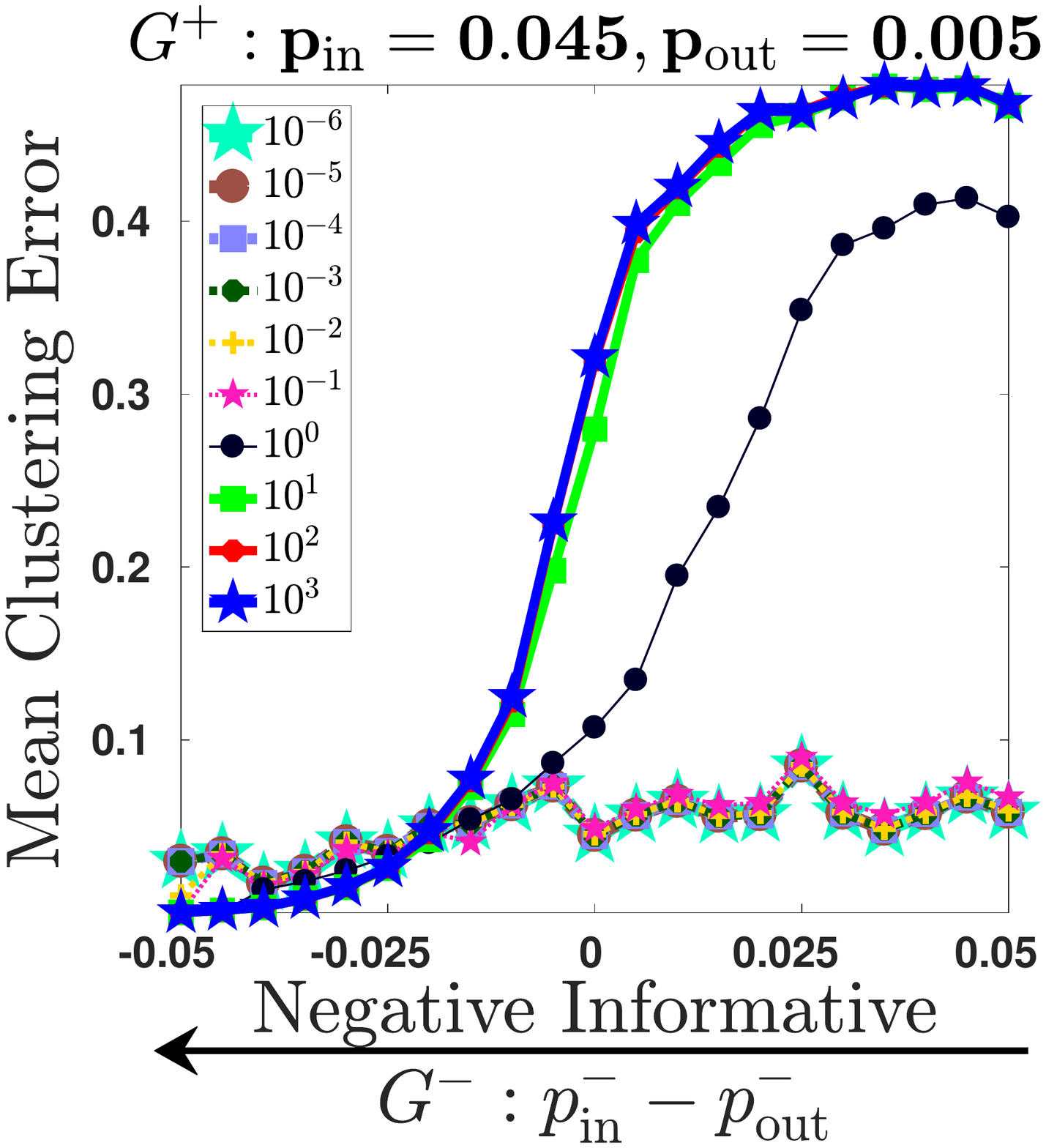}\hspace*{\fill}
\caption{$L_{-2}$}
\label{subfig:fig:SBM:diagonal_shift:sparsity_5:fix_Wpos:L_{-2}}
\end{subfigure}%
\hfill
\begin{subfigure}[]{0.24\linewidth}
\includegraphics[width=1\linewidth, clip,trim=130 40 170 40]{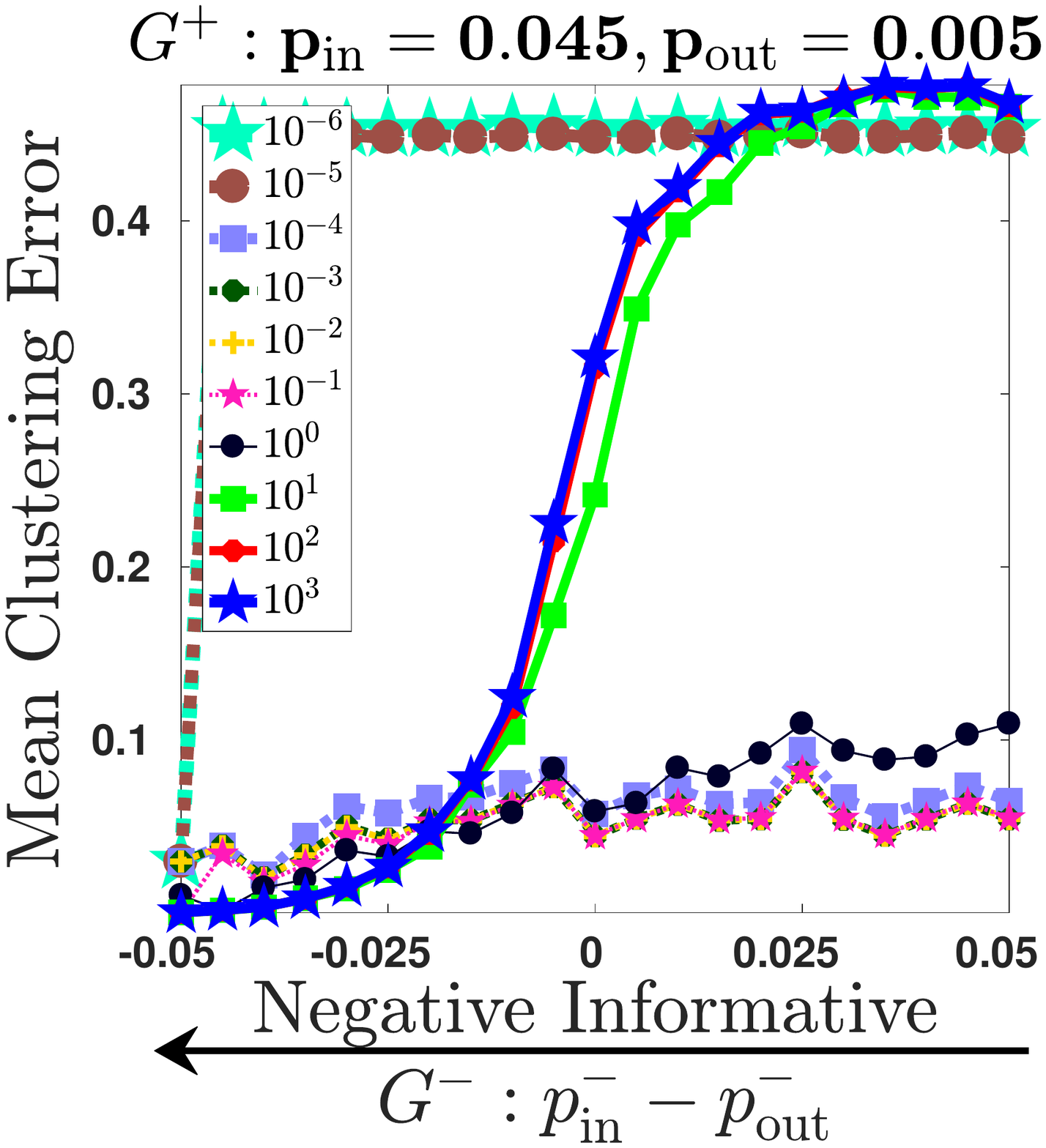}\hspace*{\fill}
\caption{$L_{-5}$}
\label{subfig:fig:SBM:diagonal_shift:sparsity_5:fix_Wpos:L_{-5}}
\end{subfigure}%
\hfill
\begin{subfigure}[]{0.24\linewidth}
\includegraphics[width=1\linewidth, clip,trim=130 40 170 40]{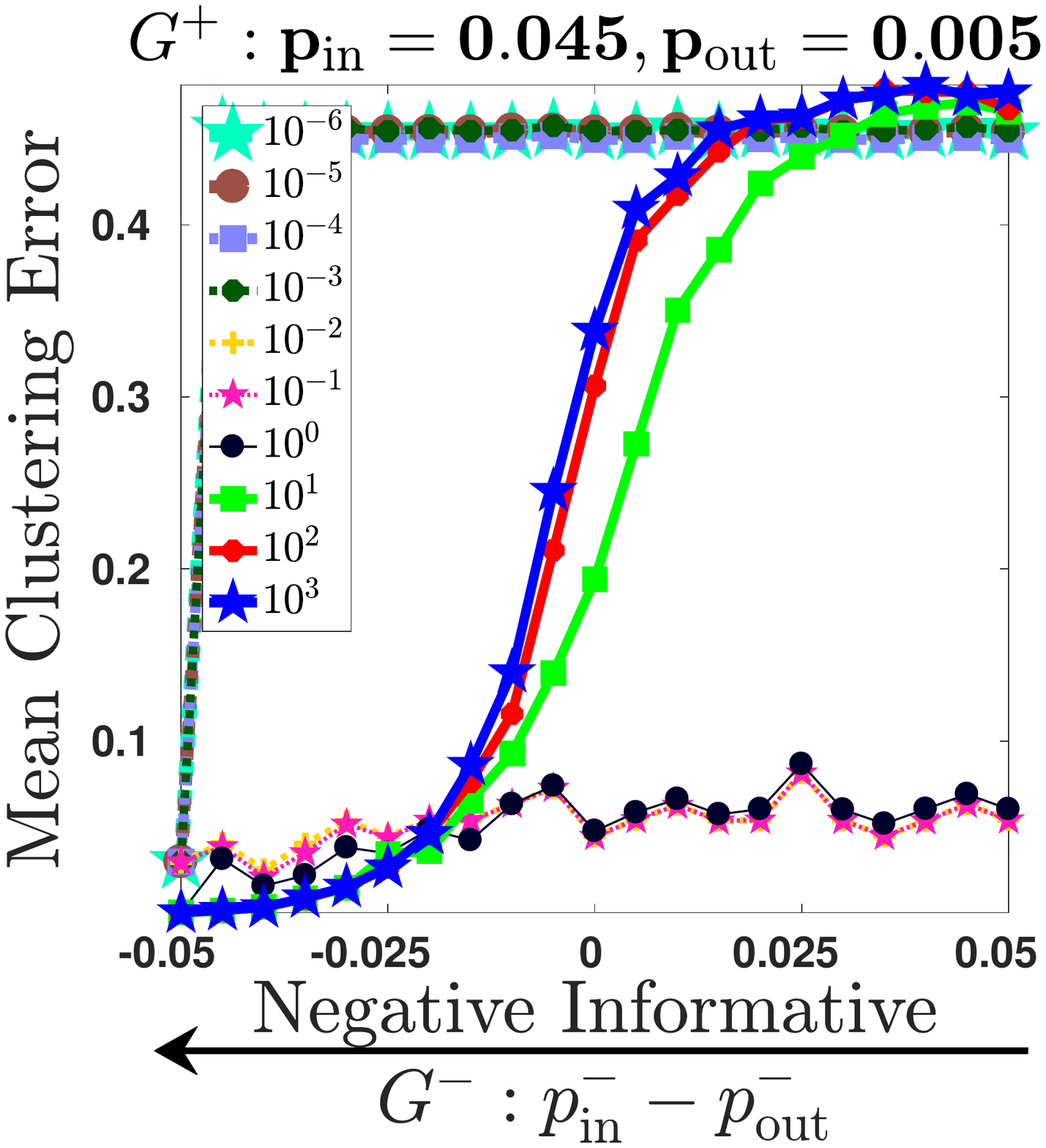}\hspace*{\fill}
\caption{$L_{-10}$}
\label{subfig:fig:SBM:diagonal_shift:sparsity_5:fix_Wpos:L_{-10}}
\end{subfigure}
\hfill
\\
 \vspace{5pt}
\begin{subfigure}[]{0.24\linewidth}
\includegraphics[width=1\linewidth, clip,trim=130 40 170 40]{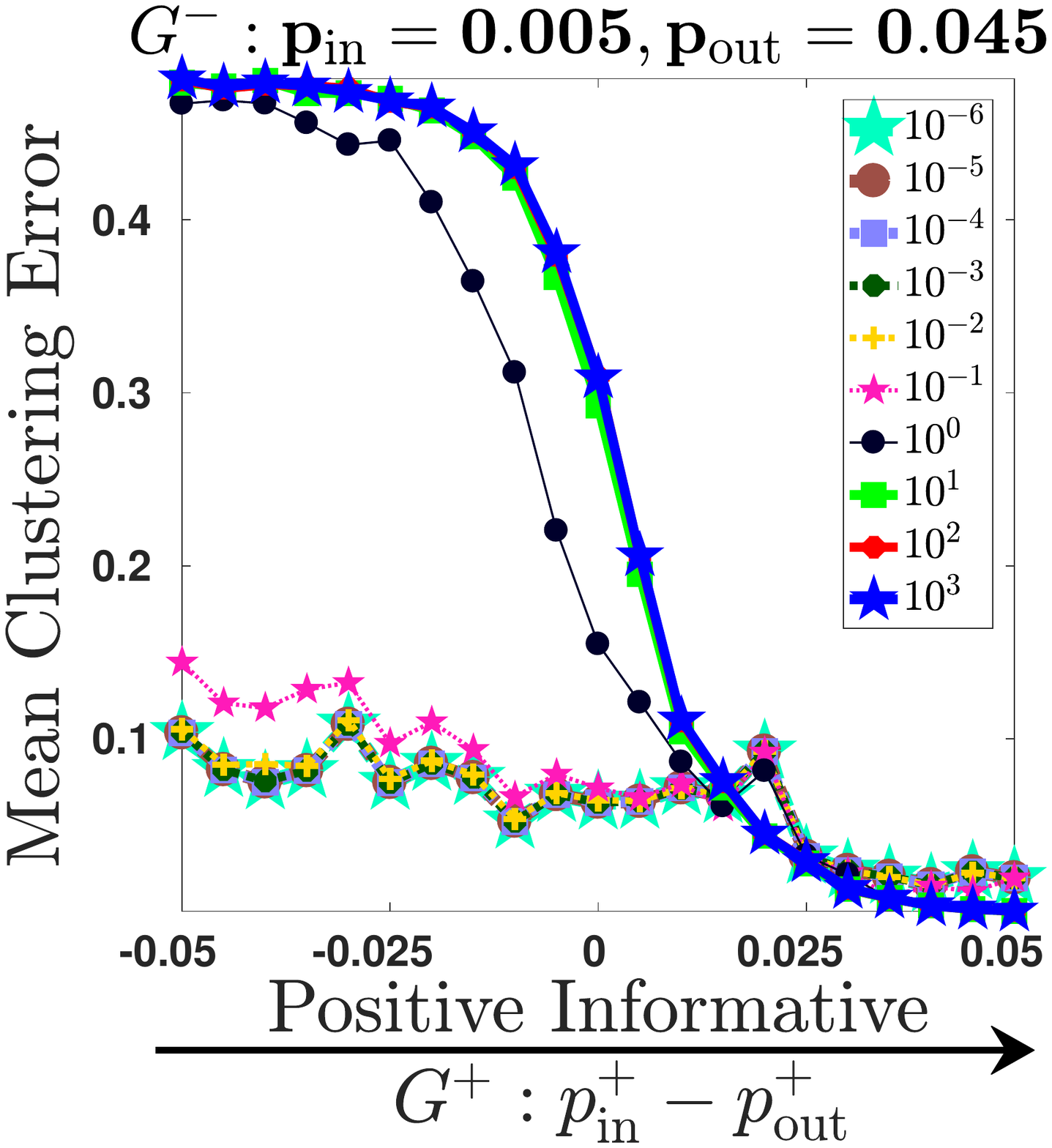}\hspace*{\fill}
\caption{$L_{-1}$}
\label{subfig:fig:SBM:diagonal_shift:sparsity_5:fix_Wneg:L_{-1}}
\end{subfigure}%
\hfill
\begin{subfigure}[]{0.24\linewidth}
\includegraphics[width=1\linewidth, clip,trim=130 40 170 40]{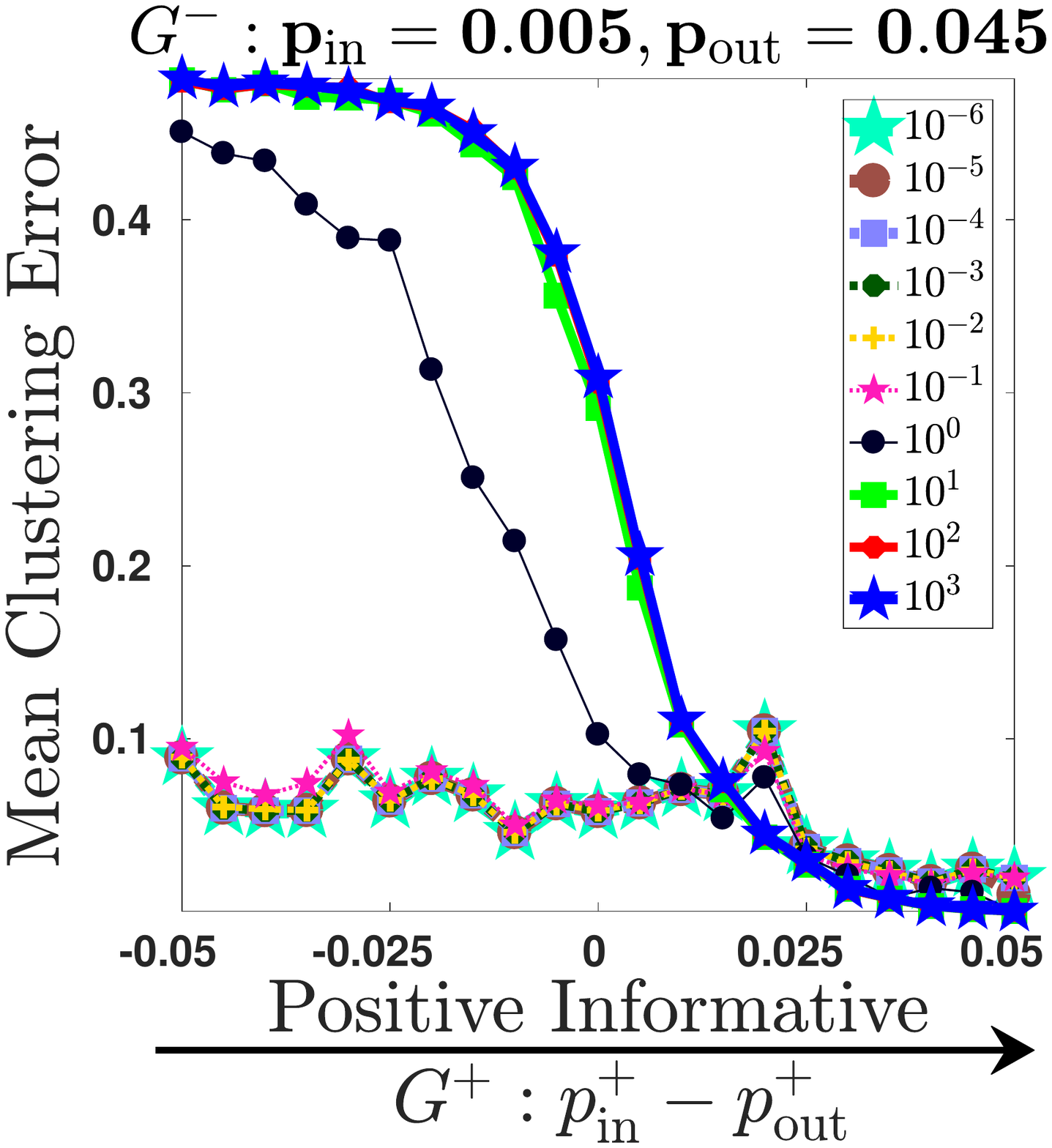}\hspace*{\fill}
\caption{$L_{-2}$}
\label{subfig:fig:SBM:diagonal_shift:sparsity_5:fix_Wneg:L_{-2}}
\end{subfigure}%
\hfill
\begin{subfigure}[]{0.24\linewidth}
\includegraphics[width=1\linewidth, clip,trim=130 40 170 40]{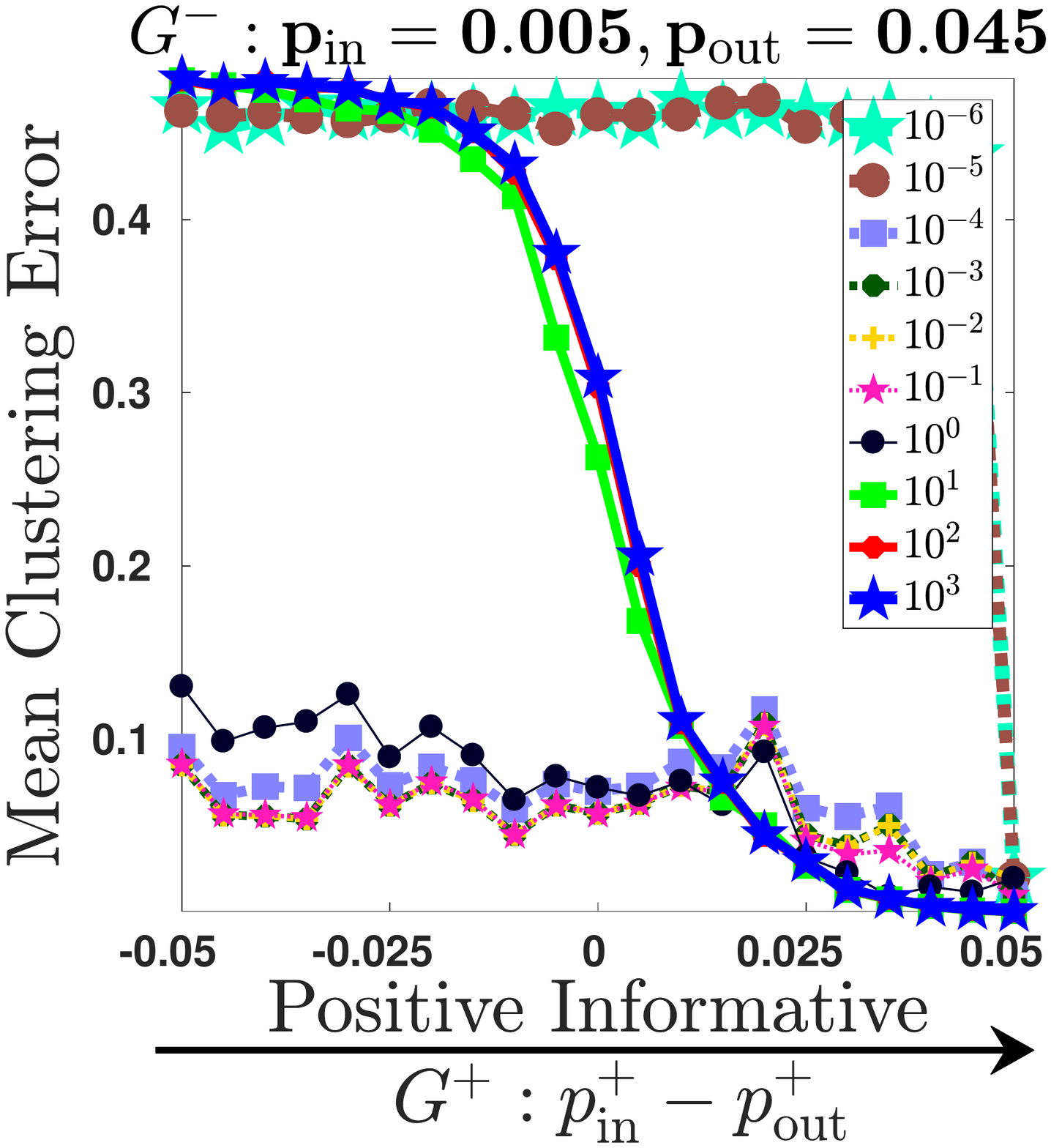}\hspace*{\fill}
\caption{$L_{-5}$}
\label{subfig:fig:SBM:diagonal_shift:sparsity_5:fix_Wneg:L_{-5}}
\end{subfigure}%
\hfill
\begin{subfigure}[]{0.24\linewidth}
\includegraphics[width=1\linewidth, clip,trim=130 40 170 40]{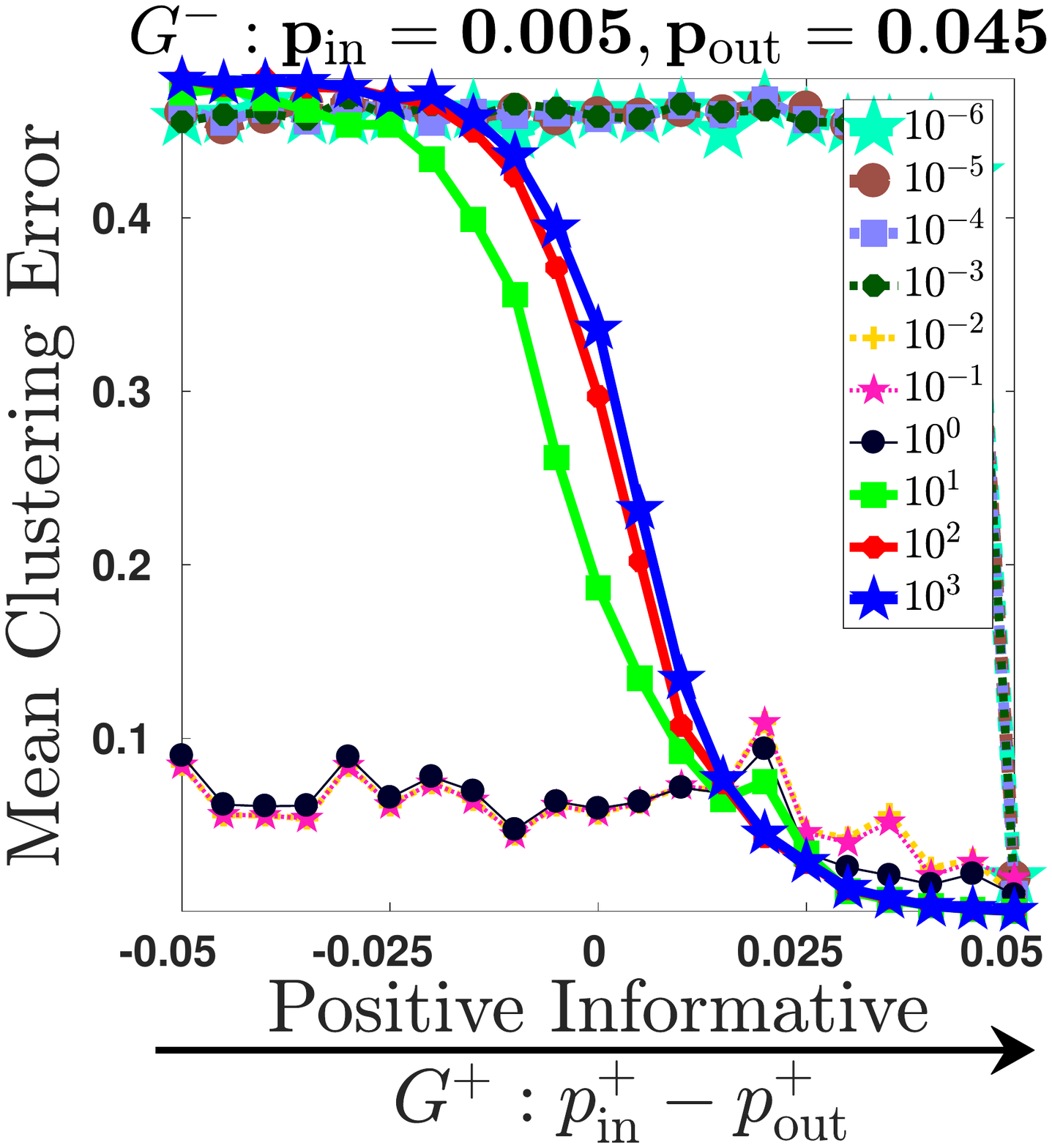}\hspace*{\fill}
\caption{$L_{-10}$}
\label{subfig:fig:SBM:diagonal_shift:sparsity_5:fix_Wneg:L_{-10}}
\end{subfigure}%
\hfill
%
\caption{
Mean clustering error under SBM for different diagonal shifts with sparsity $0.05$. Details in Sec.~\ref{sec:OnDiagonalShift}.
}
\label{fig:SBM:diagonal_shift_sparsity_5}
\end{figure*}
\begin{figure*}[t]
\centering
\vskip.6em

%
\begin{subfigure}[]{0.24\linewidth}
\includegraphics[width=1\linewidth, clip,trim=105 40 160 40]{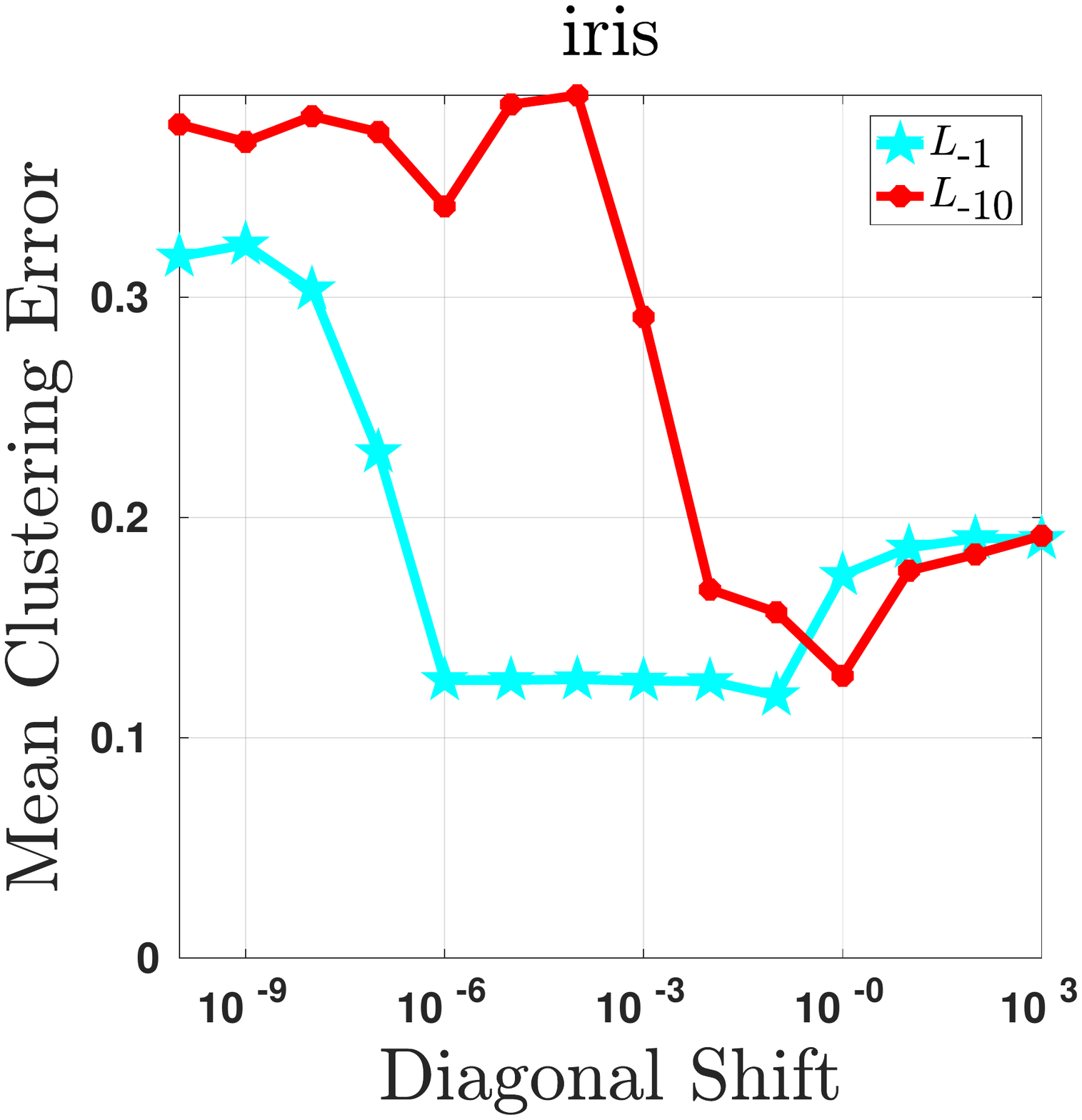}\hspace*{\fill}
\end{subfigure}%
\hfill
\begin{subfigure}[]{0.24\linewidth}
\includegraphics[width=1\linewidth, clip,trim=105 40 160 40]{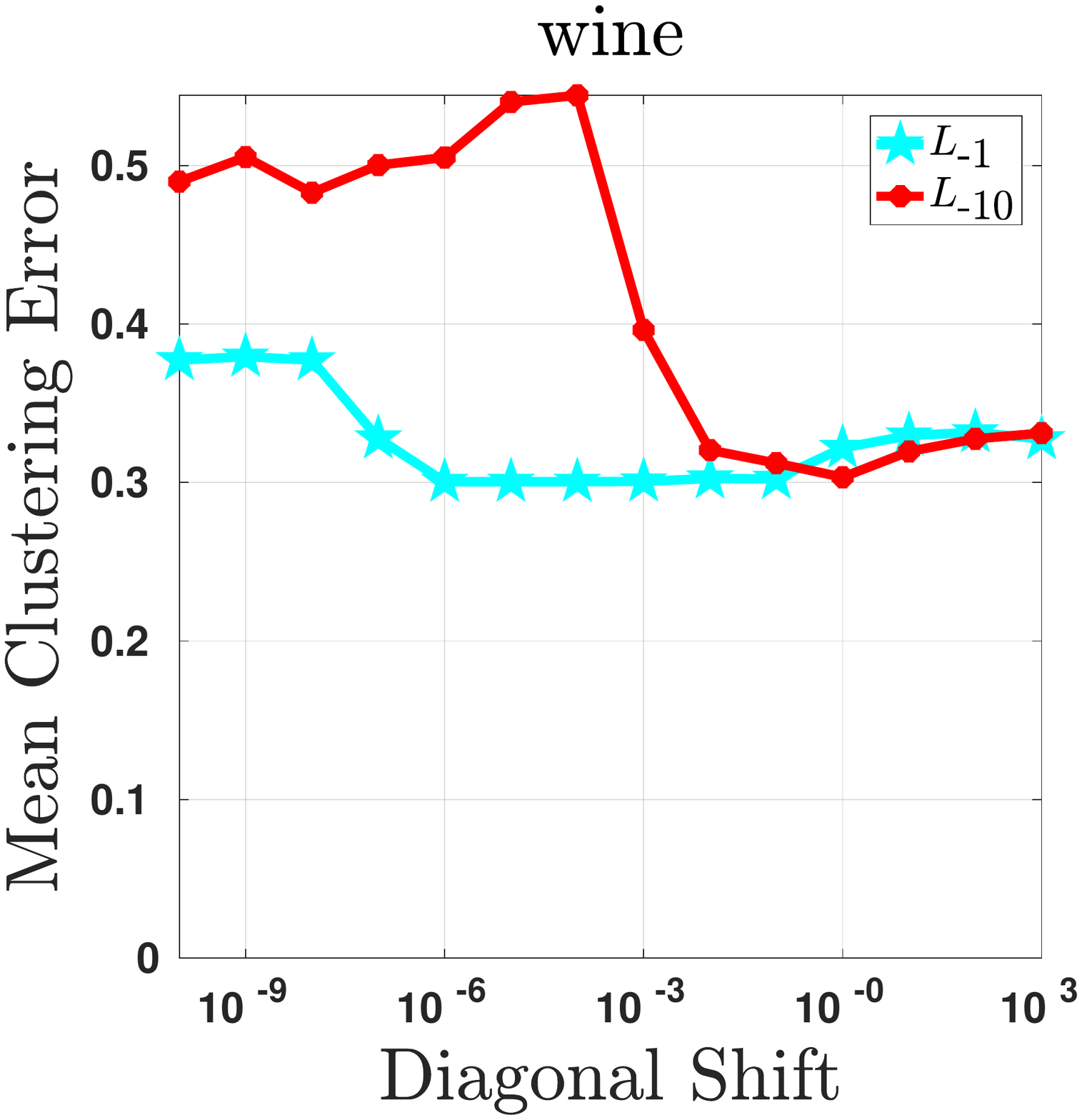}\hspace*{\fill}
\end{subfigure}%
\hfill
\begin{subfigure}[]{0.24\linewidth}
\includegraphics[width=1\linewidth, clip,trim=105 40 160 40]{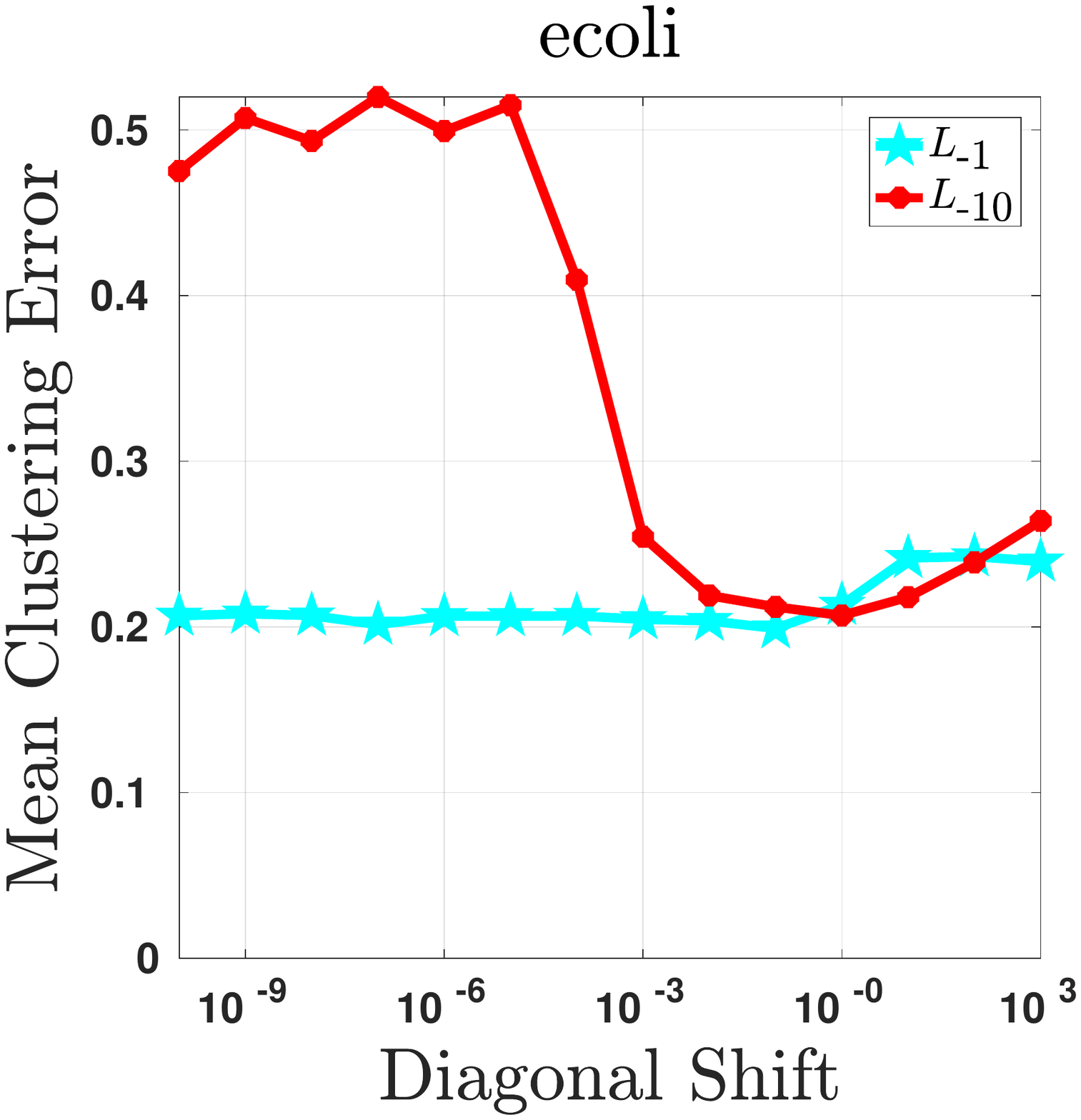}\hspace*{\fill}
\end{subfigure}%
\hfill
\begin{subfigure}[]{0.24\linewidth}
\includegraphics[width=1\linewidth, clip,trim=105 40 160 40]{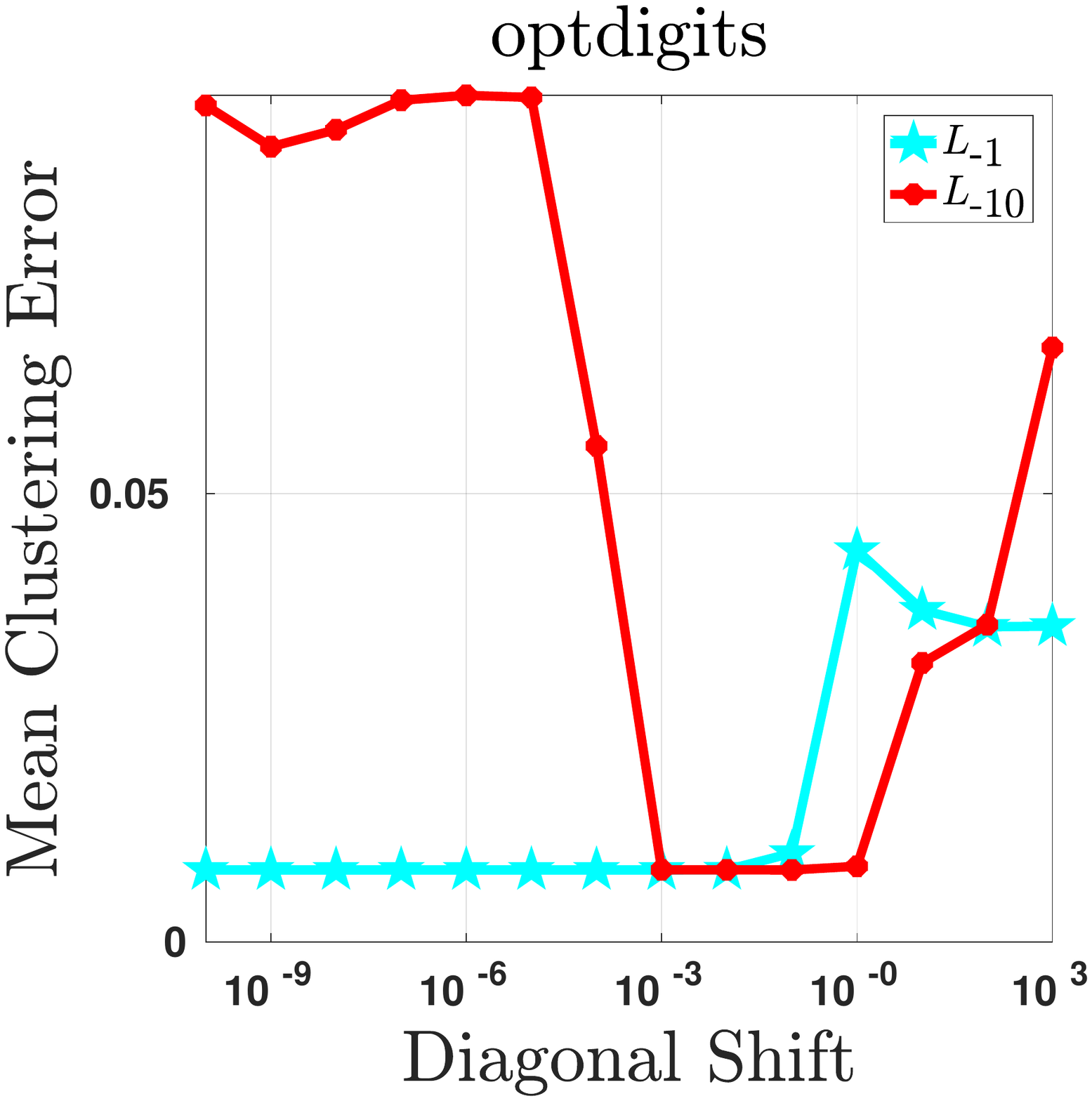}\hspace*{\fill}
\end{subfigure}%
\caption{
Mean clustering error of the power mean Laplacians $L_{-1}$ and $L_{-10}$ with diagonal shifts $\{10^{-10},10^{-9},\ldots,10^{3}\}$.
\vspace{-10pt}
}
\label{fig:diagonal_shift:real_datasets}
\end{figure*}
\section{On Diagonal Shift}\label{sec:OnDiagonalShift}
In this section we briefly discuss the effect of the diagonal shift on the power mean Laplacian $L_p$ for $p\leq 0$.
In the definition of power mean Laplacian in Eq.~\ref{eq:signedMatrixMeanLaplacian} it is mentioned that for negative powers $p\leq 0$ a diagonal shift is necessary. 
To evaluate the influence of the magnitude of the diagonal shift we perform numerical evaluations on two different kinds of signed graphs: on one side we consider signed graphs generated through the Signed Stochastic Block Model introduced in Section~\ref{sec:SBM}, and on the otherside we consider signed graphs built from standard machine learning benchmark datasets following Section~\ref{section:experiments}.

\textbf{Experiments with SSBM.}
We begin with experiments based on signed graphs following the SSBM. The corresponding results are presented in Fig.~\ref{fig:SBM:diagonal_shift}. 
We study the performance of the power mean Laplacians $L_{-1},L_{-2},L_{-5},L_{-10}$ with diagonal shifts $\{10^{-10},10^{-9},\ldots,10^{3}\}$.
Moreover, the case where either $G^+$ or $G^-$ are informative i.e. assortative and disassortative, respectively. In particular, in top (resp. bottom) row of Fig.~\ref{fig:SBM:diagonal_shift} the results correspond to the case where $G^+$ (resp.$G^-$) is fixed to be assortative (resp. disassortative). 

We can observe that the larger the value of $p$, the more robust the performance of the corresponding power mean Laplacian $L_p$ to the values of the diagonal shift. For instance, we can see for $L_{-1}$ (see Figs.~\ref{subfig:fig:SBM:diagonal_shift:fix_Wpos:L_{-1}} and~\ref{subfig:fig:SBM:diagonal_shift:fix_Wneg:L_{-1}}) that the smaller the diagonal shift, the better the smaller the clustering error, whereas for diagonal shifts $10^{0},10^{1},10^{2},10^{3}$ its performance clearly deteriorates.

On the other side we can see that the power mean Laplacian $L_{-10}$ presents a high sensibility towards the value of the diagonal shift (see Figs.~\ref{subfig:fig:SBM:diagonal_shift:fix_Wpos:L_{-10}} and~\ref{subfig:fig:SBM:diagonal_shift:fix_Wneg:L_{-10}}) where the diagonal shift should be neither too large nor too small, being the values $\{10^{-2},10^{-1},10^{0}\}$ the more suitable for this particular case. 
\makeatletter
\newcommand{\removelatexerror}{\let\@latex@error\@gobble}
\makeatother
\removelatexerror
\begin{figure*}[h]
    \makeatletter
    \def\@captype{algocf}
    \makeatother
    \begin{minipage}{0.45\textwidth}
    \vspace{-10pt}
          \begin{algorithm2e}[H]
          \DontPrintSemicolon
		\caption{\footnotesize{PM applied to $M_p(L_\sym^+,Q_\sym^-).$}}\label{alg:PM}
		{\footnotesize
		\KwIn{$\x_0$, $p<0$}
		\KwOut{Eigenpair  $(\lambda, \x)$ of $M_p(L_\sym^+,Q_\sym^-)$}
		\Repeat{tolerance reached}{
			   $\u^{(1)}_k$ $\gets$ $(L_\sym^+)^{p} \x_k$ \hfill(\text{Compute with Alg.~\ref{alg:kyrlov})}\;
		        $\u^{(2)}_k$ $\gets$ $(Q_\sym^-)^{p} \x_k$ \hfill(\text{Compute with Alg.~\ref{alg:kyrlov})}\;
			$\y_{k+1}$ $\gets$ ${\frac  {1}{2}} (\u^{(1)}_k + \u^{(2)}_k)$\;
			$\x_{k+1}$ $\gets$ $\y_{k+1} / \|\y_{k+1} \|_2$\;
			}
		$\lambda$ $\gets$ $(\x_{k+1}^T \x_k)^{1/p}$,\quad  $\x$ $\gets$ $\x_{k+1}$
		}
		\hspace{100pt}
		\end{algorithm2e}
		\hspace{100pt}
    \end{minipage}\hfill
    \begin{minipage}{0.45\textwidth}
         \begin{algorithm2e}[H]
          \DontPrintSemicolon
		\caption{\footnotesize{PKSM for the computation of $A^p\y$}}\label{alg:kyrlov}
		{\footnotesize
		\KwIn{$\u_0 = \y$, $V_0 = [\,\cdot\,  ], p<0$ }
		\KwOut{$\x=A^p\y$}
		$\v_0$ $\gets$ $\y/\vectornorm{\y}_2$\;
		\For{$s=0,1,2,\dots,n$}{
			$\tilde V_{s+1}$ $\gets$ $[V_s,\v_{s}]$\;
			$V_{s+1}$ $\gets$ Orthogonalize columns of $\tilde V_{s+1}$\;
			$H_{s+1}$ $\gets$ $V_{s+1}^T AV_{s+1}$\;
			$\x_{s+1}$ $\gets$ $V_{s+1} (H_{s+1})^p\e_1 \vectornorm{\y}_2$\;
			\textbf{if} {\it tolerance reached} \textbf{then} {\it break} \;
			$\v_{s+1}$ $\gets$  $A\v_{s}$	\;	
		}
		$\x$ $\gets$ $\x_{s+1}$\;
		}
		\end{algorithm2e}
    \end{minipage}
\end{figure*}
This observations are confirmation for the setting with sparse graphs, as it is observed in Fig.~\ref{fig:SBM:diagonal_shift_sparsity_5}.

\textbf{Experiments with benchmark datasets.}
We now perform a numerical evaluation on different real world networks, following the procedure of Section~\ref{section:experiments}. 
Moreover, we perform this analysis for $p\in\{-1,-10\}$ and diagonal shifts $\{10^{-10},10^{-9},\ldots,10^{3}\}$. The corresponding results are presented in Fig.~\ref{fig:diagonal_shift:real_datasets},
where we present the average clustering error taken across all values of $k^+$ and $k^-$ (for more details on the construction of the corresponding signed graphs please see Section~\ref{section:experiments}).

We can observe a general behaviour for $L_{-10}$ across datasets, where for a small diagonal shift, the clustering error is high, and decreases for larger shifts, generally reaching its mininum clustering error around diagonal shifts equal to one, to later present a slight increase in clustering error. This confirms the proposed approach to set the diagonal shift to $log_{10}(1+\abs{p})+10^{-6}$ which for the case of $p=-10$ is $\approx 1.04$.
For the case of the harmonic mean Laplacian $L_{-1}$ we can observe that it presents a more stable behaviour that slightly resembles the one of $L_{-10}$. In particular, we can observe that there is a region from $10^{-6}$ to $10^{-1}$ where the smallest average clustering error is achieved. Hence, $L_{-1}$ is relatively more robust to different diagonal shifts. This confirms the observations made based on signed graphs following the SBM.

\textbf{On condition number.} We now consider a condition number approach to study the effect of the diagonal shift. Recall that the eigenvalue computation scheme considered in this paper is described in Section~\ref{section:computation} with the corresponding Algorithm~\ref{alg:PM}. We can observe that the main computation steps are related to the matrix vector operations $(L_\sym^+)^{p} \x_k$ and $(Q_\sym^-)^{p} \x_k$ with $p<0$. We highlight that this framework considers only the case where $p<0$.

Observe that in the operation $(L_\sym^+)^{p} \x_k$, with $p<0$, the condition number plays a influential place due to the inverse operation implied by the negativity of $p$. Note that the eigenvalues of the normalized Laplacians $L_\sym^+$ are contained in the interval $[0,2]$, hence, it is a singular matrix. As mentioned in definition of the power mean Laplacian in Eq.~\ref{eq:signedMatrixMeanLaplacian}, a suitable diagonal shift is necessary for the case where $p<0$. Hence, the eigenvalues of the shifted Laplacian $L_\sym^+ + \mu I$ are contained in the interval $[\mu,2+\mu]$, therefore, condition number is equal to $\frac{\lambda_{\textrm{max}}(L_\sym^+)}{\lambda_{\textrm{min}}(L_\sym^+)}$ which in this case reduces to $\frac{2+\mu}{\mu}$. Thus, it follows that the condition number of $(L_\sym^+ + \mu I)^{p}$ is 
$g(\mu,p):=\Big( \frac{2+\mu}{\mu} \Big)^{\abs{p}}$. 
It is easy to see that $\frac{2+\mu}{\mu} > 1$ and hence $g(\mu,p)$ grows with larger values of $\abs{p}$, hence the condition number is larger for smaller values of the power mean Laplacian. Moreover, the growth rate of $g(\mu,p)$ is larger for smaller values of $\mu$, suggesting that the shift $\mu$ should be set as large as possible. Yet, very large values of $\mu$ overcome the information contained in the Laplacian matrix. Hence, the diagonal shift should not be too small (due to numerical stability) and should not be too large (due to information ofuscation). This confirms the behaviour presented in Figs.~\ref{fig:SBM:diagonal_shift},~\ref{fig:SBM:diagonal_shift_sparsity_5} and ~\ref{fig:diagonal_shift:real_datasets}.

%


\section{Computation Of the Smallest Eigenvalues and Eigenvectors of $L_p$}\label{section:computation}
For the computation of the eigenvectors corresponding to the smallest eigenvalues of the signed power mean Laplacian $L_p$ with $p<0$, we take the Polynomial Krylov Subspace Method for multilayer graphs presented in~\cite{Mercado:2018:powerMean} and apply it to our case. 
The corresponding adaption 
is presented in Algorithms~\ref{alg:PM} and~\ref{alg:kyrlov}. 

We briefly explain Algorithm~\ref{alg:PM}. Let $\lambda_1\leq\cdots\leq\lambda_n$ be the eigenvalues of $L_p=M_p(L_\sym^+, Q_\sym^+)$. Let $p<0$. Then the eigenvalues of $L_p^p$ are $\lambda_1^p\geq\cdots\geq\lambda_n^p$, that is, the eigenvectors corresponding to the smallest eigenvalues of $L_p$ correspond to the largest eigenvalues of $L_p^p$. 
Thus, in order to obtain the eigenvectors corresponding to the smallest eigenvalues of $L_p$ we have to apply the power method to $L_p^p$. This is depicted in Algorithm~\ref{alg:PM} . 
However, the main computational task now is the matrix-vector multiplications $(L_\sym^+)^p\x$ and $(Q_\sym^-)^p\x$. This is approximated through the Polynomial Krylov Subspace Method (PKSM). This approximation method allows to obtain $(L_\sym^+)^p\x$ and $(Q_\sym^-)^p\x$ without ever computing the matrices $(L_\sym^+)^p$ and $(Q_\sym^-)^p$, respectively. This is depicted in Algorithm~\ref{alg:kyrlov}. 

The main idea of PKSM $s$-step is to project a given matrix $A$ onto the space $\mathbb{K}^s(A,\y)=\{ \y, A\y,\ldots,A^{s-1}\y\}$ and solve the corresponding problem there. The projection on to $\mathbb{K}^s(A,\y)$ is done by means of the Lanczos process, producing a sequence of matrices $V_s$ with orthogonal columns where the first column of $V_s$ is $\y/\norm{\y}$ and $\text{range}(V_s)=\mathbb{K}^s(A,\y)$. Moreover, at each step we have $AV_s=V_s H_s + \v_{s+1}\e^T_s$ where $H_s$ is $s \times s$ symmetric tridiagonal, and $\e_i$ is the $i$-th canonical vector. 
The matrix product vector ${\x=A^p\y}$ is the approximated by $\x_s=V_s(H_s)^p\e_1\norm{\y}\approx A^p\y$.

 \textbf{Time Execution Analysis}. We present a time execution analysis in Fig.~\ref{fig:time}. We depict the mean time execution out of 10 runs of the power mean Laplacian $L_p$ with $p\in\{-1,-2,-5-10\}$. In particular $L_{-1}(\text{ours})$, $L_{-2}(\text{ours})$, $L_{-5}(\text{ours})$ and $L_{-10}(\text{ours})$ depict the time execution using our proposed method based on Algorithm~\ref{alg:PM} together with the polynomial Krylov subspace method described in Algorithm~\ref{alg:kyrlov}. For comparison we consider $L_{-1}(\text{eigs})$ which is computed with the function \texttt{eigs} from MATLAB instead of using Algorithm~\ref{alg:kyrlov}. All experiments are performed using one thread. For evaluation random signed graphs following the SSBM are generated, with parameters $\pp=\qm=0.05$ and $\ppm=\qp=0.025$ with two equal sized clusters, and graph size $\abs{V}\in\{10000,20000,30000,40000\}$.
 We can observe that our computational matrix-free approach based on the polynomial Krylov subspace method systematically outperforms the natural approach based on the explicity computation of power matrices per layer.
\begin{figure}[t]
 \centering\hfill
 \includegraphics[width=0.8\textwidth,trim=160 40 50 60]{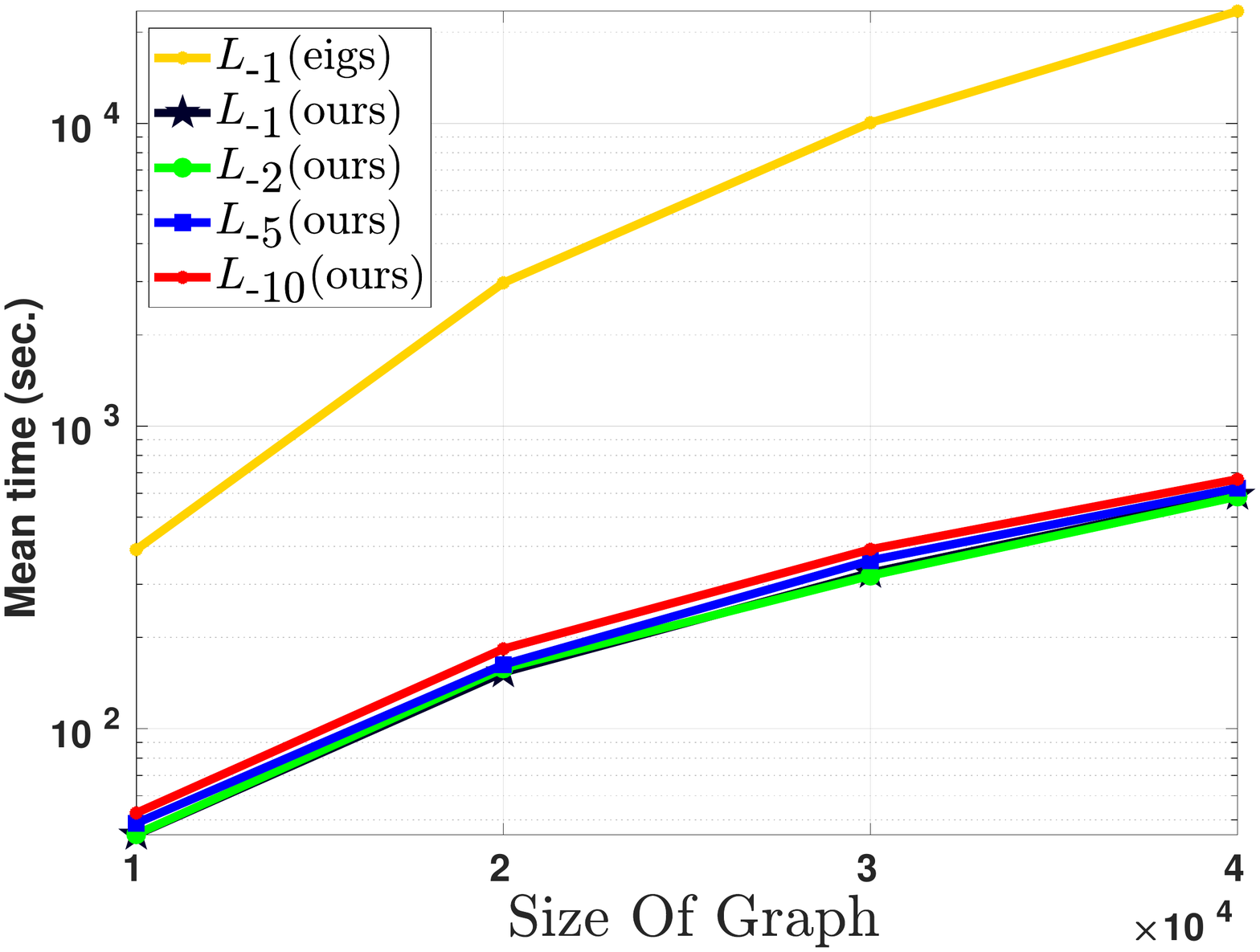}
 \hfill
 \caption{Time execution analysis }
 \label{fig:time}
\end{figure}
%
%
%

 \begin{figure*}[ht]
 \centering
 \includegraphics[width=0.5\linewidth, clip,trim=0 180 0 390]{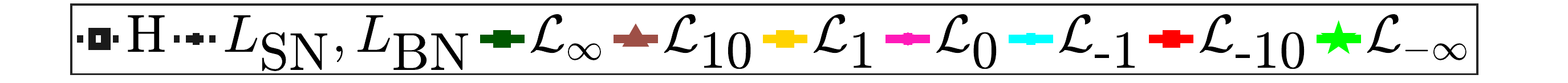}\\
 \vspace{5pt}
 \hspace{10pt}
 \begin{subfigure}[b]{0.32\textwidth}
 \includegraphics[width=1\textwidth,trim=100 60 20 60]{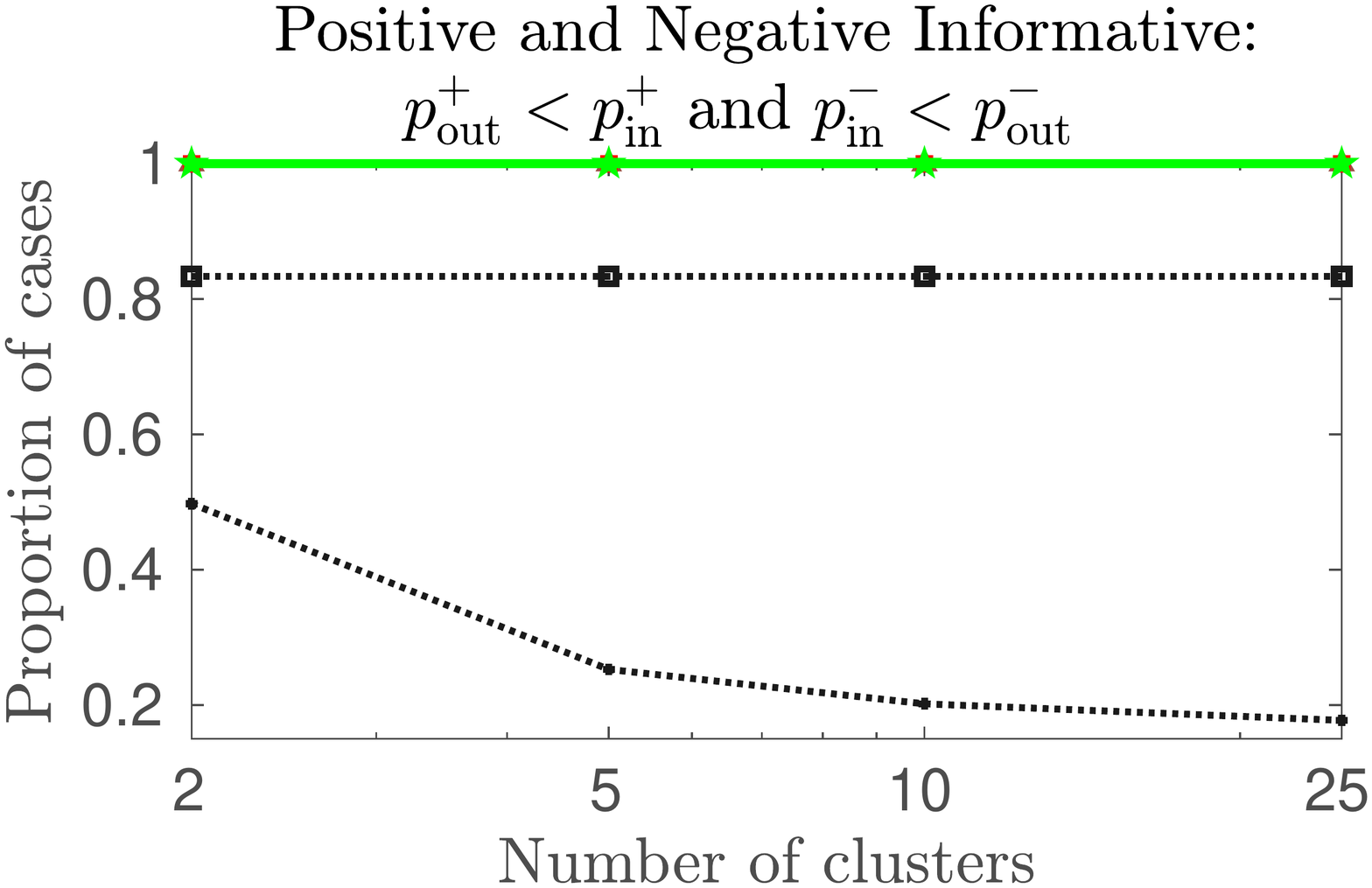}\hfill
 \vspace{-10pt}
 \caption{}
  \vspace{-10pt}
 \label{fig:numEigsVsLabeledNodes:AND}
 \end{subfigure}
 \hfill
  \begin{subfigure}[b]{0.32\textwidth}
 \includegraphics[width=1\textwidth,trim=100 60 20 60]{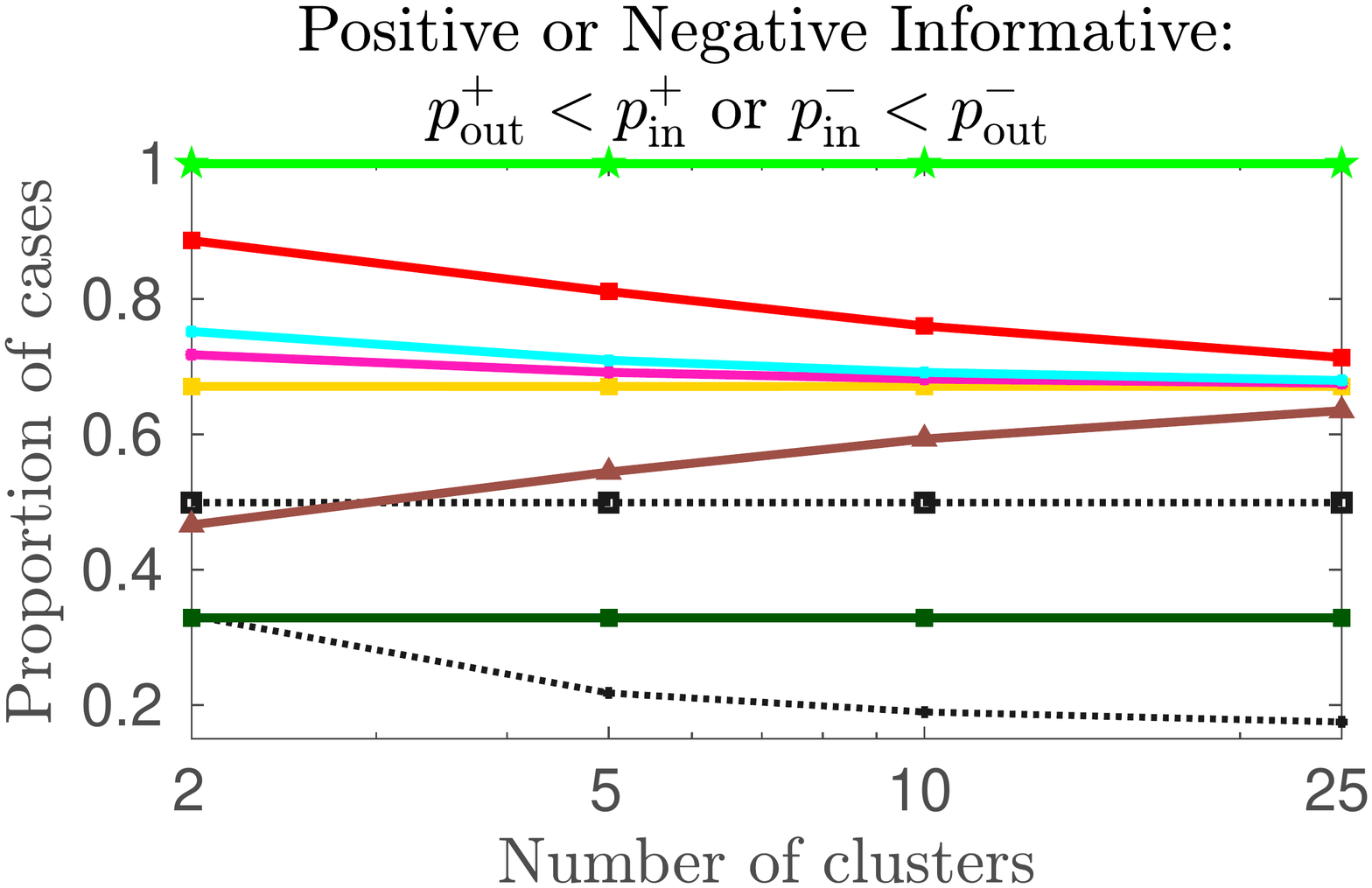}\hfill
  \vspace{-10pt}
 \caption{}
  \vspace{-10pt}
 \label{fig:numEigsVsLabeledNodes:OR}
 \end{subfigure}
 \hfill
 \begin{subfigure}[b]{0.32\textwidth}
 \includegraphics[width=1\textwidth,trim=100 60 20 60]{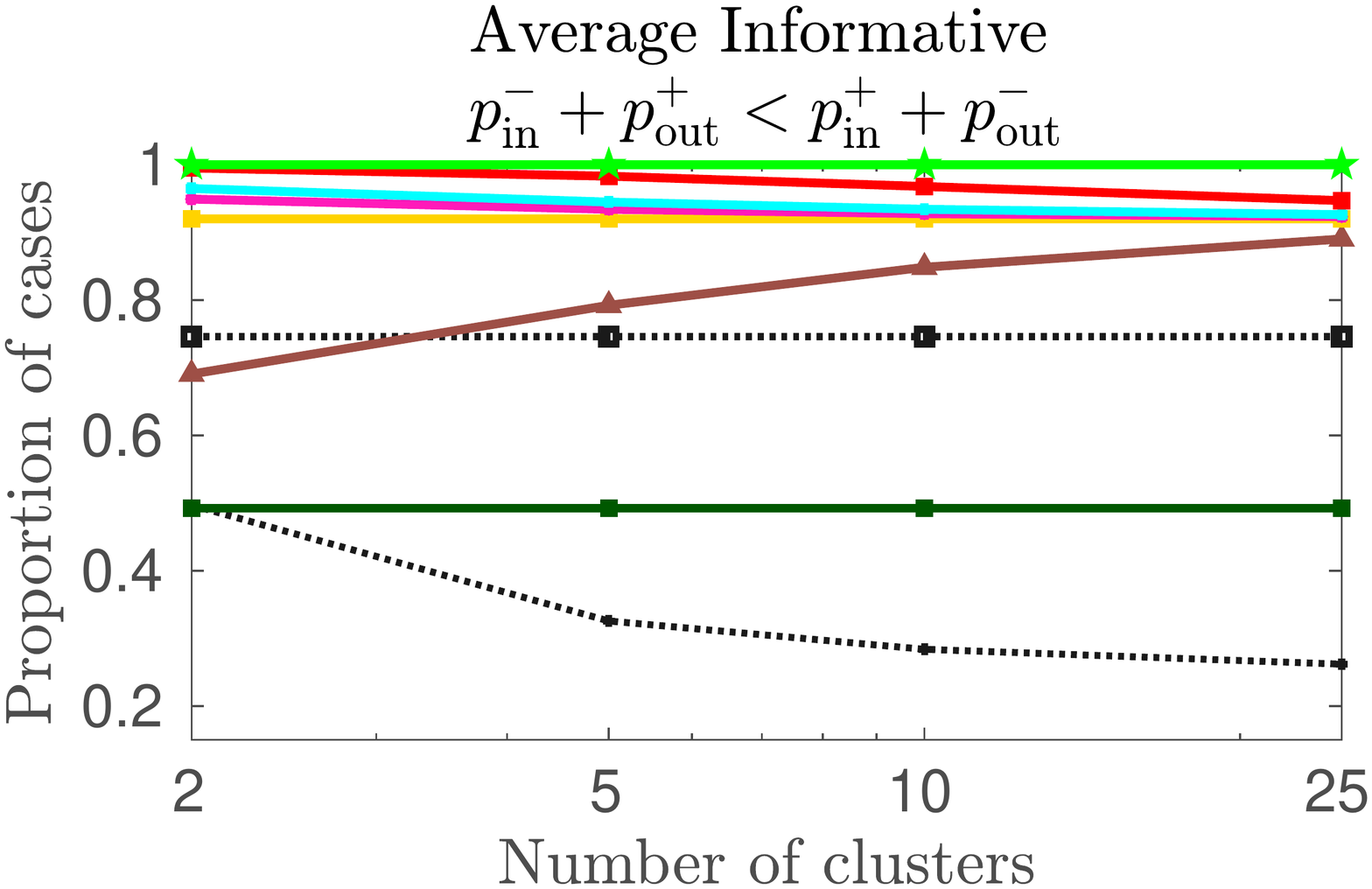}\hfill
 \vspace{-10pt}
  \caption{}
   \vspace{-10pt}
 \label{fig:numEigsVsLabeledNodes:AVERAGE}
 \end{subfigure}
\hspace{-15pt}
\caption{Proportion of cases where conditions of Theorems~\ref{theorem:mp_in_expectation},~\ref{theorem:signedLaplacians} and~\ref{theorem:bethe_hessian_V1}
hold under different settings.}
 \label{fig:numEigsVsLabeledNodes}
\end{figure*}
%
%
\section{Proportion of cases where conditions hold}\label{section:proportionOfCasesWhereConditionsHold}
%
%

In order to understand how often the conditions from 
Theorems~\ref{theorem:mp_in_expectation},~\ref{theorem:signedLaplacians} and~\ref{theorem:bethe_hessian_V1}, 
we perform a series of experiments. 

For the Bethe Hessian we take the limit result when $\abs{V}\rightarrow\infty$ as the corresponding conditions do not have as a parameter the size of graph.
The corresponding results are depicted in Fig.~\ref{fig:numEigsVsLabeledNodes}. 
We discretize each of the parameters $\pp,\qp,\pp,\qm$ in $[0,1]$ in one hundred steps and count how many times the conditions of  Theorems~\ref{theorem:mp_in_expectation},~\ref{theorem:signedLaplacians} and ~\ref{theorem:bethe_hessian_V1} hold under different settings.
In Fig.~\ref{fig:numEigsVsLabeledNodes:AND} we analyze the case when both $G^+$ and $G^-$ are informative ($\pp>\qp$ and $\ppm<\qm$). We can see that the conditions for the signed power mean Laplacian $\mathcal L_p$ are always fulfilled, whereas those of $L_{SN}$ and $L_{BN}$ hold in a significantly smaller fraction of cases, whereas the case of the Bethe Hessian $H$ are closer to the power mean Laplacians than to $L_{SN}$ and $L_{BN}$.
 In Fig.~\ref{fig:numEigsVsLabeledNodes:OR} we analyze the case when $G^+$ or $G^-$ is informative ($\pp>\qp$ or $\ppm<\qm$). Now we see an ordering between different $\mathcal L_p$ where the smaller the value of $p$ the larger the proportion of cases leading to recovery of the clusters in expectation. In particular, $\mathcal L_{-\infty}$ always fulfills the conditions, whereas $\mathcal L_{\infty}$ realizes the smallest proportion of cases where its conditions hold comparable to the one of $L_{SN}$ and $L_{BN}$, while the Bethe Hessian holds for $50\%$ of the cases. 
In Fig.~\ref{fig:numEigsVsLabeledNodes:AVERAGE} we treat the case where on average $G^+$ and $G^-$ are informative ($\ppm+\qp<\pp+\qm$). We observe the same ordering
as in the previous case and again all signed power mean Laplacians outperform $L_{SN}$ and $L_{BN}$,while the Bethe Hessian holds for around $75\%$ of the cases. 
In Figs~\ref{fig:numEigsVsLabeledNodes:OR} and~\ref{fig:numEigsVsLabeledNodes:AVERAGE} we observe that the difference between the signed power mean Laplacians with finite $p$ gets smaller as the number of clusters $k$ increases. The reason is that the eigenvalues of $\mathcal{L}_\sym$ and $\mathcal{Q}_\sym$ are of the form $1\pm\rho(k)$, where $\lim_{k\to\infty}\rho(k) = 0$. Thus, as $k$ increases the eigenvalues become equal and thus the gap vanishes.
\section{On Wikipedia Experiments}\label{sec:wikipedia-experiments-appendix}
We provide a more detailed inspection of the results from Sec.~\ref{section:Wikipedia-experiments}. In Fig.~\ref{fig:wikipedia-supplementary} we present the sorted adjacency  matrices according to the identified clusters. In the first two clumns, (left to right), we can see that there is a large cluster (upper-left corner of each adjacency matrix) that does not resemble any structure, whereas the remaining part of the graph does present certain clustering structure. The following third and fourth columns zoom in into this region, which corresponds to results presented in Fig.~\ref{fig:wikipedia}.
\begin{figure*}[!htb]
 \centering
 \begin{subfigure}[b]{0.23\textwidth}
 \includegraphics[angle=-90,width=1\textwidth,trim=160 40 10 60]{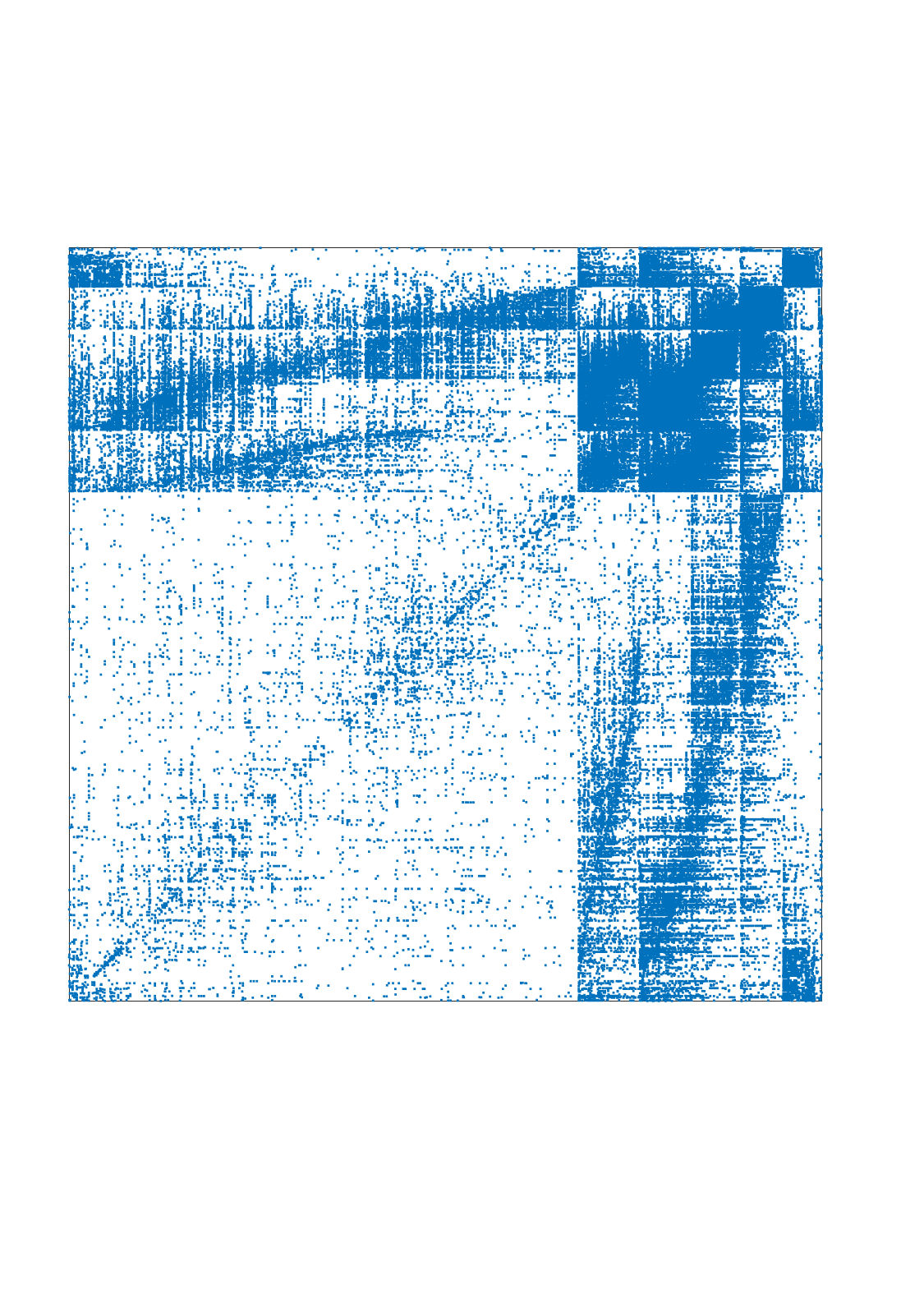}
 \vspace{-10pt}\caption{$L_{-10}$ : $W^+$}
 \end{subfigure}
 \hfill
 \begin{subfigure}[b]{0.23\textwidth}
 \includegraphics[angle=-90,width=1\textwidth,trim=160 40 10 60]{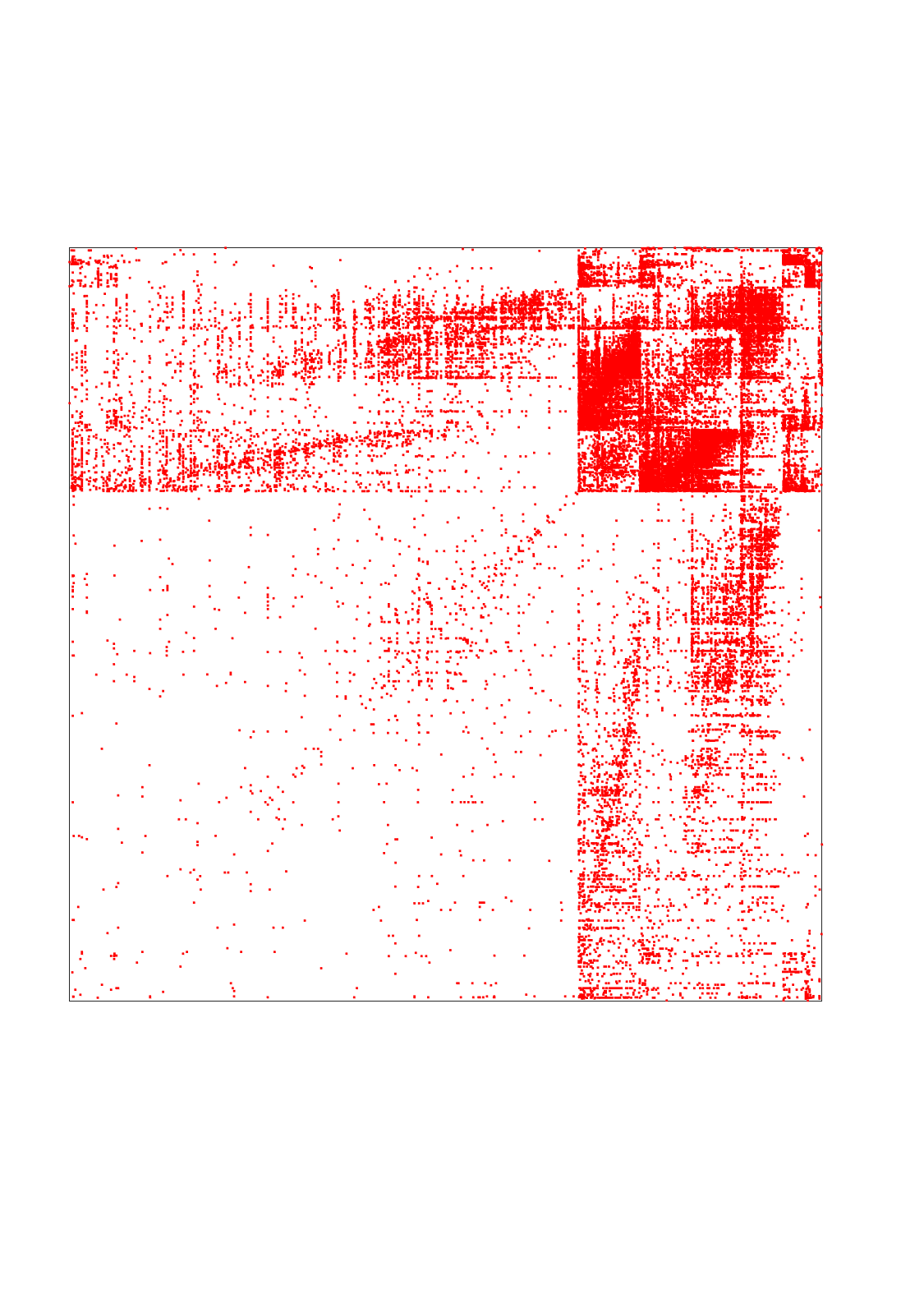}
 \vspace{-10pt}\caption{$L_{-10}$ : $W^-$}
 \end{subfigure}
 \hfill
 \begin{subfigure}[b]{0.23\textwidth}
 \includegraphics[angle=-90,width=1\textwidth,trim=160 40 10 60]{p_minus_10_blue_Zoom.png}
 \vspace{-10pt}\caption{$L_{-10}$ : $W^+$ (Zoom)}
 \label{fig:wikipedia:blue:Minus2}
 \end{subfigure}
 \hfill
 \begin{subfigure}[b]{0.23\textwidth}
 \includegraphics[angle=-90,width=1\textwidth,trim=160 40 10 60]{p_minus_10_red_Zoom.png}
 \vspace{-10pt}\caption{$L_{-10}$ : $W^-$ (Zoom)}
 \label{fig:wikipedia:blue:Minus2}
 \end{subfigure}
 \hfill
\\
\vspace{25pt}
 \begin{subfigure}[b]{0.23\textwidth}
 \includegraphics[angle=-90,width=1\textwidth,trim=160 40 10 60]{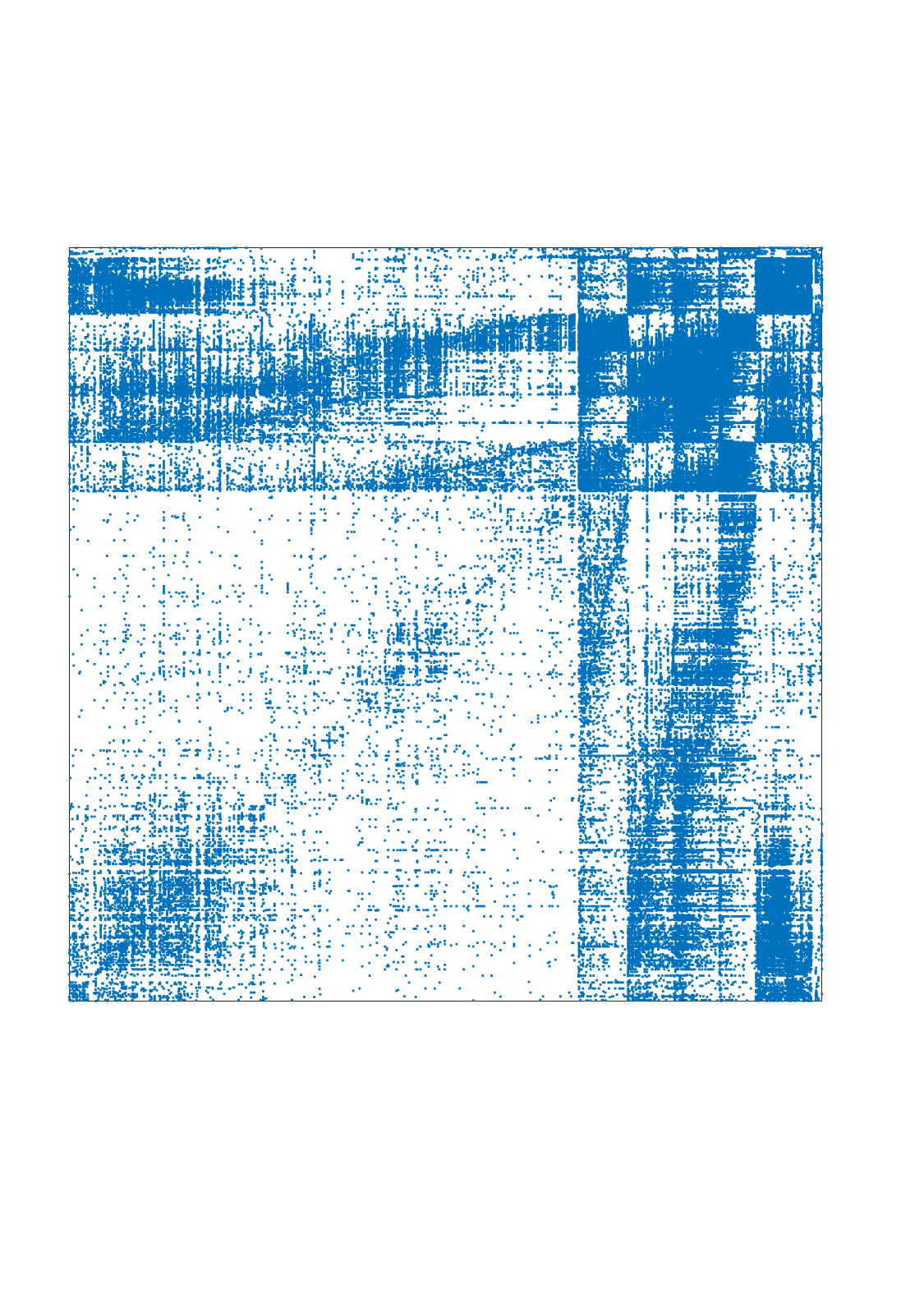}
 \vspace{-10pt}\caption{$L_{-5}$ : $W^+$}
 \end{subfigure}
 \hfill
 \begin{subfigure}[b]{0.23\textwidth}
 \includegraphics[angle=-90,width=1\textwidth,trim=160 40 10 60]{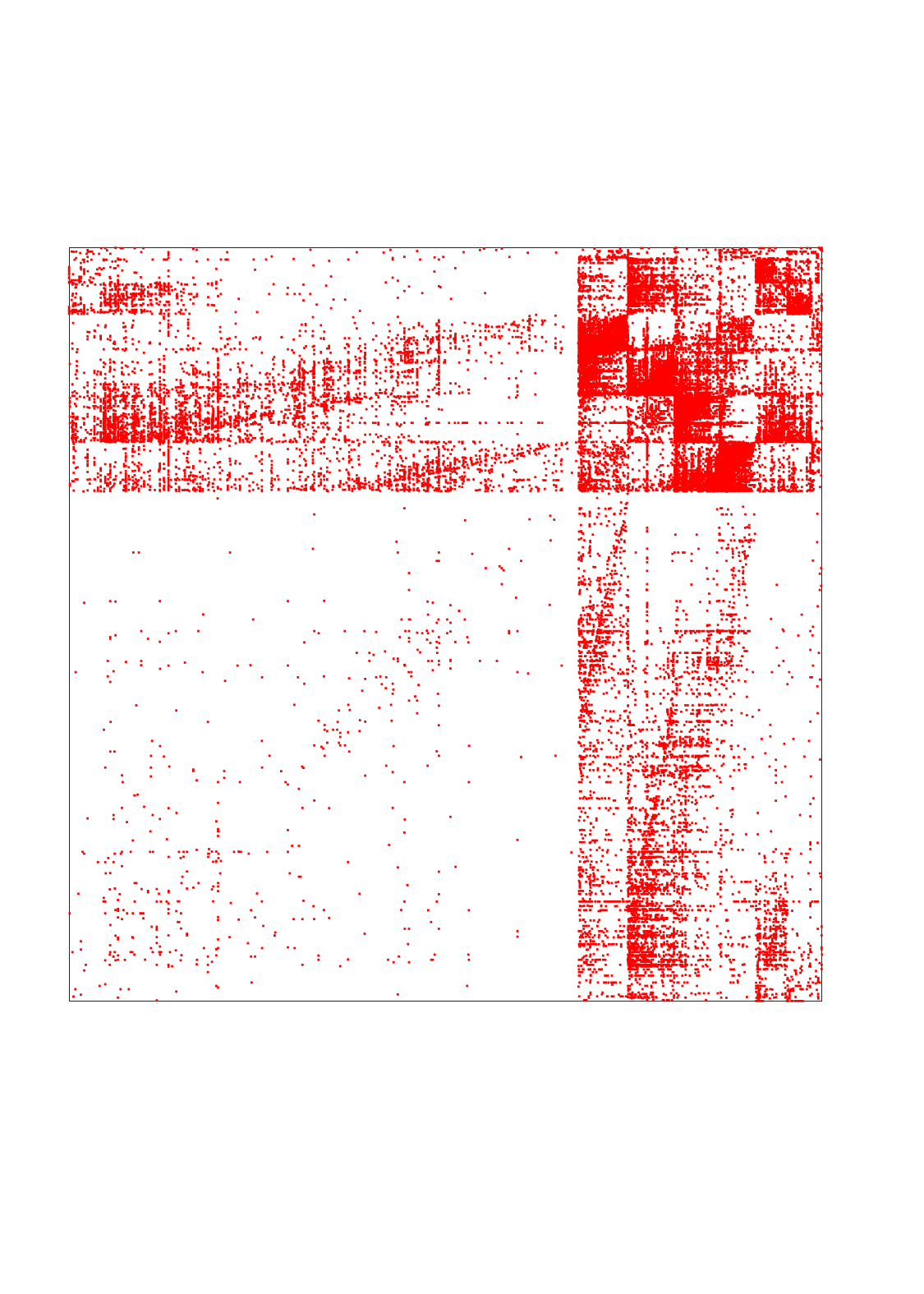}
 \vspace{-10pt}\caption{$L_{-5}$ : $W^-$}
 \end{subfigure}
 \hfill
 \begin{subfigure}[b]{0.23\textwidth}
 \includegraphics[angle=-90,width=1\textwidth,trim=160 40 10 60]{p_minus_5_blue_Zoom.png}
 \vspace{-10pt}\caption{$L_{-5}$ : $W^+$ (Zoom)}
 \label{fig:wikipedia:blue:Minus2}
 \end{subfigure}
 \hfill
 \begin{subfigure}[b]{0.23\textwidth}
 \includegraphics[angle=-90,width=1\textwidth,trim=160 40 10 60]{p_minus_5_red_Zoom.png}
 \vspace{-10pt}\caption{$L_{-5}$ : $W^-$ (Zoom)}
 \label{fig:wikipedia:blue:Minus2}
 \end{subfigure}
 \hfill
\\
\vspace{25pt}
 \begin{subfigure}[b]{0.23\textwidth}
 \includegraphics[angle=-90,width=1\textwidth,trim=160 40 10 60]{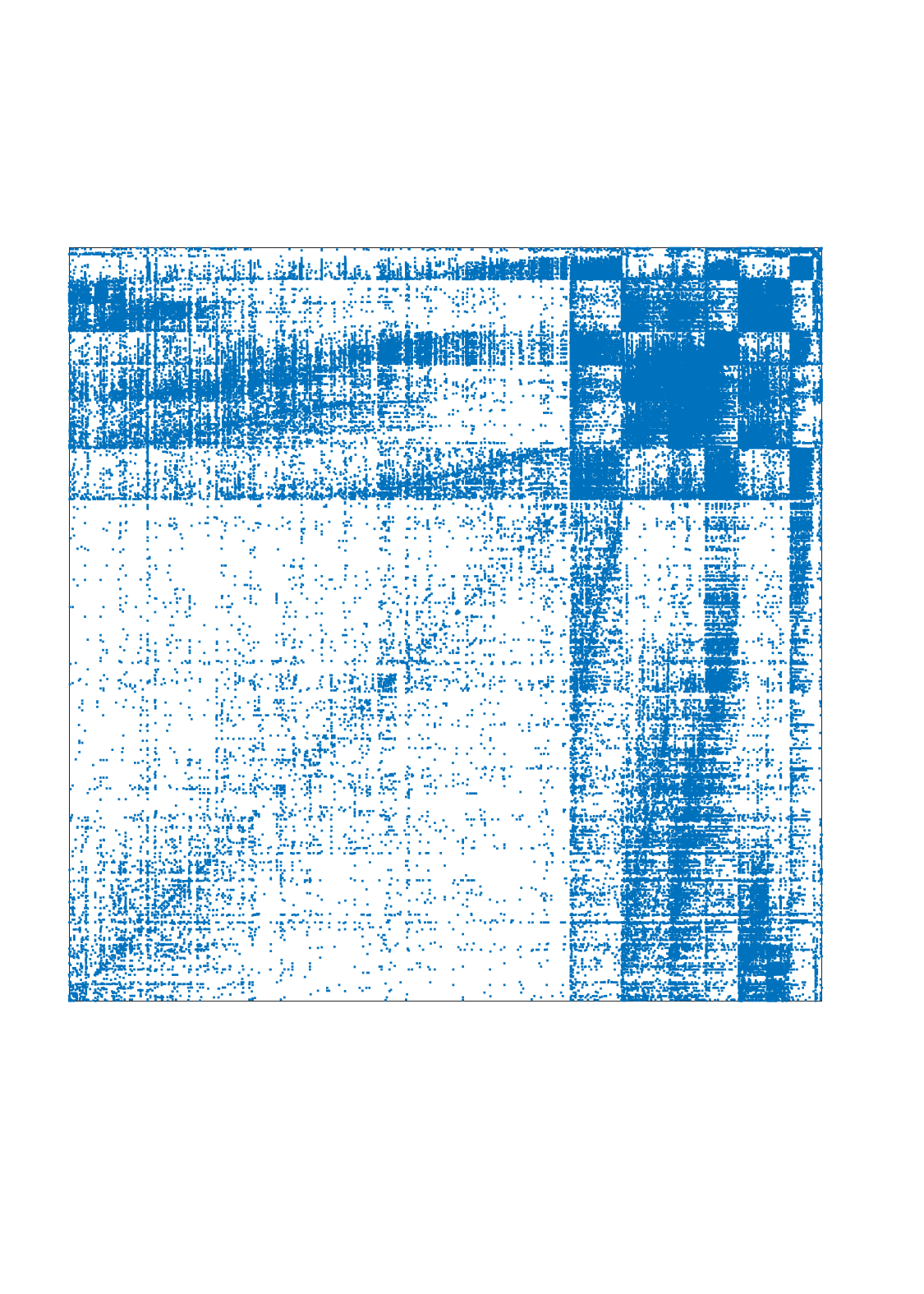}
 \vspace{-10pt}\caption{$L_{-2}$ : $W^+$}
 \end{subfigure}
 \hfill
 \begin{subfigure}[b]{0.23\textwidth}
 \includegraphics[angle=-90,width=1\textwidth,trim=160 40 10 60]{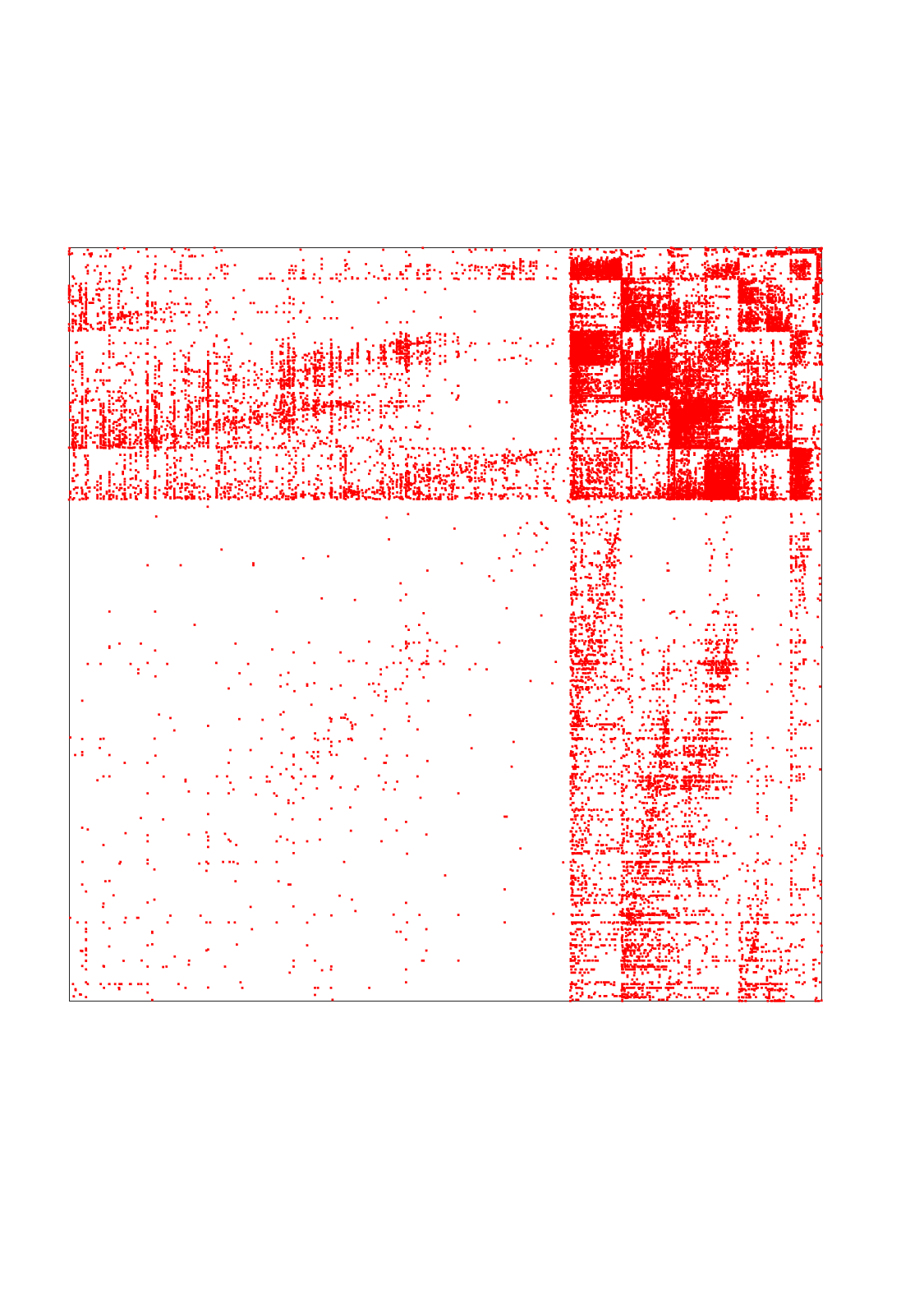}
 \vspace{-10pt}\caption{$L_{-2}$ : $W^-$}
 \end{subfigure}
 \hfill
 \begin{subfigure}[b]{0.23\textwidth}
 \includegraphics[angle=-90,width=1\textwidth,trim=160 40 10 60]{p_minus_2_blue_Zoom.png}
 \vspace{-10pt}\caption{$L_{-2}$ : $W^+$ (Zoom)}
 \label{fig:wikipedia:blue:Minus2}
 \end{subfigure}
 \hfill
 \begin{subfigure}[b]{0.23\textwidth}
 \includegraphics[angle=-90,width=1\textwidth,trim=160 40 10 60]{p_minus_2_red_Zoom.png}
 \vspace{-10pt}\caption{$L_{-2}$ : $W^-$ (Zoom)}
 \label{fig:wikipedia:blue:Minus2}
 \end{subfigure}
 \hfill
\\
\vspace{25pt}
 \begin{subfigure}[b]{0.23\textwidth}
 \includegraphics[angle=-90,width=1\textwidth,trim=160 40 10 60]{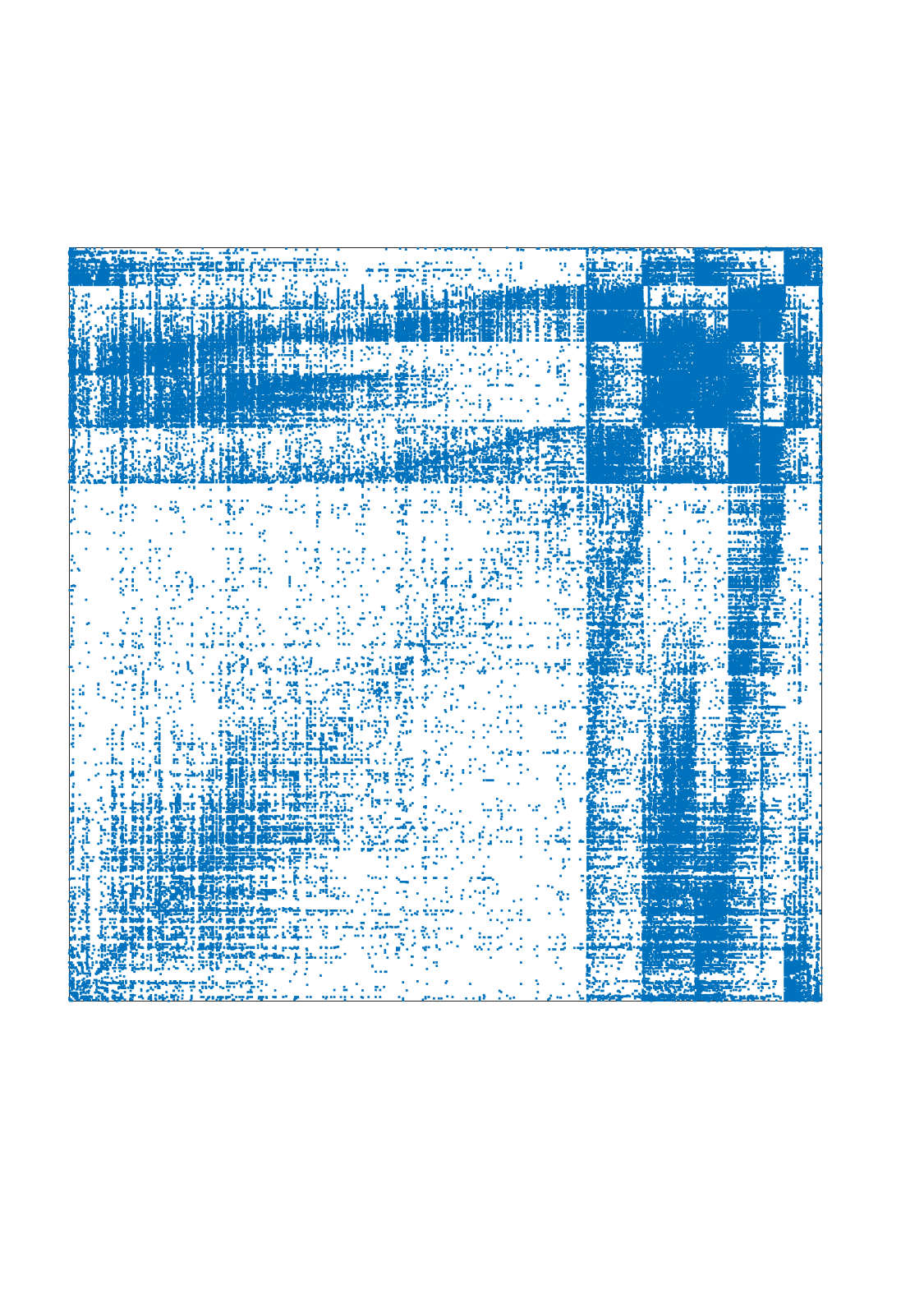}
 \vspace{-10pt}\caption{$L_{-1}$ : $W^+$}
 \end{subfigure}
 \hfill
 \begin{subfigure}[b]{0.23\textwidth}
 \includegraphics[angle=-90,width=1\textwidth,trim=160 40 10 60]{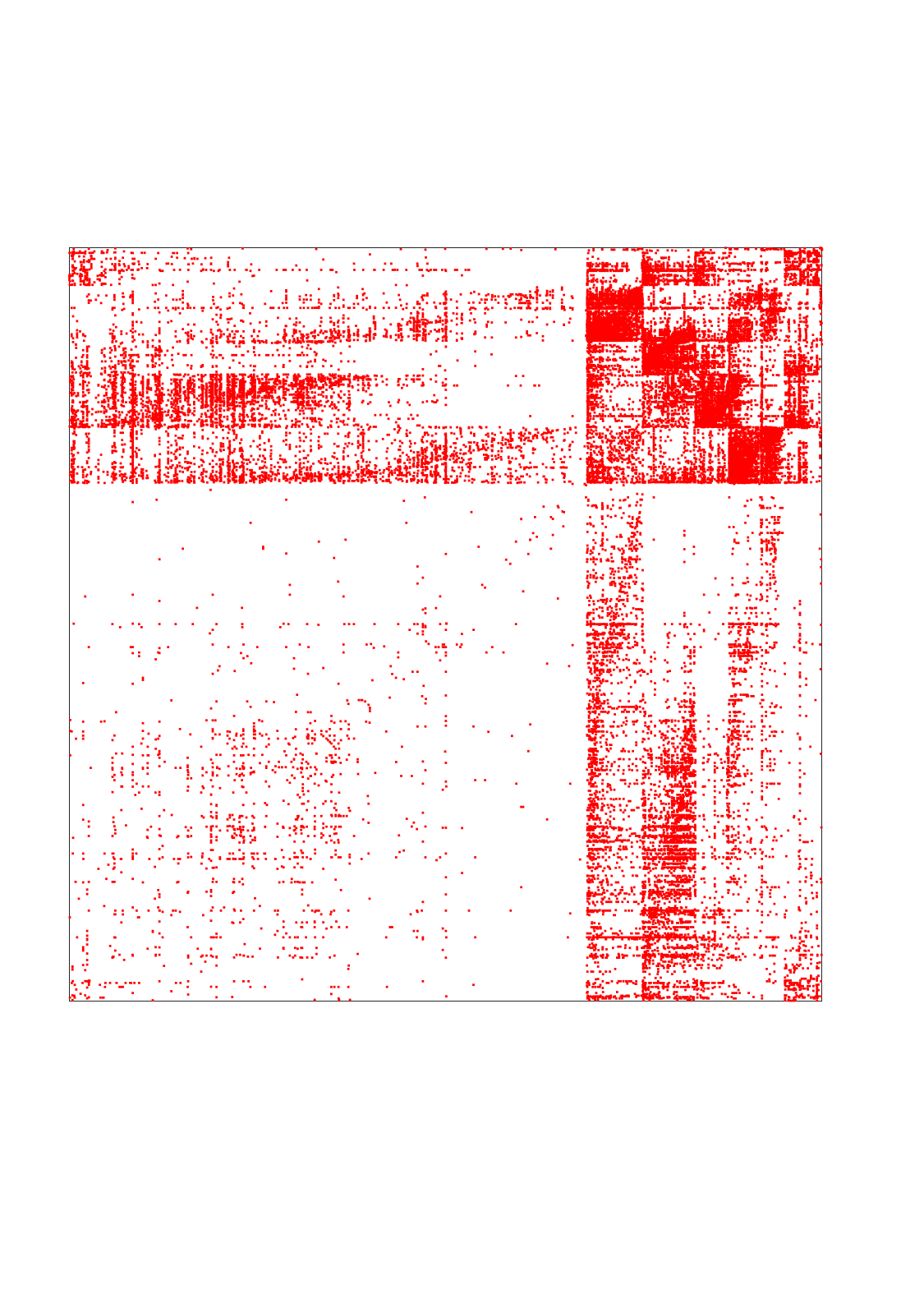}
 \vspace{-10pt}\caption{$L_{-1}$ : $W^-$}
 \end{subfigure}
 \hfill
 \begin{subfigure}[b]{0.23\textwidth}
 \includegraphics[angle=-90,width=1\textwidth,trim=160 40 10 60]{p_minus_1_blue_Zoom.png}
 \vspace{-10pt}\caption{$L_{-1}$ : $W^+$ (Zoom)}
 \label{fig:wikipedia:blue:Minus2}
 \end{subfigure}
 \hfill
 \begin{subfigure}[b]{0.23\textwidth}
 \includegraphics[angle=-90,width=1\textwidth,trim=160 40 10 60]{p_minus_1_red_Zoom.png}
 \vspace{-10pt}\caption{$L_{-1}$ : $W^-$ (Zoom)}
 \label{fig:wikipedia:blue:Minus2}
 \end{subfigure}
 \hfill
\\
\vspace{25pt}
 \begin{subfigure}[b]{0.23\textwidth}
 \includegraphics[angle=-90,width=1\textwidth,trim=160 40 10 60]{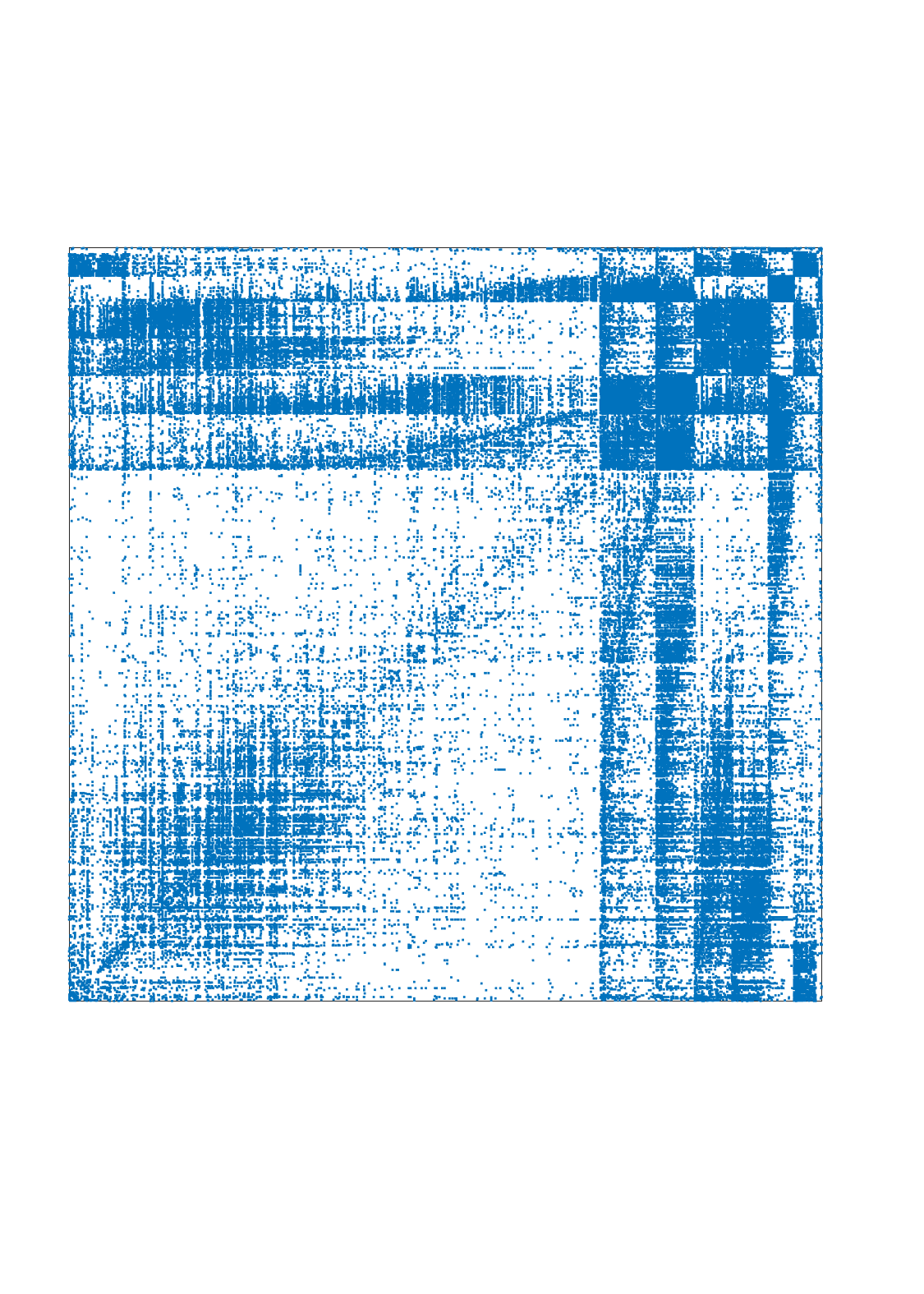}
 \vspace{-10pt}\caption{$L_{0}$ : $W^+$}
 \end{subfigure}
 \hfill
 \begin{subfigure}[b]{0.23\textwidth}
 \includegraphics[angle=-90,width=1\textwidth,trim=160 40 10 60]{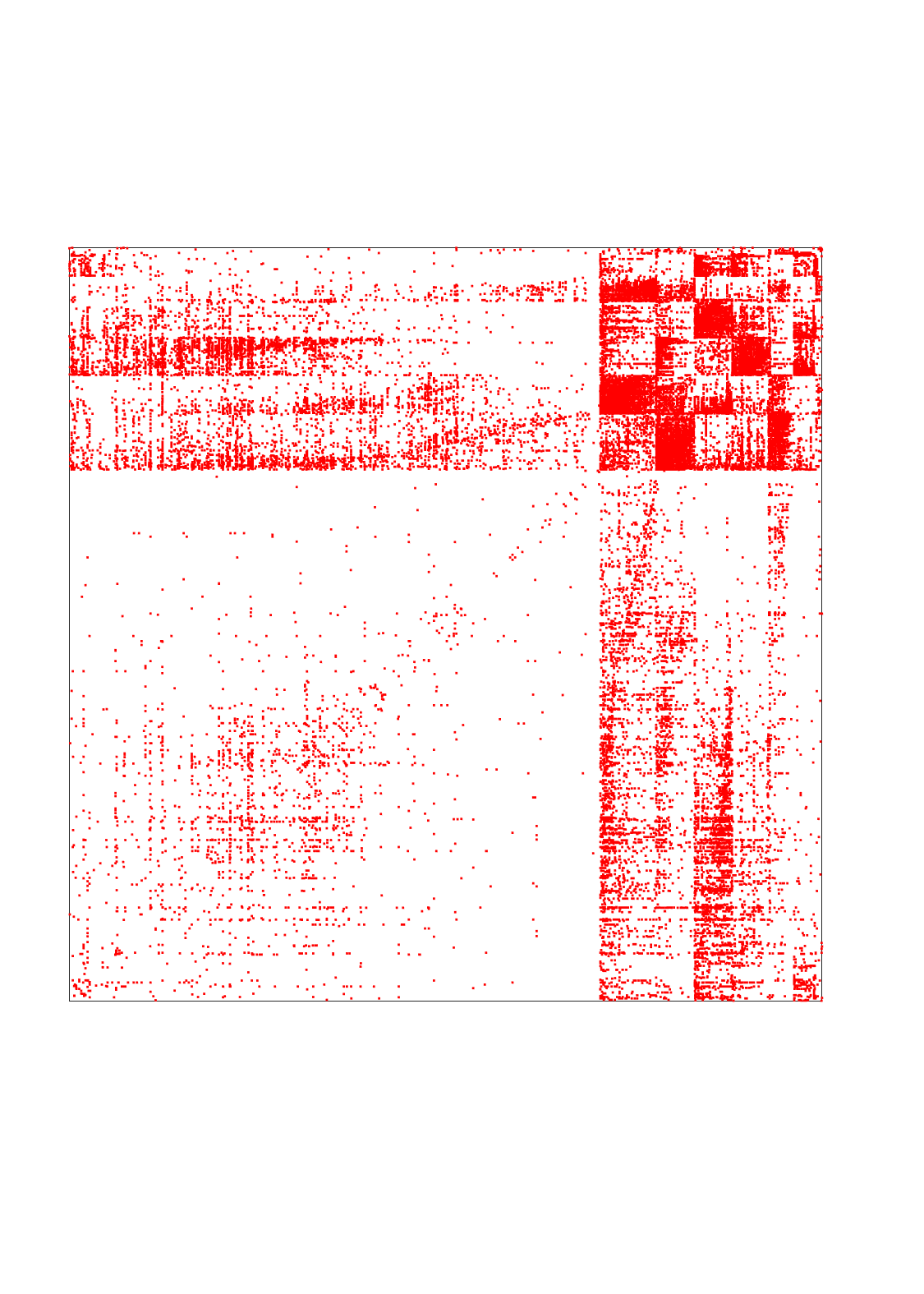}
 \vspace{-10pt}\caption{$L_{0}$ : $W^-$}
 \end{subfigure}
 \hfill
 \begin{subfigure}[b]{0.23\textwidth}
 \includegraphics[angle=-90,width=1\textwidth,trim=160 40 10 60]{p_plus_0_blue_Zoom.png}
 \vspace{-10pt}\caption{$L_{0}$ : $W^+$ (Zoom)}
 \label{fig:wikipedia:blue:Minus2}
 \end{subfigure}
 \hfill
 \begin{subfigure}[b]{0.23\textwidth}
 \includegraphics[angle=-90,width=1\textwidth,trim=160 40 10 60]{p_plus_0_red_Zoom.png}
 \vspace{-10pt}\caption{$L_{0}$ : $W^-$ (Zoom)}
 \label{fig:wikipedia:blue:Minus2}
 \end{subfigure}
 \hfill
\\
\vspace{25pt}
 \begin{subfigure}[b]{0.23\textwidth}
 \includegraphics[angle=-90,width=1\textwidth,trim=160 40 10 60]{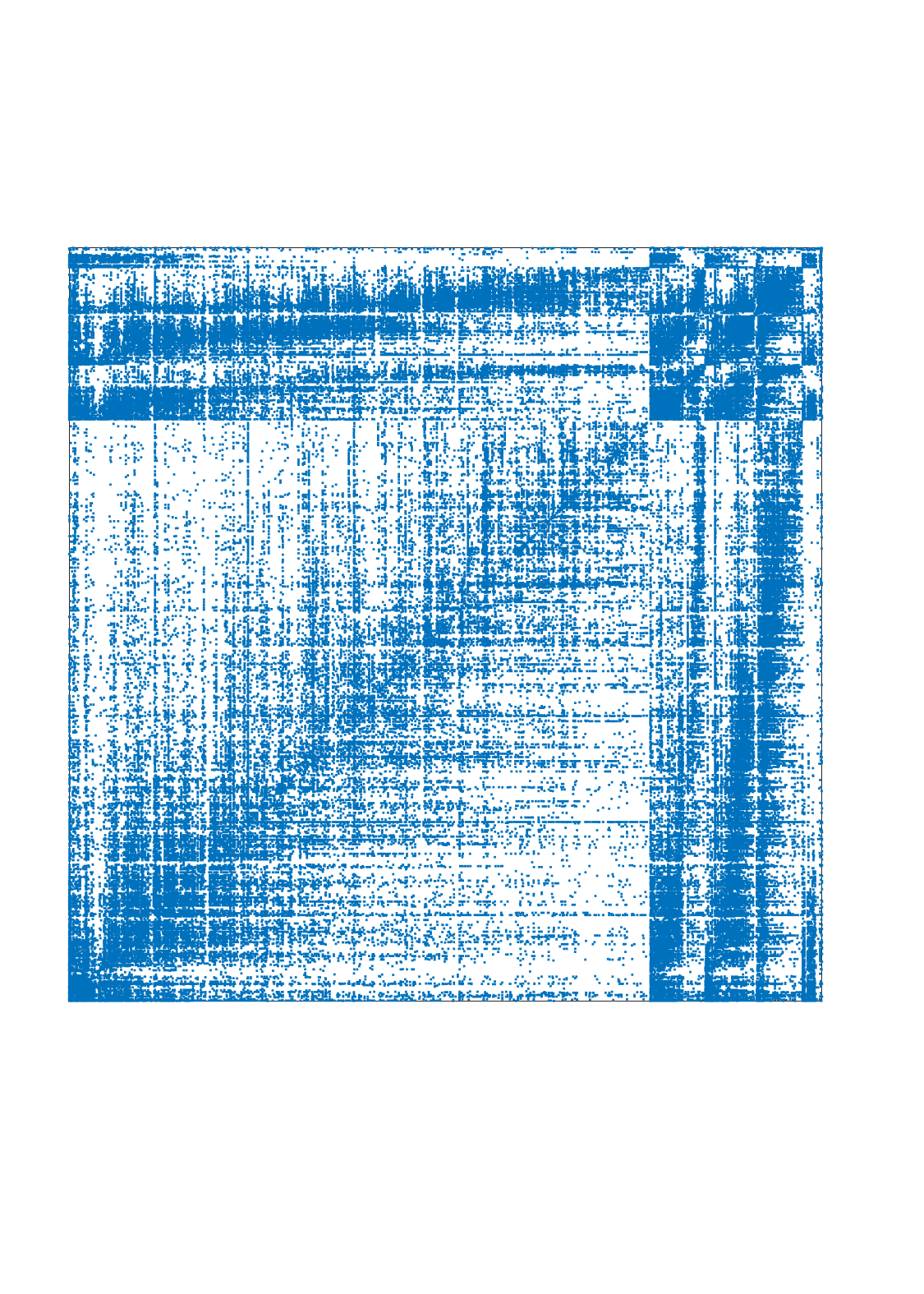}
 \vspace{-10pt}\caption{$L_{1}$ : $W^+$}
 \end{subfigure}
 \hfill
 \begin{subfigure}[b]{0.23\textwidth}
 \includegraphics[angle=-90,width=1\textwidth,trim=160 40 10 60]{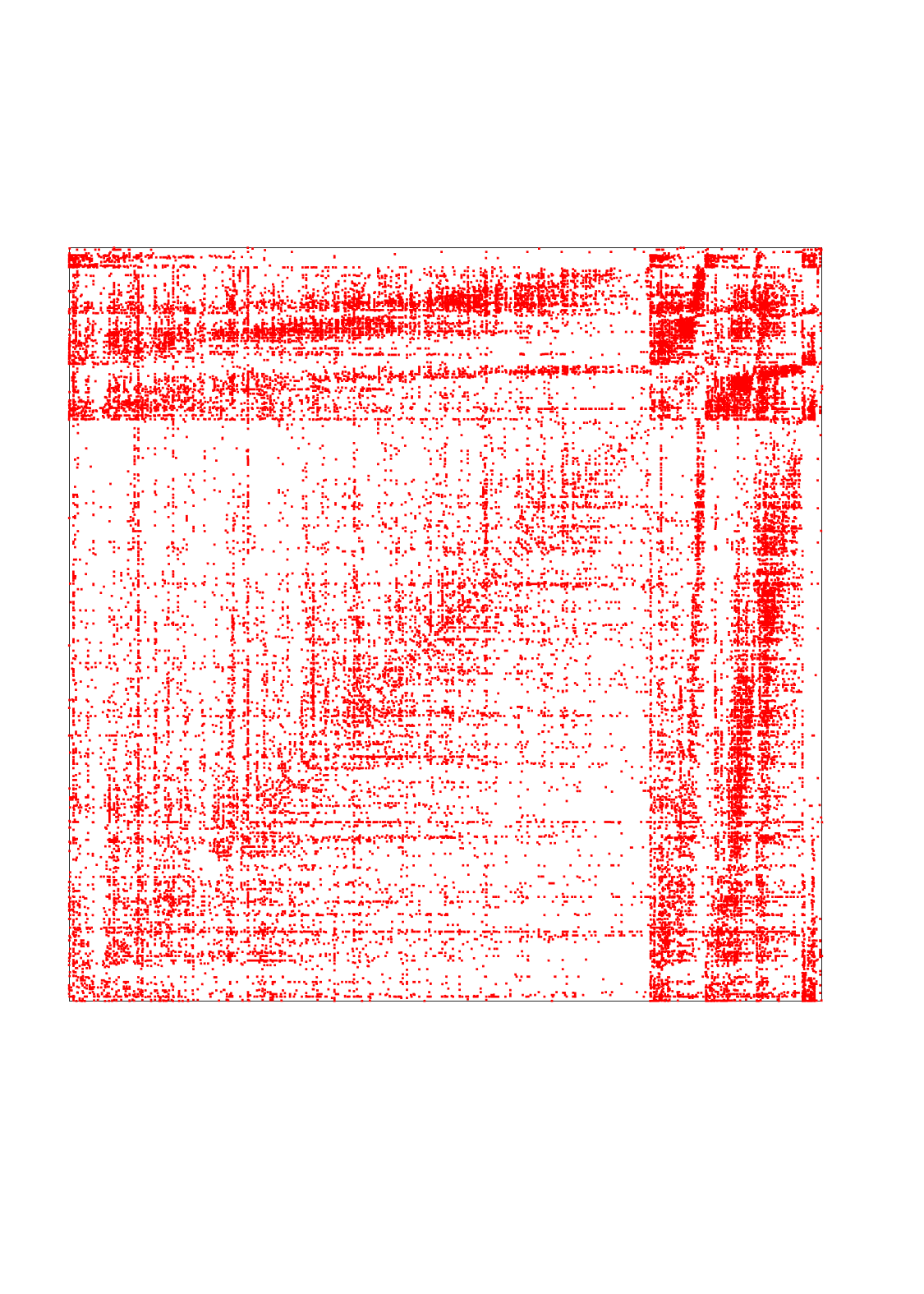}
 \vspace{-10pt}\caption{$L_{1}$ : $W^-$}
 \end{subfigure}
 \hfill
 \begin{subfigure}[b]{0.23\textwidth}
 \includegraphics[angle=-90,width=1\textwidth,trim=160 40 10 60]{p_plus_1_blue_Zoom.png}
 \vspace{-10pt}\caption{$L_{1}$ : $W^+$ (Zoom)}
 \label{fig:wikipedia:blue:Minus2}
 \end{subfigure}
 \hfill
 \begin{subfigure}[b]{0.23\textwidth}
 \includegraphics[angle=-90,width=1\textwidth,trim=160 40 10 60]{p_plus_1_red_Zoom.png}
 \vspace{-10pt}\caption{$L_{1}$ : $W^-$ (Zoom)}
 \label{fig:wikipedia:blue:Minus2}
 \end{subfigure}
 \hfill
 \caption{
 Sorted adjacency matrices according to clusters identified by the Power Mean Laplacian $L_p$ with $p\in\{-10,-5,-2,-1,0,1\}$.
 Columns from left to right: First two columns depict adjacency matrices $W^+$ and $W^-$ sorted through the corresponding clustering.
 Third and fourth columns depict the portion of adjacency matrices $W^+$ and $W^-$ corresponding to the $k-1$ identified clusters.
 Rows from top to bottom: Clustering corresponding to $L_{-10},L_{-5},L_{-2},L_{-1},L_{0},L_{-1}$.
 \vspace{-30pt}
}
 \label{fig:wikipedia-supplementary}
\end{figure*}
%

\end{document}